\documentclass[11pt]{article}
\usepackage{geometry}
\usepackage{geometry}
\usepackage{amsmath, amssymb, amsthm}
\usepackage[ruled,vlined]{algorithm2e}

\newtheorem{assumption}{Assumption}

\usepackage{float}

\newcommand{\Var}{\mathrm{Var}}
\newcommand{\Corr}{\mathrm{Corr}}
\newcommand{\SNR}{\mathrm{SNR}}
\newcommand{\SR}{\mathrm{SR}}

\usepackage[ruled,vlined]{algorithm2e}
\usepackage{xr-hyper}
\usepackage[affil-it]{authblk}
\usepackage{graphicx,subfigure}
\usepackage{amssymb,amsmath}
\usepackage{comment}
\usepackage{multirow}
\newtheorem{prop}{Proposition}
\newtheorem{lem}{Lemma}
\newtheorem{remark}{Remark}

\newtheorem{assump}{Assumption}

\newtheorem{thm}{Theorem}
\newtheorem{definition}{Definition}

\newcommand{\bbE}{\mathbb{E}} 
\newcommand{\bbR}{\mathbb{R}} 
 
\newcommand{\bbP}{\mathbb{P}}

\newcommand{\bX}{\boldsymbol{X}} 
 
\newcommand{\bY}{\boldsymbol{Y}} 
\newcommand{\bZ}{\boldsymbol{Z}} 
\newcommand{\bq}{\boldsymbol{q}} 
\newcommand{\cS}{\mathcal{S}}
\newcommand{\cP}{\mathcal{P}}
\newcommand{\cU}{\mathcal{U}}

\newcommand{\cN}{\mathcal{N}} 
\newcommand{\err}{\mathrm{Error}}

\newcommand{\ind}{\mathbb{I}}
\newcommand{\del}{\backslash}
\newcommand{\no}{\backslash}
\newcommand{\til}[1]{\widetilde{#1}}

\newcommand{\RNum}[1]{\mathrm{\uppercase\expandafter{\romannumeral #1\relax}}}

\newcommand{\hh}[1]{\hat{h}^{(#1)}}
\newcommand{\hD}[2]{\hat{\Delta}^{(#1)}_{#2}}
\newcommand{\D}[2]{\Delta^{(#1)}_{#2}}

\newcommand{\Q}[2]{Q^{#1(#2)}}

\newcommand{\DErr}{\varepsilon_{\Delta}}
\newcommand{\DErrj}[2]{\varepsilon_{#1}^{(#2)}}

\usepackage{enumitem} 
\usepackage{amsmath, amssymb}
\usepackage{xcolor}

\usepackage[utf8]{inputenc}
\usepackage[T1]{fontenc}  
\usepackage{hyperref}    
\usepackage{url}     
\usepackage{booktabs}    
\usepackage{amsfonts}    
\usepackage{nicefrac}    
\usepackage{microtype}   
\usepackage{xcolor}    
\usepackage[symbol]{footmisc}

\usepackage{hyperref}
\usepackage{enumitem}

\usepackage{natbib}

\usepackage{tabularx}
\usepackage{array}
\usepackage{multirow}
\newcolumntype{Y}{>{\raggedright\arraybackslash}X}

\usepackage{amsmath, amssymb, amsthm}
\usepackage[ruled,vlined]{algorithm2e}

\newif\ifshownotationchanges
\shownotationchangesfalse
\DeclareRobustCommand{\notationchange}[1]{%
  \ifshownotationchanges{\color{blue}#1}\else#1\fi}

\newcommand{\suppfigure}[4][0.95\textwidth]{%
\begin{figure}[tbp]
    \centering
    \IfFileExists{#2}{%
        \includegraphics[width=#1]{#2}
    }{%
        \fbox{%
            \parbox[c][0.34\textheight][c]{0.9\textwidth}{%
                \centering
                \textbf{Figure placeholder}\\[0.5em]
                Insert figure file here:\\
                \texttt{#2}
            }%
        }%
    }
    \caption{#4}
    \label{#3}
\end{figure}
}

\hypersetup{
  pdftitle={Model-Agnostic Feature Selection via LOCO-Guided Adaptive Minipatch Sampling},
  pdfauthor={Xuhui Liu; Lili Zheng},
  pdfkeywords={Black-box models, feature selection, model-agnostic feature importance, minipatch ensembles, feature subsampling, adaptive sampling}}

\begin{document}

\title{Model-Agnostic Feature Selection via LOCO-Guided Adaptive Minipatch Sampling}
\author{Xuhui Liu}
\author{Lili Zheng}
\affil{Department of Statistics, University of Illinois Urbana-Champaign}
\date{\today}
\maketitle

\begin{abstract}
    Black-box machine learning models increasingly deliver strong predictions, but extracting useful information from them, such as a set of important features, remains challenging. Existing model-agnostic methods primarily estimate feature importance or conduct inference on it rather than directly selecting features, whereas many feature selection methods are model-specific or rely on the model-X assumption. We introduce LOCO-guided Adaptive Minipatch Sampling (LAMPS), a model-agnostic ensemble framework that uses any black-box regression algorithm as its base learner to select features important for predicting the response. The base learner need only produce predictions and need not perform feature selection itself. LAMPS operates within a minipatch ensemble framework that subsamples both observations and features, allowing leave-one-covariate-out (LOCO) feature importance scores to be easily computed. It adaptively concentrates minipatch sampling on features with high LOCO scores while maintaining exploration. The resulting sampling probabilities rapidly separate signal from noise features after a few iterations, enabling selection through simple thresholding. We establish that LAMPS achieves exact feature selection in high-dimensional settings, provided that the base predictive models are sufficiently well trained on average. Extensive experiments on synthetic and real data show that LAMPS outperforms state-of-the-art feature selection methods, with particularly strong performance in the presence of correlated features.
\end{abstract}

\noindent%

\section{Introduction}
Flexible machine learning models arise with great prediction performance but often lack interpretability. In response, a growing body of interpretation methods has been developed to leverage these powerful models to shed light on the underlying structure of large-scale data~\citep{allen2023interpretable}. One fundamental type of interpretation is the set of important features relevant for a given response, such as important genes associated with Alzheimer's disease, effective treatment combination for better health outcomes, etc. Extensive literature has been devoted to model-specific feature selection methods, ranging from linear \citep{Tibshirani1996Lasso,ZouHastie2005ElasticNet} to nonparametric models \citep{friedman1991multivariate,Ravikumar2009SpAM}. 

However, there has been limited study on model-agnostic feature selection that can leverage any predictive model. Classical wrapper methods search for the best predictive model amongst all feature subsets using greedy heuristics but lack convergence or statistical guarantees \citep{guyon2003introduction}; recent model-agnostic frameworks require the base model to output selection (frequency) \citep{LiuRockova2023TVS} or a ranking \citep{guyon2002gene} for features, which are then aggregated to produce the final feature set; the knockoff framework~\citep{Barber2014ControllingTF} can utilize any predictive model and attain FDR control for feature selection, while requiring access to the joint distribution of all features, which can be limiting in some applications. Other recent works explore flexible ensemble frameworks for feature ranking and screening~\citep{Tian2021RaSEAV,chen2025top}, but not for feature selection. A more extensive discussion of related literature is included in Section~\ref{sec:additional_related_work} of Appendix.

Motivated by this gap, we aim to develop a new framework that can utilize any supervised learning algorithm to perform efficient and accurate feature selection for large-scale data sets, and that enjoys theoretical guarantees without assuming access to the joint feature distribution. One key inspiration for our approach lies in the recent advances in model-agnostic feature importance estimation/inference.

\subsection{Leave-one-covariate-out (LOCO) and Minipatch Ensembles} Recently, a flurry of model-agnostic feature importance estimation/inference methods has been proposed~\citep{williamson2023general,zhang2020floodgate,watson2021testing,Shah2018TheHO}. One of the popular frameworks is the leave-one-covariate-out (LOCO) method~\citep{lei2018distribution}, assessing the prediction error change due to feature-occlusion. Formally, the LOCO importance for feature $j$ takes the following form:
\begin{equation}\label{eq:loco_def}
\Delta_j = \bbE_{X^*,Y^*}\big[\err(Y^*,\mu_{\no j}(X^*))-\err(Y^*,\mu(X^*))\mid \bX,\bY\big],    
\end{equation}
where $\mu(\cdot)$ and $\mu_{\no j}(\cdot)$ are predictive models trained on data $(\bX,\bY)$, $\err(\cdot,\cdot)$ is a prespecified prediction error metric (set as squared loss throughout this paper), and $(\notationchange{\bX_{\cdot,\no j}},\bY)$; $(X^*,Y^*)$ is an unseen test sample independent from and identically distributed as the training samples.
The original LOCO method~\citep{lei2018distribution} estimates and provides inference for $\Delta_j$ via data-splitting and refitting a model $\mu_{\no j}$ for each feature of interest.

To improve the statistical and computational efficiency of LOCO, \citet{GanZhengAllen2022MinipatchCI} propose a new ensemble framework, LOCO-MP, which avoids data-splitting and model-refitting after training. In minipatch ensembles, a base learner is repeatedly trained on randomly subsampled observation and feature indices, termed as minipatches; then the final prediction is obtained through model averaging. \citet{GanZhengAllen2022MinipatchCI} show that after the minipatch ensembles are trained, LOCO inference is essentially free: no held-out data is needed as one can assess prediction errors in a leave-one-out manner; no model-refitting is needed as we can obtain $\mu_{\no j}$ simply by averaging minipatches which exclude feature $j$. \citet{GanZhengAllen2022MinipatchCI}
show that LOCO-MP boosts statistical efficiency by avoiding data-splitting; it is especially strong under feature correlations due to the implicit regularization effect of minipatch ensembles.

\subsection{From Feature Importance to Feature Selection}
One natural question is whether we could adapt these LOCO methods to perform feature selection. An immediate idea might be developing an FDR control method built on top of the feature-wise p-values. However, due to the unknown dependence structure among the LOCO importance estimates, going from individual feature inference to feature selection with FDR control remains a highly non-trivial route. 
Another idea is to simply threshold the importance scores, while it is unclear how to choose the threshold in a principled way, given different scales of the LOCO importance scores across data sets. Nevertheless, to examine the potential of this idea, we conduct a toy simulation with a high-dimensional sparse linear model. We estimate the LOCO importance of all features via data-splitting and via minipatch ensembles (LOCO-MP), where the base models are Lasso and least squares, respectively. The results are shown in Figure~\ref{fig:locomp_split}. We first note that LOCO via data-splitting can hardly distinguish half of the signal features from noise features, possibly because (i) data-splitting loses power and (ii) the original LOCO method assigns low importance to correlated features. LOCO-MP yields much more promising results, while the separation between signal and noise features is still not that clear; it is also difficult to find a data-driven threshold given only the importance scores.

\begin{figure}[t]
\centering
\includegraphics[width=.82\linewidth]{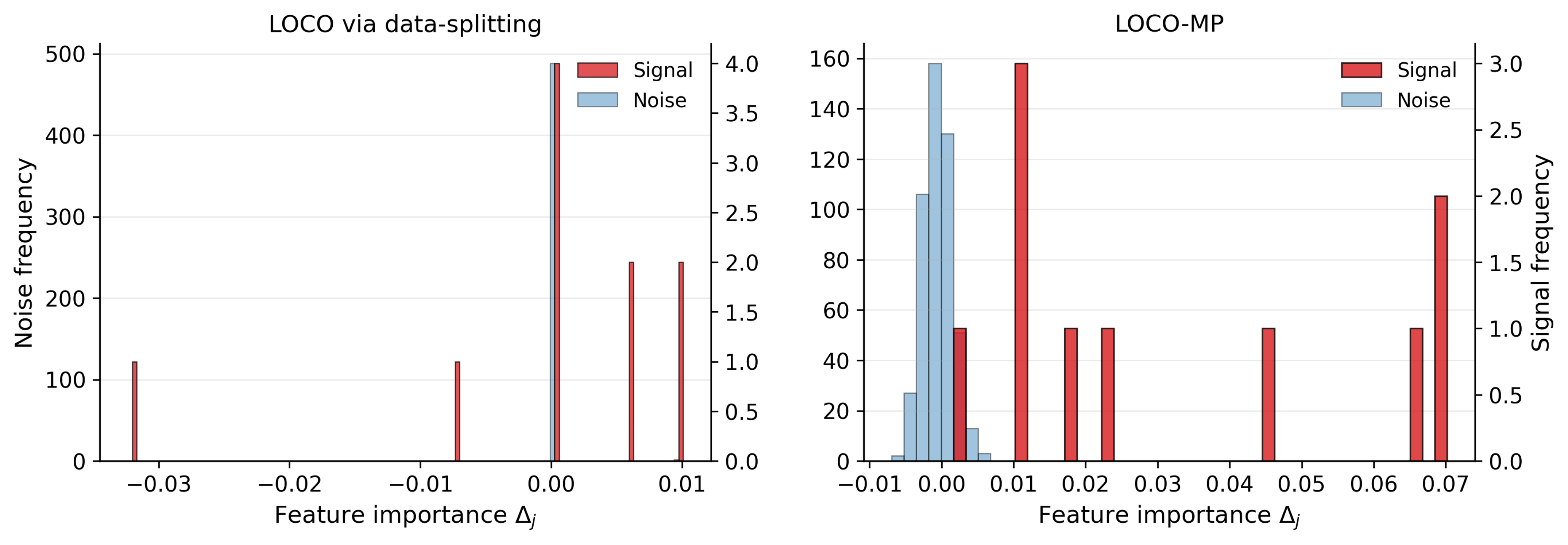}
\caption{
\small Histograms of LOCO feature importance scores obtained via data-splitting and via minipatch ensembles in a correlated linear-model setting with $200$ observations and $500$ features.
Detailed simulation setup is provided in Section~\ref{sec:intro_fig_sim_details} of Appendix.
}
\label{fig:locomp_split}
\end{figure}

\paragraph{From uniform to adaptive feature subsampling for LOCO-MP.} Motivated by Figure~\ref{fig:locomp_split}, we would like to further refine the LOCO-MP feature importance and adapt it to the feature selection problem. We note that one key bottleneck of the LOCO-MP method is the massive uniform feature subsampling. For computational and statistical efficiency of minipatch training, we often set the number of subsampled features to be much smaller than the total number of features (12\% in this example), and hence the signal can be easily buried under the errors due to insufficient sampling. This motivates us to consider an adaptive feature subsampling strategy that gradually assigns higher sampling probabilities to promising features and hence achieves higher accuracy for their importance.

More specifically, we propose the {\bf LOCO-Guided Adaptive Minipatch Sampling (LAMPS)} framework, where we iteratively run the following two steps:
\begin{itemize}[leftmargin=*]
    \item[(i)] Randomly sample minipatches and train a prespecified base model repeatedly;
    \item[(ii)] Compute the (free) LOCO importance for all features given the trained ensemble and update the feature sampling probabilities accordingly.
\end{itemize} 
We adopt uniform observational subsampling and initialize feature sampling probabilities to be the same small value. Figure~\ref{fig:prob_delta_vs_epoch} shows how the LOCO importance and feature sampling probabilities change over three iterations in LAMPS under the same simulation setup as in Figure~\ref{fig:locomp_split}. We can see that there is a clearer separation between signal and noise features as we update the sampling probabilities. Interestingly, the sampling probability itself manifests an even clearer separation; it is also easier to set a universal threshold for the probability given its fixed range $[0,1]$. Therefore, in LAMPS, we choose to threshold the updated sampling probability to obtain the selected feature set $\hat{S}$.

\paragraph{Contribution and organization.} We propose LAMPS, a new ensemble regression framework that can employ any base regression algorithm for feature selection. The base models can be black-box and supply only predictions. The paper is organized as follows.
\begin{itemize}[leftmargin=*]
    \item Section~\ref{sec:method} introduces the detailed procedure of LAMPS. One core innovation is our adaptive feature subsampling scheme guided by the LOCO importance of each feature, easily computed given the trained minipatch ensemble. To make use of the unbounded, possibly negative LOCO importance scores, we develop a new adaptive sampling scheme that is substantially different from prior works~\citep{LiuRockova2023TVS,Tian2021RaSEAV} that often assume access to binary scores, e.g., selection event of a base selector or feature inclusion in the best predictive submodel.
    \item In Section~\ref{sec:theory}, we theoretically establish exact feature selection guarantees for LAMPS in high-dimensional settings after a logarithmic number of iterations. Our theory does not follow directly from prior results on LOCO importance or minipatch ensembles; instead, we develop new proofs to characterize the behavior of our adaptive sampling algorithm and establish tighter probabilistic bounds on LOCO estimation error to accommodate high-dimensional settings. 
    \item In Sections~\ref{sec:simulation} and \ref{sec:casestudy}, we present extensive empirical studies across simulated and real data sets, demonstrating that LAMPS achieves superior feature selection accuracy compared to state-of-the-art methods across linear and nonlinear settings. LAMPS is especially strong in the correlated feature setting; we hypothesize that this is due to the implicit regularization induced by feature subsampling in minipatch ensembles, as has been noted in \cite{GanZhengAllen2022MinipatchCI} and also hinted by our illustrative theory in Section~\ref{sec:correlateddetailtheory}. 
    
\end{itemize}

\begin{figure}[t]
\centering
\includegraphics[width=.82\linewidth]{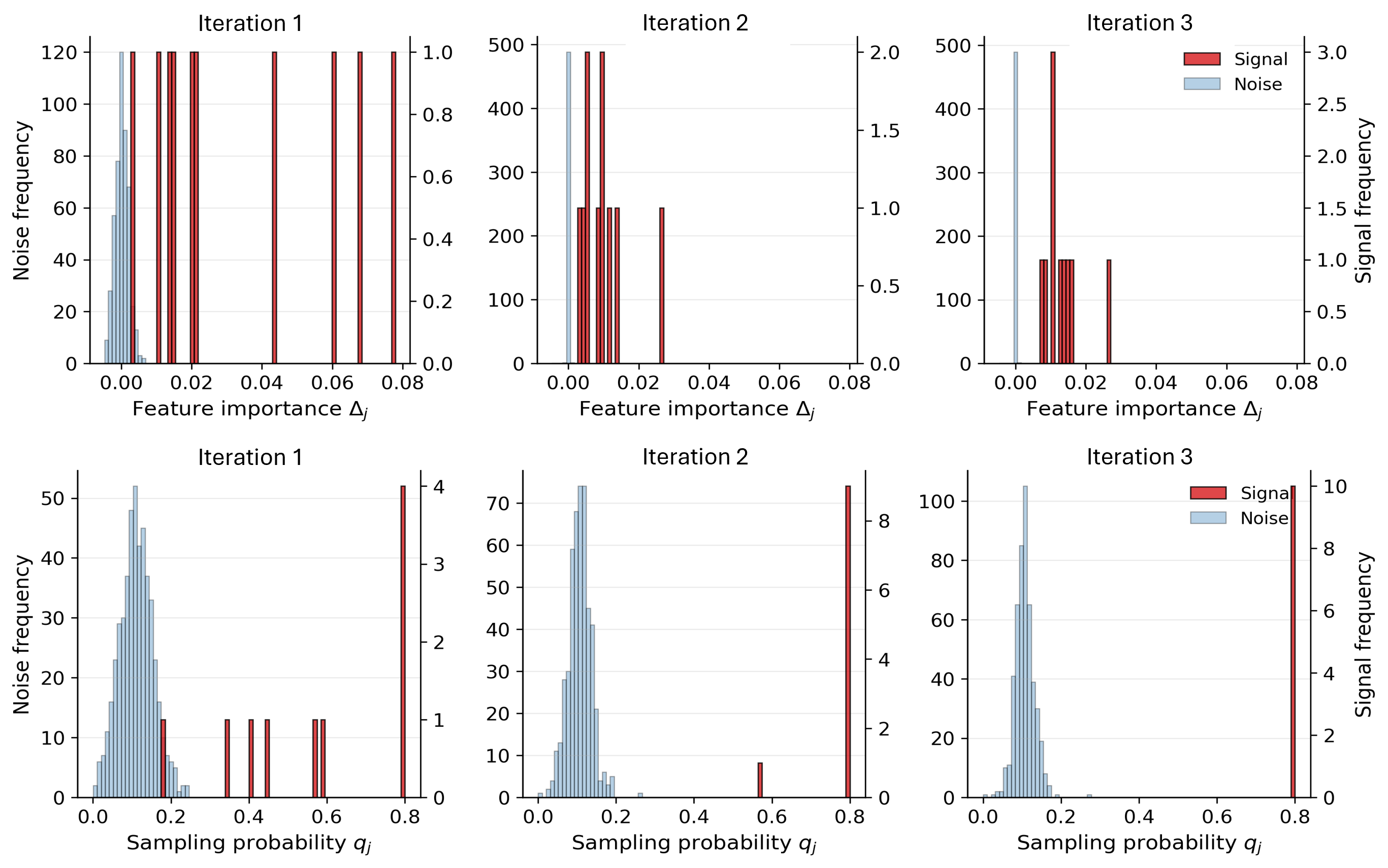}
\caption{\small
Evolution of LOCO feature importance scores and sampling probabilities in LAMPS under the same setup as in Figure~\ref{fig:locomp_split}.
The base learner is linear regression with the least squares estimate. Additional simulation details are provided in Section~\ref{sec:intro_fig_sim_details} of Appendix.}
\label{fig:prob_delta_vs_epoch}

\end{figure}

\section{Proposed Method: LOCO-Guided Adaptive Minipatch Sampling (LAMPS)}\label{sec:method}

Suppose we observe $N$ samples of feature-response pairs $(X_i,Y_i)_{i=1}^N$, where $X_i\in \bbR^M$ includes $M$ features and $Y_i\in \bbR$ is a continuous response. \notationchange{We denote the $j$th feature of sample $i$ by $X_i^{(j)}$.} We organize the observations as a design matrix $\boldsymbol{X}=(X_1^\top,\dots,X_N^\top)^\top \in \mathbb{R}^{N\times M}$ and a response vector $\boldsymbol{Y}=(Y_1,\dots,Y_N)^\top \in \mathbb{R}^N$. We assume that there exists a subset of features indexed by $S\subseteq[M]$, and a function $f:\bbR^{|S|}\rightarrow\bbR$, such that
\begin{equation}\label{eq:model}
Y_i=f(\notationchange{X_i^{(S)}}) + \varepsilon_i, \quad i=1,\dots,N,
\end{equation}
where $\{\varepsilon_i\}_{i=1}^N$ are independent mean-zero random variables, also independent of $\bX$. Furthermore, we assume $S$ is minimal in the sense that for any $S'\subsetneq S$, $\bbE(Y_i|X_i)=\bbE(Y_i|\notationchange{X_i^{(S)}})\neq \bbE(Y_i|\notationchange{X_i^{(S')}})$ holds with positive probability. Our goal is to identify the set $S$ given data $\bZ:=(\bX,\bY)$, using any regression algorithm suitable for the data.

To achieve this, we adopt the minipatch ensemble framework with built-in LOCO importance~\citep{GanZhengAllen2022MinipatchCI} and propose a new adaptive sampling scheme for feature selection. Specifically, we iterate between the following two modules. 
\begin{itemize}[leftmargin=*]
    \item {\bf Minipatch ensemble training and LOCO importance estimate}: given any prespecified base learner, we repeatedly apply it to $K$ randomly subsampled minipatches $(\bX_{I_k,F_k},\bY_{I_k})_{k=1}^K$ which yields $\{\mu_k(\cdot)\}_{k=1}^K$, where $I_k$ is a size-$n$ set uniformly sampled from $[N]$ without replacement, and $F_k\subset[M]$ is an independently sampled index set for features. To better accommodate non-uniform sampling weights, we adopt a Bernoulli-based sampling for $F_k$, different from \cite{GanZhengAllen2022MinipatchCI}: let $F_k=\{j:A_{k,j}=1\}$ where $A_{k,j}$ follows an independent Bernoulli distribution with parameter $q_j$, conditioned on the event that $\sum_{j=1}^M A_{k,j}>0$. $q_j$'s are initialized as $m/M$, where $m$ is a prespecified size parameter.
    
    Given the trained $\{\mu_k(\cdot)\}_{k=1}^K$, one can compute the LOCO importance estimates for each feature via simple model averaging, following the same steps as in \cite{GanZhengAllen2022MinipatchCI}. We adopt the squared error loss and the detailed procedure is summarized in Alg.~\ref{algo:minipatch_train}.
    \item {\bf Update feature sampling probability}: Given the unbounded LOCO importance estimates $\{\hat{\Delta}_j\}_{j=1}^M$, we propose a new adaptive scheme that transforms these scores to updated sampling probabilities $\bq=(q_1,\dots,q_M)^\top$. In particular, we incorporate the following three considerations.
    \begin{itemize}[leftmargin=*]
        \item {\bf Exploitation:} features with larger LOCO importance should be sampled more. We set $q_j\propto \omega_j$, where $\omega_j$ is a non-negative increasing function of $\hat\Delta_j$ to be fixed later.
        \item {\bf Training and computational efficiency:} to ensure the efficiency of minipatch training, the minipatch feature size should be controlled. We set $\sum_{j=1}^M q_j=m$ so that $\bbE(|F_k|)$ is close to a prespecified size parameter $0<m\ll M$; hence $q_j=\frac{m\omega_j}{\sum_{k=1}^M \omega_k}$.
        \item {\bf Exploration and regularization}: the maximum sampling probability for all features should be bounded away from 1 to encourage exploration, regularization, and accurate estimation of each LOCO importance score. We achieve this through thresholding: let the weight parameter $\omega_j=\til{\Delta}_j\wedge t^*$, where $\tilde{\Delta}_j=\hat{\Delta}_j-\min_l\hat\Delta_l+\frac{c_0}{M}$ shifts $\hat{\Delta}_j$ by the same margin across $j$ to ensure positiveness and lower bounded su, while $t^*$ is chosen as the largest threshold to ensure $\max_j q_j=\frac{m\max_j\omega_j}{\sum_{k=1}^M \omega_k}\leq \delta$ for a constant $0<\delta<1$. Therefore, the update behaves like a capped reweighting scheme: $q_j$ is approximately proportional to $\hat\Delta_j$ for moderately important features, but is truncated for extremely large scores.
    \end{itemize}
    Our capped normalization scheme ensures exploitation of promising features while also maintaining continued exploration. The detailed steps are summarized in Alg.~\ref{algo:IndepSampWeight}.
\end{itemize}
Iterating between these two modules for $B$ iterations, one can expect that larger values of the final-iteration sampling probability $q_j^{(B)}$ and LOCO importance scores $\hat\Delta_j^{(B)}$ should both be good indicators for the signal features. As illustrated in Figure~\ref{fig:prob_delta_vs_epoch}, $q_j^{(B)}$ often enjoys a clearer separation than $\hat\Delta_j^{(B)}$. Furthermore, since $q_j^{(B)}\in [0,\delta]$, we can simply estimate the signal feature set by $\hat{S}=\{j: q_j^{(B)}/\delta > 0.5\}$, where the threshold $0.5$ follows the majority-selection principle commonly used in stability selection~\citep{MeinshausenBuhlmann2010}. The full LAMPS algorithm is summarized in Alg.~\ref{algo:loco_adamp}. We note here that our LAMPS algorithm is inspired by but differs substantially from prior adaptive subspace sampling methods~\citep{YaoAllen2020Minipatch,LiuRockova2023TVS,Tian2021RaSEAV}, most of which assume access to binary reward for each feature through base model selection or examination of the inclusion of each feature in the best subspace within a sufficiently large group of subspaces. In contrast, our method designs a new adaptive strategy tailored to the unbounded LOCO importance metric.

\begin{algorithm}[t]
\caption{LOCO-guided Adaptive Minipatch Sampling (LAMPS) for Feature Selection}
\label{algo:loco_adamp}
\noindent\textbf{Input}: Training data $(\boldsymbol{X}, \boldsymbol{Y})$, minipatch sizes $n$, $m$; number of iterations $B$; number of minipatches per iteration $K$; base learner $H$; adaptivity parameter $\delta \in (0,1)$; constant buffer $c_0>0$.

\medskip
\noindent\textbf{Initialization}: Let $M$ be the number of features. Initialize sampling probabilities $\boldsymbol{q}^{(0)} = (m/M, \dots, m/M)$.

\noindent\textbf{Adaptive Minipatch Loop:} For $b = 1$ to $B$,
    \begin{enumerate}
        \item \textbf{Compute Feature Importance Scores:} \\
        Call Alg.~\ref{algo:minipatch_train} with input $(\boldsymbol{X}, \boldsymbol{Y}, n, K, \boldsymbol{q}^{(b-1)}, H)$ to obtain $\{\hat{\Delta}_j^{(b)}\}_{j=1}^M$.

        \item \textbf{Update Sampling Probabilities:} \\
        Call Alg.~\ref{algo:IndepSampWeight} with input $(\hat{\Delta}_1^{(b)}, \dots, \hat{\Delta}_M^{(b)}, m, \delta, c_0)$ to obtain updated sampling probability vector $\boldsymbol{q}^{(b)}$.
\end{enumerate}

\noindent\textbf{Output}: Feature subset $\hat{S}=\{\,j : q_j^{(B)} > \frac{\delta}{2}\}$.
\end{algorithm}

\begin{algorithm}[t]
\caption{LOCO-guided Updates for Feature Sampling Probabilities}
\label{algo:IndepSampWeight}
\noindent\textbf{Input}: Feature importance scores $\{\hat{\Delta}_1, \dots, \hat{\Delta}_M\}$; probability upper bound $\delta$;  expected minipatch feature size $m\leq \delta M$; constant buffer $c_0>0$.

\begin{enumerate}
    \item Compute non-negative scores: $\tilde{\Delta}_j = \hat{\Delta}_j - \min_\ell \hat{\Delta}_\ell + \frac{c_0}{M}$ for all $j$.

        \item Define $\omega_j(t) = \tilde{\Delta}_j \wedge t$ as the thresholded weight of feature $j$, a function of $t>0$. Find the desired threshold
        \begin{equation}\label{eq:threshold}
            t^*: = \sup\left\{0<t\leq \max_l\tilde{\Delta}_l:\frac{m\max_l\omega_l(t)}{\sum_{l=1}^M\omega_l(t)}\leq \delta\right\}
        \end{equation} 
        using Alg. \ref{algo:find_threshold}.
        \item Compute updated feature sampling probabilities $q_j = \frac{m \omega_j(t^*)}{\sum_{k=1}^M \omega_k(t^*)}$ for $j = 1, \dots, M$.

    \end{enumerate}
\noindent\textbf{Output}: Updated sampling probability vector $\boldsymbol{q} = (q_1, \dots, q_M)$.
\end{algorithm}

\subsection{Hyperparameter tuning}
\label{sec:tuning}
Although LAMPS involves many hyperparameters such as the minipatch size $(n,m)$, here we note that most of them are easy to tune and the rest can be set based on a universal rule. One convenient tool we leverage for hyperparameter tuning is the leave-one-out (LOO) error: given the trained minipatch ensemble, one can easily compute its leave-one-out prediction $\hat{\mu}_{-i}(\cdot)$ for each sample $i$ through simple model averaging without model-refitting as in Alg.~\ref{algo:minipatch_train}; the average mean squared LOO error can then be computed as $\frac{1}{N}\sum_{i=1}^N \big(Y_i-\hat{\mu}_{-i}(X_i)\big)^2$.

To tune the number of iterations $B$, we propose to track the LOO errors over iterations. Starting from $b=2$, we check if the LOO error at iteration $b$ decreases significantly from iteration $b-1$, and if not, we will choose $B$ as $b-1$. We detect the significant decrease through a heuristic paired z-test for LOO error comparison. Detailed steps of LAMPS with data-driven selection of $B$ are summarized in Alg.~\ref{algo:loco_adamp_stopping}. We can also tune the minipatch size $(n,m)$ by running LAMPS over a grid of candidate pairs and choosing the pair with the lowest final-iteration LOO error. Similarly, when there are multiple candidate base learners to choose from, we can use LOO errors as the criterion without resorting to a separate cross-validation.

For the selection of the number of minipatches $K$, we note that a larger $K$ reduces variance and is always helpful statistically. We propose to choose $K$ large enough so that each feature and feature pair is sampled at least $C_1$ and $C_2$ times on average given initial sampling probability: $K=\left\lceil\max\left\{
\frac{C_1M}{(1-n/N)m},
\frac{C_2M^2}{(1-n/N)m^2}\right\}\right\rceil$. Throughout our empirical studies, we set $C_1=200$ and $C_2=50$.

The parameter $\delta\in (0,1)$ is the capped feature sampling probability, which ensures sufficient exploration and accurate estimation of the LOCO importance of each feature. In all our empirical studies, we set $\delta=0.8$, yielding robust performance.

\section{Theoretical Guarantees}\label{sec:theory}
In this section, we establish exact feature selection guarantees for LAMPS. Section~\ref{sec:thm1} presents a general deterministic convergence guarantee for LAMPS under a sparse additive model with independent features, subject to small LOCO importance estimation errors. We then study when the LOCO importance can be estimated sufficiently accurately with high probability in Section~\ref{sec:thm2}, characterizing the required conditions on sample size, base-learner training errors, and ensemble hyperparameters. Section~\ref{sec:interaction_theory} extends our main theory beyond additive models, showing that LAMPS can identify features contributing through feature interactions. Finally, Section~\ref{sec:correlated_theory} discusses why LAMPS can be a promising strategy for handling correlated features via illustrative theory and supporting simulations.

We start by describing a general functional ANOVA model that sets the basis for our discussion throughout Sections~\ref{sec:thm1}-\ref{sec:interaction_theory}. For any $F\subseteq[M]$, let $\notationchange{X^{(F)}}$ denote the subvector of a generic feature vector $\notationchange{X=(X^{(1)},\dots,X^{(M)})^\top}$ indexed by $F$. Suppose that $f(\notationchange{X^{(S)}})$ in \eqref{eq:model} can be decomposed as
\begin{equation}
\label{eq:general_interaction_model}
f(\notationchange{X^{(S)}})=f_0+\sum_{w\in \cS} f_w(\notationchange{X^{(w)}}),
\end{equation}
where $f_0$ is a constant, $\cS\subset 2^{S}$ includes all feature subsets $w\subset S$ which jointly contribute to $f(\cdot)$ through $f_w(\notationchange{X^{(w)}})$. To ensure identifiability and following the convention in the functional ANOVA literature~\citep{Hoeffding1948ACO}, we assume that $\bbE(f_w(\notationchange{X^{(w)}})\mid \notationchange{X^{(u)}})=0$ a.s. for any $u\subsetneq w$, which can always be satisfied by defining $f_w(\cdot)$ appropriately, under the independent features assumption made throughout Sections~\ref{sec:thm1}-\ref{sec:interaction_theory}. We also define $\|f_w\|_2^2:= \bbE[f_w^2(\notationchange{X^{(w)}})]$ as the signal strength of feature set $w$. Given i.i.d. samples $(X_i,Y_i)_{i=1}^N$ from model~\ref{eq:general_interaction_model}, we aim to recover the signal feature set $S=\{j:\exists w\in \cS\text{ s.t. }w\owns j\}$.

\subsection{Exact Feature Selection under Additive Models: Deterministic Guarantees}\label{sec:thm1}
Here, we consider the additive model, a special case of \eqref{eq:general_interaction_model}: 
\begin{equation}\label{eq:additive_mod}
f(X)=\sum_{j\in S} f_j(\notationchange{X^{(j)}}).
\end{equation}
To show the exact feature selection, we first require two mild assumptions on the minipatch feature size and feature signal strength heterogeneity. 
\begin{assumption}[Minipatch feature size scaling]
\label{ass:minipatch_feature_scaling}
The minipatch feature size $m$ satisfies

$$\max\Big\{C_1\log M, C_1(s\delta+1)\Big\} \;\le\; m \;<\; \frac{\delta}{4} M$$
for some universal constant $C_1>1$, where $s = \lvert S \rvert$ and $\delta\in(0,1)$ is the probability upper bound used in Alg.~\ref{algo:loco_adamp}. 
\end{assumption}
We require $m$ large enough compared to the sparsity to cover all signal features, but small enough compared to the total number of features $M$. Assumption~\ref{ass:minipatch_feature_scaling} can be easily satisfied by choosing $m$ as a small fraction of $M$, as long as $M\gg 1$ and $s\ll M$. 

\begin{assumption}[Mild conditioning amongst feature importance]
\label{ass:signal_strength_ratio}
The total and minimum feature importance satisfy $$\sum_{k\in \cS}\|f_k\|_2^2\leq c_1M^{(C_1-1)}\min_{j\in \cS}\|f_j\|_2^2$$ 
for a sufficiently small $c_1>0$, where $C_1>1$ is the universal constant in Assumption~\ref{ass:minipatch_feature_scaling}.
\end{assumption}
Assumption~\ref{ass:signal_strength_ratio} is much weaker than assuming all signal strengths to be on the same order; instead, with bounded sparsity $s$, it is implied by $\max_{j\in S}\|f_j\|_2^2\lesssim M^{C_1-1} \min_{j \in S}\|f_j\|_2^2$. 

For our LOCO-guided procedure to succeed, one key requirement is that the estimated LOCO importance accurately reflects the population importance. In each iteration $b$, we define the oracle importance $\Delta_j^{\star(b)}$ of feature $j$ as in \eqref{eq:loco_def}, but replacing $\mu$ (resp. $\mu_{\no j}$) by the ensembled minipatch predictor (resp. without $j$) when there is an infinite number of minipatches and when each minipatch predictor $\mu_{I,F}(\notationchange{X^{(F)}})=\bbE(Y\mid \notationchange{X^{(F)}})$ is the oracle predictor (see Definition~\ref{def:Delta_defs} in the Appendix). Define the LOCO importance estimation error vector $\varepsilon^{(b)}\in \bbR^M$ as $\DErrj{j}{b} = \hat{\Delta}_j^{(b)} - \min_l\hat{\Delta}_l^{(b)} - (\Delta_j^{\star(b)} - \min_l\Delta_l^{\star(b)}),\,1\leq j\leq M$. 
The following Assumptions impose requirements on $\DErrj{j}{b}$.
\begin{assumption}[Signal strength requirement]
\label{ass:min_sig} 
    For any $j\in S$ and any $1\leq b\leq B$,
    \begin{equation}\label{eq:SNR_additive}
    \|f_j\|_2^2
    \;>\;
    \frac{6\big|\varepsilon_j^{(b)}\big|}{q_j^{(b-1)}}\vee \frac{24\|\varepsilon^{(b)}_{S^c}\|_1}{m-s\delta}.
\end{equation}
In addition, the small constant $c_0>0$ in Alg.~\ref{algo:IndepSampWeight} satisfies $c_0<\frac{m-s\delta}{24}\min_{k\in S}\|f_k\|_2^2$.
\end{assumption}
Assumption~\ref{ass:min_sig} requires each signal strength $\|f_j\|_2^2$ to dominate its relative LOCO importance estimation error $|\varepsilon_j^{(b)}|/q_j^{(b-1)}$ and the aggregated error $\|\varepsilon_{S^c}^{(b)}\|_1/(m-s\delta)$ of noise features. We scale the estimation error $\varepsilon_j^{(b)}$ by the inverse sampling probability $1/q_j^{(b-1)}$ since the magnitude of $\Delta_j^{\star(b)}$ is roughly scaled as $q_j^{(b-1)}\|f_j\|_2^2$.  
Similarly, we divide $\|\varepsilon_{S^c}^{(b)}\|_1$ by $m-s\delta$ since the sum of sampling probabilities of all noise features is lower bounded by $m-s\delta$. We also impose an upper bound on $c_0$ to prevent the positive shift $c_0/M$ in Alg.~\ref{algo:IndepSampWeight} from weakening signal-noise separation. 

\begin{assumption}[Small estimation error of noise feature importance]\label{ass:noise_err}
    The constant $c_0>0$ in Alg.~\ref{algo:IndepSampWeight} satisfies
    $c_0>\frac{4m}{\delta}\Big(\max_{1\leq b\leq B}\big\|\varepsilon_{S^c}^{(b)}\big\|_{\infty} + 10M^{-C_1}\sum_{k\in \cS}\|f_k\|_2^2\Big)$.
\end{assumption}
Assumption~\ref{ass:noise_err} requires $c_0$ in Alg.~\ref{algo:IndepSampWeight} to dominate the maximum LOCO estimation error among noise features, preventing false positives; see more discussion in Remark~\ref{rmk:c_0}. 

\begin{thm}[Exact support recovery after $B$ iterations]
\label{thm:main1}
Consider the additive model~\eqref{eq:additive_mod} with independent features $\notationchange{X^{(1)}},\dots,\notationchange{X^{(M)}}$. Suppose Alg.~\ref{algo:loco_adamp} is run for $B$ iterations with the squared error loss $\err(Y,\hat Y)=(Y-\hat Y)^2$, yielding the output $\hat{S}$. The following hold:
\begin{enumerate}
    \item {\bf (No false negative)}: If Assumptions~\ref{ass:minipatch_feature_scaling}--\ref{ass:min_sig} hold, then there exists a constant
$C>0$ such that whenever $B \;>\; C\log\!\Big(\frac{\delta M}{2m}\Big)$, $\hat S \supset S$.
    \item {\bf (No false positive)}: If Assumptions~\ref{ass:minipatch_feature_scaling}, \ref{ass:signal_strength_ratio}, and \ref{ass:noise_err} hold, then $\hat S \subset S$.
\end{enumerate}
\end{thm}

Therefore, as long as the signal strength is sufficiently large compared to the estimation error of LOCO importance and the hyperparameters $m,\,c_0$ are appropriately chosen, Alg.~\ref{algo:loco_adamp} can exactly recover the true signal feature set $S$ after only $O(\log(M/m))$ iterations. 
\begin{remark}[Choice of $c_0$]\label{rmk:c_0}
The parameter $c_0$ in Alg.~\ref{algo:IndepSampWeight} plays a critical role in Assumptions~\ref{ass:min_sig} and \ref{ass:noise_err}. As shown in Theorem~\ref{thm:main1}, to ensure no false negative, $c_0$ needs to be small enough, since we do not want the baseline shift $c_0/M$ to diminish the relative advantage of true signal features. On the other hand, to ensure no false positives, $c_0$ needs to be sufficiently large to provide a lower bound on the sum of the thresholded LOCO importance scores $\omega_j(t^*)$ in Alg.~\ref{algo:IndepSampWeight}, so that a small error in the noise feature importance score does not easily lead to a false positive. Here we note that false positives are a less serious concern since the total number of selected features by LAMPS is typically small (theoretically bounded by $2m/\delta$). Therefore, we recommend that practitioners choose a small $c_0>0$ just to ensure the stability of Alg.~\ref{algo:IndepSampWeight}. Throughout our experiments, we set $c_0=0.01$. 
\end{remark}

\subsection{Exact Feature Selection under Additive Models: Probabilistic Guarantees}\label{sec:thm2} 
Theorem~\ref{thm:main1} illustrates that when the LOCO importance estimation errors are sufficiently small (Assumptions~\ref{ass:min_sig} and \ref{ass:noise_err}), our LAMPS algorithm yields exact selection for the signal features. In this section, we examine when such assumptions hold with high probability. W.L.O.G., throughout Section~\ref{sec:thm2}, we assume that $\bbE(Y^2)=\sum_{j\in \cS}\bbE(f_j^2(\notationchange{X^{(j)}})) +\bbE(\varepsilon^2)=1$. Let $\SNR_j = \frac{\|f_j\|_2^2}{\bbE(\varepsilon^2)}$ denote the SNR associated with the $j$th feature, and let $\SR_j = \frac{\|f_j\|_2^2}{\sum_{k\in S}\|f_k\|_2^2}$ be the proportion of signal tied to feature $j$. In the special case where $\{\|f_j\|_2:j\in S\}$ are of the same order, $\SR_j \asymp \frac{1}{s}$. In addition, let $\SNR_{\min}=\min_{j\in S}\SNR_j$, $\SR_{\min} = \min_{j\in S}\SR_j$.

\begin{assumption}[Bounded predictions and responses]
\label{ass:adp_bounded}
     Suppose that the response variable and the base minipatch predictions are bounded:
    $|Y|\leq C,\,|\mu_{I,F}(X;\bZ)|\le C$, where $\mu_{I,F}(\cdot;\bZ)$ is the predictor fitted on minipatch indices $(I,F)$ given training data $\bZ$.
\end{assumption}

\begin{assumption}[Minipatch size]
\label{ass:small_n}
     There exits a constant $c>0$ such that $m\geq -\log(1-\delta-c)$; $\frac{n}{N}\leq c'(\log M+\log N)^{-1}\big(\SNR_{\min}^2\wedge \SR_{\min}^2\big)$ for a sufficiently small constant $c'>0$; 
\end{assumption}
\noindent The requirement on $m$ in Assumption~\ref{ass:small_n} is mild: when $\delta=0.8$, $m\geq 2$ suffices. The key condition requires the minipatch sample size to be much smaller than the full sample size, ensuring sufficient stability of the ensemble predictor and hence small estimation error of the LOO estimates for the LOCO importance in Algorithm~\ref{algo:minipatch_train}. The condition is more relaxed when the SNR is larger and the signal strengths are more homogeneous. {\em We note that $M$ appears in a $\log$ term and thus is allowed to grow faster than $N$.}

\begin{assumption}[Number of minipatches]
\label{ass:adp_n_minipatches}
     The number of random minipatches $K$ in each iteration satisfies 
     $K \geq C\frac{M}{m}(\log M+\log N)\big(\SNR_{\min}^{-2}\vee \SR_{\min}^{-2}\big)$ for a sufficiently large $C>0$.
\end{assumption}
\noindent We require that sufficiently many minipatches be sampled to accurately estimate the LOCO importance.
$K$ needs to grow as the feature sampling ratio $m/M$ decreases.
\begin{assumption}[Data-independent Sampling-path Approximation]\label{ass:Q_stb}
    There exists a series of probability vectors $\{\bq^{\star(b)}\in(0,\delta]^M\}_{b=1}^B$ independent of $\bZ$ and satisfying $\sum_{j=1}^M (\bq^{\star(b)})_j=m$, such that for a sufficiently small $c>0$, 
    $d(\bq^{(b)},\bq^{\star(b)})\leq c\big(\SNR_{\min}\wedge \SR_{\min}\big)$, where $d(\cdot,\cdot)$ is a sum of $\ell_1$ distance and entry-wise relative distances, with detailed definition in the Appendix (see Assumption~\ref{ass:Q_stb_full}).
\end{assumption}
\noindent Because the same data are reused across iterations, the LOCO estimation error at iteration $b$ shares randomness with the sampling probabilities updated from prior LOCO importance. Assumption~\ref{ass:Q_stb} controls only this extra dependence. It compares the realized sampling path with a nearby data-independent reference path, which may still evolve across iterations. For example, $\bq^{\star(b)} = \bbE(\bq^{(b)})$ is one possible choice when $\bq^{(b)}$ concentrates around its mean.
\begin{assumption}[Small training error of bagged predictors]\label{ass:train_err}
    For any $F\subset[M]$, define $$\varepsilon_{\rm train}(F) = \bbE_X\left[\left|\frac{1}{\binom{N}{n}}\sum_{I\subset[N]:|I|=n}\mu_{I,F}(X;\bZ)-\mu_F^*(X)\right|\mid\bZ\right],$$ 
    where $\mu_{I,F}(\cdot;\bZ)$ is defined in Assumption~\ref{ass:adp_bounded} and $\mu_F^*(X)=\bbE(Y|\notationchange{X^{(F)}})$ is the oracle predictor.
    \begin{equation}\label{eq:train_err_bnd}
    \bbE_{F\sim Q^{\star(b)}}\varepsilon_{\rm train}(F),\max_{1\leq j\leq M}\bbE_{F\sim Q^{\star(b)}_{\no j}}\varepsilon_{\rm train}(F),\max_{1\leq j\leq M}\bbE_{F\sim Q^{\star(b)}_j}\varepsilon_{\rm train}(F)\leq \frac{c(m-s\delta)}{m}\min_{j\in S}\|f_j\|_2^2
    \end{equation}
    hold for a sufficiently small constant $c>0$, where $Q^{\star(b)}$ corresponds to the feature sampling probability vector $\bq^{\star(b)}$ in Assumption~\ref{ass:Q_stb}, $Q^{\star(b)}_j$ (resp. $Q^{\star(b)}_{\no j}$) is the conditional distribution of $F\sim Q^{\star(b)}$ given that $j\in F$ (resp. $j\notin F$).
\end{assumption}
\noindent In order for the LAMPS framework to select the truly important features, Assumption~\ref{ass:train_err} requires that the bagged base learners be well-trained. Note here that we don't require the ensembled minipatch predictor in each iteration to be close to the full population model $f$, or the base model $\mu_{I,F}(\cdot)$ trained on each minipatch to be sufficiently accurate; instead, $\varepsilon_{\rm train}(F)$ is defined for the bagged predictor for a fixed feature subset $F$, and then \eqref{eq:train_err_bnd} looks at the training errors averaged over the random feature indices $F$. Furthermore, $m$ is typically small, so with high probability, $\varepsilon_{\rm train}(F)$ concerns the prediction error of a low-dimensional problem instead of the $M$-dimensional problem. 
\begin{assumption}[Minimum-importance features are not heavily sampled]
\label{ass:no_min_for_large_q}
For each iteration $1\le b \le B$, let $l_{\min}^{(b)} := \arg\min_{l\in[M]}\hat{\Delta}_l^{(b)}$, $l_{\min}^{*(b)} := \arg\min_{l\in[M]}\Delta_l^{*(b)}$ denote indices attaining the minimum empirical and oracle feature importance scores, respectively. 
There exists a constant $C>0$ such that for all $1\le b \le B$,
$q^{(b-1)}_{\,l_{\min}^{(b)}},\quad q^{(b-1)}_{\,l_{\min}^{*(b)}}  \leq \frac{Cm}{M}$.
\end{assumption}
\noindent Assumption~\ref{ass:no_min_for_large_q} requires that for each iteration, heavily sampled features do not have the lowest importance score. This ensures small estimation error in $\min_l \hat{\Delta}_l^{(b)}$, and hence the probability updating procedure (Alg.~\ref{algo:IndepSampWeight}), where all feature importance scores are shifted by $-\min_l \hat{\Delta}_l^{(b)}$, would be less sensitive to data randomness.
\begin{thm}\label{thm:main2}
    Consider the additive model~\eqref{eq:additive_mod} with independent features. Suppose that Assumptions~\ref{ass:minipatch_feature_scaling}, \ref{ass:signal_strength_ratio}, \ref{ass:adp_bounded}-\ref{ass:no_min_for_large_q} hold, and $c_0 = c(m-s\delta)\min_{k\in S}\|f_k\|_2^2$ for some constant $c>0$. Then with probability at least $1-(MN)^{-c}$, the output $\hat{S}$ given by Alg.~\ref{algo:loco_adamp} with squared loss achieves exact recovery $\hat{S}=S$ whenever $C\log\big(\frac{\delta M}{2m}\big)\leq B \leq C'M$ for some $C,\,C'>0$.
\end{thm}
\noindent Therefore, LAMPS can achieve exact feature selection provided that the base models are bounded and estimated sufficiently well, the hyperparameters are chosen appropriately, and the adaptive sampling distribution is not too sensitive to data randomness. 

\subsection{Beyond Additive Models: Selecting Features with Interactive Importance}\label{sec:interaction_theory}

In this section, we consider a special case of \eqref{eq:general_interaction_model}: an additive model with one interaction term involving two features, where one element in $\cS$ is $\{u,v\}$ and all other elements are singletons. To better illustrate how LAMPS works for detecting interactively important features, we assume that $u$, $v$ only enter the model through their pairwise interaction; that is, the singleton sets $\{u\},\,\{v\}\notin \cS$. Formally, we write
\begin{equation}
\label{eq:interaction_mod}
f(X)
=
f_0+\sum_{l\in S\setminus\{u,v\}} f_l(\notationchange{X^{(l)}})+f_{uv}(\notationchange{X^{(u)}},\notationchange{X^{(v)}}).
\end{equation}
In the following, we will show Alg.~\ref{algo:loco_adamp} can achieve exact recovery for the true signal feature set $S$ under similar assumptions to the additive model, only with some adaptation to the original signal strength condition (Assumption~\ref{ass:min_sig}).
\begin{assumption}[Signal strength requirement]
\label{ass:interaction_pair_min_sig} 
    The following holds for $1\leq b\leq B$: for any $j\in S\setminus\{u,v\}$, \eqref{eq:SNR_additive} holds; while for $j\in \{u,v\}$,
    \begin{equation}\label{eq:SNR_interaction}
    \|f_{uv}\|_2^2
    >
    \frac{72|\varepsilon_j^{(b)}|}
    {41 q_u^{(b-1)}q_v^{(b-1)}}
    \vee
    \frac{12\|\varepsilon^{(b)}_{S^c}\|_1}
    {(m-s\delta)(q_u^{(b-1)}\wedge q_v^{(b-1)})}.
\end{equation}
In addition, the small constant $c_0>0$ in Alg.~\ref{algo:IndepSampWeight} satisfies $c_0<\frac{m(m-s\delta)}{12M}\|f_{uv}\|_2^2 \wedge \frac{m-s\delta}{24}\min_{k\in \cS}\|f_k\|_2^2$.
\end{assumption}

\begin{thm}[Exact support recovery after $B$ iterations]
\label{thm:interaction}
Consider the interaction model~\eqref{eq:interaction_mod}. We recover the same guarantee as in Theorem~\ref{thm:main1} provided that Assumption~\ref{ass:min_sig} is replaced by Assumption~\ref{ass:interaction_pair_min_sig}, and the factor $M^{C_1-1}$ in Assumption~\ref{ass:signal_strength_ratio} is replaced by $M^{C_1-2}$ with $C_1>2$. 
\end{thm}

\noindent Therefore, Alg.~\ref{algo:loco_adamp} can also recover features $u,\,v$ even if they only contribute through their interaction, as long as the interactive importance $\|f_{uv}\|_2^2$ is sufficiently strong: 
note the extra factor $1/(q_u^{(b-1)}\wedge q_v^{(b-1)})$ on the R.H.S. of \eqref{eq:SNR_interaction}. This is because the contribution of $f_{uv}$ would only show up when both features $u$ and $v$ are sampled in the same minipatch, thus harder to detect than $j\in S\backslash \{u,v\}$.

\subsection{Illustrative Theory for Correlated Features}\label{sec:correlated_theory}
Feature correlation arises frequently in practice and remains a long-standing challenge for feature selection and feature importance inference. In particular, the LOCO metric can suffer from correlation-induced distortion and yield low importance for a feature $j$ if it is highly correlated to another, as feature $j$'s contribution can be mostly replaced by its correlated proxy~\citep{Verdinelli2021DecorrelatedVI}. On the other hand, as has been shown in \cite{YaoAllen2020Minipatch,GanZhengAllen2022MinipatchCI}, minipatch-based approaches are inherently suited for correlated settings, as feature subsampling breaks the feature correlation structure. Our LAMPS algorithm also enjoys similar advantages in this context: the sampling probability of each feature is capped at $\delta<1$, thus for any signal feature, its correlated proxy cannot fully replace its contribution. To make this intuition more concrete, we show the following illustrative theory under a simple linear model setup. 

\begin{prop}[Informal]
\label{prop:correlated_informal}
Suppose $Y=\beta \notationchange{X^{(1)}}+\varepsilon$ with zero-mean $\varepsilon$ independent of all features; $(\notationchange{X^{(1)}},\notationchange{X^{(2)}})$ is independent of other features, satisfying $\bbE(\notationchange{X^{(1)}}\mid \notationchange{X^{(2)}})=\rho \notationchange{X^{(2)}}$. An oracle version of LAMPS with squared error loss and $B$ iterations satisfies:
\begin{enumerate}
    \item $\forall\rho\in [-1,1]$, with appropriately chosen $m$ and $c_0$, $q_1^{(b)}$ increases geometrically and $S\subset \hat{S}$ if $B=\Theta(\log (M/m))$.
    \item $\hat S\subseteq S$ if $\rho$ is bounded by an increasing function of $B$.
\end{enumerate}
\end{prop}

A formal version of Proposition~\ref{prop:correlated_informal} is included in the Appendix. Although Proposition~\ref{prop:correlated_informal} considers a very simple setup, it enables a clean illustration and shows interesting insights for LAMPS. We would like to highlight two points: (i) a signal feature can be selected even if it is perfectly correlated with other features, in sharp contrast to the behavior of LOCO importance without feature subsampling; (ii) a noise feature correlated with a signal feature can act as a weak signal feature; adaptive sampling could help eliminate the correlated proxy as our condition on correlation relaxes as the iteration increases. A full-fledged theory that characterizes feature selection performance of LAMPS given general feature correlation is an interesting direction that we leave for future work.
 
We also conduct linear and nonlinear simulation studies with one correlated signal-noise feature pair to validate these insights. Figure~\ref{fig:corr_theory} shows the feature sampling probability trajectory of the feature pair under a nonlinear model, where the signal feature probability quickly increases to the cap $\delta=0.8$ regardless of the correlation, while the noise feature probability first increases, then drops back to below the threshold when $\rho\leq 0.7$. The detailed simulation setup and results across varying dimensions and for the corresponding linear-model setting are reported in Section~\ref{sec:intro_fig_sim_details}-\ref{sec:supp_correlated_theory_validation} of the Appendix.

\begin{figure}[t]
\centering
\includegraphics[width=.82\linewidth]{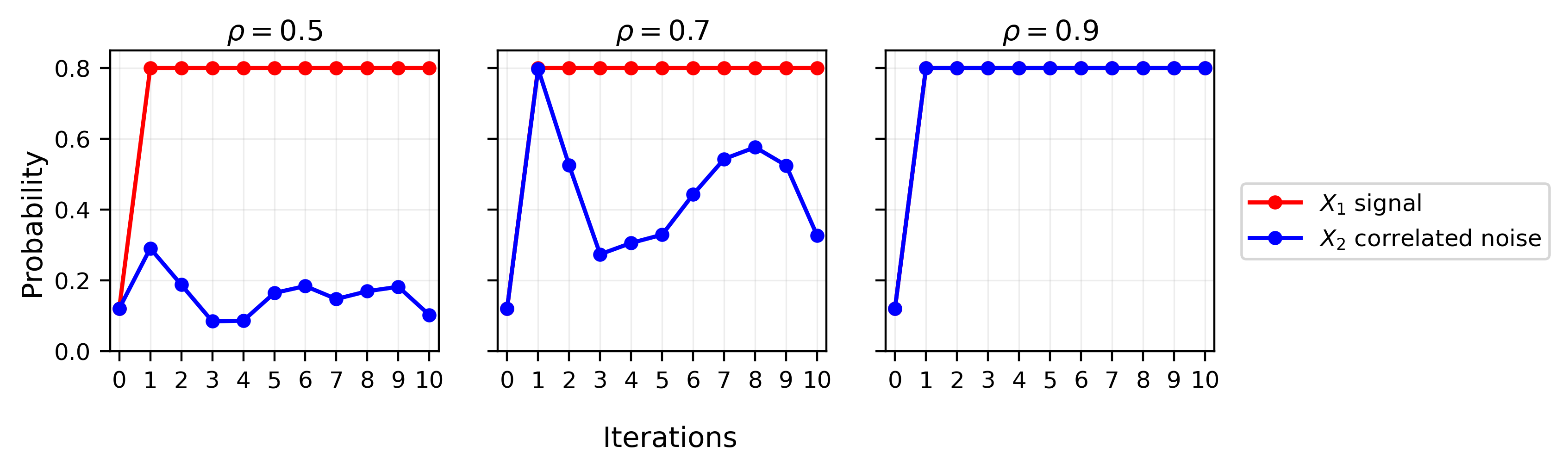}
\vspace{-1em}
\caption{
\small Sampling probabilities in the one-signal correlated-noise setting with $N=500$ and $M=100$.
The base learner is regression tree. Additional simulation details are provided in Section~\ref{sec:intro_fig_sim_details} of the Appendix. 
}
\label{fig:corr_theory}
\end{figure}

\subsection{Proof Sketch}

To better understand why LAMPS works, we provide a proof sketch for our main theory (no false negative results in Theorems~\ref{thm:main1} and \ref{thm:main2}) and highlight the main technical novelty.

\paragraph{Dynamics of the probability updates.} When proving Theorem~\ref{thm:main1}, one key building block is the geometric increase of $q_j^{(b)}$ for any signal feature $j$: 
\begin{equation}\label{eq:proof_sketch_lwrbnd}
    q_j^{(b)}\geq Cq_j^{(b-1)}\wedge\delta, 
\end{equation}
where $C>1$ is a universal constant, immediately implying $\hat{S}\owns S$ given logarithmic iterations $B=\Theta(\log(M/m))$. We prove \eqref{eq:proof_sketch_lwrbnd} by exploiting the fact that $q_j^{(b)}$ relies on the contrast between signal feature importance $\tilde{\Delta}_j^{(b)}$ and noise feature importance $\tilde{\Delta}_{S^c}^{(b)}$:
\begin{equation}\label{eq:proof_sketch_lwrbnd2}
    q_j^{(b)}\geq \frac{\til{\Delta}_j^{(b)}(m-s\delta)}{\sum_{l\in S^c}\til{\Delta}_l^{(b)}}\wedge\delta,
\end{equation}
where $\til{\Delta}_j^{(b)}$ can be approximately viewed as an estimate for $q_j^{(b-1)}\|f_j\|_2^2$. Hence our SNR condition $(m-s\delta)\|f_j\|_2^2>24\|\varepsilon_{S^c}^{(b)}\|_1$ in Assumption~\ref{ass:min_sig} allows us to establish \eqref{eq:proof_sketch_lwrbnd}.

\paragraph{Verifying the SNR condition.} In Theorem~\ref{thm:main2}, we validate that our SNR requirement can be satisfied (e.g., $\|\varepsilon_{S^c}^{(b)}\|_1$ can be well controlled by $(m-s\delta)\|f_j\|_2^2$) with high probability even when $M$ is much larger than $N$. In fact, this is highly non-trivial and does not follow directly from prior work on minipatch LOCO inference: \citet{GanZhengAllen2022MinipatchCI} conduct inference for one feature at a time and the proof only bounds $\varepsilon_j^{(0)}$ for one given $j$ under the initial uniform feature sampling. In contrast, to accommodate the high dimensionality, we bound the relative estimation error $\varepsilon_j^{(b)}/q_j^{(b-1)}$ uniformly over $1\leq j\leq M$; hence $\|\varepsilon_{S^c}\|_1$ can be well controlled on the scale of $\sum_{j=1}^M q_j^{(b-1)}=m$, the feature subsampling size, instead of $M$. To achieve this, we develop a new and fine-grained analysis built on top of (i) a careful decomposition of the LOCO error $\varepsilon_j^{(b)}$ to show the role of $q_j^{(b-1)}$ and (ii) an exponential tail bound for $\varepsilon_j^{(b)}$ by rewriting it as an average of $N^3/n^2$ weakly dependent r.v.s. Furthermore, we prove a smoothness property of the LOCO importance $(\hat\Delta_j^{(b)})_{j=1}^M$ as a function of the sampling probability $\bq^{(b-1)}$ (see Lemma~\ref{lem:DeltaErrDiffQ}), allowing us to safely replace $q_j^{(b-1)}$ by a nearby data-independent probability $q_j^{\star(b-1)}$ in our analysis of $\varepsilon_j^{(b)}/q_j^{(b-1)}$, thereby decoupling the dependence between $q_j^{(b-1)}$ and $\varepsilon_j^{(b)}$. We prove Lemma~\ref{lem:DeltaErrDiffQ} by constructing a coupling that keeps two random feature subsets close when their marginal distributions are close (Lemma~\ref{lem:F1_F2_diff}).

\section{Simulations}\label{sec:simulation}\label{sec:simulation_settings}
We evaluate the feature selection performance of LAMPS in linear, nonlinear additive, and nonlinear interaction models under varying signal strengths and correlation structures. Software for LAMPS implementation and code for reproducing our empirical results can be found at \hyperlink{LAMPS}{https://github.com/liu-xuhui/lamps/tree/main}.
Throughout, the signal feature set is $S=\{1,\ldots,10\}$. The feature–response relationship is specified separately for each experiment below. For the joint feature distribution, we consider two settings: in correlated setting~1, we generate feature vectors $X_i\in \bbR^M$ independently from a zero-mean Gaussian distribution with Toeplitz covariance
$\Sigma_{ab}=\rho^{|a-b|},\;\rho\in\{0,0.5,0.9\}$; in correlated setting~2, we apply a random permutation $\pi:[M]\rightarrow [M]$ to the feature indices and use $\Sigma_\pi=\Sigma[\pi,\pi]$ instead. Thus, when $\rho>0$, signal features form a correlated block in setting~1, whereas correlations span signal and noise features in setting~2.

Throughout Sections~\ref{sec:simulation} and \ref{sec:casestudy}, LAMPS is run with squared-error loss, minipatch subsampling ratios $n/N=0.4$ and $m/M=0.12$, $\delta=0.8$, and an ensemble size $K$ satisfying the coverage conditions in Section~\ref{sec:tuning}. As described in Section~\ref{sec:tuning}, we use data-driven selection for the number of iterations $B$ (see Algorithm~\ref{algo:loco_adamp_stopping}) while keeping it capped at 5. We also include a simulation study on data-driven tuning of minipatch sizes in the Appendix.

\paragraph{Linear Model Experiments.} We set $(N,M)=(200,500)$. For $j\in S^c$, we set $\beta_j=0$, while for $j\in S$, $\beta_j=\xi_jU_j$, where $U_j\stackrel{\mathrm{i.i.d.}}{\sim}\mathrm{Unif}(2,3)$, and $\xi_j$ independently follows the Rademacher distribution.

The response follows $Y=X\beta+\varepsilon \;\;\varepsilon\sim\cN(0,\sigma^2)$, and the $\mathrm{SNR}$ is enforced via $\sigma^2=\Var(X\beta)/\mathrm{SNR}$.
We run 10 independent replicates for every combination of $\rho$ and SNR, and use ordinary least squares without an intercept as the LAMPS base learner. 
We compare LAMPS with widely used feature-selection procedures for linear models including Lasso \citep{Tibshirani1996Lasso} with five-fold cross-validation,
Lasso with eBIC \citep{chen2008extended}, Stability Selection \citep{MeinshausenBuhlmann2010}, CPSS \citep{ShahSamworth2013}, and
Elastic Net \citep{ZouHastie2005ElasticNet} with five-fold cross-validation. 

We report F1 scores across SNRs and correlation levels in Figure~\ref{fig:linear-nonoracle}.
In correlated setting~1, LAMPS maintains the highest F1 scores under moderate and strong correlation and outperforms all baselines.
In correlated setting~2, introducing permutation correlation reduces performance for all methods, but LAMPS remains among the strongest methods. These results suggest that LAMPS is robust to both correlated signal groups and signal--noise correlation.
To isolate performance gain/loss due to differing sparsity-levels of different methods, we additionally report oracle
sparsity-tuning results given by Lasso and Elastic Net as well as precision and recall of each method in the Appendix, where similar comparison results hold.

\begin{figure}[t]
\centering
\includegraphics[width=.82\linewidth]{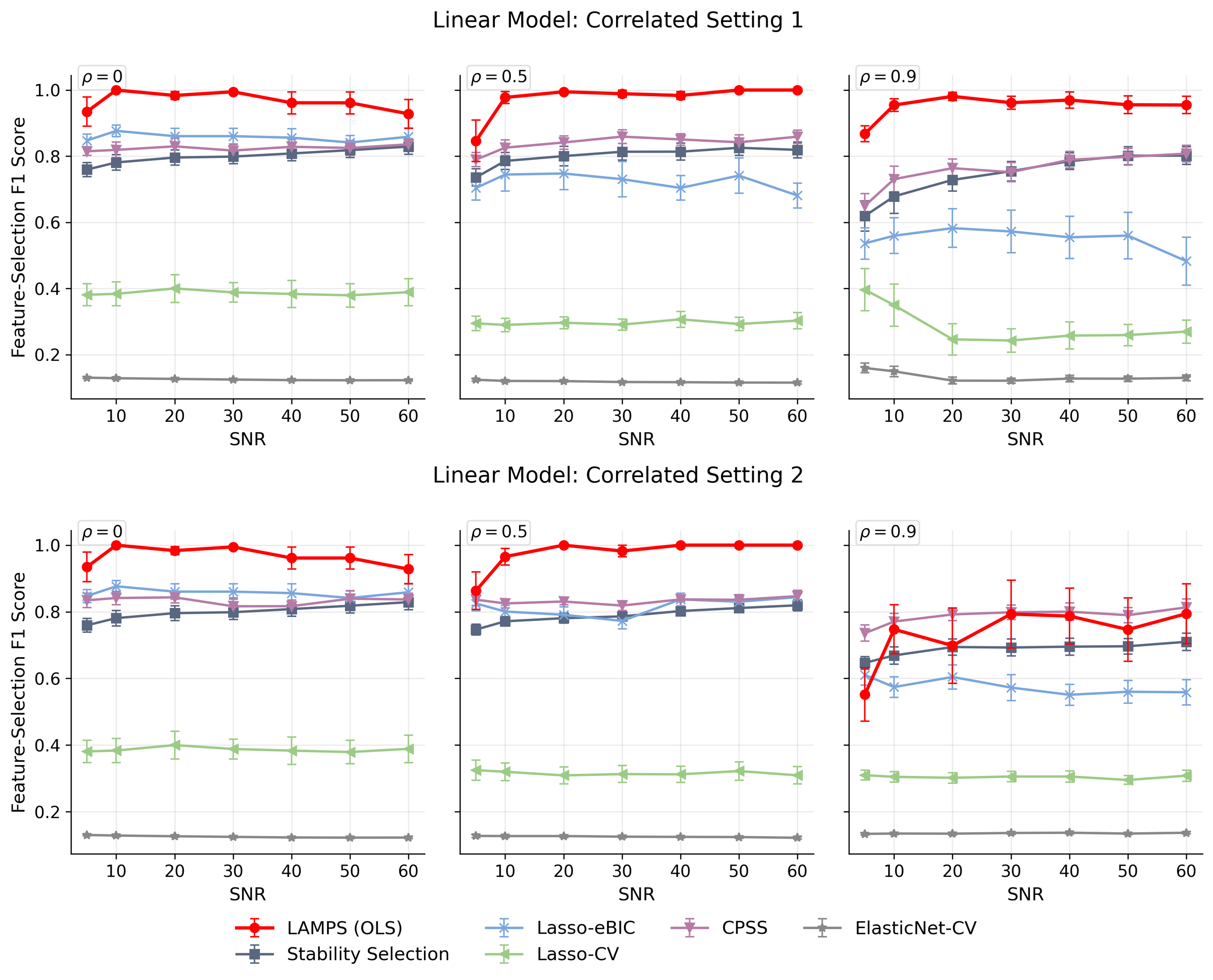}
\caption{
\small Feature-selection performance (F1 score) in linear models across signal-to-noise ratios (SNR).
Top row: Linear Setting 1 with non-permuted Toeplitz covariance; bottom row: Linear Setting 2 with permuted covariance.
Error bars represent $\pm$ one standard error over 10 Monte Carlo replicates.
}
\label{fig:linear-nonoracle}
\end{figure}

\paragraph{Nonlinear Additive Model Experiments.} We set $(N,M)=(200,500)$ and let
$f(\notationchange{X^{(S)}})=\sum_{\ell=1}^{10}\beta_\ell\,g_\ell(\notationchange{X^{(\ell)}}),\;
\beta_\ell \stackrel{\text{i.i.d.}}{\sim}\mathrm{Unif}(2,3)$,
where the component functions are $g_1(x)=x,\, g_2(x)=x,\, g_3(x)=\sin(\pi x),\, g_4(x)=x^2,\, g_5(x)=\text{exp}(x),\,g_6(x)=\log(|x|+1),\, g_7(x)=\max(0,x),\, g_8(x)=\mathbf{1}\{x>0\},\,
g_9(x)=x,\, g_{10}(x)=x.$

Each $g_\ell(\notationchange{\bX_{\cdot,\ell}})$ is normalized by its sample variance on the design. Observations follow
$Y=f(\notationchange{X^{(S)}})+\varepsilon,\;\; \varepsilon\sim\cN(0,\sigma^2),$ and the $\mathrm{SNR}$ is enforced via $\sigma^2=\Var(f(\notationchange{X^{(S)}}))/\mathrm{SNR}.$ We use 10 replicates for every combination of $\rho$ and SNR.
For this nonlinear setting, we use MARS~\citep{friedman1991multivariate, MilborrowEarth} and SpAM~\citep{Ravikumar2009SpAM, jiang2026package} as base learners for LAMPS. Although LAMPS can accommodate any black-box learner, we choose MARS and SpAM because they can also serve as feature selection baselines, allowing us to better isolate the effect of the LAMPS framework. Aside from MARS and SpAM (tuned with cross-validation), two additional nonlinear feature-selection baselines are considered: the Model-X Knockoffs \citep{Barber2014ControllingTF, candes2018panning} and HSIC-Lasso \citep{Yamada2014HSICLasso}. Knockoff is implemented with the random forest feature statistic with swap importance and target FDR levels $q\in\{0.1,0.2,0.3\}$; HSIC-Lasso requires the desired sparsity as input and is therefore reported only in oracle sparsity-tuning comparisons. Figure~\ref{fig:additive-nonoracle} reports F1 scores of the two LAMPS variants, MARS, SpAM, and knockoffs. Under independence, only MARS performs comparably to LAMPS. In the correlated setting~1, the performance of MARS decreases as within-signal correlation increases, whereas both LAMPS variants remain robust against correlation. In correlated setting~2, signal--noise correlation lowers accuracy for all methods, but LAMPS generally retains the highest or near-highest F1 scores. Notably, LAMPS with both base learners shows improved selection performance compared to the corresponding base learners applied to the full data set (MARS and SpAM) without minipatch ensembling and adaptive sampling.
We additionally report precision and recall as well as oracle sparsity-tuning results in the Appendix, allowing us to separately assess the roles of sparsity tuning and signal/noise feature ranking in the performance gains.

\begin{figure}[t]
\centering
\includegraphics[width=.82\linewidth]{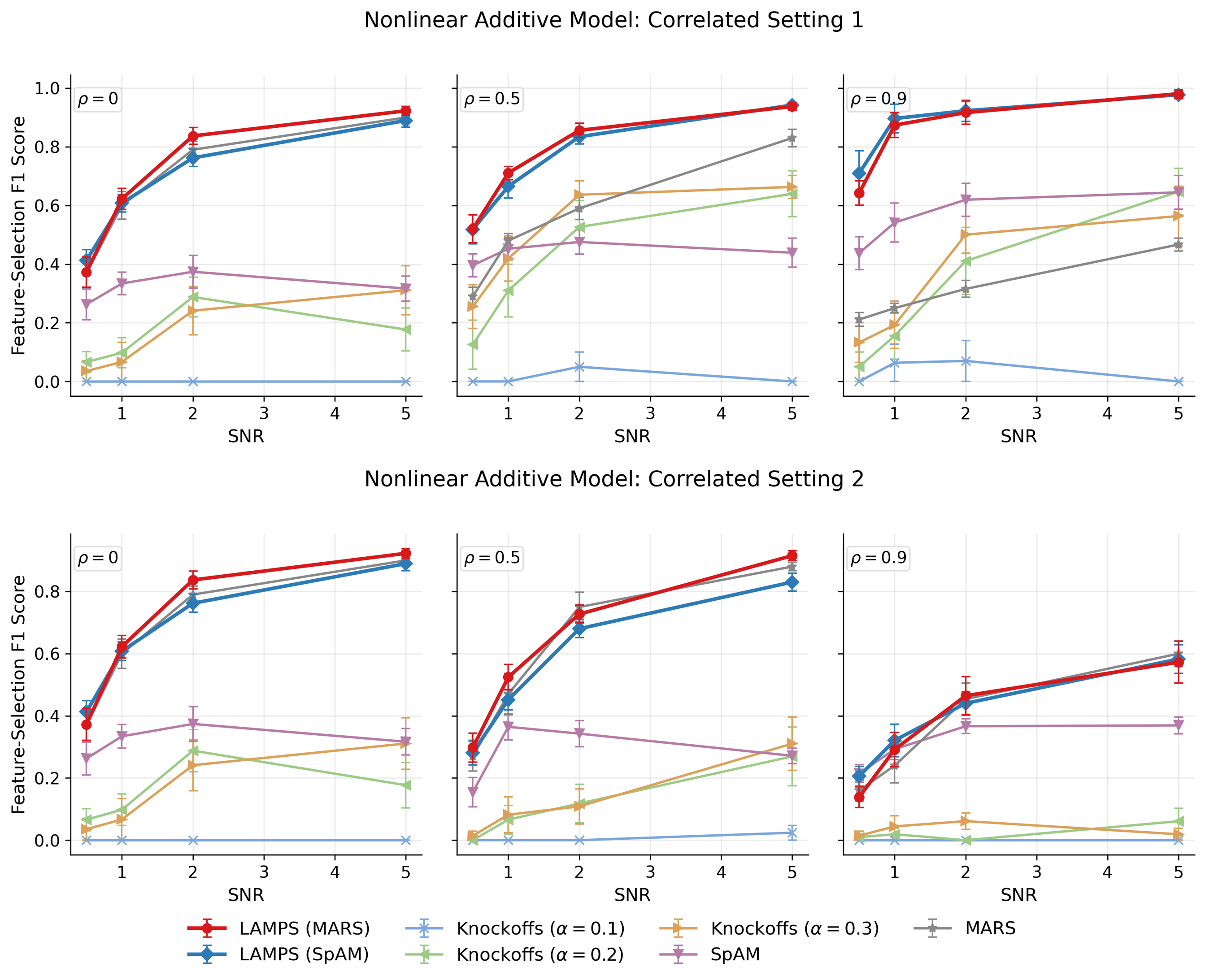}
\caption{
\small Feature-selection performance (F1 score) for nonlinear additive models across signal-to-noise ratios (SNR). Top row: Nonlinear Additive Setting~1 with non-permuted Toeplitz covariance; bottom row: Nonlinear Additive Setting~2 with permuted covariance.
Error bars represent $\pm$ one standard error over 10 Monte Carlo replicates.
}
\label{fig:additive-nonoracle}
\end{figure}

\paragraph{Nonlinear Nonadditive Model Experiments: Detection of Features in Interactions.} Here, we investigate LAMPS' performance in selecting features that contribute only through interactions. We consider $(N,M)=(200,50)$ and modify the nonlinear additive model by replacing the first two univariate components with the interaction term $\notationchange{X^{(1)}}\notationchange{X^{(2)}}$. The remaining components on features $3,\ldots,10$ are unchanged and are normalized as above.
The regression function is therefore of the form
$f(X)=\kappa \notationchange{X^{(1)}}\notationchange{X^{(2)}}+\sum_{\ell=3}^{10}\beta_\ell g_\ell(\notationchange{X^{(\ell)}}),
\;
\beta_\ell\stackrel{\mathrm{i.i.d.}}{\sim}\mathrm{Unif}(2,3),
\;
\kappa\in\{4,6,8,10\}$.
Observations follow
$Y=f(X)+\varepsilon, \;\; \varepsilon\sim\mathcal N(0,\sigma^2)$ with  $\sigma^2=\operatorname{Var}\{f(X)\}$.
We generate 30 replicates for each configuration. LAMPS again uses either MARS or SpAM as the base learner; the baseline methods are the same as in the nonlinear additive study.

Our main evaluation metric is the rate of simultaneously detecting the interacting features ${1,2}$. Figure~\ref{fig:interaction-rate} reports the detection rate across interaction strengths and correlation levels. LAMPS with MARS achieves the highest and often near-perfect detection rates and substantially outperforms MARS applied to the full dataset. SpAM and LAMPS with SpAM also perform well in correlated setting~1 when $\rho=0.5$ or $0.9$, despite SpAM’s additive structure. This occurs because the strong correlation between $\notationchange{X^{(1)}}$ and $\notationchange{X^{(2)}}$ allows the interaction $\notationchange{X^{(1)}}\notationchange{X^{(2)}}$ to be approximated by a univariate nonlinear effect such as $\notationchange{\bigl(X^{(1)}\bigr)^2}$. In setting~2, the permutation weakens the correlation between $\notationchange{X^{(1)}}$ and $\notationchange{X^{(2)}}$, making this additive approximation less effective.

\begin{figure}[t]
\centering
\includegraphics[width=.82\linewidth]{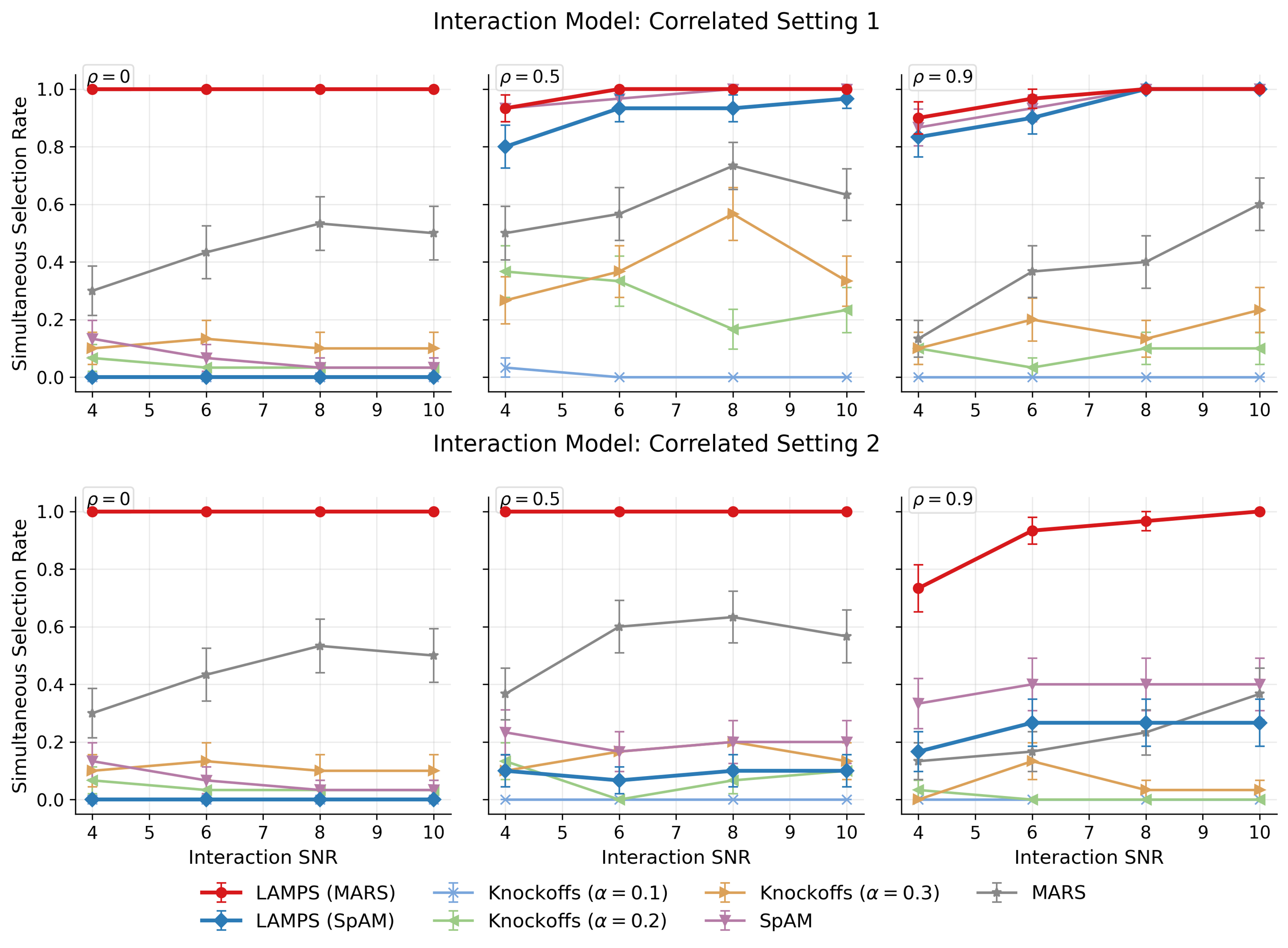}
\caption{
\small Detection rate for interactively important feature pair $(u,\,v)$ in nonlinear nonadditive models across interaction strength $\kappa$.
Top row: Nonlinear Interaction Setting~1 with non-permuted covariance; bottom row: Nonlinear Interaction Setting~2 with permuted covariance.
Error bars represent $\pm$ one standard error over 30 Monte Carlo replicates.
}
\label{fig:interaction-rate}
\end{figure}

\section{Case Studies}\label{sec:casestudy}

We evaluate the feature selection performance of LAMPS on two real-world data sets. 

\subsection{Riboflavin Data Analysis}
\label{sec:riboflavin}

The riboflavin data \citep{dezeure2015high} consist of $N=71$ observations and $M=4{,}088$ continuous covariates measuring log gene-expression levels in Bacillus subtilis. The response is the logarithm of the riboflavin production rate \citep{ZamboniEtAl2005}, and the goal is to identify genes associated with riboflavin production. In this data set, the two genes
\texttt{YXLD\_at} and \texttt{YXLE\_at} have a large sample
correlation $0.978$, providing a representative scenario in which correlated signals pose
challenges for feature selection.

Since the feature size is much larger than the sample size, we apply LAMPS with Lasso \citep{Tibshirani1996Lasso} as the base learner to select the important genes, and compare it with commonly used linear feature selection baselines: complementary-pairs stability selection (CPSS)
\citep{ShahSamworth2013}, cross--validated Elastic Net
\citep{ZouHastie2005ElasticNet}, cross--validated Lasso
\citep{Tibshirani1996Lasso}, EBIC-tuned Lasso
\citep{chen2008extended}, and Stability Selection
\citep{MeinshausenBuhlmann2010}. The hyperparameter configuration for LAMPS follows the same set up as in Section~\ref{sec:simulation}; more implementation details are included in the Appendix.

We report the selection sizes given by all methods in Table~\ref{tab:riboflavin-method-compare}. Lasso and Elastic Net tuned via cross--validation select substantially larger sets than others; CPSS returns an empty set, while EBIC--Lasso, Stability Selection, and our LAMPS select compact gene sets. We further observe that LAMPS selects three genes that are commonly reported in benchmark studies~\citep{chichignoud2016practical,LedererMueller2015TREX,bar2020scalable,GongEtAl2021StatisticaSinica}: \texttt{YOAB\_at}, \texttt{YXLD\_at}, and \texttt{YXLE\_at}. Biologically, all three genes are plausibly related to production rate:
\texttt{YXLD\_at} and \texttt{YXLE\_at} are co-transcribed genes that act as major negative regulators of the extracytoplasmic function sigma factor SigY~\citep{cao2003regulation}; \texttt{YOAB\_at} has
been annotated as a transporter-related gene involved in central metabolic processes~\citep{uniprot2025uniprot}. Notably, LAMPS simultaneously selects the highly correlated pair, \texttt{YXLD\_at} and \texttt{YXLE\_at}, which are difficult to detect simultaneously using the original LOCO importance metric or $\ell_1$-penalty based methods.

\begin{table}[t]
\centering
\caption{Riboflavin data: comparison across selection procedures. We report the number of
selected features ($|\widehat S|$) and indicate whether the most stable core signals are
retained. Full feature lists for methods selecting many genes are provided in the
supplement.}
\label{tab:riboflavin-method-compare}
\scalebox{0.85}{\begin{tabular}{l c c c c}
\hline
Method & $|\widehat S|$ & \texttt{YOAB\_at} & \texttt{YXLD\_at} & \texttt{YXLE\_at} \\
\hline
CPSS \citep{ShahSamworth2013} & 0  & -- & -- & -- \\
Elastic Net (CV) \citep{ZouHastie2005ElasticNet} & 73 & \checkmark & \checkmark & \checkmark \\
Lasso (CV) \citep{Tibshirani1996Lasso} & 41 & \checkmark & \checkmark & \checkmark \\
Lasso (EBIC) \citep{chen2008extended} & 1  & -- & -- & -- \\
Stability Selection \citep{MeinshausenBuhlmann2010} & 2 & \checkmark & \checkmark & -- \\
{\bf LAMPS (Lasso)} & 3 & \checkmark & \checkmark & \checkmark \\
\hline
\end{tabular}}
\end{table}

\subsection{ROSMAP Data Analysis}
\label{sec:rosmap}

We next evaluate LAMPS on a transcriptomic prediction task using data from the Religious Orders Study and Rush Memory and Aging Project (ROSMAP), a pair of ongoing longitudinal clinicopathologic cohort studies of aging and Alzheimer’s disease (AD) \citep{bennett2018religious}.
We use a preprocessed data set with $N=507$ observations and $M=200$ covariates measuring gene expressions from RNA sequencing. We aim to detect important genes associated with a continuous response: the global cognition score of each subject.

To accommodate potentially nonlinear relationships between gene expression and cognitive outcome, we use multivariate adaptive regression splines (MARS) \citep{friedman1991multivariate} as the base learner within LAMPS. Throughout this experiment, we use the same LAMPS hyperparameters as in Section~\ref{sec:simulation}. We compare LAMPS with HSIC Lasso \citep{Yamada2014HSICLasso}, SpAM \citep{Ravikumar2009SpAM}, and Lasso \citep{Tibshirani1996Lasso}. To quantitatively evaluate feature-selection performance, we assess the predictive utility of the selected features using random forest regressors. To isolate the effect of sparsity, we fix the feature selection size for each method as $k\in \{1,\dots,20\}$, by choosing the probability threshold in LAMPS and the regularization parameters in SpAM and Lasso appropriately. For HSIC Lasso, we simply select top-k features following
the strategy used in~\cite{Yamada2014HSICLasso}.
For evaluation, we use 10-fold cross-validation where the training fold is used for feature selection and downstream random forest training given the selected features, with mean squared error computed on the test fold.

Figure~\ref{fig:rosmap-testmse} shows the 10-fold cross-validated test MSE as a function of the number of selected genes. The LAMPS curve decreases over the first few selected features and then remains relatively stable. Across all model sizes considered, LAMPS achieves substantially lower mean test MSE than HSIC Lasso, SpAM, and Lasso.
When using our default selection threshold $\delta/2 = 0.4$, LAMPS selects a compact set of genes, reported in Table~\ref{tab:rosmap_selection_frequency} of the Appendix. Across the 10 splits, a total of 9 features are selected, among which 4 genes are consistently selected. Additional biological context for the selected genes is provided in Section~\ref{sec:supp_case_study_details} of the Appendix.

\begin{figure}[t]
\centering
\includegraphics[width=.7\linewidth]{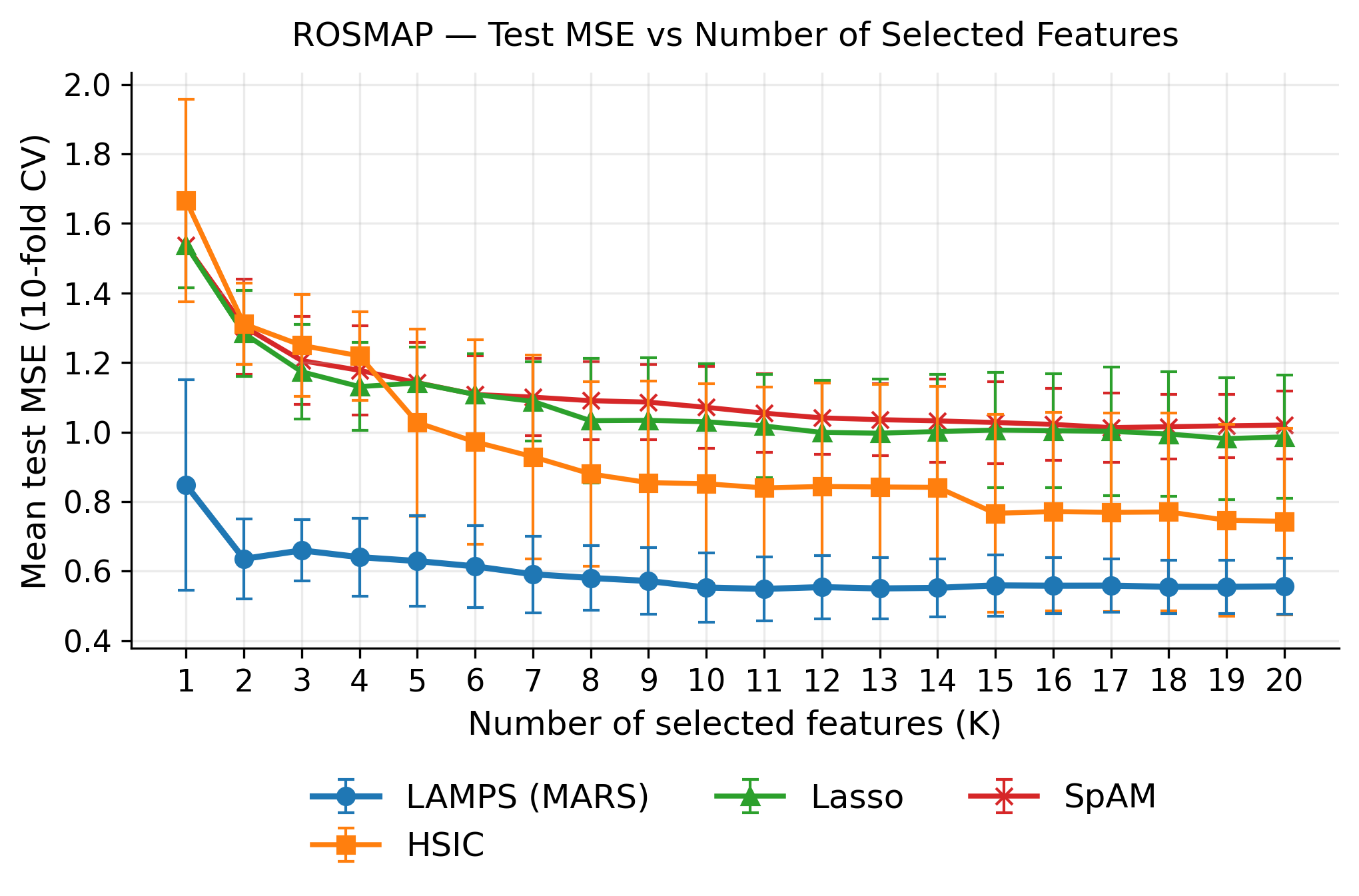}
\caption{\small Cross-validated prediction performance in the ROSMAP data as a function of the number of selected features. Error bars represent standard deviation over 10 folds.
}
\label{fig:rosmap-testmse}
\end{figure}

\section{Discussion}
In this paper, we propose LAMPS (LOCO-guided Adaptive Minipatch Sampling), a model-agnostic feature selection framework that can leverage any regression algorithm for selecting important features relevant for the response. By using LOCO importance as the feature-relevance criterion, LAMPS can be paired with arbitrary base learners that need only provide predictions.
Leveraging the minipatch ensemble framework with double-subsampling of observations and features, the LOCO importance is estimated without data splitting and easily computed even in high dimensions. Our key innovation is a new adaptive feature subsampling scheme, tailored for the unbounded LOCO feature importance, that gradually shifts the sampling distribution towards features with higher LOCO importance. Our adaptive feature sampling leads to clearer signal-noise separation and consequently, a compact and accurate set of important features.
Theoretically, we establish exact support recovery guarantees for LAMPS under high-dimensional sparse models with independent features, when the base learners are sufficiently well trained and the hyperparameters are appropriately chosen. We further provide illustrative theory and simulations to shed light on LAMPS' behavior under feature correlation. Finally, our extensive simulations and case studies demonstrate strong feature-selection performance of LAMPS across linear and nonlinear models, with particular advantages in correlated settings and in identifying interactively important features.

As a new framework, LAMPS also opens up several interesting directions for future work. First, our analysis of correlated features is restricted to an illustrative setting, and developing a general theory under complex feature dependence would provide a more complete understanding of LAMPS. Second, an interesting direction is to extend LAMPS to accommodate structured data, such as time-series observations, thereby further broadening its applicability. Finally, the adaptive sampling idea underlying LAMPS may have the potential to be extended to efficiently detect higher-order feature interactions.

\section*{Acknowledgments}
The authors thank Genevera I. Allen and Jason M. Klusowski for helpful discussions during the preparation of this manuscript.

\bibliographystyle{apalike}
\bibliography{main_ref,Supplement_ref}

\begin{thebibliography}{}

\bibitem[Allen et~al., 2023]{allen2023interpretable}
Allen, G.~I., Gan, L., and Zheng, L. (2023).
\newblock Interpretable machine learning for discovery: Statistical challenges
  and opportunities.
\newblock {\em Annual Review of Statistics and Its Application}, 11.

\bibitem[Bar et~al., 2020]{bar2020scalable}
Bar, H.~Y., Booth, J.~G., and Wells, M.~T. (2020).
\newblock A scalable empirical bayes approach to variable selection in
  generalized linear models.
\newblock {\em Journal of Computational and Graphical Statistics},
  29(3):535--546.

\bibitem[Barber and Cand{\`e}s, 2014]{Barber2014ControllingTF}
Barber, R.~F. and Cand{\`e}s, E.~J. (2014).
\newblock Controlling the false discovery rate via knockoffs.
\newblock {\em Annals of Statistics}, 43:2055--2085.

\bibitem[Bennett et~al., 2018]{bennett2018religious}
Bennett, D.~A., Buchman, A.~S., Boyle, P.~A., Barnes, L.~L., Wilson, R.~S., and
  Schneider, J.~A. (2018).
\newblock Religious orders study and rush memory and aging project.
\newblock {\em Journal of Alzheimer’s disease}, 64(s1):S161--S189.

\bibitem[B{\"u}hlmann and Yu, 2002]{BuhlmannYu2002Bagging}
B{\"u}hlmann, P. and Yu, B. (2002).
\newblock Analyzing bagging.
\newblock {\em The annals of Statistics}, 30(4):927--961.

\bibitem[Candes et~al., 2018]{candes2018panning}
Candes, E., Fan, Y., Janson, L., and Lv, J. (2018).
\newblock Panning for gold:‘model-x’knockoffs for high dimensional
  controlled variable selection.
\newblock {\em Journal of the Royal Statistical Society Series B: Statistical
  Methodology}, 80(3):551--577.

\bibitem[Cannings and Samworth, 2017]{cannings2017random}
Cannings, T.~I. and Samworth, R.~J. (2017).
\newblock Random-projection ensemble classification.
\newblock {\em Journal of the Royal Statistical Society Series B: Statistical
  Methodology}, 79(4):959--1035.

\bibitem[Cao et~al., 2003]{cao2003regulation}
Cao, M., Salzberg, L., Tsai, C.~S., Mascher, T., Bonilla, C., Wang, T., Ye,
  R.~W., M{\'a}rquez-Magana, L., and Helmann, J.~D. (2003).
\newblock Regulation of the bacillus subtilis extracytoplasmic function protein
  $\sigma$y and its target promoters.
\newblock {\em Journal of bacteriology}, 185(16):4883--4890.

\bibitem[Cattaneo et~al., 2026]{cattaneo2026inference}
Cattaneo, M.~D., Klusowski, J.~M., and Underwood, W.~G. (2026).
\newblock Inference with mondrian random forests.
\newblock {\em Journal of the Royal Statistical Society Series B: Statistical
  Methodology}, 88(3):1060--1085.

\bibitem[Chen and Chen, 2008]{chen2008extended}
Chen, J. and Chen, Z. (2008).
\newblock Extended bayesian information criteria for model selection with large
  model spaces.
\newblock {\em Biometrika}, 95(3):759--771.

\bibitem[Chen et~al., 2026]{chen2025revisiting}
Chen, X., Klusowski, J.~M., Tan, Y.~S., and Yu, C. (2026).
\newblock Revisiting randomization in greedy model search.

\bibitem[Chen et~al., 2025]{chen2025top}
Chen, Y., Tang, T., and Allen, G. (2025).
\newblock Top-$k$ feature importance ranking.

\bibitem[Chichignoud et~al., 2016]{chichignoud2016practical}
Chichignoud, M., Lederer, J., and Wainwright, M.~J. (2016).
\newblock A practical scheme and fast algorithm to tune the lasso with
  optimality guarantees.
\newblock {\em Journal of Machine Learning Research}, 17(229):1--20.

\bibitem[Consortium, 2025]{uniprot2025uniprot}
Consortium, T.~U. (2025).
\newblock Uniprot: the universal protein knowledgebase in 2025.
\newblock {\em Nucleic acids research}, 53(D1):D609--D617.

\bibitem[Dezeure et~al., 2015]{dezeure2015high}
Dezeure, R., B{\"u}hlmann, P., Meier, L., and Meinshausen, N. (2015).
\newblock High-dimensional inference: confidence intervals, p-values and
  r-software hdi.
\newblock {\em Statistical science}, pages 533--558.

\bibitem[Fan and Li, 2001]{FanLi2001SCAD}
Fan, J. and Li, R. (2001).
\newblock Variable selection via nonconcave penalized likelihood and its oracle
  properties.
\newblock {\em Journal of the American statistical Association},
  96(456):1348--1360.

\bibitem[Friedman, 1991]{friedman1991multivariate}
Friedman, J.~H. (1991).
\newblock Multivariate adaptive regression splines.
\newblock {\em The annals of statistics}, 19(1):1--67.

\bibitem[Gan and Allen, 2022]{gan2022fast}
Gan, L. and Allen, G.~I. (2022).
\newblock Fast and interpretable consensus clustering via minipatch learning.
\newblock {\em PLOS Computational Biology}, 18(10):e1010577.

\bibitem[Gan et~al., 2026]{GanZhengAllen2022MinipatchCI}
Gan, L., Zheng, L., and Allen, G.~I. (2026).
\newblock Loco feature importance inference without data splitting via
  minipatch ensembles.

\bibitem[Gavinsky et~al., 2015]{gavinsky2015tail}
Gavinsky, D., Lovett, S., Saks, M., and Srinivasan, S. (2015).
\newblock A tail bound for read-k families of functions.
\newblock {\em Random Structures \& Algorithms}, 47(1):99--108.

\bibitem[Gil and Ulitsky, 2020]{gil2020regulation}
Gil, N. and Ulitsky, I. (2020).
\newblock Regulation of gene expression by cis-acting long non-coding rnas.
\newblock {\em Nature Reviews Genetics}, 21(2):102--117.

\bibitem[Gong et~al., 2021]{GongEtAl2021StatisticaSinica}
Gong, S., Zhang, K., and Liu, Y. (2021).
\newblock Penalized linear regression with high-dimensional pairwise screening.
\newblock {\em Statistica Sinica}, 31(1):391--420.

\bibitem[Guyon and Elisseeff, 2003]{guyon2003introduction}
Guyon, I. and Elisseeff, A. (2003).
\newblock An introduction to variable and feature selection.
\newblock {\em Journal of machine learning research}, 3(Mar):1157--1182.

\bibitem[Guyon et~al., 2002]{guyon2002gene}
Guyon, I., Weston, J., Barnhill, S., and Vapnik, V. (2002).
\newblock Gene selection for cancer classification using support vector
  machines.
\newblock {\em Machine learning}, 46(1):389--422.

\bibitem[He and Allen, 2026]{he2026cluster}
He, C.~M. and Allen, G.~I. (2026).
\newblock Cluster loco: Feature importance for interpreting clusters.

\bibitem[Hoeffding, 1948]{Hoeffding1948ACO}
Hoeffding, W. (1948).
\newblock A class of statistics with asymptotically normal distribution.
\newblock {\em Annals of Mathematical Statistics}, 19:308--334.

\bibitem[Jiang et~al., 2026]{jiang2026package}
Jiang, H., Ma, Y., Liu, H., Roeder, K., Li, X., Zhao, T., and Jiang, M.~H.
  (2026).
\newblock {\em Package ‘SAM’}.

\bibitem[Lederer and M{\"u}ller, 2015]{LedererMueller2015TREX}
Lederer, J. and M{\"u}ller, C. (2015).
\newblock Don't fall for tuning parameters: tuning-free variable selection in
  high dimensions with the trex.
\newblock In {\em Proceedings of the AAAI conference on artificial
  intelligence}, volume~29.

\bibitem[Lei et~al., 2018]{lei2018distribution}
Lei, J., G’Sell, M., Rinaldo, A., Tibshirani, R.~J., and Wasserman, L.
  (2018).
\newblock Distribution-free predictive inference for regression.
\newblock {\em Journal of the American Statistical Association},
  113(523):1094--1111.

\bibitem[LeJeune et~al., 2024]{lejeune2024asymptotics}
LeJeune, D., Patil, P., Javadi, H., Baraniuk, R.~G., and Tibshirani, R.~J.
  (2024).
\newblock Asymptotics of the sketched pseudoinverse.
\newblock {\em SIAM Journal on Mathematics of Data Science}, 6(1):199--225.

\bibitem[Lemhadri et~al., 2021]{Lemhadri2021LassoNet}
Lemhadri, I., Ruan, F., Abraham, L., and Tibshirani, R. (2021).
\newblock Lassonet: A neural network with feature sparsity.
\newblock {\em Journal of Machine Learning Research}, 22(127):1--29.

\bibitem[Li et~al., 2005]{li2005lasso}
Li, F., Yang, Y., and Xing, E. (2005).
\newblock From lasso regression to feature vector machine.
\newblock {\em Advances in neural information processing systems}, 18.

\bibitem[Lin and Zhang, 2006]{Lin2006COSSO}
Lin, Y. and Zhang, H.~H. (2006).
\newblock {Component selection and smoothing in multivariate nonparametric
  regression}.
\newblock {\em The Annals of Statistics}, 34(5):2272 -- 2297.

\bibitem[Liu and Ro{\v{c}}kov{\'a}, 2023]{LiuRockova2023TVS}
Liu, Y. and Ro{\v{c}}kov{\'a}, V. (2023).
\newblock Variable selection via thompson sampling.
\newblock {\em Journal of the American Statistical Association},
  118(541):287--304.

\bibitem[Luo and Daniels, 2024]{luo2024variable}
Luo, C. and Daniels, M.~J. (2024).
\newblock Variable selection using bayesian additive regression trees.
\newblock {\em Statistical science: a review journal of the Institute of
  Mathematical Statistics}, 39(2):286.

\bibitem[Ma and Sun, 2026]{ma2026feature}
Ma, Y. and Sun, Q. (2026).
\newblock Feature bagging provides stability.

\bibitem[Meinshausen and B{\"u}hlmann, 2010]{MeinshausenBuhlmann2010}
Meinshausen, N. and B{\"u}hlmann, P. (2010).
\newblock Stability selection.
\newblock {\em Journal of the Royal Statistical Society: Series B (Statistical
  Methodology)}, 72(4):417--473.

\bibitem[Mentch and Hooker, 2016]{mentch2016quantifying}
Mentch, L. and Hooker, G. (2016).
\newblock Quantifying uncertainty in random forests via confidence intervals
  and hypothesis tests.
\newblock {\em Journal of Machine Learning Research}, 17(26):1--41.

\bibitem[Milborrow et~al., 2017]{MilborrowEarth}
Milborrow, S., Hastie, T., Tibshirani, R., Miller, A., and Lumley, T. (2017).
\newblock earth: Multivariate adaptive regression splines.
\newblock {\em R package version}, 5(2).

\bibitem[Patil and LeJeune, 2024]{patil2024asymptotically}
Patil, P. and LeJeune, D. (2024).
\newblock Asymptotically free sketched ridge ensembles: Risks,
  cross-validation, and tuning.
\newblock In {\em International Conference on Learning Representations}, volume
  2024, pages 52793--52828.

\bibitem[Ravikumar et~al., 2009]{Ravikumar2009SpAM}
Ravikumar, P., Lafferty, J., Liu, H., and Wasserman, L. (2009).
\newblock Sparse additive models.
\newblock {\em Journal of the Royal Statistical Society Series B: Statistical
  Methodology}, 71(5):1009--1030.

\bibitem[Schoch et~al., 2001]{schoch2001snare}
Schoch, S., De{\'a}k, F., K{\"o}nigstorfer, A., Mozhayeva, M., Sara, Y.,
  S{\"u}dhof, T.~C., and Kavalali, E.~T. (2001).
\newblock Snare function analyzed in synaptobrevin/vamp knockout mice.
\newblock {\em Science}, 294(5544):1117--1122.

\bibitem[Serrano et~al., 2019]{serrano2019s100a4}
Serrano, A., Apolloni, S., Rossi, S., Lattante, S., Sabatelli, M., Peric, M.,
  Andjus, P., Michetti, F., Carr{\`\i}, M.~T., Cozzolino, M., et~al. (2019).
\newblock The s100a4 transcriptional inhibitor niclosamide reduces
  pro-inflammatory and migratory phenotypes of microglia: implications for
  amyotrophic lateral sclerosis.
\newblock {\em Cells}, 8(10):1261.

\bibitem[Shah and Peters, 2020]{Shah2018TheHO}
Shah, R.~D. and Peters, J. (2020).
\newblock The hardness of conditional independence testing and the generalised
  covariance measure.
\newblock {\em The Annals of Statistics}, 48(3).

\bibitem[Shah and Samworth, 2013]{ShahSamworth2013}
Shah, R.~D. and Samworth, R.~J. (2013).
\newblock Variable selection with error control: another look at stability
  selection.
\newblock {\em Journal of the Royal Statistical Society Series B: Statistical
  Methodology}, 75(1):55--80.

\bibitem[Soloff et~al., 2024]{SoloffBarberWillett2024BaggingStability}
Soloff, J.~A., Barber, R.~F., and Willett, R. (2024).
\newblock Bagging provides assumption-free stability.
\newblock {\em Journal of Machine Learning Research}, 25(131):1--35.

\bibitem[Staerk et~al., 2021]{Staerk2021AdaptiveLearning}
Staerk, C., Kateri, M., and Ntzoufras, I. (2021).
\newblock {High-dimensional variable selection via low-dimensional adaptive
  learning}.
\newblock {\em Electronic Journal of Statistics}, 15(1):830 -- 879.

\bibitem[Sutera et~al., 2018]{Sutera2018RandomSubspaceTrees}
Sutera, A., Ch{\^a}tel, C., Louppe, G., Wehenkel, L., and Geurts, P. (2018).
\newblock Random subspace with trees for feature selection under memory
  constraints.
\newblock In {\em International conference on artificial intelligence and
  statistics}, pages 929--937. PMLR.

\bibitem[Tada et~al., 2007]{tada2007role}
Tada, T., Simonetta, A., Batterton, M., Kinoshita, M., Edbauer, D., and Sheng,
  M. (2007).
\newblock Role of septin cytoskeleton in spine morphogenesis and dendrite
  development in neurons.
\newblock {\em Current Biology}, 17(20):1752--1758.

\bibitem[Tian and Feng, 2021]{Tian2021RaSE}
Tian, Y. and Feng, Y. (2021).
\newblock Rase: Random subspace ensemble classification.
\newblock {\em Journal of Machine Learning Research}, 22(45):1--93.

\bibitem[Tian and Feng, 2023]{Tian2021RaSEAV}
Tian, Y. and Feng, Y. (2023).
\newblock Rase: A variable screening framework via random subspace ensembles.
\newblock {\em Journal of the American Statistical Association},
  118(541):457--468.

\bibitem[Tibshirani, 1996]{Tibshirani1996Lasso}
Tibshirani, R. (1996).
\newblock Regression shrinkage and selection via the lasso.
\newblock {\em Journal of the Royal Statistical Society Series B: Statistical
  Methodology}, 58(1):267--288.

\bibitem[Toghani and Allen, 2021]{9373169}
Toghani, M.~T. and Allen, G.~I. (2021).
\newblock Mp-boost: Minipatch boosting via adaptive feature and observation
  sampling.
\newblock In {\em 2021 IEEE International Conference on Big Data and Smart
  Computing (BigComp)}, pages 75--78.

\bibitem[Verdinelli and Wasserman, 2024]{Verdinelli2021DecorrelatedVI}
Verdinelli, I. and Wasserman, L. (2024).
\newblock Decorrelated variable importance.
\newblock {\em Journal of Machine Learning Research}, 25(7):1--27.

\bibitem[Wang et~al., 2025]{wang2025sharp}
Wang, T., Dobriban, E., Gataric, M., and Samworth, R.~J. (2025).
\newblock Sharp-ssl: Selective high-dimensional axis-aligned random projections
  for semi-supervised learning.
\newblock {\em Journal of the American Statistical Association},
  120(549):395--407.

\bibitem[Watson and Wright, 2021]{watson2021testing}
Watson, D.~S. and Wright, M.~N. (2021).
\newblock Testing conditional independence in supervised learning algorithms.
\newblock {\em Machine Learning}, 110(8):2107--2129.

\bibitem[Westin et~al., 2012]{westin2012ccl2}
Westin, K., Buchhave, P., Nielsen, H., Minthon, L., Janciauskiene, S., and
  Hansson, O. (2012).
\newblock Ccl2 is associated with a faster rate of cognitive decline during
  early stages of alzheimer's disease.
\newblock {\em PloS one}, 7(1):e30525.

\bibitem[Wieland et~al., 2022]{wieland2022epstein}
Wieland, L., Schwarz, T., Engel, K., Volkmer, I., Kr{\"u}ger, A., Tarabuko, A.,
  Junghans, J., Kornhuber, M.~E., Hoffmann, F., Staege, M.~S., et~al. (2022).
\newblock Epstein-barr virus-induced genes and endogenous retroviruses in
  immortalized b cells from patients with multiple sclerosis.
\newblock {\em Cells}, 11(22):3619.

\bibitem[Williamson et~al., 2023]{williamson2023general}
Williamson, B.~D., Gilbert, P.~B., Simon, N.~R., and Carone, M. (2023).
\newblock A general framework for inference on algorithm-agnostic variable
  importance.
\newblock {\em Journal of the American Statistical Association},
  118(543):1645--1658.

\bibitem[Yamada et~al., 2014]{Yamada2014HSICLasso}
Yamada, M., Jitkrittum, W., Sigal, L., Xing, E.~P., and Sugiyama, M. (2014).
\newblock High-dimensional feature selection by feature-wise kernelized lasso.
\newblock {\em Neural computation}, 26(1):185--207.

\bibitem[Yang et~al., 2024]{Yang2024ENNS}
Yang, K., Ganguli, A., and Maiti, T. (2024).
\newblock Enns: Variable selection, regression, classification, and deep neural
  network for high-dimensional data.
\newblock {\em Journal of Machine Learning Research}, 25(335):1--45.

\bibitem[Yao and Allen, 2021]{YaoAllen2020Minipatch}
Yao, T. and Allen, G.~I. (2021).
\newblock Feature selection for huge data via minipatch learning.

\bibitem[Yao et~al., 2023]{yao2021gaussian}
Yao, T., Wang, M., and Allen, G.~I. (2023).
\newblock Fast and accurate graph learning for huge data via minipatch
  ensembles.

\bibitem[Zamboni et~al., 2005]{ZamboniEtAl2005}
Zamboni, N., Fischer, E., Muffler, A., Wyss, M., Hohmann, H.-P., and Sauer, U.
  (2005).
\newblock Transient expression and flux changes during a shift from high to low
  riboflavin production in continuous cultures of bacillus subtilis.
\newblock {\em Biotechnology and bioengineering}, 89(2):219--232.

\bibitem[Zhang and Janson, 2022]{zhang2020floodgate}
Zhang, L. and Janson, L. (2022).
\newblock Floodgate: inference for model-free variable importance.

\bibitem[Zou, 2006]{Zou2006AdaptiveLasso}
Zou, H. (2006).
\newblock The adaptive lasso and its oracle properties.
\newblock {\em Journal of the American statistical association},
  101(476):1418--1429.

\bibitem[Zou and Hastie, 2005]{ZouHastie2005ElasticNet}
Zou, H. and Hastie, T. (2005).
\newblock Regularization and variable selection via the elastic net.
\newblock {\em Journal of the Royal Statistical Society Series B: Statistical
  Methodology}, 67(2):301--320.

\end{thebibliography}

\clearpage
\appendix
\renewcommand{\thethm}{S\arabic{thm}}
\renewcommand{\thelem}{S\arabic{lem}}
\renewcommand{\thecor}{S\arabic{cor}}
\renewcommand{\theassump}{S\arabic{assump}}
\renewcommand{\thealgocf}{S\arabic{algocf}}
\renewcommand{\theprop}{S\arabic{prop}}
\renewcommand{\thefigure}{S\arabic{figure}}
\renewcommand{\thetable}{S\arabic{table}}
\renewcommand{\theequation}{S\arabic{equation}}
\setcounter{section}{0}
\setcounter{thm}{0}
\setcounter{lem}{0}
\setcounter{cor}{0}
\setcounter{assump}{0}
\setcounter{prop}{0}
\setcounter{algocf}{0}
\setcounter{figure}{0}
\setcounter{table}{0}
\setcounter{equation}{0}

\section{Additional Related Work}
\label{sec:additional_related_work}
\paragraph{Feature selection literature.}

Under the linear model setup, a number of penalization-based methods have been proposed for feature selection, including the Lasso~\citep{Tibshirani1996Lasso}, SCAD~\citep{FanLi2001SCAD}, the adaptive Lasso \citep{Zou2006AdaptiveLasso}, and the elastic net \citep{ZouHastie2005ElasticNet}. Stability selection \citep{MeinshausenBuhlmann2010} and its complementary-pairs refinement CPSS \citep{ShahSamworth2013} can further improve stability of feature selection in linear models and provide explicit error bounds.
Various nonlinear feature selection methods have also been proposed. For instance, MARS \citep{friedman1991multivariate} approximates the regression function by a sum of spline basis functions, while feature selection is induced by forward basis construction followed by backward pruning. Feature-wise nonlinear Lasso \citep{li2005lasso} and HSIC-Lasso \citep{Yamada2014HSICLasso} apply implicit nonlinear transformations to individual feature vectors through the kernel trick, enabling nonlinear feature selection. Sparse additive modeling \citep{Ravikumar2009SpAM} and COSSO \citep{Lin2006COSSO} assume that the regression function admits a sparse decomposition into nonlinear functional components belonging to an RKHS, and impose $\ell_1$-type penalties on the component norms to select important components.
\cite{luo2024variable} discuss several variable selection methods based on Bayesian additive regression trees (BART). Recent neural-network-based embedded selectors such as LassoNet \citep{Lemhadri2021LassoNet} and ENNS \citep{Yang2024ENNS} further extend feature selection to highly nonlinear predictors. Complementing these approaches, model-agnostic methods decouple feature selection from a particular predictive learner. Classical wrapper methods evaluate candidate feature subsets according to predictive performance \citep{guyon2003introduction}, while Thompson variable selection formulates feature selection as a multi-armed bandit problem, the reward for each feature (arm) being its selection probability of a given base selector~\citep{LiuRockova2023TVS}.
Finally, the knockoff framework constructs synthetic control variables based on the feature joint distribution and uses flexible feature statistics to select variables with false discovery rate control \citep{Barber2014ControllingTF,candes2018panning}.

\paragraph{Model-agnostic feature importance literature.}

There is a surge of literature on model-agnostic inference for feature importance. \citep{lei2018distribution} introduce leave-one-covariate-out (LOCO) inference which assess each feature's contribution through data-splitting and feature occlusion. \citep{GanZhengAllen2022MinipatchCI} propose a minipatch-based ensemble framework that has built-in LOCO inference without model refitting or data splitting. \citep{zhang2020floodgate} propose Floodgate, which constructs lower confidence bounds for the minimum mean squared error gap due to feature occlusion using a model-X approach. \citep{williamson2023general} develop a general framework that provides inference for a population feature importance metric defined through feature occlusion for the oracle predictor.  \citep{watson2021testing} propose the conditional predictive impact, a knockoff-based measure of conditional importance that applies to arbitrary supervised learners and loss functions. Many of these methods come with validity guarantees for feature importance inference, while they do not provide a principled rule for converting feature-importance scores into a selected set of features.

\paragraph{Ensemble literature.}
Ensemble methods based on data perturbation and/or feature subsampling provide a natural mechanism for improving predictive performance and stabilizing variable rankings in high-dimensional problems. Bagging has been shown to reduce instability of base learners and can yield substantial variance reduction \citep{BuhlmannYu2002Bagging}, with recent work establishing finite-sample, assumption-free stability guarantees for broad classes of bagging schemes \citep{SoloffBarberWillett2024BaggingStability}. Many have leveraged ensemble structures for statistical inference: \citep{mentch2016quantifying} represent predictions from subsampled random forests as incomplete U-statistics to derive confidence intervals and tests for feature significance; \citep{cattaneo2026inference} establish bias and variance characterizations and distributional approximations for Mondrian random forests to construct valid inference for the unknown regression function. Recently, the idea of feature subsampling has been explored in ensemble frameworks: \citep{cannings2017random} propose aggregating prediction results on lower dimensional projection of the feature space to solve high-dimensional classification problem. Building on this idea, \citep{wang2025sharp} propose Sharp-SSL, which aggregates semi-supervised learners fitted on axis-aligned random projections of the data to enable effective variable selection in high-dimensional semi-supervised learning; \cite{Tian2021RaSE,Tian2021RaSEAV} further designs an iterative sampling scheme within the random subspace framework, performing high-dimensional classification and feature screening. Recent theoretical work further characterizes feature subsampling and sketching: \citep{ma2026feature} show that feature bagging improves algorithmic stability. \citep{lejeune2024asymptotics} develop asymptotic theory for sketched pseudoinverses and show that random sketching induces a ridge-like implicit regularization effect; \citep{patil2024asymptotically} show that the prediction risk of sketched ridge ensembles decomposes into an equivalent ridge bias and a sketching-induced variance component, and establish consistent generalized cross-validation for tuning the ensembles; \citep{chen2025revisiting} show that combining feature subsampling with greedy forward selection can reduce both bias and variance under an orthogonal design. Most closely related to our work, the minipatch ensembles have been exploited for graphical modeling, boosting, and clustering. \citep{yao2021gaussian} proposes MPGraph, which fits graph estimators on randomly sampled minipatches and aggregates edge selection frequencies to recover stable graph structure. \citep{9373169} propose MP-Boost, which trains successive weak learners on adaptively sampled minipatches and aggregates them into a boosting ensemble. \citep{he2026cluster} \citep{gan2022fast} use minipatches to construct an ensemble of clustering solutions that are aggregated through consensus clustering.

\paragraph{Adaptive sampling for feature selection, screening, and ranking.} In spirit, our work is related to many prior works that design adaptive sampling strategies for feature selection/screening/ranking, while many impose restrictions on the model class or require the base model to output selection or ranking. For instance, \cite{Sutera2018RandomSubspaceTrees} proposes an adaptive sampling strategy for trees, where sequential random-subspace trees are fitted, retaining variables previously assigned positive tree importance while reserving part of each subspace for unexplored variables \citep{Sutera2018RandomSubspaceTrees}. AdaSTAMPS~\citep{YaoAllen2020Minipatch} and AdaSub~\citep{Staerk2021AdaptiveLearning} both consider linear models: AdaSTAMPS adapts the balance between exploration and exploitation using aggregated selection frequencies from a base selector, while AdaSub updates feature-inclusion probabilities according to the solutions of criterion-based low-dimensional model-selection problems. The Thompson variable selection framework~\citep{LiuRockova2023TVS} can incorporate more general base models, but since it formulates subset search as a combinatorial bandit problem with binary rewards, it requires the base model to output feature selection given subsampled features. RAMPART~\citep{chen2025top} considers the top-$k$ feature ranking problem and it aggregate feature ranking results from the base models applied on randomly sampled minipatches and recursively trims lower-ranked features. Interestingly, \cite{Tian2021RaSE,Tian2021RaSEAV} propose the iterative RaSE, a flexible ensemble framework for high-dimensional sparse classification and variable screening, which can be incorporated with any base learner equipped with an informative criterion such as BIC or cross-validation error. In contrast to our work, RaSE focuses on variable screening rather than exact feature selection. The key methodological idea of our LAMPS framework also significantly differs from RaSE: RaSE iteratively samples groups of subspaces, and for each group, they identify the best subspace within this group using a prespecified criterion, then examine the inclusion event of all features only in the best subspace; while LAMPS is built on top of the estimated LOCO importance from all trained minipatches. Collectively, these methods demonstrate that adaptive allocation can improve computational efficiency and concentrate estimation on promising features. Their update rules, however, rely on discrete selection events, criterion-selected subspaces, or within-subspace rankings, and hence are not directly transferable to our setting, in which an arbitrary base learner supplies only predictions and the continuous LOCO importance scores are used to guide the update of feature sampling distributions.

\section{Detailed Notations for Algorithms, Theoretical Results and Proofs}\label{sec:supp_notations}

\notationchange{Recall that for a generic feature vector $X=(X^{(1)},\dots,X^{(M)})^\top$ and $F\subseteq[M]$, $X^{(F)}$ denotes the subvector indexed by $F$; correspondingly, $X_i^{(F)}$ denotes the feature subvector of sample $i$.}

\begin{definition}[Probability mass function for $F$]
\label{def:prob_Q} Given a feature sampling probability vector $\bq\in [0,1]^M$, we define its associated probability mass function of subsampled feature sets as $Q_{\bq}: 2^{[M]} \rightarrow [0,1]$, satisfying
\begin{align*}
Q_{\bq}(F) = \frac{\prod_{j\in F}q_j\prod_{k\notin F}(1-q_k)}{1-\prod_{k=1}^M(1-q_k)}\ind(F\neq\emptyset).
\end{align*}
We define $Q_{\bq,\no j}$ and $Q_{\bq,j}$ as the probability mass functions conditioning on the exclusion and inclusion of feature $j$, respectively: for any $F \subset [M]\no j$, we have
\begin{align*}
Q_{\bq,\no j}(F) = &\frac{Q_{\bq}(F)}{\sum_{F' \subset [M]}\ind(j \notin F')Q_{\bq}(F')},\\
Q_{\bq, j}(F) = &\frac{Q_{\bq}(F)}{\sum_{F' \subset [M]}\ind(j \in F')Q_{\bq}(F')}.
\end{align*}
\end{definition}
Throughout this proof, we denote $Q_{\bq^{(b)}}$, $Q_{\bq^{(b)},\no j}$, and $Q_{\bq^{(b)},j}$  by $Q^{(b)}$, $Q^{(b)}_{\no j}$ and $Q^{(b)}_{j}$.

\begin{definition}[Oracle and minipatch predictors]
\label{def:mu_defs}
Let $\bZ:=\{(X_i,Y_i)\}_{i=1}^N$ denote the training sample. For any feature subset
$F\subset[M]$, define the oracle regression function
\[
\mu_F^*(x)
:=\bbE\!\left[Y\mid \notationchange{X^{(F)}}=\notationchange{x^{(F)}}\right],
\qquad x\in\bbR^M.
\]
For any index set $I\subset[N]$ with $|I|=n$, let
$\bZ_{I,F}:=\{(\notationchange{X_i^{(F)}},Y_i): i\in I\}$ denote the minipatch restricted to features $F$,
and define the corresponding fitted predictor
\[
\mu_{I,F}(\cdot)
:= H(\bZ_{I,F})(\cdot),
\]
where $H$ is the base learner. We further define the combinatorial averaged minipatch predictor by
\[
\bar\mu_F(x)
:= \frac{1}{\binom{N}{n}}
\sum_{\substack{I\subset[N]\\ |I|=n}}
\hat{\mu}_{I,F}(x),
\qquad x\in\bbR^M.
\]
When evaluating at a random covariate $X$, we write $\mu_F^*(\notationchange{X^{(F)}})$, $\hat{\mu}_{I,F}(\notationchange{X^{(F)}})$,
and $\bar\mu_F(\notationchange{X^{(F)}})$ accordingly.
\end{definition}

\begin{definition}[LOCO feature-importance targets]
\label{def:Delta_defs}
Fix a feature sampling probability vector $\bq\in[0,1]^M$, and let $Q_{\bq}$ and $Q_{\bq,\no j}$
be as in Definition~\ref{def:prob_Q}. We then define the oracle feature-importance scores and LOCO inference target by

\[
\Delta_j^*(\bq)
=
\bbE\!\left[(Y-\bbE_{F\sim Q_{\bq,\no j}}\mu_F^*(\notationchange{X^{(F)}}))^2
-
(Y-\bbE_{F\sim Q_{\bq}}\mu_F^*(\notationchange{X^{(F)}}))^2\right],
\]
\[
\Delta_j(\bq)
=
\bbE\!\left[(Y-\bbE_{F\sim Q_{\bq,\no j}}\bar\mu_F(\notationchange{X^{(F)}}))^2
-
(Y-\bbE_{F\sim Q_{\bq}}\bar\mu_F(\notationchange{X^{(F)}}))^2\right].
\]
where $\mu_F^*$ and $\bar\mu_F$ are defined in Definition~\ref{def:mu_defs}. When working at iteration $b$, we write $\Delta_j^{*(b)}:=\Delta_j^*(\bq^{(b-1)})$ and
$\Delta_j^{(b)}:=\Delta_j(\bq^{(b-1)})$.
\end{definition}

\paragraph{Additional notations.} Given sets $\mathcal{A}$ and positive integer $n>0$, we denote the uniform distribution over size-$n$ subsets of $\mathcal{A}$ by $\mathcal{U}_n^{\mathcal{A}}$. We summarize additional notations used throughout the proofs in Table~\ref{tab:notations}.

 \begin{table}[!htb]
     \centering
     \begin{tabular}{|c|c|}
     \hline
         $\mu_{I,F}(X;\bZ)$ & $H(\bZ_{I,F})(\notationchange{X^{(F)}})$\\
         \hline
         $\mu^*(X;Q,\bZ)$ & $\frac{1}{\binom{N}{n}}\sum_{I\subset[N], F\subset[M]}\mu_{I,F}(X;\bZ)Q(F)$\\
         \hline
         $h_j(Z, \bZ, Q)$  &  $\ell(Z,\mu^*(\cdot;Q_{\no j},\bZ)) - \ell(Z,\mu^*(\cdot;Q,\bZ))$\\
        \hline
         \multirow{2}{*}{$\hat{h}_j^{(t)}(Z_i,\bZ_{\no i,:})$ }& LOCO-LOO score $\hat{\Delta}^{(t)}_j(Z_i)$ at sample $i$, calculated in step 2(c) \\& of Algorithm~\ref{algo:minipatch_train} with input sampling vector $q^{(t)}$.\\
         \hline
         $\hat{\Delta}^{(t)}_j$& Estimated feature importance given in the $t$-th iteration of Algorithm~\ref{algo:loco_adamp}\\
         \hline
         $\bq^{(t)}\in [0,1]^{M}$& The feature sampling probability vector given in the $t$th iteration of Algorithm~\ref{algo:loco_adamp}\\
         \hline
         $Q^{(t)}:2^{[M]}\rightarrow [0,1]$& The feature sampling distribution function corresponding to $\bq^{(t)}$\\
         \hline
         \multirow{2}{*}{$\bq^{*(t)}\in [0,1]^{M}$}& The feature sampling probability vector in Assumption~\ref{ass:Q_stb},\\
         &close to $\bq^{(t)}$ but independent of $\bZ$\\
         \hline
         $Q^{*(t)}:2^{[M]}\rightarrow [0,1]$& The feature sampling distribution function corresponding to $\bq^{*(t)}$\\
     \hline
     \end{tabular}
     \caption{List of notations for predictors and feature importance functions. For training data $\bZ=(\bX,\bY)$, we let $\bZ_{I,F}$ denote $(\bX_{I,F},\bY_I)$. When the training data $\bZ=(\bX,\bY)$ is the full training set, we omit $\bZ$ in the notations above for simplicity. $\bZ_{\no i,:} = (\bX_{\no i,:}, \bY_{\no i})$, $\bZ_{\no i,\no j} = (\bX_{\no i,\no j}, \bY_{\no i})$.}
     \label{tab:notations}
 \end{table}

\section{Additional Algorithms}
Here, we present the detailed algorithms mentioned in the main paper. Alg.~\ref{algo:minipatch_train} summarizes how we compute the LOCO importance estimates via minipatch ensembles; Alg.~\ref{algo:find_threshold} describes how we solve~\eqref{eq:threshold} in Alg.~\ref{algo:IndepSampWeight}; Alg.~\ref{algo:loco_adamp_stopping} shows how we select the number of iterations $B$ by tracking LOO errors, as mentioned in Section~\ref{sec:tuning}.
\begin{algorithm}[!htbp]
\caption{Minipatch Sampling, Training, and Feature Importance Computation}
\label{algo:minipatch_train}
\noindent\textbf{Input}: Training data $(\boldsymbol{X}, \boldsymbol{Y})$ with observation size $N$ and feature size $M$; expected minipatch size $n$; number of minipatches $K$; feature sampling probabilities $\boldsymbol{q} = (q_1, \dots, q_M)$; base learner $H$.

\begin{enumerate}
    \item For $k = 1$ to $K$:
    \begin{enumerate}
        \item Sample $I_k \subset [N]$ uniformly without replacement with size $|I_k| = n$.
        \item Sample feature set $F_k\subset [M]\sim Q_{\boldsymbol{q}}$.
        \item Train base model: $\mu_k = H(\boldsymbol{X}_{I_k, F_k}, \boldsymbol{Y}_{I_k})$.
    \end{enumerate}
    \item For each data point $i \in [N]$ and feature $j \in [M]$:
    \begin{enumerate}
        \item Compute LOO prediction:
        \[
        \mu_{-i}(X_i) = \frac{1}{\sum_{k=1}^K \mathbb{I}(i \notin I_k)} \sum_{k=1}^K \mathbb{I}(i \notin I_k) \hat{\mu}_k(X_i)
        \]
        \item Compute LOO+LOCO prediction:
        \[
        \mu_{-i}^{-j}(X_i) = \frac{1}{\sum_{k=1}^K \mathbb{I}(i \notin I_k)\mathbb{I}(j \notin F_k)} \sum_{k=1}^K \mathbb{I}(i \notin I_k)\mathbb{I}(j \notin F_k) \hat{\mu}_k(X_i)
        \]
        \item Compute difference in prediction error:
        \[
        \hat{\Delta}_{j}(X_i,Y_i) = \mathrm{Error}(Y_i, \mu_{-i}^{-j}(X_i)) - \mathrm{Error}(Y_i, \mu_{-i}(X_i))
        \]
    \end{enumerate}
    \item Average across $i$ to obtain final score for each feature:
    \[
    \hat{\Delta}_j = \frac{1}{N} \sum_{i=1}^N \hat{\Delta}_{j}(X_i,Y_i), \quad \text{for } j = 1, \dots, M.
    \]
\end{enumerate}

\noindent\textbf{Output}: Feature importance scores $\{\hat{\Delta}_j\}_{j=1}^M$.
\end{algorithm}

\begin{algorithm}[!htbp]
\caption{Solving \eqref{eq:threshold} for the Threshold}
\label{algo:find_threshold}
\noindent\textbf{Input}: Feature weights $\{\tilde{\Delta}_1, \dots, \tilde{\Delta}_M\}$; minipatch feature size parameter $m$; probability upper bound $\delta \in (0, 1)$.
\begin{enumerate}
    \item Sort $\{\tilde{\Delta}_1, \dots, \tilde{\Delta}_M\}$ in decreasing order: $\tilde{\Delta}_{(1)}\geq \tilde{\Delta}_{(2)}\geq \cdots \geq \tilde{\Delta}_{(M)}$.
    \item Find smallest index $0\leq k <\frac{m}{\delta}$ such that:
    \begin{equation}\label{eq:find_threshold}
        \sum_{\ell > k} \tilde{\Delta}_{(\ell)} > \left(\frac{m}{\delta} - k\right)\tilde{\Delta}_{(k+1)}.
    \end{equation}
    \item  Let $t^* := \frac{\sum_{\ell > k} \tilde{\Delta}_{(\ell)}}{m/\delta - k} \wedge \tilde{\Delta}_{(1)}$.
\end{enumerate}
\noindent\textbf{Output}: Threshold $t^*$.
\end{algorithm}

\begin{algorithm}[!htbp]
\caption{Adaptive Minipatch Feature Selection with Data-driven Selection of Iterations}
\label{algo:loco_adamp_stopping}
\noindent\textbf{Input}: Training data $(\boldsymbol{X}, \boldsymbol{Y})$, minipatch sizes $n$, $m$; maximum number of iterations $B$; number of minipatches per iteration $K$; base learner $H$; adaptivity parameter $\delta \in (0,1)$.

\medskip
\noindent\textbf{Initialization}: Let $M$ be the number of features. Initialize sampling probabilities $\boldsymbol{q}^{(0)} = (m/M, \dots, m/M)$, $b^*=B-1$.

\begin{enumerate}
    \item \textbf{Adaptive Minipatch Loop:} For $b = 1$ to $B$,
    \begin{enumerate}
        \item \textbf{Compute Feature Importance Scores:} \\
        Call Algorithm~\ref{algo:minipatch_train} with input $(\boldsymbol{X}, \boldsymbol{Y}, n, K, \boldsymbol{q}^{(b-1)}, H)$ to obtain $\{\hat{\Delta}_j^{(b)}\}_{j=1}^M$.

        \item \textbf{Update Sampling Probabilities:} \\
        Call Algorithm~\ref{algo:IndepSampWeight} with input $(\hat{\Delta}_1^{(b)}, \dots, \hat{\Delta}_M^{(b)}, m, \delta)$ to obtain updated sampling probability vector $\boldsymbol{q}^{(b)}$.

        \item \textbf{Check Stopping Criterion:} \\
        If $b>1$, let $\ell_i^{(b)} := \mathrm{Error}(Y_i,\mu_{-i}^{(b)}(X_i))$ denote the LOO loss for observation $i$ at iteration $b$ where $\mu_{-i}^{(b)}(X_i)$ is the iteration-b LOO prediction, and define paired differences
        $$d_i^{(b)} := \ell_i^{(b)} - \ell_i^{(b-1)},\qquad i=1,\dots,N.$$
        Compute the one-sided test statistic
        $$T_b := \frac{\bar d^{(b)}}{s_d^{(b)}/\sqrt{N}},\qquad\bar d^{(b)}=\frac{1}{N}\sum_{i=1}^N d_i^{(b)},\quad(s_d^{(b)})^2=\frac{1}{N-1}\sum_{i=1}^N\big(d_i^{(b)}-\bar d^{(b)}\big)^2.$$
        If $T_b \ge z_{0.025}$ (2.5\% lower-tail quantile of normal distribution), stop and set $b^{\star} = \max\{b-2,\,1\}$; break the loop.
    \end{enumerate}

    \item \textbf{Output Selected Feature Set:} \\
    Return $\hat{S} = \{\,j : q_j^{(b^{\star})} > \frac{1}{2}\delta\}$.
\end{enumerate}

\noindent\textbf{Output}: Feature subset $\hat{S}$.
\end{algorithm}

\section{Property of Algorithm~\ref{algo:find_threshold}}
\begin{prop}\label{prop:sampling_property}
   If $m < \delta M$, then the solution of \eqref{eq:threshold} exists and can be found by Alg. \ref{algo:find_threshold} in the Appendix. We further have the following two consequences: 
    \begin{enumerate}
        \item the feature sampling probability vectors $\bq^{(b)},\,0\leq b\leq B$ in Alg. \ref{algo:loco_adamp} satisfy $\boldsymbol{q}^{(b)}\in[0, \delta]^{M}$, with $\sum_{j=1}^Mq^{(b)}_j=m$;
        \item if $F_k\sim Q_{\boldsymbol{q}^{(b)}}$, its expected size satisfies $\bbE|F_k| \in (m, \frac{m}{1-e^{-m}})$.
    \end{enumerate}
\end{prop}

\begin{proof}[Proof of Proposition \ref{prop:sampling_property}]
We first note that $t=\tilde{\Delta}_{(M)}\geq \frac{c_0}{M}$ satisfies 
$$\frac{m\max_l\omega_l(t)}{\sum_{l=1}^M\omega_l(t)} = \frac{m}{M} \leq \delta.$$
Therefore, $\left\{0<t\leq \max_l\tilde{\Delta}_l:\frac{m\max_l\omega_l(t)}{\sum_{l=1}^M\omega_l(t)}\leq \delta\right\}$ is a non-empty, bounded set, with a supremum $t^*\geq \tilde{\Delta}_{(M)}$. In the following, we show that $t^*$ can be exactly found by Algorithm \ref{algo:find_threshold}. 

Define function $g(t) := \sum_{l=1}^M \frac{\tilde{\Delta}_l \wedge t}{\tilde{\Delta}_{(1)}\wedge t}$. One can rewrite $t^*$ as 
$$t^* = \sup\left\{0<t\leq \tilde{\Delta}_{(1)}:g(t)\geq \frac{m}{\delta}\right\}.$$ 
Note that $g(t)$ is a continuous, non-increasing function on $(0,\infty)$. In particular, $g(t)$ is strictly decreasing on $[\tilde{\Delta}_{(M)}, \tilde{\Delta}_{(1)}]$. We then have that either (i) $g(\tilde{\Delta}_{(1)})> \frac{m}{\delta}$ and $t^* = \tilde{\Delta}_{(1)}$, or (ii) $g(t^*) = \frac{m}{\delta}$ with $t^* \leq \tilde{\Delta}_{(1)}$. If case (i) is true, we have $\sum_{l=1}^M\tilde{\Delta}_l > \frac{m}{\delta}\tilde{\Delta}_{(1)}$, implying \eqref{eq:find_threshold} with $k=0$. Hence, the output threshold of Algorithm \ref{algo:find_threshold} is $\big(\frac{\delta}{m}\sum_{l}\tilde{\Delta}_l\big)\wedge \tilde{\Delta}_{(1)} = \tilde{\Delta}_{(1)} = t^*$. 

Otherwise, case (ii) is true, $g(\tilde{\Delta}_{(1)}) \leq \frac{m}{\delta}$, and $t^*\in [\tilde{\Delta}_{(M)}, \tilde{\Delta}_{(1)}]$. Hence, $k=0$ does not satisfy \eqref{eq:find_threshold}. Define $k^* = \max\{k:1\leq k\leq M, \tilde{\Delta}_{(k)}\geq t^*\}$; note that $t^*\in [\tilde{\Delta}_{(M)}, \tilde{\Delta}_{(1)}]$ already implies that $k^*$ is well-defined. For any $0\leq k\leq k^*-1$, $\tilde{\Delta}_{(k+1)}\geq \tilde{\Delta}_{(k^*)}\geq t^*$. The decreasing property of $g(\cdot)$ suggests that 
\begin{equation*}
    \begin{split}
        g(\tilde{\Delta}_{(k+1)}) \leq g(t^*) =\frac{m}{\delta},\\
        k+\sum_{l>k}\frac{\tilde{\Delta}_{(l)}}{\tilde{\Delta}_{(k+1)}}\leq \frac{m}{\delta},\\
        \sum_{l>k}\tilde{\Delta}_{(l)}\leq \Big(\frac{m}{\delta}-k\Big)\tilde{\Delta}_{(k+1)}.
    \end{split}
\end{equation*}
Furthermore, since $\frac{m}{\delta}\geq g(\tilde{\Delta}_{(k^*)}) = k^*+\sum_{l>k^*}\frac{\tilde{\Delta}_{(l)}}{\tilde{\Delta}_{(k^*)}}>k^*$, we know that $k^*<\frac{m}{\delta}\leq M$. Therefore, $k^*+1\leq M$ and $\tilde{\Delta}_{(k^*+1)}< t^*$ by the definition of $t^*$. Applying the decreasing property of $g(\cdot)$ again, we have
\begin{equation*}
    \begin{split}
        g(\tilde{\Delta}_{(k^*+1)})>g(t^*)=\frac{m}{\delta},\\
        k^*+\sum_{l>k^*}\frac{\tilde{\Delta}_{(l)}}{\tilde{\Delta}_{(k^*+1)}}>\frac{m}{\delta},\\
        \sum_{l>k^*}\tilde{\Delta}_{(l)}>\Big(\frac{m}{\delta}-k^*\Big)\tilde{\Delta}_{(k^*+1)}.
    \end{split}
\end{equation*}
Therefore, in case (ii), $k^*$ is the minimum integer satisfying \eqref{eq:find_threshold}. By the definition of $k^*$, we have that $\tilde{\Delta}_{(k^*+1)}<t^*\leq \tilde{\Delta}_{(k^*)}$, implying that $\frac{m}{\delta}= g(t^*)=k^* +\frac{\sum_{l>k^*}\tilde{\Delta}_{(l)}}{t^*}$, which further yields $t^* = \frac{\sum_{l>k^*}\tilde{\Delta}_{(l)}}{\frac{m}{\delta}-k^*}\leq \tilde{\Delta}_{(1)}$. Therefore, Algorithm \ref{algo:find_threshold} gives the exact solution for \eqref{eq:threshold} also in case (ii).

It remains to verify the two stated consequences. By the definition of the sampling weights,
$$
q_j^{(b)}
=\frac{m\omega_j(t^*)}{\sum_{\ell=1}^M\omega_\ell(t^*)},
\qquad j\in[M].
$$
Hence, $\sum_{j=1}^M q_j^{(b)}=m$, and by the definition of $t^*$,
$$
\max_{j\in[M]}q_j^{(b)}
=\frac{m\max_j\omega_j(t^*)}{\sum_{\ell=1}^M\omega_\ell(t^*)}
\leq \delta.
$$
Thus, $\bq^{(b)}\in[0,\delta]^M$ and $\sum_{j=1}^M q_j^{(b)}=m$. The same properties hold for the initialization $\bq^{(0)}=(m/M,\ldots,m/M)$ since $m/M\leq\delta$.

Next, for $F_k\sim Q_{\bq^{(b)}}$, Definition~\ref{def:prob_Q} gives
$$
\bbE|F_k|
=\frac{\sum_{j=1}^M q_j^{(b)}}{1-\prod_{j=1}^M(1-q_j^{(b)})}
=\frac{m}{1-\prod_{j=1}^M(1-q_j^{(b)})}.
$$
Since $q_j^{(b)}\leq\delta<1$ for all $j$ and $\sum_jq_j^{(b)}=m>0$,
$$
0<\prod_{j=1}^M(1-q_j^{(b)})
<\exp\left\{-\sum_{j=1}^M q_j^{(b)}\right\}=e^{-m},
$$
where we use $1-x\leq e^{-x}$, with strict inequality for at least one $q_j^{(b)}>0$. Therefore,
$$
    m<\bbE|F_k|<\frac{m}{1-e^{-m}},
$$
which proves both claims.
\end{proof}

\section{Proofs of Theorem 1 and Its Auxiliary Lemmas}

\begin{proof}[Proof of Theorem \ref{thm:main1}]
We prove the two inclusions
\[
S \subseteq \hat S
\qquad\text{and}\qquad
\hat S \subseteq S
\]
separately.

\paragraph{No false negatives.}
Our goal is to show that every signal feature $j\in S$ satisfies
$q_j^{(B)} > \frac{1}{2}\delta$, which implies $j\in \hat S$. To establish this, we use the following growth property.

\begin{lem}\label{lem:q_growth}
    Suppose Assumptions~\ref{ass:minipatch_feature_scaling}--\ref{ass:min_sig} hold. Then 
\begin{equation}
\label{eq:growth_factor_C_in_thm1_proof}
q_j^{(b)} \geq \big(1.5\,q_j^{(b-1)}\big)\wedge \delta,
\qquad \forall j\in S,\ \forall\,1\le b\le B.
\end{equation}
\end{lem}

Iterating
\eqref{eq:growth_factor_C_in_thm1_proof} yields that $\forall j\in S$,
\[
q_j^{(B)}\geq 
\big(1.5^B q_j^{(0)}\big)\wedge \delta= \Big(1.5^B \frac{m}{M}\Big)\wedge \delta,
\]
where we utilized the fact that the initial sampling probability $q_k^{(0)} = \frac{m}{M}$ for all feature $k$. Hence, as long as $B> (\log(1.5))^{-1}\log\left(\frac{\delta M}{2m}\right)$, we have $q_j^{(B)}> \frac{\delta M}{2m}\frac{m}{M}=\frac{1}{2}\delta$. Therefore, by the definition of $\hat{S}$ in Algorithm~\ref{algo:loco_adamp}, as long as we set the constant $C$ in Theorem~\ref{thm:main1} as $(\log(1.5))^{-1}$, we have $j\in \hat{S}$. Since the arguments above holds for all $j\in S$, this implies that $S \subseteq \hat{S}$.

\paragraph{No false positives.}
We next show that no noise feature is selected, namely,
\[
\hat S \subseteq S.
\]

For this part, we use the following lemma.

\begin{lem}\label{lem:noise_prob_bound}
Suppose Assumptions~\ref{ass:minipatch_feature_scaling},  \ref{ass:signal_strength_ratio}, and \ref{ass:noise_err} hold.
For any $j\in S^c$ and $1\le b\le B$, $q_j^{(b)}\le
\frac{1}{2}\delta$.
\end{lem}

\medskip

Applying Lemma~\ref{lem:noise_prob_bound} at iteration $b=B$, we obtain
\[
q_j^{(B)} \le \frac{1}{2}\delta,
\qquad \forall j\in S^c,
\]
Therefore no noise feature can cross the selection threshold, and hence
\[
\hat S \subseteq S.
\]

Combining the no-false-negative and no-false-positive parts, we conclude that
\[
\hat S = S
\]
This completes the proof.
\end{proof}

\subsection{Proof of Auxiliary Lemmas Needed for Theorem~\ref{thm:main1}}

\begin{proof}[Proof of Lemma \ref{lem:q_growth}]
For any iteration $1\leq b\leq B$, we have the following lemma that lower bounds the sampling probability of each feature:
\begin{lem}\label{prop:q_lwrbnd}
    Suppose that $m\leq \delta M$. For any feature $j$, the output feature sampling probability of Algorithm \ref{algo:IndepSampWeight} satisfies $q_j \geq \frac{\til{\Delta}_j(m-k\delta)}{\sum_{l>k}\til{\Delta}_{(l)}}\wedge \delta$ for any $0\leq k <\|\til{\Delta}\|_0$.
\end{lem}
Since $\|\til{\Delta}\|_0=M$, we can simply invoke Lemma~\ref{prop:q_lwrbnd} with $k=s$ for the $b$th iteration. We then have
\begin{equation*}
    q_j^{(b)}\geq \frac{\til{\Delta}_j^{(b)}(m-s\delta)}{\sum_{l>s}\til{\Delta}^{(b)}_{(l)}}\wedge \delta \geq \frac{\til{\Delta}^{(b)}_j(m-s\delta)}{\sum_{l\in S^c}\til{\Delta}^{(b)}_{l}}\wedge \delta,
\end{equation*}
where the second inequality is due to the fact that $\sum_{l\in S^c}\til{\Delta}^{(b)}_{l}\geq \sum_{l>s}\til{\Delta}_{(l)}.$ Therefore, it suffices to show that $\frac{\til{\Delta}^{(b)}_j(m-s\delta)}{\sum_{l\in S^c}\til{\Delta}^{(b)}_{l}}\geq 1.5q_j^{(b-1)}$.

We note that for any $1\leq k\leq M$, the estimated and shifted feature importance score satisfies $\til{\Delta}^{(b)}_k = \Delta_k^{\star(b)} - \Delta_{\min}^{\star(b)} + \frac{c_0}{M} +\varepsilon_k^{(b)}$. The following lemma provides lower and upper bounds for the population feature importance $\Delta_k^{\star(b)}$ when $k\in S$ and $k\in S^c$ separately.
\begin{lem}\label{lem:Delta_oracle}
   Suppose Assumptions~\ref{ass:minipatch_feature_scaling}
and~\ref{ass:signal_strength_ratio} hold. For any $k\in S$,
\begin{equation}
\label{eq:delta_oracle_signal_bound}
\Delta_k^{\star(b)}\geq \frac{(1-2e^{-m})q_k^{(b-1)}}{2\,(1-e^{-m})}\|f_k\|_2^2.
\end{equation}

For any $k\in S^c$,
\begin{equation}
\label{eq:delta_oracle_noise_bound}
-\frac{2e^{-(m-1)}}{1-e^{-(m-1)}}\frac{q_k^{(b-1)}}{1-\prod_l(1-q_l^{(b-1)})}\sum_{l\in S}q_l^{(b-1)}\|f_l\|_2^2\leq \Delta_k^{\star(b)}\leq 0.
\end{equation}
\end{lem}
Since $m\geq 2>\ln{4}$, we can further simplify the bound \eqref{eq:delta_oracle_signal_bound} in Lemma~\ref{lem:Delta_oracle} as follows: for $k\in S$, 
\begin{equation}\label{eq:delta_oracle_signal_bnd2}
\begin{split}
    \Delta_k^{\star(b)}\geq &\frac{(1-2e^{-m})q_k^{(b-1)}}{2\,(1-e^{-m})}\|f_k\|_2^2\\
    = &\frac{1}{2}\Big(1-\frac{e^{-m}}{1-e^{-m}}\Big) q_k^{(b-1)}\|f_k\|_2^2\\
    \geq &\frac{1}{3}q_k^{(b-1)}\|f_k\|_2^2.
\end{split}
\end{equation}
While for the lower bound \eqref{eq:delta_oracle_noise_bound}, Assumption~\ref{ass:minipatch_feature_scaling} helps to simplify it as follows: for $k\in S^c$,
\begin{equation}\label{eq:delta_oracle_noise_bnd2}
    \begin{split}
        -\Delta_k^{\star(b)}\leq &\frac{2e^{-(m-1)}}{1-e^{-(m-1)}}\frac{q_k^{(b-1)}}{1-\prod_l(1-q_l^{(b-1)})}\sum_{l\in S}q_l^{(b-1)}\|f_l\|_2^2\\
        \leq &\frac{2e}{(1-1/e)(1-e^{-m})}M^{-C_1}q_k^{(b-1)}\sum_{l\in S}q_l^{(b-1)}\|f_l\|_2^2\\
        \leq &10M^{-C_1}q_k^{(b-1)}\sum_{l\in S}q_l^{(b-1)}\|f_l\|_2^2,
    \end{split}
\end{equation}
where the second inequality is due to Assumption~\ref{ass:minipatch_feature_scaling} and the fact that $\prod_l(1-q_l^{(b-1)})=\exp\Big\{\sum_{l=1}^M \log(1-q_l^{(b-1)})\Big\} \leq \exp\big\{-\sum_{l=1}^Mq_l^{(b-1)}\big\} =e^{-m}$.

Therefore, $\Delta_{\min}^{\star(b)} = \min_{1\leq k\leq M} \Delta_k^{\star(b)}\leq 0$. Hence, for any $j\in S$ and $1\leq b\leq B$, we can write
\begin{equation}\label{eq:q_numerator_lwr}
    \begin{split}
        \til{\Delta}_j^{(b)} \geq &\Delta_j^{\star(b)} - \Delta_{\min}^{\star(b)} +\frac{c_0}{M} +\varepsilon_j^{(b)}\\
        \geq &\Big(\frac{1}{3}\|f_j\|_2^2 - \big|\varepsilon_j^{(b)}\big|/q_j^{(b-1)}\Big)q_j^{(b-1)}\\
        \geq &\frac{1}{6}\|f_j\|_2^2 q_j^{(b-1)},
    \end{split}
\end{equation}
where we applied Assumption~\ref{ass:min_sig} in the last line above. In addition, 
define
\begin{equation}\label{eq:Derr}
    \DErr = \max_{1\leq b\leq B}\Big\|\varepsilon^{(b)}_{S^c}\Big\|_1.
\end{equation}.
Thus,
\begin{equation}\label{eq:q_denominator_upp}
    \begin{split}
        \sum_{l\in S^c}\til{\Delta}_l^{(b)} = &\sum_{l\in S^c}\Delta_l^{\star(b)} - (M-|S|)\Delta_{\min}^{\star(b)} + \frac{c_0(M-|S|)}{M} + \sum_{l\in S^c}\varepsilon_l^{(b)}\\
        \leq &10M^{-(C_1-1)}\sum_{l\in S}\|f_l\|_2^2+  c_0 + \varepsilon_{\Delta}.
    \end{split}
\end{equation}
where the second line utilizes \eqref{eq:delta_oracle_noise_bnd2}. Suppose that we choose the constant $c_1\leq \frac{m-s\delta}{360}$ in Assumption~\ref{ass:signal_strength_ratio}, we then have 
\begin{equation*}
    \begin{split}
        \frac{\til{\Delta}^{(b)}_j(m-s\delta)}{\sum_{l\in S^c}\til{\Delta}^{(b)}_{l}}\geq &\frac{\|f_j\|_2^2(m-s\delta)}{\|f_j\|_2^2(m-s\delta)/6 + 6\big(c_0+\varepsilon_{\Delta}\big)}q_j^{(b-1)}\\
        \geq&1.5q_j^{(b-1)}.
    \end{split}
\end{equation*}
where the first inequality is due to \eqref{eq:q_numerator_lwr}, \eqref{eq:q_denominator_upp}, and Assumption~\ref{ass:signal_strength_ratio}; the second line is due to Assumption~\ref{ass:min_sig}. The proof is now complete.
\end{proof}

\begin{proof}[Proof of Lemma \ref{lem:noise_prob_bound}]

By Algorithm~\ref{algo:IndepSampWeight}, for any $1\leq b\leq B$, we have

\[
q_j^{(b)}
=
\frac{
m\big(\til{\Delta}_j^{(b)}\wedge t^{\star(b)}\big)
}{
\sum_l\big(\til{\Delta}_l^{(b)}\wedge t^{\star(b)}\big)
},
\]
where $t^{\star(b)}$ is the output of Algorithm~\ref{algo:find_threshold} when applied to $\{\til{\Delta}_l^{(b)}\}_{l=1}^M$. After examining the steps of Algorithm~\ref{algo:find_threshold}, we note that $t^{\star(b)}\geq \til{\Delta}_{(\lfloor \frac{m}{\delta}\rfloor +1)}^{(b)} \geq \frac{c_0}{M}$, since $\min_l \til{\Delta}_l^{(b)} \geq \frac{c_0}{M}$ by definition. Therefore, we can write $\sum_l\big(\til{\Delta}_l^{(b)}\wedge t^\star\big)\ge M\frac{c_0}{M} = c_0$. Hence, $q_j^{(b)}\le
\frac{m}{c_0}\til{\Delta}_j^{(b)}$.

By definition of $\til{\Delta}_j^{(b)}$,

\begin{equation*}
\begin{split}
    q_j^{(b)}\le &\frac{m}{c_0}\left(\Delta_j^{\star(b)}-\Delta_{\min}^{\star(b)} + \frac{c_0}{M} + \varepsilon_j^{(b)}\right)\\
    \le &\frac{m}{c_0}\left(10M^{-C_1}\sum_{l\in S}\|f_l\|_2^2 +\frac{c_0}{M} + \big|\varepsilon_j^{(b)}\big|\right)\\
    \le & \frac{1}{2}\delta,
\end{split}
\end{equation*}
where the second line is due to \eqref{eq:delta_oracle_noise_bnd2} and the last line is due to Assumptions ~\ref{ass:minipatch_feature_scaling} and \ref{ass:noise_err}. By the definition of $\hat{S}$ in Algorithm~\ref{algo:loco_adamp}, the proof of Lemma~\ref{lem:noise_prob_bound} is now complete.
\end{proof}

\begin{proof}[Proof of Lemma \ref{prop:q_lwrbnd}]
    We first consider the scenario where $\hat{s} = \|\tilde{\Delta}\|_0 < \frac{m}{\delta}\leq M$. In this case, $q_j = \delta\ind(\tilde{\Delta}_j >0) + \frac{m-\hat{s}\delta}{M-\hat{s}}\ind(\tilde{\Delta}_j=0)$. When $\tilde{\Delta}_j = 0$, $q_j=\frac{m-\hat{s}\delta}{M-\hat{s}} > 0 = \frac{\tilde{\Delta}_j(m-k\delta)}{\sum_{l>k}\tilde{\Delta}_{(l)}}\wedge \delta$. While if $\tilde{\Delta}_j > 0$, $q_j = \delta\geq \frac{\tilde{\Delta}_j(m-k\delta)}{\sum_{l>k}\tilde{\Delta}_{(l)}}\wedge \delta$ also holds.

    In the second scenario, we apply Algorithm \ref{algo:find_threshold} to find $t^\star=\sup\left\{0<t\leq \max_l\tilde{\Delta}_l:\frac{m\max_l\omega_l(t)}{\sum_{l=1}^M\omega_l(t)}\leq \delta\right\}$ and set $q_j = \frac{m \omega_j(t^\star)}{\sum_{k=1}^M \omega_k(t^\star)}$. As shown in the proof of Proposition \ref{prop:sampling_property}, we know that one of the following two cases must hold: (i) $\sum_{l=1}^M \tilde{\Delta}_l >\frac{m}{\delta} \tilde{\Delta}_{(1)}$ and $t^\star =\tilde{\Delta}_{(1)}$; (ii) $\sum_{l=1}^M (\tilde{\Delta}_l\wedge t^\star) = \frac{m}{\delta} t^\star$ and $t^\star\leq \tilde{\Delta}_{(1)}$. If case (i) holds, we have $q_j = \frac{m \tilde{\Delta}_j}{\sum_{l=1}^M \tilde{\Delta}_l}$. In addition, 
\begin{equation*}
\begin{split}
    \frac{\sum_{l=1}^M \tilde{\Delta}_l}{m}=&\frac{\sum_{l>k}\tilde{\Delta}_{(l)}}{m}\frac{\sum_{l=1}^M\tilde{\Delta}_l}{\sum_{l>k}\tilde{\Delta}_{(l)}}\\
    =&\frac{\sum_{l>k}\tilde{\Delta}_{(l)}}{m}\left(1-\frac{\sum_{l\leq k}\tilde{\Delta}_l}{\sum_{l=1}^M\tilde{\Delta}_l}\right)^{-1}\\
    \leq &\frac{\sum_{l>k}\tilde{\Delta}_{(l)}}{m}\left(1-\frac{k\tilde{\Delta}_{(1)}}{\sum_{l=1}^M\tilde{\Delta}_l}\right)^{-1}\\
    <&\frac{\sum_{l>k}\tilde{\Delta}_{(l)}}{m}\left(1-\frac{k\delta}{m}\right)^{-1}\\
    =&\frac{\sum_{l>k}\tilde{\Delta}_{(l)}}{m-k\delta}
\end{split}
\end{equation*}
where the fourth line is due to that $\tilde{\Delta}_{(1)}<\frac{\delta}{m}\sum_{l=1}^M\tilde{\Delta}_l$ holds when case (i) is true. Hence, $q_j = \frac{m \tilde{\Delta}_j}{\sum_{l=1}^M \tilde{\Delta}_l}>\frac{\tilde{\Delta}_j(m-k\delta)}{\sum_{l>k}\tilde{\Delta}_{(l)}}$ holds.

While for case (ii), we have that
\begin{equation}\label{eq:q_case2}
    \begin{split}
        q_j = &\frac{m (\tilde{\Delta}_j\wedge t^\star)}{\sum_{l=1}^M (\tilde{\Delta}_l\wedge t^\star)}\\
        = &\frac{\tilde{\Delta}_j\wedge t^\star}{t^\star}\frac{m t^\star}{\sum_{l=1}^M (\tilde{\Delta}_l\wedge t^\star)}\\
        =&\frac{\tilde{\Delta}_j\delta}{t^\star}\wedge \delta,
    \end{split}
\end{equation}
where the third line utilizes the fact that $\sum_{l=1}^M (\tilde{\Delta}_l\wedge t^\star) = \frac{m}{\delta} t^\star$ in case (ii). Furthermore, by the definition of $t^\star$ and the continuity of $\frac{m\max_l\omega_l(t)}{\sum_{l=1}^M\omega_l(t)}$ w.r.t. $t$, we have $\delta\geq \frac{m\max_l\omega_l(t^\star)}{\sum_{l=1}^M\omega_l(t^\star)}=\frac{mt^\star}{\sum_{l=1}^M(\tilde{\Delta}_l\wedge t^\star)}$, which implies $\sum_{l=1}^M\left(\frac{\tilde{\Delta}_l}{t^\star}\wedge 1\right)\geq \frac{m}{\delta}$. Therefore, for any $0\leq k<\|\tilde{\Delta}\|_0$,
    \begin{equation*}
        \begin{split}
            &\frac{m}{\delta}\leq \sum_{l=1}^M\left(\frac{\tilde{\Delta}_l}{t^\star}\wedge 1\right)\leq k + \frac{\sum_{l>k}\tilde{\Delta}_{(l)}}{t^\star}\\
            \Rightarrow &t^\star\leq \frac{\delta}{m-k\delta}\sum_{l>k}\tilde{\Delta}_{(l)}.
        \end{split}
    \end{equation*}
Plugging in $t^\star\leq \frac{\delta}{m-k\delta}\sum_{l>k}\tilde{\Delta}_{(l)}$ into \eqref{eq:q_case2}, we have $q_j = \frac{\tilde{\Delta}_j\delta}{t^\star}\wedge \delta\geq \frac{\tilde{\Delta}_j(m-k\delta)}{\sum_{l>k}\tilde{\Delta}_{(l)}}\wedge \delta$.

Therefore, we have proved that in all three cases, $q_j\geq \frac{\tilde{\Delta}_j(m-k\delta)}{\sum_{l>k}\tilde{\Delta}_{(l)}}\wedge \delta$ for any $0\leq k<\|\tilde{\Delta}\|_0$.
\end{proof}

\begin{proof}[Proof of Lemma \ref{lem:Delta_oracle}]
    Recall that $$\bbE_{F\sim Q_{\no j}}\mu_F^\star(\notationchange{X^{(F)}}) = \bbE_{F\sim Q_{\no j}}\Bigg(\sum_{k\in S}f_k(\notationchange{X^{(k)}})\ind(k\in F)\Bigg) = \sum_{k\in S}f_k(\notationchange{X^{(k)}})\bbP_{F\sim Q}(k\in F|j\notin F),$$ and
    $$\bbE_{F\sim Q}\mu_F^\star(\notationchange{X^{(F)}}) = \bbE_{F\sim Q}\Bigg(\sum_{k\in S}f_k(\notationchange{X^{(k)}})\ind(k\in F)\Bigg) = \sum_{k\in S}f_k(\notationchange{X^{(k)}})\bbP_{F\sim Q}(k\in F).$$
    We can then write $$\bbE(Y-\bbE_{F\sim Q_{\no j}}\mu_F^\star(\notationchange{X^{(F)}}))^2=\sigma_{\varepsilon}^2 + \sum_{k\in S}\Var((f_k(\notationchange{X^{(k)}}))\bbP^2_{F\sim Q}(k\notin F|j\notin F),$$
    $$\bbE(Y-\bbE_{F\sim Q}\mu_F^\star(\notationchange{X^{(F)}}))^2=\sigma_{\varepsilon}^2 + \sum_{k\in S}\Var(f_k(\notationchange{X^{(k)}}))\bbP^2_{F\sim Q}(k\notin F),$$
    where $\sigma_\varepsilon^2:=\Var(\varepsilon)$.
    Hence, 
    \begin{equation*}
        \begin{split}
            \Delta_j^\star = \sum_{k\in S}\Var(f_k(\notationchange{X^{(k)}}))\left(\bbP^2_{F\sim Q}(k\notin F|j\notin F) - \bbP^2_{F\sim Q}(k\notin F)\right)
        \end{split}
    \end{equation*}

For notational simplicity we fix an iteration $b$ and omit the superscript $(b)$ in this proof.  
Thus we write $q_j$ for $q_j^{(b-1)}$, $Q$ for $Q^{(b)}$, and $Q_{\backslash j}$ for $Q^{(b)}_{\backslash j}$.

\paragraph{Step 1: Expressing $\Delta_j^\star$ in terms of inclusion probabilities.}

By the additive model and assumptions,
\[
Y =
\sum_{k\in S} f_k(\notationchange{X^{(k)}}) + \varepsilon,
\quad
\mathbb{E}[\varepsilon]=0,\quad
\varepsilon\perp X,
\]
and each $f_k(\notationchange{X^{(k)}})$ is centered: $\mathbb{E}[f_k(\notationchange{X^{(k)}})]=0$. Moreover, the features $\notationchange{X^{(1)}},\dots,\notationchange{X^{(M)}}$ are independent, so for $k\neq \ell$,
\[
\mathbb{E}\big[f_k(\notationchange{X^{(k)}}) f_\ell(\notationchange{X^{(\ell)}})\big] = 0,
\]
and
\[
\|f_k\|_2^2 = \mathbb{E}_X f_k^{2}(\notationchange{X^{(k)}}).
\]

By definition,
\[
\mu_F^\star(\notationchange{X^{(F)}})
= 
\sum_{k\in S} f_k(\notationchange{X^{(k)}})\,\mathbb{I}(k\in F),
\]
and the aggregated oracle predictors under $Q$ and $Q_{\backslash j}$ are
\[
\mathbb{E}_{F\sim Q}\mu_F^\star(\notationchange{X^{(F)}})
=
\mathbb{E}_{F\sim Q}\Bigg[\sum_{k\in S} f_k(\notationchange{X^{(k)}})\,\mathbb{I}(k\in F)\Bigg]
=
\sum_{k\in S} f_k(\notationchange{X^{(k)}})\,\mathbb{P}_{F\sim Q}(k\in F),
\]
\[
\mathbb{E}_{F\sim Q_{\backslash j}}\mu_F^\star(\notationchange{X^{(F)}})
=
\sum_{k\in S} f_k(\notationchange{X^{(k)}})\,\mathbb{P}_{F\sim Q}(k\in F\mid j\notin F).
\]
Here we used that $F$ is independent of $X$.

For brevity, define
\[
p_k := \mathbb{P}_{F\sim Q}(k\in F),
\qquad
p_k^{(-j)} := \mathbb{P}_{F\sim Q}(k\in F\mid j\notin F),
\]
and note that
\[
1-p_k = \mathbb{P}_{F\sim Q}(k\notin F),
\qquad
1-p_k^{(-j)} = \mathbb{P}_{F\sim Q}(k\notin F\mid j\notin F).
\]

Then
\[
\mathbb{E}_{F\sim Q}\mu_F^\star(\notationchange{X^{(F)}})
=
\sum_{k\in S} p_k\, f_k(\notationchange{X^{(k)}}),
\quad
\mathbb{E}_{F\sim Q_{\backslash j}}\mu_F^\star(\notationchange{X^{(F)}})
=
\sum_{k\in S} p_k^{(-j)}\, f_k(\notationchange{X^{(k)}}).
\]

We now compute the risks
\[
R_Q := \mathbb{E}\Big(Y - \mathbb{E}_{F\sim Q}\mu_F^\star(\notationchange{X^{(F)}})\Big)^2,
\quad
R_{Q,\backslash j} := \mathbb{E}\Big(Y - \mathbb{E}_{F\sim Q_{\backslash j}}\mu_F^\star(\notationchange{X^{(F)}})\Big)^2.
\]

Write
\[
Y - \mathbb{E}_{F\sim Q}\mu_F^\star(\notationchange{X^{(F)}})
=
\sum_{k\in S}\big(1-p_k\big)f_k(\notationchange{X^{(k)}}) + \varepsilon.
\]
Using centering and independence,
\[
\mathbb{E}\big[f_k(\notationchange{X^{(k)}})\big]=0,\quad
\mathbb{E}\big[f_k(\notationchange{X^{(k)}})\,f_\ell(\notationchange{X^{(\ell)}})\big]=0\ (k\neq \ell),
\quad
\mathbb{E}[f_k(\notationchange{X^{(k)}})\varepsilon]=0.
\]
Thus
\[
R_Q
=
\mathbb{E}\Big(Y - \mathbb{E}_{F\sim Q}\mu_F^\star(\notationchange{X^{(F)}})\Big)^2
=
\sigma_\varepsilon^2 
+ \sum_{k\in S}(1-p_k)^2\,\|f_k\|_2^2,
\]
.

Similarly,
\[
Y - \mathbb{E}_{F\sim Q_{\backslash j}}\mu_F^\star(\notationchange{X^{(F)}})
=
\sum_{k\in S}\big(1-p_k^{(-j)}\big)f_k(\notationchange{X^{(k)}}) + \varepsilon,
\]
hence
\[
R_{Q,\backslash j}
=
\sigma_\varepsilon^2
+ \sum_{k\in S}(1-p_k^{(-j)})^2\,\|f_k\|_2^2.
\]

Therefore,
\begin{align*}
\Delta_j^\star
&=
R_{Q,\backslash j} - R_Q\\
&=
\sum_{k\in S}\big[(1-p_k^{(-j)})^2 - (1-p_k)^2\big]\,\|f_k\|_2^2.
\end{align*}
Equivalently, in terms of ``exclusion'' probabilities,
\[
1-p_k = \mathbb{P}_{F\sim Q}(k\notin F),\quad
1-p_k^{(-j)} = \mathbb{P}_{F\sim Q}(k\notin F\mid j\notin F),
\]
so we can also write
\begin{equation}\label{eq:Delta_oracle_prob_form}
\Delta_j^\star
=
\sum_{k\in S}\|f_k\|_2^2
\left[
\mathbb{P}_{F\sim Q}(k\notin F\mid j\notin F)^2
-
\mathbb{P}_{F\sim Q}(k\notin F)^2
\right].
\end{equation}

\paragraph{Step 2: The contribution of feature $j$ when $j\in S$.}

Assume $j\in S$. Take $k=j$ in \eqref{eq:Delta_oracle_prob_form}.  
Under $Q_{\backslash j}$, the set $F$ never contains $j$, so
\[
\mathbb{P}_{F\sim Q}(j\notin F\mid j\notin F) = 1.
\]
Hence the $k=j$ term in \eqref{eq:Delta_oracle_prob_form} equals
\begin{align*}
\|f_j\|_2^2
\left[
1^2 - \mathbb{P}_{F\sim Q}(j\notin F)^2
\right]
&=
\|f_j\|_2^2\Big(1 - (1-p_j)^2\Big) \\
&=
\|f_j\|_2^2\big(2p_j - p_j^2\big) \\
&\ge
\|f_j\|_2^2\,p_j,
\end{align*}
since $p_j\in[0,1]$ implies $2p_j - p_j^2 \ge p_j$.

Next we relate $p_j$ to the sampling probability $q_j$. By Definition~\ref{def:prob_Q}, $Q$ is obtained by independently including each feature $\ell$ with probability $q_\ell$, and then conditioning on the event that $F$ is non-empty. Let
\[
r := \prod_{\ell=1}^M(1-q_\ell)
= \mathbb{P}(F=\emptyset \ \text{under the independent Bernoulli model}).
\]
Then
\[
p_j
= \mathbb{P}_{F\sim Q}(j\in F)
= \frac{\mathbb{P}(j\in F)}{\mathbb{P}(F\neq\emptyset)}
= \frac{q_j}{1-r}.
\]
Using $\sum_{\ell=1}^M q_\ell = m$ and the inequality
\[
\prod_{\ell=1}^M(1-q_\ell)
\le
\exp\Big(-\sum_{\ell=1}^M q_\ell\Big)
=
e^{-m},
\]
we have $r\le e^{-m}$ and hence $1-r \ge 1-e^{-m}$. Therefore
\[
p_j = \frac{q_j}{1-r}
\ge
q_j
\]
Combining this with the lower bound on the $k=j$ term gives
\begin{equation}\label{eq:j_term_lower_bound}
\text{(contribution of $k=j$ in $\Delta_j^\star$)}
\;\ge\;
q_j\|f_j\|_2^2.
\end{equation}

\paragraph{Step 3: Contributions from $k\neq j$ and the assumption on signal strength.}

For $k\neq j$, the summands in \eqref{eq:Delta_oracle_prob_form} are
\[
\|f_k\|_2^2
\left[
\mathbb{P}_{F\sim Q}(k\notin F\mid j\notin F)^2
-
\mathbb{P}_{F\sim Q}(k\notin F)^2
\right].
\]
Each bracketed term is bounded in absolute value by
\[
\left|
\mathbb{P}(k\notin F\mid j\notin F)^2
-
\mathbb{P}(k\notin F)^2
\right|
\le
2\,
\big|
\mathbb{P}(k\notin F\mid j\notin F)
-
\mathbb{P}(k\notin F)
\big|,
\]

Notice that
\[
\bigl|\mathbb{P}_{F\sim Q}(k\in F\mid j\notin F)-\mathbb{P}_{F\sim Q}(k\in F)\bigr|
=
\bigl|\mathbb{P}_{F\sim Q}(k\notin F\mid j\notin F)-\mathbb{P}_{F\sim Q}(k\notin F)\bigr|.
\]
Write
\[
r := \prod_{\ell=1}^M(1-q_\ell),\qquad 
r_{-j}:=\prod_{\ell\neq j}(1-q_\ell),
\]
so that $r=(1-q_j)r_{-j}$ and $1-r = 1-\prod_{\ell=1}^M(1-q_\ell)$.
By the law of total probability,
\begin{align*}
\mathbb{P}_{F\sim Q}(k\notin F)
&=
\mathbb{P}_{F\sim Q}(k\notin F\mid j\notin F)\,\mathbb{P}_{F\sim Q}(j\notin F)
+
\mathbb{P}_{F\sim Q}(k\notin F\mid j\in F)\,\mathbb{P}_{F\sim Q}(j\in F),
\end{align*}
hence
\begin{align*}
\mathbb{P}_{F\sim Q}(k\notin F\mid j\notin F)
-
\mathbb{P}_{F\sim Q}(k\notin F)
=
\mathbb{P}_{F\sim Q}(j\in F)\bigl(
\mathbb{P}_{F\sim Q}(k\notin F\mid j\notin F)
-
\mathbb{P}_{F\sim Q}(k\notin F\mid j\in F)
\bigr).
\end{align*}

We now compute each term explicitly under $Q$.  
Using the Bernoulli construction and the fact that $j\in F$ already implies $F\neq\emptyset$, we have
\[
\mathbb{P}_{F\sim Q}(k\in F\mid j\in F)=q_k,
\qquad
\mathbb{P}_{F\sim Q}(k\notin F\mid j\in F)=1-q_k.
\]
Moreover,
\[
\mathbb{P}_{F\sim Q}(k\in F) = \frac{q_k}{1-r},\quad
\mathbb{P}_{F\sim Q}(j\in F) = \frac{q_j}{1-r},\quad
\mathbb{P}_{F\sim Q}(k\in F,j\in F) = \frac{q_kq_j}{1-r},
\]
so
\begin{align*}
\mathbb{P}_{F\sim Q}(k\notin F,j\notin F)
&=
1-\mathbb{P}_{F\sim Q}(k\in F)-\mathbb{P}_{F\sim Q}(j\in F)
+\mathbb{P}_{F\sim Q}(k\in F,j\in F) \\
&=
1-\frac{q_k}{1-r}-\frac{q_j}{1-r}+\frac{q_kq_j}{1-r}
=
\frac{(1-r)-q_k-q_j+q_kq_j}{1-r},
\end{align*}
and therefore
\begin{align*}
\mathbb{P}_{F\sim Q}(k\notin F\mid j\notin F)
&=
\frac{\mathbb{P}_{F\sim Q}(k\notin F,j\notin F)}
{\mathbb{P}_{F\sim Q}(j\notin F)}
=
\frac{\displaystyle\frac{(1-r)-q_k-q_j+q_kq_j}{1-r}}
{\displaystyle 1-\frac{q_j}{1-r}}\\
&=
\frac{(1-r)-q_k-q_j+q_kq_j}{1-r-q_j}.
\end{align*}
Hence
\begin{align*}
\mathbb{P}_{F\sim Q}(k\notin F\mid j\notin F)
-
\mathbb{P}_{F\sim Q}(k\notin F\mid j\in F)
&=
\frac{(1-r)-q_k-q_j+q_kq_j}{1-r-q_j}-(1-q_k) \\
&= -\,\frac{q_kr}{1-r-q_j},
\end{align*}
so that
\begin{align*}
\bigl|
\mathbb{P}_{F\sim Q}(k\notin F\mid j\notin F)
-
\mathbb{P}_{F\sim Q}(k\notin F\mid j\in F)
\bigr|
=
\frac{q_kr}{1-r-q_j}.
\end{align*}
Combining with the decomposition above gives
\begin{align*}
\bigl|
\mathbb{P}_{F\sim Q}(k\notin F\mid j\notin F)
-
\mathbb{P}_{F\sim Q}(k\notin F)
\bigr|
&=
\mathbb{P}_{F\sim Q}(j\in F)\,
\frac{q_kr}{1-r-q_j}
=
\frac{q_j}{1-r}\cdot \frac{q_kr}{1-r-q_j}.
\end{align*}

Next observe that
\[
r=(1-q_j)r_{-j}
\quad\text{and}\quad
1-r-q_j=(1-q_j)(1-r_{-j}),
\]
so
\[
\frac{r}{1-r-q_j}
=
\frac{r_{-j}}{1-r_{-j}}.
\]
By $1-x\le e^{-x}$ for $x\in[0,1]$ we have
\[
r_{-j}
=
\prod_{\ell\neq j}(1-q_\ell)
\le
\exp\!\left(-\sum_{\ell\neq j}q_\ell\right),
\qquad
1-r_{-j}
\ge
1-\exp\!\left(-\sum_{\ell\neq j}q_\ell\right).
\]
Since $\sum_{\ell=1}^Mq_\ell = m$ by construction, we have
$\sum_{\ell\neq j}q_\ell \ge m-1$, and therefore
\[
\frac{r}{1-r-q_j}
=
\frac{r_{-j}}{1-r_{-j}}
\le
\frac{e^{-(m-1)}}{1-e^{-(m-1)}}.
\]
Putting everything together yields
\[
\bigl|
\mathbb{P}_{F\sim Q}(k\notin F\mid j\notin F)
-
\mathbb{P}_{F\sim Q}(k\notin F)
\bigr|
\le
\frac{e^{-(m-1)}}{1-e^{-(m-1)}}\cdot
\frac{q_jq_k}{1-\prod_{\ell=1}^M(1-q_\ell)}.
\]

Hence, for each $k\neq j$,
\begin{equation*}
    \begin{split}
\bigl|
\mathbb{P}_{F\sim Q}(k\notin F\mid j\notin F)^2
-
\mathbb{P}_{F\sim Q}(k\notin F)^2
\bigr|
\le&
2\,
\bigl|
\mathbb{P}_{F\sim Q}(k\notin F\mid j\notin F)
-
\mathbb{P}_{F\sim Q}(k\notin F)
\bigr|\\
\le&
\frac{2e^{-(m-1)}}{1-e^{-(m-1)}}\cdot
\frac{q_jq_k}{1-\prod_{\ell=1}^M(1-q_\ell)}.
\end{split}
\end{equation*}

Plugging this into \eqref{eq:Delta_oracle_prob_form} gives

\begin{equation}\label{eq:cross_term_bound}
\bigg|
\sum_{k\in S\setminus\{j\}}
\|f_k\|_2^2
\left[
\mathbb{P}(k\notin F\mid j\notin F)^2
-
\mathbb{P}(k\notin F)^2
\right]
\bigg|
\le
\frac{2e^{-(m-1)}}{1-e^{-(m-1)}}
\cdot
\frac{q_j}{1-\prod_{\ell=1}^M(1-q_\ell)}
\sum_{k\in S\setminus\{j\}} q_k\|f_k\|_2^2.
\end{equation}

Now apply Assumption~\ref{ass:signal_strength_ratio},

\[
\sum_{\ell\in S}\|f_\ell\|_2^2
\le
c_1M^{C_1-1}\min_{k\in S}\|f_k\|_2^2
\le
\frac{1-e^{-(m-1)}}{4e^{-(m-1)}}
\cdot
\min_{k\in S}\|f_k\|_2^2.
\]
where the second inequality is from Assumption~\ref{ass:minipatch_feature_scaling}, $M >e$, $m > 2$, and choosing $c_1 < \frac{1}{10}$ in Assumption~\ref{ass:signal_strength_ratio}.

For a fixed $j\in S$ we can bound
\[
\sum_{k\in S\setminus\{j\}} q_k\|f_k\|_2^2
\le
\sum_{k\in S\setminus\{j\}}\|f_k\|_2^2
\le
\sum_{\ell\in S}\|f_\ell\|_2^2
\le
\frac{1-e^{-(m-1)}}{4e^{-(m-1)}}\,\|f_j\|_2^2.
\]
Combining this with \eqref{eq:cross_term_bound}, and using $1-\prod_\ell(1-q_\ell)\ge 1-e^{-m}$, yields
\[
\bigg|
\sum_{k\in S\setminus\{j\}}
\|f_k\|_2^2
\left[
\mathbb{P}(k\notin F\mid j\notin F)^2
-
\mathbb{P}(k\notin F)^2
\right]
\bigg|
\le
\frac{q_j}{2\,(1-e^{-m})}\,\|f_j\|_2^2.
\]

\paragraph{Step 4: Lower bound for $j\in S$.}

Putting together the $k=j$ contribution \eqref{eq:j_term_lower_bound} and the bound on the sum over $k\neq j$ from Step~3, we obtain
\[
\Delta_j^\star
\;\ge\;
q_j\|f_j\|_2^2
-
\frac{q_j}{2\,(1-e^{-m})}\|f_j\|_2^2
=
\frac{(1-2e^{-m})q_j}{2\,(1-e^{-m})}\|f_j\|_2^2,
\]
which proves the first claim of the lemma.

\paragraph{Step 5: Bounds for $j\in S^c$.}

Now suppose $j\in S^c$. In this case we want to show $\Delta_j^\star\le 0$.
From \eqref{eq:Delta_oracle_prob_form} we have
\begin{align*}
\Delta_j^\star
&=
\sum_{k\in S} \Big\{(1-p_k^{(-j)})^2 - (1-p_k)^2\Big\}\,\|f_k\|_2^2 \\
&=
\sum_{k\in S} (p_k - p_k^{(-j)})(2 - p_k^{(-j)} - p_k)\,\|f_k\|_2^2,
\end{align*}
it suffices to show that
\[
p_k - p_k^{(-j)} \;\le\; 0 \qquad\text{for all }k\in S.
\]
Note that $j\in S^c$ and $k\in S$ imply $j\neq k$.

Recall that $p_k=\mathbb{P}_{F\sim Q}(k\in F)$ and
$p_k^{(-j)}=\mathbb{P}_{F\sim Q_{\backslash j}}(k\in F)
= \mathbb{P}_{F\sim Q}(k\in F\mid j\notin F)$.
Writing $r := \prod_{\ell=1}^M (1-q_\ell)$, we know
\[
\mathbb{P}_{F\sim Q}(k\in F) = \frac{q_k}{1-r},
\qquad
\mathbb{P}_{F\sim Q}(k\notin F\mid j\notin F)
= \frac{1-r - q_j + q_j q_k - q_k}{1-r-q_j},
\]
so that
\[
p_k^{(-j)} 
= 1 - \mathbb{P}_{F\sim Q}(k\notin F\mid j\notin F)
= 1 - \frac{1-r - q_j + q_j q_k - q_k}{1-r-q_j}.
\]
Therefore
\begin{align*}
p_k - p_k^{(-j)}
&= \mathbb{P}_{F\sim Q}(k\in F) - \mathbb{P}_{F\sim Q_{\backslash j}}(k\in F) \\
&= \mathbb{P}_{F\sim Q}(k\in F) - \Bigl(1-\mathbb{P}_{F\sim Q}(k\notin F\mid j\notin F)\Bigr) \\
&= \frac{q_k}{1-r} - 1
    + \frac{1-r - q_j + q_j q_k - q_k}{1-r-q_j} \\
&= \frac{q_k}{1-r}
   + \frac{q_k(q_j-1)}{1-r-q_j} \\
&= q_k\left(
   \frac{1-r-q_j}{(1-r)(1-r-q_j)}
   + \frac{(q_j-1)(1-r)}{(1-r)(1-r-q_j)}
   \right) \\
&= -\,\frac{q_k r q_j}{(1-r)(1-r-q_j)} \;\le\; 0,
\end{align*}
where we used $q_k\ge 0$, $q_j\ge 0$, $r\in[0,1)$ and
$1-r>0$, $1-r-q_j>0$.
Consequently $p_k\le p_k^{(-j)}$ for all $k\in S$, and since
$2-p_k^{(-j)}-p_k\ge 0$ we obtain
\[
(1-p_k^{(-j)})^2 - (1-p_k)^2
= (p_k - p_k^{(-j)})(2-p_k^{(-j)}-p_k)\le 0,
\qquad\forall\,k\in S.
\]
Every summand in the expression for $\Delta_j^\star$ is therefore non-positive,
and we conclude that
\[
\Delta_j^\star \;\le\; 0 \qquad\text{whenever } j\in S^c.
\]

Moreover, applying the same bound as in \eqref{eq:cross_term_bound} but now summing over all $k\in S$ (since there is no $k=j$ term),
\[
\left|\Delta_j^\star\right|
=
\left|
\sum_{k\in S}\|f_k\|_2^2
\left[
\mathbb{P}(k\notin F\mid j\notin F)^2
-
\mathbb{P}(k\notin F)^2
\right]
\right|
\le
\frac{2e^{-(m-1)}}{1-e^{-(m-1)}}
\cdot
\frac{q_j}{1-\prod_{\ell=1}^M(1-q_\ell)}
\sum_{k\in S} q_k\|f_k\|_2^2.
\]
Combining this with $\Delta_j^\star\le 0$ gives the two-sided bound in the second part of the lemma:
\[
-\frac{2e^{-(m-1)}}{1-e^{-(m-1)}}
\cdot
\frac{q_j}{1-\prod_{\ell=1}^M(1-q_\ell)}
\sum_{k\in S} q_k\|f_k\|_2^2
\;\le\;
\Delta_j^\star
\;\le\; 0.
\]

This completes the proof.

\end{proof}

\section{More Details and Proofs for Theorem 2}\label{sec:supp_thm2}

We first present a formal version of Assumption~\ref{ass:Q_stb} presented in the main paper.
\begin{assump}[Detailed version of Assumption~\ref{ass:Q_stb}]\label{ass:Q_stb_full}
    There exists a series of probability vectors $\{\bq^{\star(b)}\in (0,\delta]^M\}_{b=1}^B$ independent of $\bZ$ and satisfying $\sum_{j=1}^M(\bq^{\star(b)})_j=m$, such that for a sufficiently small $c>0$, 
    $d(\bq^{(b)},\bq^{\star(b)})\leq c\big(\SNR_{\min}\wedge \SR_{\min}\big)$ for $1\leq b\leq B$, where 
    
    \begin{equation}\label{eq:dist_q_qstar}
        d(\bq^{(b)},\bq^{\star(b)})=\big\|\bq^{(b)}-\bq^{\star(b)}\big\|_1+\max_{j\in S}\big|q_j^{(b)}-q_j^{\star(b)}\big|/q_j^{\star(b)}+\max_{k\in L^{(b+1)}}\frac{M}{m}\big|q_k^{(b)}-q_k^{\star(b)}\big|,
    \end{equation}
    and $L^{(b+1)}$ is the set of feature indices attaining the minimum empirical and oracle feature importance scores in the $(b+1)$th iteration: $L^{(b+1)}=\{l_{\min}^{(b+1)},\,l_{\min}^{\star(b+1)}\}$, with $l_{\min}^{(b+1)}= \arg\min_{l\in[M]}\hat{\Delta}_l^{(b+1)}$, $l_{\min}^{\star(b+1)} := \arg\min_{l\in[M]}\Delta_l^{\star(b+1)}$.
\end{assump}
The factor $\frac{M}{m}$ in the third term of~\ref{eq:dist_q_qstar} can also be replaced by $1/q_k^{(b)}$ due to Assumption~\ref{ass:no_min_for_large_q}; hence, both the second and third terms represent relative entrywise differences, with a maximum taken over the signal feature set and the least important feature set. The least important feature set $L^{(b)}$ plays a key role as we use $\min_l\hat\Delta_l^{(b)}$ to shift all LOCO importance scores to be positive in Alg.~\ref{algo:IndepSampWeight}. Additionally, we note that although $\bq^{(b)},\,\bq^{\star(b)}$ are high-dimensional vectors, their own $\ell_1$-norms scale as $m\ll M$, and hence requiring the $\ell_1$ distance $\big\|\bq^{(b)}-\bq^{\star(b)}\big\|_1$ to be small is not a stringent assumption. For proof convenience, we also let $\bq^{\star(0)}=\bq^{(0)}$.

\subsection{Proof of Theorem~\ref{thm:main2}}
Recall our definition of $\varepsilon^{(b)}$:
\begin{equation}\label{eq:epsilon_decomp}
    \DErrj{j}{b} = \til{\Delta}_j^{(b)}-\Big(\Delta_j^{*(b)} - \Delta_{\min}^{*(b)} + \frac{c_0}{M}\Big) = \hat{\Delta}_j^{(b)} - \Delta_j^{*(b)}  - \big(\hat{\Delta}_{\min}^{(b)}- \Delta_{\min}^{*(b)}\big).
\end{equation}

The following Lemma shows that $\Big|\hat{\Delta}_{\min}^{(b)}- \Delta_{\min}^{*(b)}\Big|$ can be controlled by the entrywise error bound for $\hat{\Delta}_l^{(b)} - \Delta_l^{*(b)}$:
\begin{lem}\label{lem:epsilon_min_bound}
$$\Big|\hat{\Delta}_{\min}^{(b)}- \Delta_{\min}^{*(b)}\Big|\leq
\Big|\hat{\Delta}_{l_{\min}^{(b)}}^{(b)} - \Delta_{l_{\min}^{(b)}}^{*(b)}\Big|
\;\vee\;
\Big|\hat{\Delta}_{l_{\min}^{*(b)}}^{(b)}-\Delta_{l_{\min}^{*(b)}}^{*(b)}\Big|,
$$
where $l_{\min}^{(b)}$ and $l_{\min}^{*(b)}$ are as defined in Assumption~\ref{ass:no_min_for_large_q}.
\end{lem}
Therefore, the key is bounding $\hat{\Delta}_j^{(b)}-\Delta_j^{*(b)}$ for each $1\leq j\leq M$. Recall that $Q^{\star(b)}$ is the data-independent feature sampling distribution close to $Q^{(b)}$, defined in Assumption~\ref{ass:Q_stb}.
Let $\Delta_j^{(b)}(Q^{*(b)}) := \bbE(h_j(Z^*,\bZ,Q^{*(b)})|\bZ)$ be the expected feature importance of $j$ conditioning on the training data $\bZ$, where the expectation is taken over an independent test point $Z^*$ sampled from the same distribution as $Z_i$, when the population feature sampling distribution $Q^{*(b)}$ is used in minipatch sampling. We sometimes also write $\Delta_j^{*(b)}$ as $\Delta_j^{*(b)}(Q^{(b-1)})$ to emphasize its dependence on the updated sampling distribution $Q^{(b-1)}$. We can then decompose the feature importance error into two terms:
$$
\hat{\Delta}_j^{(b)} -\Delta_j^{*(b)}(Q^{(b-1)}) = \hat{\Delta}_j^{(b)}-\Delta_j^{(b)}(Q^{*(b-1)}) + \Delta_j^{(b)}(Q^{*(b-1)}) - \Delta_j^{*(b)}(Q^{(b-1)}).
$$

The following Lemmas establish upper bounds for each term above.
\begin{lem}\label{lem:LOCO_LOO_err}
    Suppose that Assumptions~\ref{ass:adp_bounded}-\ref{ass:Q_stb} hold. Then with probability at least $1-(MN)^{-c}$, for any $0\leq b\leq B-1$ and $1\leq j\leq M$,
   \begin{equation*}
    \begin{split}
        \big|\hD{b+1}{j} - \D{b+1}{j}(Q^{*(b)})\big| \lesssim &\sqrt{\frac{q_j^{(b)}(\log M+\log N)}{K}} + \frac{\log M+\log N}{K} \\
        &+\big(q_j^{(b)}+q_j^{*(b)}\big)\|\bq^{(b)}-\bq^{*(b)}\|_1+\big|q_j^{(b)}-q_j^{*(b)}\big|\Big) \\
        &+Cq_j^{*(b)}\sqrt{\frac{n}{N}(\log M+\log N)}.
    \end{split}
\end{equation*}
\end{lem}
\begin{lem}\label{lem:training_err}
Suppose that Assumptions\ref{ass:adp_bounded} and \ref{ass:Q_stb} hold. Then for any $0\leq b\leq B-1$ and $1\leq j\leq M$,
    \begin{equation*}
        \begin{split}
        \big|\Delta^{(b+1)}_j(Q^{*(b)}) - \Delta^{*(b+1)}_j(Q^{(b)})\big|\lesssim& q_j^{*(b)}\left[\bbE_{F\sim Q^{*(b)}}\varepsilon_{\rm train}(F)+\bbE_{F\sim Q^{*(b)}_{j}}\varepsilon_{\rm train}(F)+\bbE_{F\sim Q^{*(b)}_{\no j}}\varepsilon_{\rm train}(F)\right]\\
            &+\big(q_j^{(b)}+q_j^{*(b)}\big)\big\|\bq^{(b)}-\bq^{*(b)}\big\|_1 + \big|q_j^{(b)}-q_j^{*(b)}\big|,
        \end{split}
    \end{equation*}
    where $\varepsilon_{\rm train}(F)$ is as defined in Assumption~\ref{ass:train_err}.
\end{lem}
Putting Lemmas~\ref{lem:LOCO_LOO_err} and \ref{lem:training_err} together, we have that with probability at least $1-(MN)^{-c}$, for any $0\leq b\leq B-1$ and $1\leq j\leq M$,
\begin{equation}\label{eq:Delta_err}
    \begin{split}
        \big|\hD{b+1}{j} - \Delta^{*(b+1)}_j\big| \lesssim& \sqrt{\frac{q_j^{(b)}(\log M+\log N)}{K}} + \frac{\log M+\log N}{K} \\
        &+\Big(q_j^{(b)}+q_j^{*(b)}\Big)\|\bq^{(b)}-\bq^{*(b)}\|_1+\Big|q_j^{(b)}-q_j^{*(b)}\Big|\\
        &+q_j^{*(b)}\sqrt{\frac{n}{N}(\log M+\log N)}\\
        &+q_j^{*(b)}\left[\bbE_{F\sim Q^{*(b)}}\varepsilon_{\rm train}(F)+\bbE_{F\sim Q^{*(b)}_{j}}\varepsilon_{\rm train}(F)+\bbE_{F\sim Q^{*(b)}_{\no j}}\varepsilon_{\rm train}(F)\right],
    \end{split}
\end{equation}
where the errors arise from finite sampling of minipatches, deviation of $Q^{(b)}$ from a data-independent $Q^{*(b)}$, minipatch sampling ratio $\frac{n}{N}$, and base model training error. The sampling probabilities $q_j^{(b)}$ and $q_j^{*(b)}$ of feature $j$ also play important roles.

By Assumptions~\ref{ass:no_min_for_large_q} and \ref{ass:Q_stb}, for $k=l_{\min}^{(b)}$ or $l_{\min}^{*(b)}$, $q_k^{*(b-1)}\leq q_k^{(b-1)}+\frac{cm}{M}\leq \frac{Cm}{M}$ for some constant $C>0$. Now, combining this fact with Lemma~\ref{lem:epsilon_min_bound}, Eq.~\eqref{eq:Delta_err}, and Assumption~\ref{ass:adp_n_minipatches}, we have 
\begin{equation}\label{eq:Delta_min_err}
    \begin{split}
        \Big|\hat{\Delta}_{\min}^{(b+1)}- \Delta_{\min}^{*(b+1)}\Big|\lesssim & \sqrt{\frac{m(\log M+\log N)}{MK}}+\frac{m}{M}\big\|\bq^{(b)}-\bq^{*(b)}\big\|_1\\
        &+ \Big|q_{l_{\min}^{(b+1)}}^{(b)}-q_{l_{\min}^{(b+1)}}^{*(b)}\Big|+\Big|q_{l_{\min}^{*(b+1)}}^{(b)}-q_{l_{\min}^{*(b+1)}}^{*(b)}\Big|\\
        &+\frac{m}{M}\sqrt{\frac{n}{N}(\log M+\log N)}\\
        &+\frac{m}{M}\left[\bbE_{F\sim Q^{*(b)}}\varepsilon_{\rm train}(F)+\max_{j\in [M]}\Big(\bbE_{F\sim Q^{*(b)}_{j}}\varepsilon_{\rm train}(F)+\bbE_{F\sim Q^{*(b)}_{\no j}}\varepsilon_{\rm train}(F)\Big)\right].
    \end{split}
\end{equation}

In the following, we prove that Assumptions~\ref{ass:min_sig} and \ref{ass:noise_err} both hold, and then Theorem~\ref{thm:main2} immediately follows from Theorem~\ref{thm:main1}.

\paragraph{Assumption~\ref{ass:min_sig} holds.} Given \eqref{eq:epsilon_decomp}, \eqref{eq:Delta_err}, and \eqref{eq:Delta_min_err}, we have that with probability at least $1-(MN)^{-c}$, for any $0\leq b\leq B-1$,
\begin{equation*}
    \begin{split}
        \|\varepsilon^{(b+1)}_{S^c}\|_1\leq &\sqrt{\frac{mM(\log M+\log N)}{K}} + m\sqrt{\frac{n}{N}(\log M+\log N)} \\
        &+m\left[\bbE_{F\sim Q^{*(b)}}\varepsilon_{\rm train}(F)+\bbE_{F\sim Q^{*(b)}_{j}}\varepsilon_{\rm train}(F)+\bbE_{F\sim Q^{*(b)}_{\no j}}\varepsilon_{\rm train}(F)\right]\\
        &+m\|\bq^{(b)}-\bq^{*(b)}\|_1 + M\Big|q_{l_{\min}^{(b+1)}}^{(b)}-q_{l_{\min}^{(b+1)}}^{*(b)}\Big|+M\Big|q_{l_{\min}^{*(b+1)}}^{(b)}-q_{l_{\min}^{*(b+1)}}^{*(b)}\Big|\\
        \leq&c(m-s\delta)\min_{k\in S}\|f_k\|_2^2,
    \end{split}
\end{equation*}
where the second inequality is due to Assumptions~\ref{ass:small_n}- \ref{ass:train_err} and \ref{ass:Q_stb}. As long as we choose the constant $c>0$ ($C>0$) sufficiently small (resp. large) in Assumptions~\ref{ass:small_n}- \ref{ass:train_err} and \ref{ass:Q_stb}, it is guaranteed that $\min_{k\in S}\|f_k\|_2^2\geq \frac{24\|\varepsilon^{(b+1)}_{S^c}\|_1}{m-s\delta}$. To show that Assumption~\ref{ass:min_sig} holds, it now remains to prove that $\min_{k\in S}\|f_k\|_2^2\geq \frac{6|\varepsilon^{(b+1)}_j|}{q_j^{(b)}}$. We prove this using an induction argument on $b=0,\dots B-1$.

First, fix $b=0$. By initialization, $\bq^{(b)}=\frac{m}{M}\mathbf{1}_M$. Assumption~\ref{ass:Q_stb} implies that $1-c\leq|q_j^{(b)}/q_j^{*(b)}|\leq 1 + c$, and hence $|q_j^{*(b)}/q_j^{(b)}|$ is bounded. For any $j\in S$, we then have
\begin{equation}\label{eq:epsilon_by_q_err}
    \begin{split}
        \frac{\big|\varepsilon_j^{(b+1)}\big|}{q_j^{(b)}}\lesssim &\sqrt{\frac{M(\log M+\log N)}{mK}} + \sqrt{\frac{n}{N}(\log M+\log N)}\\
        &+\big\|\bq^{(b)}-\bq^{*(b)}\big\|_1 + \frac{\big|q_j^{(b)}-q_j^{*(b)}\big|}{q_j^{*(b)}} \\
        &+\bbE_{F\sim Q^{*(b)}}\varepsilon_{\rm train}(F)+\bbE_{F\sim Q^{*(b)}_{j}}\varepsilon_{\rm train}(F)+\bbE_{F\sim Q^{*(b)}_{\no j}}\varepsilon_{\rm train}(F)\\
        &+\frac{M}{m}\Big|q_{l_{\min}^{(b+1)}}^{(b)}-q_{l_{\min}^{(b+1)}}^{*(b)}\Big|+\frac{M}{m}\Big|q_{l_{\min}^{*(b+1)}}^{(b)}-q_{l_{\min}^{*(b+1)}}^{*(b)}\Big|.
    \end{split}
\end{equation}
Assumptions~\ref{ass:small_n}-\ref{ass:no_min_for_large_q} then imply that $\min_{k\in S}\|f_k\|_2^2\geq \frac{6|\varepsilon^{(b+1)}_j|}{q_j^{(b)}}$ for all $j\in S$ when $b=0$.

Now suppose that $\min_{k\in S}\|f_k\|_2^2\geq \frac{6|\varepsilon^{(l+1)}_j|}{q_j^{(l)}}$ holds for $j\in S$ and $l\leq b$. Then Lemma~\ref{lem:q_growth} implies that $q_j^{(l+1)}\geq 1.5q_j^{(l)}\wedge \delta$, $\forall j\in S$, $\forall 0\leq l\leq b$. Hence, $q_j^{(l+1)}\geq (1.5)^{l+1}q_j^{(0)}\wedge \delta\geq \frac{m}{M}$. Therefore, \eqref{eq:epsilon_by_q_err} still holds for $l=b+1$. Hence, $\forall j\in S$, $\min_{k\in S}\|f_k\|_2^2\geq \frac{6|\varepsilon^{(l+1)}_j|}{q_j^{(l)}}$ also holds for $l=b+1$. By induction, $\min_{k\in S}\|f_k\|_2^2\geq \max_{0\leq b\leq B-1}\frac{6|\varepsilon^{(b+1)}_j|}{q_j^{(b)}}$ and consequently, Assumption~\ref{ass:min_sig} holds, as long as we choose $c_0=c(m-s\delta)\min_{k\in S}\|f_k\|_2^2$ for some appropriate $c>0$.

\paragraph{Assumption~\ref{ass:noise_err} holds.} First, Assumption~\ref{ass:minipatch_feature_scaling} immediately implies that $\frac{40m}{\delta}M^{-C_1}\sum_{k\in S}\|f_k\|_2^2\leq \frac{cm}{M}\min_{j\in S}\|f_j\|_2^2$. Furthermore, \eqref{eq:epsilon_decomp}, \eqref{eq:Delta_err}, and \eqref{eq:Delta_min_err} imply that
\begin{equation*}
    \begin{split}
        \Big\|\varepsilon^{(b+1)}_{S^c}\Big\|_{\infty}\lesssim &\sqrt{\frac{\log M+\log N}{K}} + \sqrt{\frac{n}{N}\big(\log M+\log N\big)} \\
        &+\bbE_{F\sim Q^{*(b)}}\varepsilon_{\rm train}(F)+\bbE_{F\sim Q^{*(b)}_{j}}\varepsilon_{\rm train}(F)+\bbE_{F\sim Q^{*(b)}_{\no j}}\varepsilon_{\rm train}(F)\\
        &+\big\|\bq^{(b)}-\bq^{*(b)}\|_1 + \Big|q_{l_{\min}^{(b+1)}}^{(b)}-q_{l_{\min}^{(b+1)}}^{*(b)}\Big|+\Big|q_{l_{\min}^{*(b+1)}}^{(b)}-q_{l_{\min}^{*(b+1)}}^{*(b)}\Big|\\
        \leq&\frac{c(m-s\delta)}{m}\min_{j\in S}\|f_j\|_2^2.
    \end{split}
\end{equation*}
Therefore, as long as we choose $c_0=c(m-s\delta)\min_{j\in S}\|f_j\|_2^2$ for some appropriate $c>0$, Assumption~\ref{ass:noise_err} also holds without violating Assumption~\ref{ass:min_sig}.

The proof of Theorem~\ref{thm:main2} is now complete by invoking Theorem~\ref{thm:main1}.

\subsection{Proof of Lemma~\ref{lem:LOCO_LOO_err}}\label{sec:LOCO_LOO_err}
For simplicity, throughout this proof, we omit $\D{b+1}{j}(Q^{*(b)})$ to $\D{b+1}{j}$. In order to concentrate $\hD{b+1}{j}$ around $\D{b+1}{j}$, we first decompose the error as follows:
\begin{equation}\label{eq:Delta_err_decomp}
\begin{split}
    \hD{b+1}{j} - \D{b+1}{j} =&\frac{1}{N}\sum_{i=1}^N\hh{b+1}_j(Z_i,\bZ_{\del i}) - \bbE[h_j(Z^*,\bZ, \Q{*}{b})|\bZ]\\
    =&\underbrace{\frac{1}{N}\sum_{i=1}^N\big(\hh{b+1}_j(Z_i,\bZ_{\del i}) - h_j(Z_i,\bZ_{\del i}, \Q{}{b})\big)}_{\varepsilon_1}\\
    &+\underbrace{\frac{1}{N}\sum_{i=1}^N\big(h_j(Z_i,\bZ_{\del i}, \Q{}{b}) - h_j(Z_i,\bZ_{\del i}, \Q{*}{b})\big)}_{\varepsilon_2}\\
    &+\underbrace{\frac{1}{N}\sum_{i=1}^N\big(h_j(Z_i,\bZ_{\del i}, \Q{*}{b}) - \bbE[h_j(Z_i,\bZ_{\del i}, \Q{*}{b})|\bZ_{\del i}]\big)}_{\varepsilon_3}\\
    &+\underbrace{\frac{1}{N}\sum_{i=1}^N\big(\bbE[h_j(Z^*,\bZ_{\del i}, \Q{*}{b})|\bZ_{\del i}] - \bbE[h_j(Z^*,\bZ, \Q{*}{b})|\bZ]\big)}_{\varepsilon_4},
\end{split}
\end{equation}
where the first line is due to the definition of $\D{b}{j}$; the expectation in $\varepsilon_3$ is taken over $Z_i$ conditioning on $\bZ_{\del i} = (\bX_{\no i,:}, \bY_{\no i})$, and the expectation in $\varepsilon_4$ is taken over $Z^*$ conditioning on $\bZ$, respectively. $Z^*$ is independent from $\bZ$ and follows the same distribution $\cP_Z$ as each $Z_i$. We will show upper bounds for each error terms above.

The following lemma bounds the error arising from finite sampling of minipatches.
\begin{lem}\label{lem:finiteKerr}
    Suppose that Assumptions \ref{ass:adp_bounded} and \ref{ass:adp_n_minipatches} hold. With probability at least $1-(MN)^{-c}$, for any $0\leq b\leq B-1$ and $1\leq j\leq M$, the following holds:
\begin{align*}
     \frac{1}{N}\sum_{i=1}^N\left|\hat{h}_j^{(b+1)}(Z_i,\bZ_{\backslash i,:}) - h_j(Z_i,\bZ_{\backslash i,:},Q^{(b)})\right|\lesssim \sqrt{\frac{q_j(\log M+\log N)}{K}}+\frac{\log M+\log N}{K}.
\end{align*}
\end{lem}
Lemma~\ref{lem:finiteKerr} immediately implies that with probability at least $1-(MN)^{-c}$, \begin{equation}\label{eq:DeltaErr1}
    |\varepsilon_1|\lesssim \sqrt{\frac{q_j(\log M+\log N)}{K}}+\frac{\log M+\log N}{K}.
\end{equation}

While for $\varepsilon_2$, we note that it arises from the difference between $Q^{(b)}$ and $Q^{*(b)}$, where $Q^{(b)}$ is the adaptively updated sampling distribution function using Algorithm~\ref{algo:loco_adamp}, and $Q^{*(b)}$ is the deterministic sampling distribution that is close to $Q^{(b)}$ in Assumption~\ref{ass:Q_stb}. The following Lemma shows how the difference in $Q$ translates to difference in the feature importance function.
\begin{lem}\label{lem:DeltaErrDiffQ}
    Suppose that Assumption~\ref{ass:adp_bounded} holds. Given two feature sampling probability vectors $\bq,\,\tilde{\bq}\in[0,1]^M$, satisfying $\sum_{j=1}^M q_j=\sum_{j=1}^M \tilde{q}_j=m$, we have
    \begin{equation*}
\begin{split}
    \left|h_j(Z,\bZ,Q_{\bq}) -h_j(Z,\bZ,Q_{\tilde{\bq}})\right|
    \lesssim \big(q_j+\tilde{q}_j\big)\big\|\bq-\tilde{\bq}\big\|_1 + \big|q_j - \tilde{q}_j\big|.
\end{split}
\end{equation*}
    where $Q_{\bq},\,Q_{\tilde{\bq}}:2^{[M]}\rightarrow [0,1]$ are the feature sampling distribution functions corresponding to $\bq$ and $\tilde{\bq}$.
\end{lem}
Therefore, Lemma~\ref{lem:DeltaErrDiffQ} shows that $\varepsilon_2$ can be bounded as follows:
\begin{equation}\label{eq:DeltaErr2_1}
   |\varepsilon_2|\lesssim \big(q_j^{(b)}+q_j^{*(b)}\big)\big\|\bq^{(b)}-\bq^{*(b)}\big\|_1 + \big|q_j^{(b)} - q_j^{*(b)}\big|.
\end{equation}

The following two lemmas each establishes a probabilistic bound for $\varepsilon_3$ and $\varepsilon_4$, respectively.
\begin{lem}\label{lem:LOCO_LOO_concentration}
    Suppose that Assumptions~\ref{ass:adp_bounded} and \ref{ass:small_n} hold, then with probability at least $1-(MN)^{-c}$, for any $1\leq b\leq B$ and $1\leq j\leq M$,
    \begin{align*}
            \left|\frac{1}{N}\sum_{i=1}^N\left(h_j(Z_i,\bZ_{\del i},Q^{*(b)}) -  \bbE\Big[h_j(Z_i,\bZ_{\del i},Q^{*(b)})|\bZ_{\del i}\Big]\right)\right|\lesssim q_j^{*(b)}\sqrt{\frac{n}{N}(\log M+\log N)}.
    \end{align*}
\end{lem}

\begin{lem}\label{lem:LOO_err}
    Suppose that Assumption~\ref{ass:adp_bounded} holds. Then for any $1\leq b\leq B$ and $1\leq j\leq M$,
    \begin{equation}
        \begin{split}
            \left|\frac{1}{N}\sum_{i=1}^N\left(\bbE\Big[h_j(Z^*,\bZ_{\del i},Q^{*(b)})|\bZ_{\del i}\Big] - \bbE\Big[h_j(Z^*,\bZ,Q^{*(b)})|\bZ\Big]\right)\right|\lesssim \frac{q_j^{*(b)}n}{N}.
        \end{split}
    \end{equation}
\end{lem}

Combining the bounds above and the decomposition in~\eqref{eq:Delta_err_decomp}, we have that with probability at least $1-(MN)^{-c}$, for any $1\leq b\leq B$, $1\leq j\leq M$,
\begin{equation*}
    \begin{split}
        |\hD{b+1}{j} - \D{b+1}{j}| \lesssim &\sqrt{\frac{q_j^{(b)}(\log M+\log N)}{K}} + \frac{\log M+\log N}{K} \\
        &+\Big((q_j^{(b)}+q_j^{*(b)})\|\bq^{(b)}-\bq^{*(b)}\|_1+|q_j^{(b)}-q_j^{*(b)}|\Big) \\
        &+q_j^{*(b)}\sqrt{\frac{n}{N}(\log M+\log N)}.
    \end{split}
\end{equation*}

\subsection{Proof of Lemma~\ref{lem:training_err}}
Recall our definition of $\Delta_j^{*(b+1)}=\Delta_j^{*(b+1)}(Q^{(b)})$:
$$
\Delta_j^{*(b+1)}(Q^{(b)}) = \bbE\Big[\Big(Y^*-\bbE_{F\sim Q^{(b)}_{\no j}}\mu_F^*(\notationchange{X^{*(F)}}\Big)^2 - \Big(Y^*-\bbE_{F\sim Q^{(b)}}\mu_F^*(\notationchange{X^{*(F)}})\Big)^2\Big],
$$
where $\mu_F^*(\cdot)$ is the conditional expectation of $Y$ given any realization of $\notationchange{X^{(F)}}$: $\mu_F^*(\notationchange{x^{(F)}})= \bbE(Y|\notationchange{X^{(F)}}=\notationchange{x^{(F)}})$. The expectation is taken over $(X^*,Y^*)\sim \cP_Z$, the same distribution as $Z_i$. We also define:
\begin{equation*}
    \Delta_j^{*(b+1)}(Q^{*(b)}) = \bbE\Big[\Big(Y^*-\bbE_{F\sim Q^{*(b)}_{\no j}}\mu_F^*(\notationchange{X^{*(F)}})\Big)^2 - \Big(Y^*-\bbE_{F\sim Q^{*(b)}}\mu_F^*(\notationchange{X^{*(F)}})\Big)^2\Big].
\end{equation*}
We can then bound $\Delta_j^{(b+1)}(Q^{*(b)})-\Delta_j^{*(b+1)}(Q^{(b)})$ as follows:
\begin{equation}\label{eq:trainerr_decomp}
    \begin{split}
      \Big|\Delta_j^{(b+1)}(Q^{*(b)})-\Delta_j^{*(b+1)}(Q^{(b)})\Big|\leq &  \underbrace{\Big|\Delta_j^{(b+1)}(Q^{*(b)})-\Delta_j^{*(b+1)}(Q^{*(b)})\Big|}_{E_1}+\underbrace{\Big|\Delta_j^{*(b+1)}(Q^{*(b)})-\Delta_j^{*(b+1)}(Q^{(b)})\Big|}_{E_2}
    \end{split}
\end{equation}
In the following, we will bound $E_1$ and $E_2$ separately.

\paragraph{Bounding $E_1$.} Let $\mu_F(\cdot) = \frac{1}{\binom{N}{n}}\sum_{I\subset [N],|I|=n}\mu_{I,F}(\cdot)$, $\eta(X):=\bbE_{F\sim Q^{*(b)}}\mu_F(\notationchange{X^{(F)}})$, $\eta_{\no j}(X):=\bbE_{F\sim Q^{*(b)}_{\no j}}\mu_F(\notationchange{X^{(F)}})$, $\eta_{j}(X):=\bbE_{F\sim Q^{*(b)}_{j}}\mu_F(\notationchange{X^{(F)}})$. We can then write
\begin{equation*}
    \begin{split}
        \Delta_j^{(b+1)}(Q^{*(b)}) = &\bbE_{X^*,Y^*}\Big[\Big(Y^*-\eta_{\no j}(X^*)\Big)^2 - \Big(Y^*-\eta(X^*)\Big)^2\Big]\\
        =&\bbE_{X^*,Y^*}\Big[\Big(\eta(X^*)-\eta_{\no j}(X^*)\Big)\Big(2Y^*-\eta(X^*)-\eta_{\no j}(\notationchange{X^*})\Big)\Big]\\
        =&\bbP_{F\sim Q^{*(b)}}(F\owns j)\bbE_{X^*,Y^*}\Big[\Big(\eta_{j}(X^*)-\eta_{\no j}(X^*)\Big)\Big(2Y^*-\eta(X^*)-\eta_{\no j}(\notationchange{X^*})\Big)\Big],
    \end{split}
\end{equation*}
where in the last line, we have used the fact that
\begin{equation*}
\begin{split}
    \eta(X)-\eta_{\no j}(X) = &\bbE_{F\sim Q}\mu_F(X)-\bbE_{F\sim Q_{\no j}}\mu_F(X)\\
    =&\bbP_{F\sim Q}(F\owns j)\bbE_{F\sim Q_j}\mu_F(X)+\bbP_{F\sim Q}(F\not\owns j)\bbE_{F\sim Q_{\no j}}\mu_F(X)-\bbE_{F\sim Q_{\no j}}\mu_F(X)\\
    =&\bbP_{F\sim Q}(F\owns j)\bbE_{F\sim Q_j}\mu_F(X)-\bbP_{F\sim Q}(F\owns j)\bbE_{F\sim Q_{\no j}}\mu_F(X).
\end{split}
\end{equation*}

Similarly, let $\eta^*(X):=\bbE_{F\sim Q^{*(b)}}\mu_F^*(\notationchange{X^{(F)}})$, $\eta^*_{\no j}(X):=\bbE_{F\sim Q^{*(b)}_{\no j}}\mu_F^*(\notationchange{X^{(F)}})$, $\eta^*_{j}(X):=\bbE_{F\sim Q^{*(b)}_{j}}\mu_F^*(\notationchange{X^{(F)}})$. We can then write
\begin{equation*}
        \Delta_j^{*(b+1)}(Q^{*(b)})=\bbP_{F\sim Q^{*(b)}}(F\owns j)\bbE_{X^*,Y^*}\Big[\Big(\eta_{j}^*(X^*)-\eta_{\no j}^*(X^*)\Big)\Big(2Y^*-\eta^*(X^*)-\eta_{\no j}^*(\notationchange{X^*})\Big)\Big].
\end{equation*}
Due to Assumption~\ref{ass:adp_bounded} and the fact that $\bbP_{F\sim Q^{*(b)}}(F\owns j) \leq Cq_j^{*(b)}$, we have
\begin{equation}\label{eq:trainerr_E1}
    \begin{split}
        \Big|\Delta_j^{(b+1)}(Q^{*(b)})-\Delta_j^{*(b+1)}(Q^{*(b)})\Big|\leq &Cq_j^{*(b)}\bbE_{X^*,Y^*}\Big|\eta_j(X^*)-\eta_{\no j}(X^*)-\eta^*_j(X^*)+\eta_{\no j}^*(X^*)\Big|\\
        &+ Cq_j^{*(b)}\bbE_{X^*,Y^*}\Big|\eta(X^*)+\eta_{\no j}(X^*)-\eta^*(X^*)+\eta_{\no j}^*(X^*)\Big|\\
        \leq&Cq_j^{*(b)}\Big(\bbE_{F\sim Q^{*(b)}}\varepsilon_{\rm train}(F)+\bbE_{F\sim Q^{*(b)}_{\no j}}\varepsilon_{\rm train}(F)+\bbE_{F\sim Q^{*(b)}_j}\varepsilon_{\rm train}(F)\Big).
    \end{split}
\end{equation}

\paragraph{Bounding $E_2$.} While for bounding $E_2=\Big|\Delta_j^{*(b+1)}(Q^{*(b)})-\Delta_j^{*(b+1)}(Q^{(b)})\Big|$, we note that the error takes a very similar form to 
$$
\Big|\bbE_{Z^*}\Big(h_j(Z^*,\bZ, Q^{*(b)})-h_j(Z^*,\bZ, Q^{(b)}\Big)\Big|,
$$
with the only difference being that the oracle predictor $\mu_F^*(\notationchange{X^{*(F)}})$ replaces the trained base model $\mu_F(\notationchange{X^{*(F)}})$. In the proof of Lemma~\ref{lem:DeltaErrDiffQ} where we bound $h_j(Z^*,\bZ, Q^{*(b)})-h_j(Z^*,\bZ, Q^{(b)}$, we only used the boundedness property of $\mu_F(\notationchange{X^{*(F)}})$, which also holds for $\mu_F^*(\notationchange{X^{*(F)}})=\bbE(Y|\notationchange{X^{*(F)}})$ due to Assumption~\ref{ass:adp_bounded}. Therefore, all arguments in the proof of Lemma~\ref{lem:DeltaErrDiffQ} still hold if we adapt them to $\Delta_j^{*(b+1)}(Q^{*(b)})-\Delta_j^{*(b+1)}(Q^{(b)})$, and we have that
\begin{equation}\label{eq:trainerr_E2}
    E_2\lesssim \big(q_j^{*(b)}+q_j^{(b)}\big)\big\|\bq^{*(b)}-\bq^{(b)}\big\|_1+\big|q_j^{*(b)}-q_j^{(b)}\big|.
\end{equation}
Combining~\eqref{eq:trainerr_decomp}, \eqref{eq:trainerr_E1}, and \eqref{eq:trainerr_E2}, the proof of Lemma~\ref{lem:training_err} is now complete.

\subsection{Proof of Other Auxiliary Lemmas Needed for Theorem~\ref{thm:main2}}
\begin{proof}[Proof of Lemma~\ref{lem:finiteKerr}] 
Recall the definition of $\hat{h}_j^{(b+1)}(Z_i,\bZ_{\no i})$ in Table~\ref{tab:notations}. We can write
\begin{equation*}
    \begin{split}
        \hat{h}_j^{(b+1)}(Z_i,\bZ_{\no i}) =& \big(Y_i-\mu_{-i}^{-j}(X_i)\big)^2 - \big(Y_i-\mu_{-i}(X_i)\big)^2\\
        =&\big(\mu_{-i}(X_i)-\mu_{-i}^{-j}(X_i)\big)\big(2Y_i-\mu_{-i}^{-j}(X_i)-\mu_{-i}(X_i)\big),
    \end{split}
\end{equation*}
where 
$$
\mu_{-i}(X_i) = \frac{\sum_{k=1}^K\ind(i\notin I_k)\mu_k(X_i)}{\sum_{k=1}^K\ind(i\notin I_k)},
$$
$$\mu_{-i}^{-j}(X_i) = \frac{\sum_{k=1}^K \ind(i\notin I_k)\ind(j\notin F_k)\mu_k(X_i)}{\sum_{k=1}^K\ind(i\notin I_k)\ind(j\notin F_k)},
$$
and $\mu_k(\cdot)$ is the predictor trained on the $k$th minipatch $(I_k,F_k)$, with $(I_k, F_k)$ sampled from $\cU^{[N]}_n\times Q^{(b)}$. For notational simplicity, throughout this proof, we will omit the superscript in $Q^{(b)}$, $\bq^{(b)}$, $\hat{h}^{(b+1)}_j$, and write $Q$, $\bq$, $\hat{h}_j$ instead.

Similarly, we can write
\begin{equation*}
    \begin{split}
        h_j(Z_i,\bZ_{\no i},Q)=&\big(Y_i-\mu^*(X_i;Q_{\no j},\bZ_{\no i})\big)^2 - \big(Y_i-\mu^*(X_i;Q,\bZ_{\no i})\big)^2\\
        =&\big(\mu^*(X_i;Q,\bZ_{\no i})-\mu^*(X_i;Q_{\no j},\bZ_{\no i})\big)\big(2Y_i-\mu^*(X_i;Q_{\no j},\bZ_{\no i})-\mu^*(X_i;Q,\bZ_{\no i})\big).
    \end{split}
\end{equation*}
where $\mu^*(X_i;Q_{\no j},\bZ_{\no i})=\bbE_{I\sim \cU^{[N]\no i}_n}\bbE_{F\sim Q_{\no j}}\mu_{I,F}(X_i)$, $\mu^*(X_i;Q,\bZ_{\no i})=\bbE_{I\sim \cU^{[N]\no i}_n}\bbE_{F\sim Q}\mu_{I,F}(X_i)$ are as defined in Table~\ref{tab:notations}, $Q_{\no j}$ is the conditional distribution of $F\sim Q$ given that $F\not\owns j$, and $\cU^{[N]\no i}_n$ is the uniform distribution over size-$n$ subsets of $[N]\no i$ as defined in Section~\ref{sec:supp_notations}. We can then decompose the error term $\hat{h}_j(Z_i,\bZ_{\no i}) - h_j(Z_i,\bZ_{\no i},Q)$ as follows:

\begin{equation}\label{eq:hat_h_err1}
    \begin{split}
        &\Big|\hat{h}_j(Z_i,\bZ_{\no i}) - h_j(Z_i,\bZ_{\no i},Q)\Big|\\
        &\hspace{1em}\leq \Big|\mu^*(X_i;Q,\bZ_{\no i})-\mu^*(X_i;Q_{\no j},\bZ_{\no i})\Big|\\
        &\hspace{2em}\cdot \Big|\mu_{-i}^{-j}(X_i)+\mu_{-i}(X_i)-\mu^*(X_i;Q_{\no j},\bZ_{\no i})-\mu^*(X_i;Q,\bZ_{\no i})\Big|\\
        &\hspace{2em}+\Big|\mu_{-i}(X_i)-\mu_{-i}^{-j}(X_i) - \mu^*(X_i;Q,\bZ_{\no i}) +\mu^*(X_i;Q_{\no j},\bZ_{\no i})\Big|\\
        &\hspace{2em}\cdot \Big|2Y_i-\mu^*(X_i;Q_{\no j},\bZ_{\no i})-\mu^*(X_i;Q,\bZ_{\no i})\Big|.
    \end{split}
\end{equation}

To further simplify~\eqref{eq:hat_h_err1}, we note that Assumption~\ref{ass:adp_bounded} immediately implies that $\Big|2Y_i-\mu^*(X_i;Q_{\no j},\bZ_{\no i})-\mu^*(X_i;Q,\bZ_{\no i})\Big|\leq C$. While for $\Big|\mu_{-i}^{-j}(X_i)+\mu_{-i}(X_i)-\mu^*(X_i;Q_{\no j},\bZ_{\no i})-\mu^*(X_i;Q,\bZ_{\no i})\Big|$, we will show that it is bounded by $Cq_j$ in the following.

Define 
\begin{equation}\label{eq:alpha_beta_gamma_def}
    \begin{split}
        \alpha_{i,k} = \ind(i\notin I_k)\mu_k(X_i),\quad &\beta_{i,j,k} = \ind(i\notin I_k)\ind(j\notin F_k)\mu_k(X_i),\\
        \gamma_{i,j,k}=&\ind(i\notin I_k)\ind(j\in F_k)\mu_k(X_i).
    \end{split}
\end{equation}

They satisfy $|\alpha_{i,k}|, |\beta_{i,j,k}|, |\gamma_{i,j,k}|\leq C$, and $\bbE|\gamma_{i,j,k}| \leq C\bbP(j\in F_k)\frac{Cq_j}{1-\prod_{l=1}^M(1-q_l)}\leq Cq_j$. Also define 
\begin{equation}\label{eq:p_no_i_def}
    \begin{split}
        p_{-i}=\bbP(i\notin I_k),\quad& p_{-i,-j}= \bbP(i\notin I_k, j\notin F_k),\\
        p_{-i,j}=&\bbP(i\notin I_k, j\in F_k),
    \end{split}
\end{equation}
as well as their empirical counterparts 
\begin{equation}\label{eq:hat_p_no_i_def}
    \begin{split}
        \hat{p}_{-i}=\frac{\sum_{k=1}^K\ind(i\notin I_k)}{K},\quad&\hat{p}_{-i,-j} = \frac{\sum_{k=1}^K\ind(i\notin I_k)\ind(j\notin F_k)}{K}\\
        \hat{p}_{-i,j} = &\frac{\sum_{k=1}^K\ind(i\notin I_k)\ind(j\in F_k)}{K}.
    \end{split}
\end{equation} 
We can then further write
\begin{equation}\label{eq:mu_diff_decomp}
    \begin{split}
        &\mu^*(X_i;Q,\bZ_{\no i})-\mu^*(X_i;Q_{\no j},\bZ_{\no i})\\
        &\hspace{1em}=\bbE_{I\sim \cU^{[N]\no i}_n}\bbE_{F\sim Q}\mu_{I,F}(X_i)-\bbE_{F\sim Q_{\no j}}\mu_{I,F}(X_i)\\
        &\hspace{1em}= \bbP(j\in F)\bbE_{I\sim \cU^{[N]\no i}_n}\Big[\bbE_{F\sim Q_j}\mu_{I,F}(X_i)-\bbE_{F\sim Q_{\no j}}\mu_{I,F}(X_i)\Big]\\
        &\hspace{1em}=p_{-i}^{-1}\bbE_{I_k,F_k}(\gamma_{i,j,k}) - p_{-i}^{-1}p_{-i,-j}^{-1}p_{-i,j}\bbE_{I_k,F_k}(\beta_{i,j,k}),
    \end{split}
\end{equation}
where $Q_j$ is the conditional distribution of $F\sim Q$ given that $F\owns j$. Since there exists constants $c,C>0$ such that $|\bbE_{I_k,F_k}(\gamma_{i,j,k})|\leq Cq_j$, $|\bbE_{I_k,F_k}(\beta_{i,j,k})|\leq C$, $p_{-i}=1-\frac{n}{N}\geq c>0$, $p_{-i,-j}=\big(1-\frac{n}{N}\big)(1-q_j)\geq c>0$, $p_{-i,j} = \big(1-\frac{n}{N}\big) q_j\leq q_j$, we then have
\begin{equation}\label{eq:mu_diff_bnd}
    \Big|\mu^*(X_i;Q,\bZ_{\no i})-\mu^*(X_i;Q_{\no j},\bZ_{\no i})\Big|\leq Cq_j.
\end{equation}
Applying Assumption~\ref{ass:adp_bounded} and plugging in \eqref{eq:mu_diff_bnd} into \eqref{eq:hat_h_err1}, we have

\begin{equation}\label{eq:finiteK_err_decomp}
    \begin{split}
        &\Big|\hat{h}_j(Z_i,\bZ_{\no i}) - h_j(Z_i,\bZ_{\no i},Q)\Big|\\
        &\hspace{1em}\leq Cq_j\underbrace{\Big|\mu_{-i}(X_i)-\mu^*(X_i;Q,\bZ_{\no i})+\mu_{-i}^{-j}(X_i)-\mu^*(X_i;Q_{\no j},\bZ_{\no i})\Big|}_{E_1}\\
        &\hspace{2em}+C\underbrace{\Big|\mu_{-i}(X_i)-\mu_{-i}^{-j}(X_i) - \mu^*(X_i;Q,\bZ_{\no i}) +\mu^*(X_i;Q_{\no j},\bZ_{\no i})\Big|}_{E_2}.
    \end{split}
\end{equation}
In the following, we will bound $E_1$ and $E_2$ separately.

\paragraph{Bounding $E_1$.} We first note that
\begin{equation*}
    \begin{split}
        \mu_{-i}(X_i)= \frac{\hat{p}_{-i}^{-1}}{K}\sum_{k=1}^K\alpha_{i,k},\quad& \mu_{-i}^{-j}(X_i) = \frac{\hat{p}_{-i,-j}^{-1}}{K}\sum_{k=1}^K\beta_{i,j,k},\\
        \mu^*(X_i;Q,\bZ_{\no i}) = p_{-i}^{-1}\bbE_{I_k,F_k}(\alpha_{i,k}),\quad&
        \mu^*(X_i;Q_{\no j},\bZ_{\no i}) = p_{-i,-j}^{-1}\bbE_{I_k,F_k}(\beta_{i,j,k}).
    \end{split}
\end{equation*}
Hence, we can write
\begin{equation}\label{eq:finiteK_E1_bnd1}
    \begin{split}
        E_1 \leq & \big|\hat{p}_{-i}^{-1}-p_{-i}^{-1}\big|\big|\bbE_{I_k,F_k}(\alpha_{i,k})\big| + \hat{p}_{-i}^{-1}\left|\frac{1}{K}\sum_{k=1}^K\big(\alpha_{i,k}-\bbE_{I_k,F_k}(\alpha_{i,k})\big)\right|\\
        &\hspace{2em}+\big|\hat{p}_{-i,-j}^{-1}-p_{-i,-j}^{-1}\big|\big|\bbE_{I_k,F_k}(\beta_{i,j,k})\big| + \hat{p}_{-i,-j}^{-1}\left|\frac{1}{K}\sum_{k=1}^K\big(\beta_{i,j,k}-\bbE_{I_k,F_k}(\beta_{i,j,k})\big)\right|.
    \end{split}
\end{equation}
By the definition of $\alpha_{i,k}$ and $\beta_{i,j,k}$ in \eqref{eq:alpha_beta_gamma_def} and Assumption~\ref{ass:adp_bounded}, we have $|\alpha_{i,k}|,|\beta_{i,j,k}|\leq C$. Applying the Hoeffding's inequality, we know that with probability at least $1-M^{-c}$, for any $1\leq j\leq M$, $1\leq i\leq N$,
\begin{equation}\label{eq:alpha_beta_concen}
    \begin{split}
        \left|\frac{1}{K}\sum_{k=1}^K(\alpha_{i,k}-\bbE_{I_k,F_k}(\alpha_{i,k})\right|\leq &C\sqrt{\frac{\log M}{K}},\\
        \left|\frac{1}{K}\sum_{k=1}^K\big(\beta_{i,j,k}-\bbE_{I_k,F_k}(\beta_{i,j,k})\big)\right|\leq&C\sqrt{\frac{\log M}{K}},
    \end{split}
\end{equation}
for some universal constants $c,C>0$. Furthermore, by the definition in \eqref{eq:p_no_i_def} and \eqref{eq:hat_p_no_i_def}, 
\begin{equation*}
    \begin{split}
        \hat{p}_{-i} - p_{-i}= &\frac{1}{K}\sum_{k=1}^K\big[\ind(i\notin I_k)- \bbE(\ind(i\notin I_k))\big]\\
        \hat{p}_{-i,-j} - p_{-i,-j}= &\frac{1}{K}\sum_{k=1}^K\big[\ind(i\notin I_k,j\notin F_k)- \bbE(\ind(i\notin I_k,j\notin F_k))\big],
    \end{split}
\end{equation*}
Hoeffding's inequality also implies that with probability at least $1-(MN)^{-c}$, for any $1\leq j\leq M$, $1\leq i\leq N$,
\begin{equation}\label{eq:p_hat_concen}
    \begin{split}
        \big|\hat{p}_{-i} - p_{-i}\big|\leq &C\sqrt{\frac{\log M+\log N}{K}},\\
        \big|\hat{p}_{-i,-j} - p_{-i,-j}\big|\leq&C\sqrt{\frac{\log M+\log N}{K}},
    \end{split}
\end{equation}
for some universal constants $c,C>0$. Under Assumptions~\ref{ass:small_n} and \ref{ass:adp_n_minipatches}, $\hat{p}_{-i}\geq p_{-i}-\big|\hat{p}_{-i} - p_{-i}\big|\geq c>0$, $\hat{p}_{-i,-j}\geq p_{-i,-j}-\big|\hat{p}_{-i,-j} - p_{-i,-j}\big|\geq c>0$ for some constant $c>0$. Therefore, \eqref{eq:p_hat_concen} implies that
\begin{equation}\label{eq:hat_p_inverse_concen}
    \begin{split}
       \big|\hat{p}_{-i}^{-1} - p_{-i}^{-1}\big| \leq C\big|\hat{p}_{-i} - p_{-i}\big|\leq &C\sqrt{\frac{\log M+\log N}{K}}\leq C,\\
       \big|\hat{p}_{-i,-j}^{-1} - p_{-i,-j}^{-1}\big| \leq C\big|\hat{p}_{-i,-j} - p_{-i,-j}\big|\leq &C\sqrt{\frac{\log M+\log N}{K}}\leq C.
    \end{split}
\end{equation}
Plugging in \eqref{eq:hat_p_inverse_concen} and \eqref{eq:alpha_beta_concen} into \eqref{eq:finiteK_E1_bnd1}, we have
\begin{equation}\label{eq:finiteK_E1_bnd2}
        \bbP\Big(E_1 >C\sqrt{\frac{\log M+\log N}{K}}\Big)\leq (MN)^{-c}.
\end{equation}

\paragraph{Bounding $E_2$.} We first note that
\begin{equation*}
    \begin{split}
        \mu_{-i}(X_i)-\mu_{-i}^{-j}(X_i)=&\frac{\sum_{k=1}^K\ind(i\notin I_k)\mu_k(X_i)}{\sum_{k=1}^K\ind(i\notin I_k)}  - \frac{\sum_{k=1}^K\ind(i\notin I_k)\ind(j\notin F_k)\mu_k(X_i)}{\sum_{k=1}^K\ind(i\notin I_k)\ind(j\notin F_k)}\\
        =&\frac{\hat{p}_{-i}^{-1}}{K}\sum_{k=1}^K\gamma_{i,j,k} - \hat{p}_{-i}^{-1}\hat{p}_{-i,-j}^{-1}\hat{p}_{-i,j}\frac{1}{K}\sum_{k=1}^K\beta_{i,j,k}.
    \end{split}
\end{equation*}
As shown in \eqref{eq:mu_diff_decomp}, 
\begin{equation*}
    \mu^*(X_i;Q,\bZ_{\no i})-\mu^*(X_i;Q_{\no j},\bZ_{\no i})=p_{-i}^{-1}\bbE_{I_k,F_k}(\gamma_{i,j,k}) - p_{-i}^{-1}p_{-i,-j}^{-1}p_{-i,j}\bbE_{I_k,F_k}(\beta_{i,j,k}).
\end{equation*}
Therefore, 
\begin{equation}\label{eq:finiteK_E2_bnd1}
    \begin{split}
        &\Big|\mu_{-i}(X_i)-\mu_{-i}^{-j}(X_i) -\mu^*(X_i;Q,\bZ_{\no i})+\mu^*(X_i;Q_{\no j},\bZ_{\no i})\Big|\\
        &\hspace{1em}\leq \big|\hat{p}_{-i}^{-1}-p_{-i}^{-1}\big|\big|\bbE_{I_k,F_k}(\gamma_{i,j,k})\big| + \hat{p}_{-i}^{-1}\left|\frac{1}{K}\sum_{k=1}^K(\gamma_{i,j,k}-\bbE_{I_k,F_k}(\gamma_{i,j,k}))\right|\\
        &\hspace{2em}+\big|\hat{p}_{-i}^{-1}\hat{p}_{-i,-j}^{-1}\hat{p}_{-i,j}-p_{-i}^{-1}p_{-i,-j}^{-1}p_{-i,j}\big|\big|\bbE_{I_k,F_k}(\beta_{i,j,k})\big| \\
        &\hspace{2em}+ \hat{p}_{-i}^{-1}\hat{p}_{-i,-j}^{-1}\hat{p}_{-i,j}\left|\frac{1}{K}\sum_{k=1}^K(\beta_{i,j,k}-\bbE_{I_k,F_k}(\beta_{i,j,k})\right|.
    \end{split}
\end{equation}
The bounds we established in \eqref{eq:alpha_beta_concen} and \eqref{eq:hat_p_inverse_concen} are also applicable here. What remain to be shown are concentration bounds for $\gamma_{i,j,k}$ and $\hat{p}_{-i,j}$. We note that Assumptions~\ref{ass:adp_bounded} and \ref{ass:small_n} imply that
\begin{equation*}
    \begin{split}
        |\gamma_{i,j,k}|\leq C, \quad &|\bbE(\gamma_{i,j,k})|\leq Cp_{-i,j}\leq Cq_j,\quad \bbE(\gamma_{i,j,k}^2)\leq C^2p_{-i,j}\leq C^2q_j,\\
        \ind(i\notin I_k,j\in F_k)\leq 1,\quad&\bbE(\ind^2(i\notin I_k,j\in F_k))= p_{-i,j}\leq q_j.
    \end{split}
\end{equation*}
Therefore, we can apply Bernstein's inequality and a union bound to obtain the following: with probability at least $1-(MN)^{-c}$, for $1\leq j\leq M$, $1\leq i\leq N$,
\begin{equation}\label{eq:gamma_bnd}
    \begin{split}
        \left|\frac{1}{K}\sum_{k=1}^K(\gamma_{i,j,k}-\bbE(\gamma_{i,j,k}))\right|\leq &C\Big(\sqrt{\frac{q_j(\log M+\log N)}{K}}+\frac{\log M+\log N}{K}\Big),\\
        |\hat{p}_{-i,j}-p_{-i,j}|=&\Big|\frac{1}{K}\sum_{k=1}^K\big[\ind(i\notin I_k,j\in F_k)-\bbE(\ind(i\notin I_k,j\in F_k))\big]\Big|\\
        \leq& C\Big(\sqrt{\frac{q_j(\log M+\log N)}{K}}+\frac{\log M+\log N}{K}\Big),
    \end{split}
\end{equation}
for some universal constants $c,C>0$. We also note that 
\begin{equation}\label{eq:hat_p3_inverse_concen}
\begin{split}
    &|\hat{p}_{-i}^{-1}\hat{p}_{-i,-j}^{-1}\hat{p}_{-i,j}-p_{-i}^{-1}p_{-i,-j}^{-1}p_{-i,j}|\\
    &\hspace{1em}\leq \hat{p}_{-i}^{-1}\hat{p}_{-i,-j}^{-1}|\hat{p}_{-i,j}-p_{-i,j}|+ p_{-i,j}|\hat{p}_{-i}^{-1}\hat{p}_{-i,-j}^{-1}-p_{-i}^{-1}p_{-i,-j}^{-1}|\\
    &\hspace{1em}\leq C|\hat{p}_{-i,j}-p_{-i,j}| + Cq_j\sqrt{\frac{\log M+\log N}{K}}\\
    &\hspace{1em}\leq C\Big(\sqrt{\frac{q_j(\log M+\log N)}{K}}+\frac{\log M+\log N}{K}\Big).
\end{split}
\end{equation}

Therefore, plugging in \eqref{eq:hat_p_inverse_concen}, \eqref{eq:gamma_bnd}, \eqref{eq:hat_p3_inverse_concen} into \eqref{eq:finiteK_E2_bnd1}, we have that with probability at least $1-(MN)^{-c}$,
\begin{equation}\label{eq:finiteK_E2_bnd2}
\begin{split}
    E_2 \leq& Cq_j\sqrt{\frac{\log M+\log N}{K}}+C\Big(\sqrt{\frac{q_j(\log M+\log N)}{K}}+\frac{\log M+\log N}{K}\Big)\\
    \leq&C\Big(\sqrt{\frac{q_j(\log M+\log N)}{K}}+\frac{\log M+\log N}{K}\Big).
\end{split}
\end{equation}
Plugging in \eqref{eq:finiteK_E1_bnd2} and \eqref{eq:finiteK_E2_bnd2} into \eqref{eq:finiteK_err_decomp}, we arrive at the following: with probability at least $1-(MN)^{-c}$, for any $1\leq j\leq M$, any $0\leq b\leq B-1$, and $1\leq i\leq N$,
\begin{equation*}
    \Big|\hat{h}_j^{(b+1)}(Z_i,\bZ_{\no i}) - h_j(Z_i,\bZ_{\no i},Q)\Big|\leq C\Big(\sqrt{\frac{q_j^{(b)}(\log M+\log N)}{K}}+\frac{\log M+\log N}{K}\Big).
\end{equation*}
The proof of Lemma~\ref{lem:finiteKerr} is now complete.
\end{proof}

\begin{proof}[Proof of Lemma~\ref{lem:DeltaErrDiffQ}]
First, define $\mu_F(X;\bZ) = \frac{1}{\binom{N}{n}}\sum_{I\subset[N], |I|=n}\mu_{I,F}(X;\bZ)$, $\varepsilon(Z,\bZ, Q) = Y-\sum_{F\subset[M]}\mu_F(X;\bZ)Q(F)$, $\delta_j(F;Q) = Q(F)-Q_{\backslash j}(F) = Q(F)\left(\ind(j\in F) - \omega_j(Q)\ind(j\notin F)\right)$, $\omega_j(Q) = \frac{\bbP_{F\sim Q}(j\in F)}{\bbP_{F\sim Q}(j\notin F)} = \frac{q_j}{1-q_j-\prod_{k=1}^M(1-q_k)}$. Due to the definition of $h_j(Z,\bZ,Q)$, we can decompose it as follows for any feature sampling distribution function $Q$:
\begin{equation}\label{eq:h_j}
\begin{split}
    h_j(Z,\bZ,Q) =&\left(Y-\sum_{F\subset[M]}\mu_F(X;\bZ)Q_{\del j}(F)\right)^2 - \left(Y-\sum_{F\subset[M]}\mu_F(X;\bZ)Q(F)\right)^2\\
    =&2\varepsilon(Z,\bZ,Q) \sum_F\mu_F(X;\bZ)\delta_j(F;Q) + \left(\sum_F\mu_F(X;\bZ)\delta_j(F;Q)\right)^2.
\end{split}
\end{equation}
Hence, we can write
    \begin{equation*}
        \begin{split}
            &\left|h_j(Z,\bZ,Q_{\bq}) -h_j(Z,\bZ,Q_{\tilde{\bq}})\right|\\
            \leq &\underbrace{2\left(\bbE_{F\sim Q_{\tilde{\bq}}}\mu_F(X;\bZ) - \bbE_{F\sim Q_{\bq}}\mu_F(X;\bZ)\right)\sum_{F}\mu_F(X;\bZ)\delta_j(F;Q_{\bq})}_{=:\eta_1}\\
            &+\underbrace{2\varepsilon(Z,\bZ,Q_{\tilde{\bq}})\sum_F\mu_F(X;\bZ)\big(\delta_j(F;Q_{\bq}) - \delta_j(F;Q_{\tilde{\bq}})\big)}_{=:\eta_2}\\
            &+\underbrace{\sum_{F,F'}\mu_F(X;\bZ)\mu_{F'}(X;\bZ)\left(\delta_j(F;Q_{\bq})\delta_j(F';Q_{\bq}) - \delta_j(F;Q_{\tilde{\bq}})\delta_j(F';Q_{\tilde{\bq}})\right)}_{=:\eta_3}.
        \end{split}
    \end{equation*}
    We bound the three terms in the following.
    \paragraph{Bounding $\eta_1$.} 

We first note that 
    \begin{equation*}
        \begin{split}
            \sum_{F}\mu_F(X;\bZ)\delta_j(F,Q_{\bq})\leq &
            \max_{F}|\mu_F(X;\bZ)|\left(\bbP_{F\sim Q_{\bq}}(j\in F) + \omega_j(Q_{\bq})\right)\\
            \leq &\left(\frac{q_j}{1-e^{-m}} + \frac{q_j}{1-\delta-e^{-m}}\right)\max_{F}|\mu_F(X;\bZ)|\\
            \leq &Cq_j\max_{F}|\mu_F(X;\bZ)|,
        \end{split}
    \end{equation*}
    where the second line is due to that when $F\sim Q_{\bq}$ \begin{equation*}
        \begin{split}
            \bbP(j\in F) = &\frac{q_j}{1-\prod_{k}(1-q_k)} \\
            = &\frac{q_j}{1-\exp\{\sum_{k}\log(1-q_k)\}} \\
            \leq &\frac{q_j}{1-\exp\{-\sum_{k=1}^Mq_k\}}\\
            = &\frac{q_j}{1-e^{-m}},
        \end{split}
    \end{equation*}
as well as the fact that $q_j\leq \delta$; the last line is due to Assumption~\ref{ass:small_n} which states that $m\geq -\log(1-\delta-c)$, implying that $1-\delta-e^{-m}\geq c>0$.

To bound $\bbE_{F\sim Q_{\tilde{\bq}}}\mu_F(X;\bZ) - \bbE_{F\sim Q_{\bq}}\mu_F(X;\bZ)$, we need the following lemma.
    \begin{lem}\label{lem:F1_F2_diff}
        If $\|\bq-\tilde{\bq}\|_{\infty}\leq \frac{1}{2}$, $m\geq\log(1+\frac{1}{1-\delta})$, then there exists a joint distribution $Q^{(1,2)}$ on $2^{[M]}\times 2^{[M]}$ such that when $(F^{(1)},\,F^{(2)})\sim Q^{(1,2)}$, their marginal distributions are $Q_{\bq}$ and $Q_{\tilde{\bq}}$, i.e., $F^{(1)}\sim Q_{\bq}$, $F^{(2)}\sim Q_{\tilde{\bq}}$, and 
        \begin{align}
            \bbP(F^{(1)}\neq F^{(2)})\leq &\frac{2\|\bq-\tilde{\bq}\|_1}{1-e^{-m}},\label{eq:F1F2_diff1}\\
                \bbE\left||F^{(1)}|-|F^{(2)}|\right|\leq &\frac{2m+1}{(1-e^{-m})^2}\|\bq - \tilde{\bq}\|_1,\label{eq:F1F2_diff2}\\
                \bbP\left(j\in F^{(1)},\,F^{(1)}\neq F^{(2)}\right)\leq &\frac{3q_j\|\bq-\tilde{\bq}\|_1 + |\tilde{q}_j - q_j|}{1-e^{-m}},\label{eq:F1F2_diff3}\\
                \bbP\left(j\in F^{(2)},\,F^{(1)}\neq F^{(2)}\right)\leq &\frac{3\tilde{q}_j\|\bq-\tilde{\bq}\|_1 + |\tilde{q}_j - q_j|}{1-e^{-m}}.\label{eq:F1F2_diff4}
        \end{align}
    \end{lem}
Let $F^{(1)}$ and $F^{(2)}$ be the random feature subsets that satisfy \eqref{eq:F1F2_diff1}-\eqref{eq:F1F2_diff3} in Lemma \ref{lem:F1_F2_diff}. We can then write
    \begin{equation*}
        \begin{split}
            &\left|\bbE_{F\sim Q_{\tilde{\bq}}}\mu_F(X;\bZ) - \bbE_{F\sim Q_{\bq}}\mu_F(X;\bZ)\right|\\
            = &\left|\bbE_{F^{(1)}, F^{(2)}}\left(\mu_{F^{(1)}}(X;\bZ) - \mu_{F^{(2)}}(X;\bZ)\right)\right|\\
            \leq&\bbE_{F^{(1)}, F^{(2)}}\big(2\max_F|\mu_{F^{(1)}}(X;\bZ)|\ind(F^{(1)}\neq F^{(2)})\big)\\
            \leq &2\max_F|\mu_{F}(X;\bZ)|\bbP(F^{(1)}\neq F^{(2)})\\
            \leq& C\max_F|\mu_{F}(X;\bZ)|\|\bq-\tilde{\bq}\|_1,
        \end{split}
    \end{equation*}
    where we have utilized the fact that $1-e^{-m}\geq c>0$ again.
    Therefore, we can bound $\eta_1$ as follows:
    \begin{equation*}
        \begin{split}
            |\eta_1| = &2\left|\bbE_{F\sim Q_{\tilde{\bq}}}\mu_F(X;\bZ) - \bbE_{F\sim Q_{\bq}}\mu_F(X;\bZ)\right|\left|\sum_{F}\mu_F(X;\bZ)\delta_j(F;Q_{\bq})\right|\\
            \leq &C\max_F\mu_{F}^2(X;\bZ)q_j\|\bq-\tilde{\bq}\|_1.
        \end{split}
    \end{equation*}

\paragraph{Bounding $\eta_2$.} By the definition $\delta_j(F;Q) = Q(F)\big(\ind(j\in F) - \omega_j(Q)\ind(j\notin F)\big)$, we have 
\begin{equation}\label{eq:eta2_bnd1}
    \begin{split}
        &\sum_F\mu_F(X;\bZ)\left(\delta_j(F;Q_{\bq}) - \delta_j(F;Q_{\tilde{\bq}})\right)\\
        =&\sum_F\mu_F(X;\bZ)\ind(j\in F)\left(Q_{\bq}(F) - Q_{\tilde{\bq}}(F))\right) \\&+\sum_F\mu_F(X;\bZ)\ind(j\notin F)\left(\omega_j(Q_{\bq})Q_{\bq}(F) - \omega_j(Q_{\tilde{\bq}})Q_{\tilde{\bq}}(F))\right)\\
        =&\bbE_{F\sim Q_{\bq}}\left(\mu_F(X;\bZ)\ind(j\in F)\right)-\bbE_{F\sim Q_{\tilde{\bq}}}\left(\mu_F(X;\bZ)\ind(j\in F)\right)\\
        &+\omega_j(Q_{\bq})\left[\bbE_{F\sim Q_{\bq}}\left(\mu_F(X;\bZ)\ind(j\notin F)\right) - \bbE_{F\sim Q_{\tilde{\bq}}}\left(\mu_F(X;\bZ)\ind(j\notin F)\right)\right]\\
        &+\left(\omega_j(Q_{\bq})-\omega_j(Q_{\tilde{\bq}})\right)\bbE_{F\sim Q_{\tilde{\bq}}}\left(\mu_F(X;\bZ)\ind(j\notin F)\right).
    \end{split}
\end{equation}
Let $(F^{(1)},\,F^{(2)})$ be a pair of random feature subsets following the joint distribution $Q^{(1,2)}$ in Lemma \ref{lem:F1_F2_diff}. We then have
\begin{equation}\label{eq:eta2_bnd2}
    \begin{split}
        &\left|\bbE_{F\sim Q_{\bq}}\left(\mu_F(X;\bZ)\ind(j\in F)\right)-\bbE_{F\sim Q_{\tilde{\bq}}}\left(\mu_F(X;\bZ)\ind(j\in F)\right)\right|\\
        =&\bbE\left|\mu_{F^{(1)}}(X;\bZ)\ind(j\in F^{(1)})-\mu_{F^{(2)}}(X;\bZ)\ind(j\in F^{(2)})\right|\\
        \leq &\max_F|\mu_F(X;\bZ)|\bbE\left(\big(\ind(j\in F^{(1)})+\ind(j\in F^{(2)})\big)\ind(F^{(1)}\neq F^{(2)})\right)\\
        =&\max_F|\mu_F(X;\bZ)|\left(\bbP(j\in F^{(1)}, F^{(1)}\neq F^{(2)}) + \bbP(j\in F^{(2)}, F^{(1)}\neq F^{(2)})\right)\\
        \leq &C\max_F|\mu_F(X;\bZ)|\left((q_j+\tilde{q}_j)\|\bq-\tilde{\bq}\|_1 + |q_j-\tilde{q}_j|\right),
    \end{split}
\end{equation}
where the last line is due to \eqref{eq:F1F2_diff3} and \eqref{eq:F1F2_diff4}. Similarly, one can show that
\begin{equation}\label{eq:eta2_bnd3}
    \begin{split}
        &\left|\bbE_{F\sim Q_{\bq}}\left(\mu_F(X;\bZ)\ind(j\notin F)\right)-\bbE_{F\sim Q_{\tilde{\bq}}}\left(\mu_F(X;\bZ)\ind(j\notin F)\right)\right|\\
        =&\bbE\left|\mu_{F^{(1)}}(X;\bZ)\ind(j\notin F^{(1)})-\mu_{F^{(2)}}(X;\bZ)\ind(j\notin F^{(2)})\right|\\
        \leq &2\max_F|\mu_F(X;\bZ)|\bbP(F^{(1)}\neq F^{(2)})\\
        \leq &C\max_F|\mu_F(X;\bZ)|\|\bq-\tilde{\bq}\|_1,,
    \end{split}
\end{equation}
and
\begin{equation}\label{eq:eta2_bnd4}
    \left|\bbE_{F\sim Q_{\tilde{\bq}}}\left(\mu_F(X;\bZ)\ind(j\notin F)\right)\right|\leq \max_F|\mu_F(X;\bZ)|.
\end{equation}
The following Lemma shows an upper bound for $|\omega_j(Q_{\bq})-\omega_j(Q_{\tilde{\bq}})|$.
\begin{lem}\label{lem:omega_j_diff}
    $$
    |\omega_j(Q_{\bq}) - \omega_j(Q_{\tilde{\bq}})|\leq C\big|q_j - \tilde{q}_j\big| + C(q_j+ \tilde{q}_j)\|\bq-\tilde{\bq}\|_1,.
    $$
\end{lem}
Therefore, \eqref{eq:eta2_bnd1}-\eqref{eq:eta2_bnd4} and Lemma \ref{lem:omega_j_diff} together imply the following:
\begin{equation*}
    \begin{split}
        |\eta_2|=&2\left|\varepsilon(Z,\bZ,Q_{\tilde{\bq}})\right|\left|\sum_F\mu_F(X;\bZ)\left(\delta_j(F;Q_{\bq}) - \delta_j(F;Q_{\tilde{\bq}})\right)\right|\\
        \leq&C\left|\varepsilon(Z,\bZ,Q_{\tilde{\bq}})\right|\max_F|\mu_F(X;\bZ)|\left(\big(q_j+\tilde{q}_j\big)\|\bq-\tilde{\bq}\|_1, + |q_j-\tilde{q}_j|\right).
    \end{split}
\end{equation*}

\paragraph{Bounding $\eta_3$.} The main idea for bounding $\eta_3$ is similar to the previous arguments, except that we need to deal with two feature subsets $F,\,F'$. First, we take a closer look at $\delta_j(F;Q_{\bq})\delta_j(F';Q_{\bq})-\delta_j(F;Q_{\tilde{\bq}})\delta_j(F';Q_{\tilde{\bq}})$. By the definition of $\delta_j(F;Q)$, we can write

\begin{equation*}
    \begin{split}
        &\delta_j(F;Q_{\bq})\delta_j(F';Q_{\bq})-\delta_j(F;Q_{\tilde{\bq}})\delta_j(F';Q_{\tilde{\bq}})\\
        =&\left(\ind(j\in F\cap F')+\omega_j^2(Q_{\bq})\ind(j\in F^c\cap F'^c)-\omega_j(Q_{\bq})\ind(j\in F\bigtriangleup F')\right)\\
        &\cdot \left(Q_{\bq}(F)Q_{\bq}(F') - Q_{\tilde{\bq}}(F)Q_{\tilde{\bq}}(F')\right)\\
        &+\left((\omega_j^2(Q_{\tilde{\bq}})-\omega_j^2(Q_{\bq}))\ind(j\in F^c\cap F'^c)+(\omega_j(Q_{\bq})-\omega_j(Q_{\tilde{\bq}}))\ind(j\in F\bigtriangleup F')\right)\\
        &\cdot Q_{\tilde{\bq}}(F)Q_{\tilde{\bq}}(F').
    \end{split}
\end{equation*}
Plugging in the identity above into the definition of $\eta_3$, we have
\begin{equation*}
    \begin{split}
        \eta_3 = &\bbE_{F,F'\overset{\rm ind.}{\sim}Q_{\bq}}\Big(\mu_F(X;\bZ)\mu_{F'}(X;\bZ)\\
        &\left(\ind(j\in F\cap F')+\omega_j^2(Q_{\bq})\ind(j\in F^c\cap F'^c)-\omega_j(Q_{\bq})\ind(j\in F\bigtriangleup F')\right)\Big)\\
        &-\bbE_{F,F'\overset{\rm ind.}{\sim}Q_{\tilde{\bq}}}\Big(\mu_F(X;\bZ)\mu_{F'}(X;\bZ)\\
        &\left(\ind(j\in F\cap F')+\omega_j^2(Q_{\bq})\ind(j\in F^c\cap F'^c)-\omega_j(Q_{\bq})\ind(j\in F\bigtriangleup F')\right)\Big)\\
        &+(\omega_j^2(Q_{\tilde{\bq}})-\omega_j^2(Q_{\bq}))\bbE_{F,F'\overset{\rm ind.}{\sim}Q_{\tilde{\bq}}}\Big(\mu_F(X;\bZ)\mu_{F'}(X;\bZ)\ind(j\in F^c\cap F'^c)\Big) \\
        &+(\omega_j(Q_{\bq})-\omega_j(Q_{\tilde{\bq}}))\bbE_{F,F'\overset{\rm ind.}{\sim}Q_{\tilde{\bq}}}\Big(\mu_F(X;\bZ)\mu_{F'}(X;\bZ)\ind(j\in F\bigtriangleup F')\Big).
    \end{split}
\end{equation*}
Here, 
\begin{equation}\label{eq:omega_sq_diff}
    \begin{split}
        |\omega_j^2(Q_{\tilde{\bq}}) - \omega_j^2(Q_{\bq})|=&|\omega_j(Q_{\tilde{\bq}}) - \omega_j(Q_{\bq})|(\omega_j(Q_{\tilde{\bq}}) + \omega_j(Q_{\bq}))\\
        \leq &\frac{q_j + \tilde{q}_j}{1-\delta-e^{-m}}|\omega_j(Q_{\tilde{\bq}}) - \omega_j(Q_{\bq})|.
    \end{split}
\end{equation}
In addition, let $(F_1^{(1)},\,F_1^{(2)})$ and $(F_2^{(1)},\,F_2^{(2)})$ be two independent, identically distributed pairs of random feature sets, both having joint distribution $Q^{(1,2)}$. Invoking Lemma \ref{lem:F1_F2_diff} on both $(F_1^{(1)},\,F_1^{(2)})$ and $(F_2^{(1)},\,F_2^{(2)})$, we have
\begin{equation*}
    \begin{split}
        &\bbP\big(F^{(1)}_1\neq F^{(2)}_1\text{ or }F_2^{(1)}\neq F_2^{(2)}\big)\\
        \leq &\bbP\big(F^{(1)}_1\neq F^{(2)}_1\big) + \bbP\big(F^{(1)}_2\neq F^{(2)}_2\big)\\
        \leq &\frac{4\|\bq - \tilde{\bq}\|_1}{1-e^{-m}};
    \end{split}
\end{equation*}
and
\begin{equation*}
    \begin{split}
        &\bbP\big(j\in F^{(1)}_1\cap F^{(1)}_2, F^{(1)}_1\neq F^{(2)}_1\text{ or }F_2^{(1)}\neq F_2^{(2)}\big)\\
        \leq &\bbP\big(j\in F^{(1)}_1, F^{(1)}_1\neq F^{(2)}_1\big) + \bbP\big(j\in F^{(1)}_2, F^{(1)}_2\neq F^{(2)}_2\big)\\
        \leq &\frac{6q_j\|\bq-\tilde{\bq}\|_1, + 2|\tilde{q}_j - q_j|}{1-e^{-m}}.
    \end{split}
\end{equation*}
Similarly, 
\begin{equation*}
    \begin{split}
        &\bbP\big(j\in F^{(2)}_1\cap F^{(2)}_2, F^{(1)}_1\neq F^{(2)}_1\text{ or }F_2^{(1)}\neq F_2^{(2)}\big)\\
      \leq &\frac{6\tilde{q}_j\|\bq-\tilde{\bq}\|_1, + 2|\tilde{q}_j - q_j|}{1-e^{-m}}.
    \end{split}
\end{equation*}
With these properties of $(F_1^{(1)},\,F_1^{(2)})$ and $(F_2^{(1)},\,F_2^{(2)})$, as well as \eqref{eq:omega_sq_diff}, we can then write
\begin{equation*}
    \begin{split}
        |\eta_3|\leq &\bbE\Big|\mu_{F_1^{(1)}}(X;\bZ)\mu_{F_2^{(1)}}(X;\bZ)\ind(j\in F_1^{(1)}\cap F_2^{(1)}) \\
        &- \mu_{F_1^{(2)}}(X;\bZ)\mu_{F_2^{(2)}}(X;\bZ)\ind(j\in F_1^{(2)}\cap F_2^{(2)})\Big|\\
        &+\omega_j^2(Q_{\bq})\bbE\Big|\mu_{F_1^{(1)}}(X;\bZ)\mu_{F_2^{(1)}}(X;\bZ)\ind(j\in F_1^{(1)c}\cap F_2^{(1)c}) \\
        &- \mu_{F_1^{(2)}}(X;\bZ)\mu_{F_2^{(2)}}(X;\bZ)\ind(j\in F_1^{(2)c}\cap F_2^{(2)c})\Big|\\
        &+\omega_j(Q_{\bq})\bbE\Big|\mu_{F_1^{(1)}}(X;\bZ)\mu_{F_2^{(1)}}(X;\bZ)\ind(j\in F_1^{(1)}\bigtriangleup F_2^{(1)}) \\
        &- \mu_{F_1^{(2)}}(X;\bZ)\mu_{F_2^{(2)}}(X;\bZ)\ind(j\in F_1^{(2)}\bigtriangleup F_2^{(2)})\Big|\\
        &+\Bigg(1+\frac{q_j + \tilde{q}_j}{1-\delta-e^{-m}}\Bigg)\big|\omega_j(Q_{\tilde{\bq}}) - \omega_j(Q_{\bq})\big|\max_F\mu_F^2(X;\bZ)\\
        \leq &\max_F\mu_F^2(X;\bZ)\Big(\bbP(j\in F_1^{(1)}\cap F_2^{(1)},\,F_1^{(1)}\neq F_1^{(2)}\text{ or }F_2^{(1)}\neq F_2^{(2)})\\ 
        &\quad +\bbP(j\in F_1^{(2)}\cap F_2^{(2)},\,F_1^{(1)}\neq F_1^{(2)}\text{ or }F_2^{(1)}\neq F_2^{(2)})\Big)\\
        &+2\omega_j^2(Q_{\bq})\max_F\mu_F^2(X;\bZ)\bbP(F_1^{(1)}\neq F_1^{(2)}\text{ or }F_2^{(1)}\neq F_2^{(2)})\\
        &+2\omega_j(Q_{\bq})\max_F\mu_F^2(X;\bZ)\bbP(F_1^{(1)}\neq F_1^{(2)}\text{ or }F_2^{(1)}\neq F_2^{(2)})\\
        &+C|\omega_j(Q_{\tilde{\bq}}) - \omega_j(Q_{\bq})|\max_F\mu_F^2(X;\bZ)\\
        \leq &C\max_F\mu_F^2(X;\bZ)\left(\big(q_j+\tilde{q}_j\big)\|\bq-\tilde{\bq}\|_1, + |q_j - \tilde{q}_j|\right),
    \end{split}
\end{equation*}
where we utilized Lemma \ref{lem:omega_j_diff} and the fact that $1-\delta-e^{-m}\leq c<1$ for some constant $c>0$, when $m$ is sufficiently large, in the last line. Therefore, putting bounds for $\eta_1,\,\eta_2,\,\eta_3$ together, we have
\begin{equation*}
\begin{split}
    &\left|h_j(Z,\bZ,Q_{\bq}) -h_j(Z,\bZ,Q_{\tilde{\bq}})\right|\\
    \leq &C\Big(\max_F\mu_F^2(X;\bZ)+\left|\varepsilon(Z,\bZ,Q_{\tilde{\bq}})\right|\max_F|\mu_F(X;\bZ)|\Big)\\
    &\cdot\Big(\big(q_j+\tilde{q}_j\big)\big\|\bq-\tilde{\bq}\big\|_1 + \big|q_j - \tilde{q}_j\big|\Big)\\
    \leq &C\Big(\big(q_j+\tilde{q}_j\big)\big\|\bq-\tilde{\bq}\big\|_1 + \big|q_j - \tilde{q}_j\big|\Big),
\end{split}
\end{equation*}
where the last line is due to Assumption~\ref{ass:adp_bounded}.

\end{proof}
\begin{proof}[Proof of Lemma~\ref{lem:LOCO_LOO_concentration}]
For notational simplicity, we omit the superscript of $Q^{*(b)}$ and use $Q$ instead throughout this proof. Recall the definition of $h_j(Z_i,\bZ_{\no i}, Q)$ in Table~\ref{tab:notations}. We can write
\begin{equation*}
    \begin{split}
        h_j(Z_i,\bZ_{\no i},Q) =& \Big(Y_i-\mu^*(X_i;Q_{\no j},\bZ_{\no i})\Big)^2 - \Big(Y_i-\mu^*(X_i;Q,\bZ_{\no i})\Big)^2\\
        =& \Big(\mu^*(X_i;Q,\bZ_{\no i}) - \mu^*(X_i;Q_{\no j}, \bZ_{\no i})\Big)\Big(2Y_i-\mu^*(X_i;Q,\bZ_{\no i}) - \mu^*(X_i;Q_{\no j}, \bZ_{\no i})\Big),
    \end{split}
\end{equation*}
where $\mu^*(X_i;Q,\bZ_{\no i})=\frac{1}{\binom{N-1}{n}}\sum_{I\subset[N]\no i,|I|=n}\sum_{F\subset [M]}Q(F)\mu_{I,F}(X_i)$. We can then write
\begin{equation*}
    \begin{split}
        h_j(Z_i,\bZ_{\no i},Q) =&\frac{1}{\binom{N-1}{n}^2}\sum_{\substack{I_1,I_2\subset[N]\no i\\|I_1|=|I_2|=n}}\Big(\sum_{F\subset [M]}\mu_{I_1,F}(X_i)\delta_j(F,Q)\Big)\Big\{2Y_i - \sum_{F\subset [M]}\mu_{I_2,F}(X_i)(Q(F)+Q_{\no j}(F))\Big\}.
    \end{split}
\end{equation*}
Denote $\Big(\sum_{F\subset [M]}\mu_{I_1,F}(X_i)\delta_j(F,Q)\Big)\Big\{2Y_i - \sum_{F\subset [M]}\mu_{I_2,F}(X_i)(Q(F)+Q_{\no j}(F))\Big\}$ by $\alpha_{i,I_1,I_2}$, a function of $Z_i,\bZ_{I_1},\bZ_{I_2}$. We then have
$$
h_j(Z_i,\bZ_{\no i},Q) = \frac{1}{\binom{N-1}{n}^2}\sum_{\substack{I_1,I_2\subset[N]\no i\\|I_1|=|I_2|=n}}\alpha_{i,I_1,I_2}.
$$
Hence,
\begin{equation*}
    \begin{split}
        &\frac{1}{N}\sum_{i=1}^N\left\{h_j(Z_i,\bZ_{\no i},Q) - \bbE\Big[h_j(Z_i,\bZ_{\no i}, Q)|\bZ_{\no i}\Big]\right\}\\
        =&\frac{1}{N\binom{N-1}{n}^2}\sum_{i=1}^N\sum_{\substack{I_1,I_2\subset[N]\no i\\|I_1|=|I_2|=n}} \Big(\alpha_{i,I_1,I_2}-\bbE\big(\alpha_{i,I_1,I_2}|\bZ_{I_1\cup I_2}\big)\Big),
    \end{split}
\end{equation*}
where the R.H.S. is an average over $N\binom{N-1}{n}^2$ dependent random variables. Although $\alpha_{i,I_1,I_2}$ are dependent from each other through the overlaps across the subsampled observations, each sample is only involved in a small number of $\alpha_{i,I_1,I_2}$'s. The following lemma provides a useful concentration bound for this situation. 

In particular, Lemma~\ref{lem:read_k_concentration} gives the Hoeffding bounds for sum of dependent random variables $\{Y_i\}_{i=1}^n$ when they are functions of a set of independent variables $\notationchange{\{X^{(j)}\}_{j=1}^m}$, where each $\notationchange{X^{(j)}}$ only influences a small number of $\notationchange{Y_i}$'s.
\begin{lem}\label{lem:read_k_concentration}
    Consider centered random variables $Y_1,\dots,Y_n$ satisfying $|Y_i|\leq C_i$. Suppose that there exist independent random variables $\notationchange{X^{(1)}},\dots,\notationchange{X^{(m)}}$ such that $Y_i=f_i(\notationchange{X^{(S_i)}})$ for some set $S_i\subset[m]$. If $\forall 1\leq j\leq m$, $|\{i:S_i\owns j\}|\leq k$, then
    $$
    \bbP\Bigg(\Bigg|\sum_{i=1}^n Y_i\Bigg|>t\Bigg)\leq 2\exp\Bigg\{-\frac{t^2}{2k\sum_{i=1}^n C_i^2}\Bigg\}.
    $$
\end{lem}
Lemma~\ref{lem:read_k_concentration} is an extension of the read-$k$ Chernoff bound previously shown for indicator variables \citep{gavinsky2015tail}. Now we apply Lemma~\ref{lem:read_k_concentration} to $\Big\{\alpha_{i,I_1,I_2}-\bbE\big(\alpha_{i,I_1,I_2}|\bZ_{I_1\cup I_2}\big): 1\leq i\leq N, I_1\subset[N]\no i, I_2\subset[N]\no i, |I_1|=|I_2|=n\Big\}$. We note the following facts:
\begin{itemize}
    \item Assumption~\ref{ass:adp_bounded} implies that 
    \begin{align}
        |\alpha_{i,I_1,I_2}|\leq &C\sum_{F\subset[M]}\big|\delta_j(F,Q)\big|\\
        =&C\sum_{F\subset[M]}\Big|Q(F)\big(\ind(j\in F)-\omega_j(Q)\ind(j\notin F)\big)\Big|\\
        \leq&CP_{F\sim Q}(j\in F) + C\omega_j(Q)\\
        =&\frac{Cq_j}{1-\prod_{l=1}^M(1-q_l)} + \frac{Cq_j}{1-q_j-\prod_{l=1}^M(1-q_l)}\\
        \leq& Cq_j,\label{eq:delta_j_sum}
\end{align}
where $\omega_j(Q) = \frac{\bbP_{F\sim Q}(F\owns j)}{\bbP_{F\sim Q}(F\not\owns j)} = \frac{q_j}{1-q_j-\prod_{k=1}^M(1-q_k)}$. In the last line above, we utilized the fact that $1-q_j-\prod_{l=1}^M(1-q_l)\geq 1-\delta-e^{-m}\geq c>0$ (Assumption~\ref{ass:small_n}). The constants $C$ may vary from line to line. 
    \item $\alpha_{i,I_1,I_2}$'s are functions of the independent random variables $Z_1,\dots,Z_N$.
    \item Each $Z_l$ only influences $\alpha_{i,I_1,I_2}$ if $I_1\cup I_2\cup\{i\}\owns l$, and the total number of influenced terms is
    \begin{equation*}
    \begin{split}
        &\binom{N-1}{n}^2 + (N-1)\Big(\binom{N-1}{n}^2-\binom{N-2}{n}^2\Big)\\
        =&\binom{N-1}{n}^2\Big(2n+1-\frac{n^2}{N-1}\Big)\\
        \leq&\binom{N-1}{n}^2\big(2n+1\big).
    \end{split}
    \end{equation*} 
\end{itemize}
Now we invoke Lemma~\ref{lem:read_k_concentration} to obtain the following:
\begin{equation*}
    \begin{split}
        &\bbP\left(\Big|\frac{1}{N}\sum_{i=1}^N\left\{h_j(Z_i,\bZ_{\no i},Q) - \bbE\Big[h_j(Z_i,\bZ_{\no i}, Q)|\bZ_{\no i}\Big]\right\}\Big|>t\right)\\
        \leq&\bbP\left(\left|\sum_{i=1}^N\sum_{\substack{I_1,I_2\subset[N]\no i\\|I_1|=|I_2|=n}} \Big(\alpha_{i,I_1,I_2}-\bbE\big(\alpha_{i,I_1,I_2}|\bZ_{I_1\cup I_2}\big)\Big)\right|>tN\binom{N-1}{n}^2\right)\\
        \leq&2\exp\left\{-c\frac{Nt^2}{q_j^2n}\right\}.
    \end{split}
\end{equation*}
Plug in $t=Cq_j\sqrt{\frac{n}{N}\log M}$ into the inequality above and take a union bound over $1\leq j\leq M$, we have that with probability at least $1-M^{-c}$, 
$$
\left|\frac{1}{N}\sum_{i=1}^N\left\{h_j(Z_i,\bZ_{\no i},Q) - \bbE\Big[h_j(Z_i,\bZ_{\no i}, Q)|\bZ_{\no i}\Big]\right\}\right|\leq Cq_j\sqrt{\frac{n}{N}\log M}.
$$
Apply the same argument to each iteration with $Q^{*(b)}$ and take a union bound over $1\leq b\leq B\leq CM$, the proof is then complete.

\end{proof}

\begin{proof}[Proof of Lemma~\ref{lem:LOO_err}]
Similar to the proof of Lemma~\ref{lem:LOCO_LOO_concentration}, we can write
\begin{equation*}
    \begin{split}
        h_j(Z^*,\bZ_{\no i},Q) =&\frac{1}{\binom{N-1}{n}^2}\sum_{\substack{I_1,I_2\subset[N]\no i\\|I_1|=|I_2|=n}}\Big(\sum_{F\subset [M]}\mu_{I_1,F}(X^*)\delta_j(F,Q)\Big)\Big[2Y^* - \sum_{F\subset [M]}\mu_{I_2,F}(X^*)(Q(F)+Q_{\no j}(F))\Big],\\
        h_j(Z^*,\bZ,Q) =&\frac{1}{\binom{N}{n}^2}\sum_{\substack{I_1,I_2\subset[N]\\|I_1|=|I_2|=n}}\Big(\sum_{F\subset [M]}\mu_{I_1,F}(X^*)\delta_j(F,Q)\Big)\Big[2Y^* - \sum_{F\subset [M]}\mu_{I_2,F}(X^*)(Q(F)+Q_{\no j}(F))\Big].
    \end{split}
\end{equation*}
Denote $\bbE_{Z^*}\left\{\Big(\sum_{F\subset [M]}\mu_{I_1,F}(X^*)\delta_j(F,Q)\Big)\Big[2Y^* - \sum_{F\subset [M]}\mu_{I_2,F}(X^*)(Q(F)+Q_{\no j}(F))\Big]|\bZ\right\}$ by $\alpha_{I_1,I_2}$. We then have
\begin{equation*}
    \begin{split}
        \bbE_{Z^*}\Big[h_j(Z^*,\bZ_{\no i},Q)|\bZ\Big] =&\frac{1}{\binom{N-1}{n}^2}\sum_{\substack{I_1,I_2\subset[N]\no i\\|I_1|=|I_2|=n}}\alpha_{I_1,I_2} = \frac{1}{\binom{N-1}{n}^2}\sum_{\substack{I_1,I_2\subset[N]\\|I_1|=|I_2|=n}}\ind(I_1\cup I_2\not\owns i)\alpha_{I_1,I_2},\\
        \bbE_{Z^*}\Big[h_j(Z^*,\bZ,Q)|\bZ\Big] =&\frac{1}{\binom{N}{n}^2}\sum_{\substack{I_1,I_2\subset[N]\\|I_1|=|I_2|=n}}\alpha_{I_1,I_2}.
    \end{split}
\end{equation*}
For simplicity, denote $\frac{n}{N}$ by $\gamma$ throughout this proof. We can then write
\begin{equation*}
    \begin{split}
        &\left|\frac{1}{N}\sum_{i=1}^N\bbE_{Z^*}\Big[h_j(Z^*,\bZ_{\no i},Q)-h_j(Z^*,\bZ,Q)|\bZ\Big]\right|\\
    = & \left|\sum_{\substack{I_1,I_2\subset[N]\\|I_1|=|I_2|=n}}\alpha_{I_1,I_2}\left(\frac{1}{N\binom{N-1}{n}^2}\sum_{i=1}^N\ind(I_1\cup I_2\not\owns i) - \frac{1}{\binom{N}{n}^2}\right)\right|\\
    \leq&\frac{1}{\binom{N}{n}^2}\sum_{\substack{I_1,I_2\subset[N]\\|I_1|=|I_2|=n}}
    |\alpha_{I_1,I_2}|\left|\frac{(2\gamma -\gamma^2)-|I_1\cup I_2|/N}{(1-\gamma)^2\binom{N}{n}^2}\right|\\
    \leq &C\frac{q_jn}{N}.
    \end{split} 
\end{equation*}
Here, the last line arises from two facts: (i) by Assumption~\ref{ass:adp_bounded}, $|\alpha_{I_1,I_2}|\leq C\sum_{F\subset[M]}\delta_j(F,Q)\leq Cq_j$; and (ii) $(2\gamma -\gamma^2)-|I_1\cup I_2|/N\in [-\gamma^2,\gamma-\gamma^2]$ for any $I_1,I_2$ with size $n$.

\end{proof}

 \begin{proof}[Proof of Lemma~\ref{lem:F1_F2_diff}]
    Here, we construct two random feature subsets $F^{(1)}\sim Q_{\bq}$ and $F^{(2)}\sim Q_{\tilde{\bq}}$ such that they are similar to each other, using Algorithm \ref{algo:sample_F1F2}. 
    
    \begin{algorithm}[H]
\caption{Coupled Sampling of Two Random Feature Subsets}
\label{algo:sample_F1F2}
\noindent{\textbf{Input}}: Number of features $M$, sampling probability vectors $\bq,\,\tilde{\bq}\in [0,1]^M$.
\begin{enumerate}
    \item Initialize $t=0$, $F^{(1)} = \emptyset$, $F^{(2)} = \emptyset$.
    \item While $F^{(1)}=\emptyset$ or $F^{(2)} = \emptyset$:
    \begin{enumerate}
        \item Set $t = t + 1$.
        \item Randomly sample $U_{t,1},\dots, U_{t,M}\overset{\rm i.i.d.}{\sim}\cU[0,1]$.
        \item Let $F^{(t,1)} = \{j: U_{t,j}\leq q_j\}$, $F^{(t,2)} = \{j: U_{t,j}\leq \tilde{q}_j\}$. 
        \item If $F^{(1)} = \emptyset$, set $F^{(1)} = F^{(t,1)}$; if $F^{(2)} = \emptyset$, set $F^{(2)} = F^{(t,2)}$.
    \end{enumerate} 
\end{enumerate}
\textbf{Output}: $(F^{(1)},\,F^{(2)})$.
\end{algorithm}

Note that $F^{(1)}\sim Q_{\bq}$ and $F^{(2)}\sim Q_{\tilde{\bq}}$. Let $T_1= \min\{t: F^{(t,1)}\neq \emptyset\}$, $T_2 = \min\{t: F^{(t,2)}\neq \emptyset\}$. The updating rule in Algorithm \ref{algo:sample_F1F2} implies that $F^{(1)} = F^{(T_1,1)}$ and $F^{(2)} = F^{(T_2,2)}$. Also let $T_0 = T_1\wedge T_2$. Note that $F^{(1)}\neq F^{(2)}$ implies that $F^{(T_0,1)}\neq F^{(T_0,2)}$, and $F^{(t,1)} = F^{(t,2)} = \emptyset$ for $1\leq t\leq T_0-1$. Hence, we can bound the probability of $F^{(1)}\neq F^{(2)}$ as follows:
\begin{equation*}
    \begin{split}
        \bbP(F^{(1)}\neq F^{(2)})\leq& \bbP(F^{(T_0,1)}\neq F^{(T_0,2)}, F^{(t,1)} = F^{(t,2)} = \emptyset \text{ for }t< T_0)\\
        \leq &\bbP(\exists s\geq 1, \text{ s.t. } F^{(s,1)}\neq F^{(s,2)}, F^{(t,1)} = F^{(t,2)} = \emptyset \text{ for }t< s)\\
        \leq& \sum_{s=1}^\infty \bbP(F^{(s,1)}\neq F^{(s,2)}, F^{(t,1)} = F^{(t,2)} = \emptyset \text{ for }t< s)\\
        = &\sum_{s=1}^\infty \bbP(F^{(s,1)}\neq F^{(s,2)})\prod_{t<s}\bbP(F^{(t,1)} = F^{(t,2)} = \emptyset).
    \end{split}
\end{equation*}
where the fourth line is due to that $\{(F^{(t,1)},\,F^{(t,2)}\}_{t=1}^\infty$ is an independent series of subset pairs. Since $\bbP\left(\ind(j\in F^{(t,1)})\ind(j\notin F^{(t,2)})=0\right) = \bbP\big(U_{t,j}\in (q_j\wedge \tilde{q}_j, q_j\vee \tilde{q}_j]\big) = \big|q_j-\tilde{q}_j\big|$, we have
\begin{equation*}
    \begin{split}
        \bbP(F^{(s,1)}\neq F^{(s,2)}) =& 1- \bbP(F^{(s,1)}= F^{(s,2)})\\
        =&1-\prod_{j=1}^M(1-|q_j - \tilde{q}_j|).
    \end{split}
\end{equation*}
While for $\bbP(F^{(t,1)} = F^{(t,2)} = \emptyset)$, we have
\begin{equation*}
    \begin{split}
        \bbP(F^{(t,1)} = F^{(t,2)} = \emptyset) = \bbP(\forall j, U_{t,j}>q_j\vee \tilde{q}_j) = \prod_{j=1}^M(1-q_j\vee \tilde{q}_j).
    \end{split}
\end{equation*}
Therefore, we can write
\begin{equation}\label{eq:F1F2_unequalProb}
    \begin{split}
         \bbP(F^{(1)}\neq F^{(2)})\leq&\sum_{s=1}^\infty \bbP(F^{(s,1)}\neq F^{(s,2)})\prod_{t<s}\bbP(F^{(t,1)} = F^{(t,2)} = \emptyset)\\
         \leq &\Big(1-\prod_{j=1}^M\big(1-\big|q_j - \tilde{q}_j\big|\big)\Big)\sum_{s=0}^{\infty}\left(\prod_{j=1}^M\Big(1-q_j\vee \tilde{q}_j\Big)\right)^s\\
         =&\Big(1-\prod_{j=1}^M\big(1-\big|q_j - \tilde{q}_j\big|\big)\Big)\left(1-\prod_{j=1}^M\Big(1-q_j\vee \tilde{q}_j\Big)\right)^{-1}\\
         \leq&\frac{1-\exp\{\sum_{j=1}^M\log(1-|q_j - \tilde{q}_j|)\}}{1-e^{-m}},
    \end{split}
\end{equation}
where the fourth line is due to the fact that $$\prod_{j=1}^M\Big(1-q_j\vee \tilde{q}_j)=\exp\{\sum_{j=1}^M\log(1-q_j\vee \tilde{q}_j)\}\leq \exp\{-\sum_{j=1}^M(q_j\vee \tilde{q}_j)\}\leq e^{-m}.$$ Furthermore, noting that the function $\log(1-x)$ satisfies $\log(1-x)\geq -\frac{x}{1-x}$, we have $\log(1-|q_j - \tilde{q}_j|)\geq -2|q_j - \tilde{q}_j|$ given that $\|\bq-\tilde{\bq}\|_{\infty}\leq \frac{1}{2}$. Therefore, we can write
\begin{equation}\label{eq:F1F2_unequalProb2}
    \begin{split}
        \bbP(F^{(1)}\neq F^{(2)})\leq&\frac{1-\exp\{-2\|\bq-\tilde{\bq}\|_1,\}}{1-e^{-m}}\\
        \leq&\frac{2\|\bq-\tilde{\bq}\|_1,}{1-e^{-m}},
    \end{split}
\end{equation}
where we have utilized the fact that $1-e^{-x}\leq x$ for $x\geq 0$.

Now we prove an upper bound for $\bbE\big||F^{(1)}|-|F^{(2)}|\big|$. We first note that
\begin{equation}\label{eq:F1F2_size_diff}
    \begin{split}
        &\bbE\big||F^{(1)}|-|F^{(2)}|\big| \\
        =&\bbE\left|\sum_{j=1}^M(\ind(j\in F^{(1)}) - \ind(j\in F^{(2)})\right|\\
        \leq &\bbE\left(\sum_{j=1}^M\ind(j\in F^{(1)}, j\notin F^{(2)}, \text{ or }j\in F^{(2)}, j\notin F^{(1)})\right)\\
        = &\sum_{j=1}^M \bbP(j\in F^{(1)},j\notin F^{(2)}\text{ or }j\in F^{(2)}, j\notin F^{(1)}).
    \end{split}
\end{equation}
For any $1\leq j\leq M$, the event that $j\in F^{(1)},j\notin F^{(2)}\text{ or }j\in F^{(2)}, j\notin F^{(1)}$ implies that at least one of the following three events are true: (i)$\ind(j\in F^{(T_0,1)})+\ind(j\in F^{(T_0,2)})=1$; (ii)$|F^{(T_0,1)}|=0$, $j\notin F^{(T_0,2)}$, $|F^{(T_0,2)}|>0$, $j\in F^{(T_2,2)}$; (iii) $|F^{(T_0,2)}|=0$, $j\notin F^{(T_0,1)}$, $|F^{(T_0,1)}|>0$, $j\in F^{(T_1,1)}$. In the following, we bound the probability of each event. First of all,

\begin{equation*}
    \begin{split}
        &\bbP\left(\ind(j\in F^{(T_0,1)})+\ind(j\in F^{(T_0,2)})=1\right)\\
        = &\bbP\left(U_{T_0,j}\in (q_j\wedge \tilde{q}_j, q_j\vee \tilde{q}_j]\right)\\
        \leq &\bbP\left(\exists s\geq 1, U_{s,j}\in (q_j\wedge \tilde{q}_j, q_j\vee \tilde{q}_j], F^{(t,1)} = F^{(t,2)} = \emptyset \text{ for all }t<s\right)\\
        \leq &|q_j - \tilde{q}_j|\sum_{s=0}^{\infty}\left(\prod_{l=1}^M(1-q_l\vee \tilde{q}_l)\right)^s\\
        \leq &\frac{|q_j-\tilde{q}_j|}{1-e^{-m}},
    \end{split}
\end{equation*}
where the last two lines follow similar arguments to \eqref{eq:F1F2_unequalProb}.

In addition, 
\begin{equation}\label{eq:F1F2_size_diff_event2}
    \begin{split}
        &\bbP\left(|F^{(T_0,1)}|=0, j\notin F^{(T_0,2)}, |F^{(T_0,2)}|>0, j\in F^{(T_2,2)}\right)\\
        \leq &\bbP\Big(\exists s\geq 1\text{ s.t., }\forall t<s, F^{(t,1)} = F^{(t,2)} = \emptyset; j\notin F^{(s,1)}, j\notin F^{(s,2)}; \\
        &\forall k\neq j, k\notin F^{(s,1)}; \exists k\neq j, k\in F^{(s,2)}; j\in F^{(T_2,2)}\Big)\\
        \leq&\sum_{s=1}^\infty \left(\prod_{t=1}^{s-1}\bbP(F^{(t,1)} = F^{(t,2)} =\emptyset)\right) \bbP(j\notin F^{(s,1)}, j\notin F^{(s,2)})\\
        &\times \bbP(\forall k\neq j, k\notin F^{(s,1)};\exists k\neq j, k\in F^{(s,2)}) \bbP(j\in F^{(T_2,2)})\\
        \leq &\sum_{s=1}^\infty \left(\prod_{k=1}^M(1-q_k\vee \tilde{q}_k)^{s-1}\right)(1-q_j\vee \tilde{q}_j)\\
        &\times \left(\prod_{k\neq j}(1-q_k)-\prod_{k\neq j}(1-q_k\vee \tilde{q}_k)\right) \frac{\tilde{q}_j}{1-\prod_{k=1}^M(1-\tilde{q}_k)},
    \end{split}
\end{equation}
where the second inequality is due to the fact that the conditional distribution of $F^{(T_2,2)}$ given the event $F^{(t,2)}=\emptyset$ for $t=1,\dots, s$ is the same as the marginal distribution of $F^{(T_2,2)}$. Following the same arguments as \eqref{eq:F1F2_unequalProb}, we have
\begin{equation*}
    \begin{split}
        \sum_{s=1}^\infty \left(\prod_{k=1}^M(1-q_k\vee \tilde{q}_k)^{s-1}\right)\leq &\frac{1}{1-e^{-m}},\\
        \left(1-\prod_{k=1}^M(1-\tilde{q}_k)\right)^{-1}\leq &\frac{1}{1-e^{-m}}.
    \end{split}
\end{equation*}
In addition, we note that for any vectors $u,\,v\in [0,\infty)^p$ with any $p>1$, we have $$\Bigg|\prod_{i=1}^p u_i - \prod_{i=1}^p v_i\Bigg|\leq \|u-v\|_1\max_j\prod_{i\neq j}\max\{u_i,\,v_i\}.$$ Plugging in $u = (1-q)_{\backslash j}$ and $v = (1-q\vee \tilde{q})_{\backslash j}$ into the inequality above, we then have
\begin{equation*}
    \begin{split}
        \Bigg|\prod_{k\neq j}(1-q_k)-\prod_{k\neq j}(1-q_k\vee \tilde{q}_k)\Bigg|\leq \|\bq-\tilde{\bq}\|_1,.
    \end{split}
\end{equation*}
Therefore, returning to \eqref{eq:F1F2_size_diff_event2}, the second event's probability is controlled as follows:
\begin{equation*}
    \begin{split}
        &\bbP\left(|F^{(T_0,1)}|=0, j\notin F^{(T_0,2)}, |F^{(T_0,2)}|>0, j\in F^{(T_2,2)}\right)\\
        \leq &\left(1-e^{-m}\right)^{-2}\|\bq-\tilde{\bq}\|_1,\tilde{q}_j.
    \end{split}
\end{equation*}
Similarly, the probability of the third event that $|F^{(T_0,2)}|=0$, $j\notin F^{(T_0,1)}$, $|F^{(T_0,1)}|>0$, $j\in F^{(T_1,1)}$ can also be controlled by $\left(1-e^{-m}\right)^{-2}\|\bq-\tilde{\bq}\|_1,q_j$. Hence,
\begin{equation*}
    \begin{split}
        &\bbP(j\in F^{(1)},j\notin F^{(2)}\text{ or }j\in F^{(2)}, j\notin F^{(1)})\\
        \leq&\frac{|q_j-\tilde{q}_j|}{1-e^{-m}} + \left(1-e^{-m}\right)^{-2}\|\bq-\tilde{\bq}\|_1,(q_j+\tilde{q}_j).
    \end{split}
\end{equation*}
This together with \eqref{eq:F1F2_size_diff} yields
\begin{equation*}
    \begin{split}
        &\bbE\big||F^{(1)}|-|F^{(2)}|\big|\\
        \leq&\sum_{j=1}^M \frac{|q_j-\tilde{q}_j|}{1-e^{-m}} + \left(1-e^{-m}\right)^{-2}\|\bq-\tilde{\bq}\|_1,(q_j+\tilde{q}_j)\\
        =& \frac{\|\bq-\tilde{\bq}\|_1,}{1-e^{-m}} + 2m\left(1-e^{-m}\right)^{-2}\|\bq-\tilde{\bq}\|_1,\\
        \leq&\frac{2m+1}{(1-e^{-m})^2}\|\bq-\tilde{\bq}\|_1,.
    \end{split}
\end{equation*}

Finally, we focus on proving \eqref{eq:F1F2_diff3} and \eqref{eq:F1F2_diff4} follows similarly. We can first write
\begin{equation}\label{eq:F_size_unequal_bnd1}
    \begin{split}
        \bbP\big(j\in F^{(1)},\,F^{(1)}\neq F^{(2)}\big)
        =&\bbP(j\in F^{(1)})-\bbP(j\in F^{(1)}, F^{(1)}=F^{(2)})\\
        =&\frac{q_j}{1-\prod_{l=1}^M(1-q_l)}-\bbP(j\in F^{(1)}, F^{(1)}= F^{(2)}).
    \end{split}
\end{equation}
In addition, 
\begin{equation}\label{eq:F_size_unequal_bnd2}
    \begin{split}
        &\bbP\big(j\in F^{(1)},\, F^{(1)}= F^{(2)}\big)\\
        \geq &\bbP\Big(T_1=T_2,\, F^{(T_1,1)}= F^{(T_1,2)},\, j\in F^{(T_1,1)}\Big)\\
        =&\sum_{t=1}^{\infty}\bbP\Big(T_1=T_2=t,\, F^{(T_1,1)}= F^{(T_1,2)},\, j\in F^{(T_1,1)})\\
        =&\sum_{t=1}^{\infty}\bbP\Big(F^{(s,1)}=F^{(s,2)}=\emptyset\text{ for all }s<t,\, F^{(t,1)}= F^{(t,2)},\, j\in F^{(t,1)})\\
        =&\sum_{t=1}^{\infty}\prod_{s=1}^{t-1}\bbP\Big(F^{(s,1)}=F^{(s,2)}=\emptyset)\bbP(F^{(t,1)}= F^{(t,2)},\, j\in F^{(t,1)})\\
        =&(q_j\wedge \tilde{q}_j)\prod_{l\neq j}\left(1-\big|q_l-\tilde{q}_l\big|\right)\sum_{t=1}^{\infty}\left[\prod_{l=1}^M(1-q_l\vee \tilde{q}_l)\right]^{t-1}\\
        =&\frac{q_j\wedge \tilde{q}_j}{1-\prod_{l=1}^M(1-q_l\vee \tilde{q}_l)}\prod_{l\neq j}\left(1-\big|q_l-\tilde{q}_l\big|\right),
    \end{split}
\end{equation}
where the sixth line utilizes the fact that for any $t\geq 1$, 
$$\bbP(F^{(t,1)}= F^{(t,2)},\, j\in F^{(t,1)}) = (q_j\wedge \tilde{q}_j)\prod_{l\neq j}\left(1-\big|q_l-\tilde{q}_l\big|\right),$$
and 
$$\bbP\Big(F^{(s,1)}=F^{(s,2)}=\emptyset) = \prod_{l=1}^M(1-q_l\vee \tilde{q}_l).$$
Hence, \eqref{eq:F_size_unequal_bnd1} and \eqref{eq:F_size_unequal_bnd2} together yield
\begin{equation}\label{eq:F_size_unequal_bnd3}
    \begin{split}
        &\bbP\big(j\in F^{(1)},\,F^{(1)}\neq F^{(2)}\big)\\
        \leq&\frac{q_j}{1-\prod_{l=1}^M\big(1-q_l\big)}-\frac{q_j\wedge \tilde{q}_j}{1-\prod_{l=1}^M\big(1-q_l\vee \tilde{q}_l\big)}\\
        &+ \frac{q_j\wedge \tilde{q}_j}{1-\prod_{l=1}^M(1-q_l\vee \tilde{q}_l)}\Bigg(1-\prod_{l=1}^M(1-|q_l-\tilde{q}_l|)\Bigg)\\
        \leq &\frac{q_j}{1-\prod_{l=1}^M\big(1-q_l\big)}-\frac{q_j\wedge \tilde{q}_j}{1-\prod_{l=1}^M\big(1-q_l\vee \tilde{q}_l\big)}\Bigg) + \frac{2(q_j\wedge \tilde{q}_j)}{1-e^{-m}}\|\bq-\tilde{\bq}\|_1,,
    \end{split}
\end{equation}
where the last inequality is established by following the same arguments as \eqref{eq:F1F2_unequalProb} and \eqref{eq:F1F2_unequalProb2}. To bound the first term in the last line above, we define function $f:\bbR^{M+1}\rightarrow \bbR$ such that $f(x)=\frac{x_{M+1}}{1-\prod_{l=1}^M(1-x_l)}$. We can then write
\begin{equation}\label{eq:F_size_unequal_term1}
    \begin{split}
        &\left|\frac{q_j}{1-\prod_{l=1}^M\big(1-q_l\big)}-\frac{q_j\wedge \tilde{q}_j}{1-\prod_{l=1}^M\big(1-q_l\vee \tilde{q}_l\big)}\right|\\
        =&\left|f((\bq,\,q_j)^\top) - f((\bq\vee \tilde{\bq})^{\top},\,q_j\wedge \tilde{q}_j)^\top)\right|\\
        =&\left|(\nabla f(\xi))_{1:M}^\top \Big(\bq-\big(\bq\vee \tilde{\bq}\big)\Big) + (\nabla f(\xi))_{M+1}\Big(q_j-q_j\wedge \tilde{q}_j\Big)\right|\\
        \leq&\|(\nabla f(\xi))_{1:M}\|_{\infty}\|\bq- \tilde{\bq}\|_1+ |(\nabla f(\xi))_{M+1}||q_j-\tilde{q}_j|,
    \end{split}
\end{equation}
where $\xi\in \bbR^{M+1}$ is set as the interpolation of $(\bq^\top,\,q_j)^\top$ and $(\bq\vee \tilde{\bq})^{\top},\,q_j\vee \tilde{q}_j)^\top$ that satisfies the the second equality above, whose existence is guaranteed by the Lagrange's mean value theorem. Now it remains to show an upper bound for $\|(\nabla f(\xi))_{1:M}\|_{\infty}$ and $|(\nabla f(\xi))_{M+1}|$. Some calculations show that when $k\leq M$,
\begin{equation}\label{eq:f_grad_k}
    \begin{split}
        |(\nabla f(\xi))_k|=\frac{\xi_{M+1}\prod_{l\neq k,\,l\leq M}(1-\xi_l)}{(1-\prod_{l=1}^M(1-\xi_l))^2}\leq \frac{q_j\prod_{l=1}^M(1-q_l)}{(1-\delta)(1-\prod_{l=1}^M(1-q_l))^2}\leq \frac{e^{-m}q_j}{(1-e^{-m})^2(1-\delta)},
    \end{split}
\end{equation}
where the first inequality is due to $q_k\leq \delta$, and the last inequality again uses $\prod_{l=1}^M(1-q_l)\leq e^{-m}$. In addition,
\begin{equation}\label{eq:f_grad_M}
    |(\nabla f(\xi))_{M+1}|=\frac{1}{1-\prod_{l=1}^M(1-\xi_l)}\leq \frac{1}{1-\prod_{l=1}^M(1-q_l)}\leq (1-e^{-m})^{-1}.
\end{equation}
Therefore, the bounds for $\nabla f(\xi)$ above together with \eqref{eq:F_size_unequal_term1} implies
\begin{equation*}
    \begin{split}
        &\left|\frac{\tilde{q}_j}{1-\prod_{l=1}^M\big(1-\tilde{q}_l\big)}-\frac{q_j\wedge \tilde{q}_j}{1-\prod_{l=1}^M\big(1-q_l\vee \tilde{q}_l\big)}\right|\\
        \leq&\frac{q_je^{-m}\|\bq-\tilde{\bq}\|_1}{(1-e^{-m})^2(1-\delta)} + \frac{|q_j-\tilde{q}_j|}{1-e^{-m}}\\
        \leq&\frac{q_j\|\bq-\tilde{\bq}\|_1 + |q_j-\tilde{q}_j|}{1-e^{-m}},
    \end{split}
\end{equation*}
where we utilized the assumption that $m\geq \log(1+\frac{1}{1-\delta})$, which immediately implies $\frac{e^{-m}}{1-e^{-m}}\leq 1-\delta$. Plugging in the bound above into \eqref{eq:F_size_unequal_bnd3}, we get
\begin{equation*}
    \begin{split}
        \bbP\big(j\in F^{(1)},\,F^{(1)}\neq F^{(2)}\big)\leq \frac{3q_j\|\bq-\tilde{\bq}\|_1, + |q_j-\tilde{q}_j|}{1-e^{-m}}.
    \end{split}
\end{equation*}
The proof of Lemma \ref{lem:F1_F2_diff} is now complete.
\end{proof}
\begin{proof}[Proof of Lemma~\ref{lem:omega_j_diff}]
    Recall that $\omega_j(Q) = \frac{q_j}{1-q_j-\prod_{l=1}^M(1-q_l)} = \frac{q_j/(1-q_j)}{1-\prod_{l\neq j}(1-q_l)}$. Define function $f:\bbR^M\rightarrow \bbR$ such that $f(x)=\frac{x_M}{1-\prod_{l<M}(1-x_l)}$. We can then write $\omega_j(Q_{\bq}) - \omega_j(Q_{\tilde{\bq}}) = f((\bq_{\backslash j}^{\top},\,q_j/(1-q_j))^\top) - f((\tilde{\bq}_{\backslash j}^{\top},\,\tilde{q}_j/(1-\tilde{q}_j))^\top).$ Due to the Lagrange's mean value theorem, we can always find a $\xi\in \bbR^M$ that lies between $(\bq_{\backslash j}^{\top},\,q_j/(1-q_j))^\top$ and $(\bq_{\backslash j}^{\top},\,\tilde{q}_j/(1-\tilde{q}_j))^\top$ such that
    \begin{equation*}
        \begin{split}
            |\omega_j(Q_{\bq}) - \omega_j(Q_{\tilde{\bq}})|\leq \|(\nabla f(\xi))_{1:(M-1)}\|_{\infty}\big\|\bq- \tilde{\bq}\big\|_1+ |(\nabla f(\xi))_{M}|\big|\tilde{q}_j-q_j\big|.
        \end{split}
    \end{equation*}
    Following similar calculations as in \eqref{eq:f_grad_k} and \eqref{eq:f_grad_M} in the proof of Lemma \ref{lem:F1_F2_diff}, we have $\|(\nabla f(\xi))_{1:(M-1)}\|_{\infty}\leq C\big(q_j + \tilde{q}_j\big)$, $|(\nabla f(\xi))_M|\leq C$. Therefore, 
    $$
    |\omega_j(Q_{\bq}) - \omega_j(Q_{\tilde{\bq}})|\leq C\big(q_j + \tilde{q}_j\big)\big\|\tilde{\bq}- \bq\big\|_1+ C\big|\tilde{q}_j-q_j\big|.$$
    
\end{proof}

\begin{proof}[Proof of Lemma \ref{lem:epsilon_min_bound}]
Denote $\hat{\Delta}_{\min}^{(b)}- \Delta_{\min}^{*(b)}$ by $\varepsilon_{\min}^{(b)}$.
By definition,

\[
\varepsilon_{\min}^{(b)}
=\hat{\Delta}_{l_{\min}^{(b)}}^{(b)}-\Delta_{l_{\min}^{*(b)}}^{*(b)}.
\]

Since $l_{\min}^{*(b)}$ minimizes $\Delta_l^{*(b)}$, we have

\[
\Delta_{l_{\min}^{*(b)}}^{*(b)}
\le
\Delta_{l_{\min}^{(b)}}^{*(b)},
\]

and therefore

\[
\varepsilon_{\min}^{(b)}
\ge
\hat{\Delta}_{l_{\min}^{(b)}}^{(b)}-\Delta_{l_{\min}^{(b)}}^{*(b)}=
\varepsilon_{l_{\min}^{(b)}}^{(b)}.
\]

Since $l_{\min}^{(b)}$ minimizes $\hat{\Delta}_l^{(b)}$, we also have

\[
\hat{\Delta}_{l_{\min}^{(b)}}^{(b)}
\le
\hat{\Delta}_{l_{\min}^{*(b)}}^{(b)},
\]

which implies

\[
\varepsilon_{\min}^{(b)}
\le
\hat{\Delta}_{l_{\min}^{*(b)}}^{(b)}-\Delta_{l_{\min}^{*(b)}}^{*(b)}
=
\varepsilon_{l_{\min}^{*(b)}}^{(b)}.
\]

Combining the two inequalities yields

\[
\big|\varepsilon_{\min}^{(b)}\big|
\le
\Big|\varepsilon_{l_{\min}^{(b)}}^{(b)}\Big|
\;\vee\;
\Big|\varepsilon_{l_{\min}^{*(b)}}^{(b)}\Big|.
\]

\end{proof}
\begin{proof}[Proof of Lemma~\ref{lem:read_k_concentration}]
    We first show an upper bound of the moment generating function for $\sum_{i=1}^n Y_i$. For $1\leq j\leq m$, let $T_j=\{i:S_i\owns j\}$. As assumed in Lemma~\ref{lem:read_k_concentration}, $|T_j|\leq k$ for $1\leq j\leq m$. We note that
    \begin{equation*}
        \begin{split}
            \bbE\Big(\exp\Big\{\lambda\sum_{i=1}^n Y_i\Big\}\Big) = &\bbE\Big(\prod_{i=1}^ne^{\lambda Y_i}\Big)\\
            =&\bbE_{\notationchange{X^{(1)}},\notationchange{X^{(2)}},\dots,\notationchange{X^{(m-1)}}}\bbE_{\notationchange{X^{(m)}}}\left(\prod_{i\notin T_m}e^{\lambda Y_i}\prod_{i\in T_m}e^{\lambda Y_i}\right)\\
            = &\bbE_{\notationchange{X^{(1)}},\notationchange{X^{(2)}},\dots,\notationchange{X^{(m-1)}}}\left[\prod_{i\notin T_m}e^{\lambda Y_i}\left(\bbE_{\notationchange{X^{(m)}}}\Big(\prod_{i\in T_m}e^{\lambda Y_i}\Big)\right)\right],
        \end{split}
    \end{equation*}
    where we used $\bbE_{\notationchange{X^{(l)}}}$ to denote the expectation taken only over $\notationchange{X^{(l)}}$, with $\notationchange{X^{(\no l)}}$ fixed. Since $|T_m|\leq k$, the H\"older's inequality immediately yields
    \begin{equation*}
        \begin{split}
            \bbE_{\notationchange{X^{(m)}}}\Big(\prod_{i\in T_m}e^{\lambda Y_i}\Big)\leq \prod_{i\in T_m}\left(\bbE_{\notationchange{X^{(m)}}}e^{k\lambda Y_i}\right)^{\frac{1}{k}}.
        \end{split}
    \end{equation*}
    For any $i\notin T_m$, $\left(\bbE_{\notationchange{X^{(m)}}}e^{k\lambda Y_i}\right)^{\frac{1}{k}} = e^{\lambda Y_i}$. Hence we can write
    \begin{equation*}
            \bbE\Big(\exp\Big\{\lambda\sum_{i=1}^n Y_i\Big\}\Big) \leq  \bbE\left[\prod_{i=1}^n\left(\bbE_{\notationchange{X^{(m)}}}e^{k\lambda Y_i}\right)^{\frac{1}{k}}\right].
    \end{equation*}
    Repeating the same arguments above when taking expectation over $\notationchange{X^{(m-1)}}$, we have
    \begin{equation*}
        \begin{split}
            \bbE\left[\prod_{i=1}^n\left(\bbE_{\notationchange{X^{(m)}}}e^{k\lambda Y_i}\right)^{\frac{1}{k}}\right]\leq \bbE_{\notationchange{X^{(1)}},\notationchange{X^{(2)}},\dots,\notationchange{X^{(m-2)}}}\left[\prod_{i=1}^n\Big(
            \bbE_{\notationchange{X^{(m-1)}},\notationchange{X^{(m)}}}e^{k\lambda Y_i}\Big)^{\frac{1}{k}}\right],
        \end{split}
    \end{equation*}
    Similarly, we can keep taking expectation over $\notationchange{X^{(m-2)}},\dots, \notationchange{X^{(1)}}$ one by one to obtain the following:
    \begin{equation*}
         \bbE\Big(\exp\Big\{\lambda\sum_{i=1}^n Y_i\Big\}\Big) \leq \prod_{i=1}^n\Big(
            \bbE e^{k\lambda Y_i}\Big)^{\frac{1}{k}}.
    \end{equation*}
    Since each $Y_i$ is centered and bounded with $|Y_i|\leq C_i$, we can apply the Hoeffding's lemma to bound $\bbE e^{k\lambda Y_i}$:
    \begin{equation*}
            \bbE e^{k\lambda Y_i}\leq \exp\{\frac{1}{2}k^2\lambda^2C_i^2\},
    \end{equation*}
    which then implies 
    \begin{equation*}
        \bbE\Big(\exp\Big\{\lambda\sum_{i=1}^n Y_i\Big\}\Big) \leq \exp\{\frac{k\sum_{i=1}^nC_i^2\lambda^2}{2}\}.
    \end{equation*}
    Therefore, we can apply the Chernoff bound to obtain the following:
    \begin{equation*}
        \bbP\left(\big|\sum_{i=1}^n Y_i\big|>t\right)\leq 2\exp\Big\{-\frac{t^2}{2k\sum_{i=1}^nC_i^2}\Big\}.
    \end{equation*}
    The proof of Lemma~\ref{lem:read_k_concentration} is now complete.
\end{proof}

\section{Proofs for Theorem~\ref{thm:interaction} and Its Auxiliary Lemmas}

Consider the general functional ANOVA model setup in Eq. \eqref{eq:general_interaction_model} and its special case \eqref{eq:interaction_mod}. We provide a complete proof of Theorem~\ref{thm:interaction} here in the supplement.

\begin{proof}[Proof of Theorem \ref{thm:interaction}]
We prove the two inclusions
\[
S \subseteq \hat S
\qquad\text{and}\qquad
\hat S \subseteq S
\]
separately.

\paragraph{No false negatives.}
Our goal is to show that every signal feature $j\in S$ satisfies
$q_j^{(B)} > \frac{1}{2}\delta$, which implies $j\in \hat S$. To establish this, we use the following growth property.

\begin{lem}\label{lem:interaction_q_growth}
    Suppose Assumptions~\ref{ass:minipatch_feature_scaling}, \ref{ass:signal_strength_ratio}, and \ref{ass:interaction_pair_min_sig} hold. Then 
\begin{equation}
\label{eq:growth_factor_C_in_interaction_proof}
q_j^{(b)} \geq \big(\frac{7}{6}\,q_j^{(b-1)}\big)\wedge \delta,
\qquad \forall j\in S,\ \forall\,1\le b\le B.
\end{equation}
\end{lem}

Iterating
\eqref{eq:growth_factor_C_in_interaction_proof} yields that $\forall j\in S$,
\[
q_j^{(B)}\geq 
\big((\frac{7}{6})^B q_j^{(0)}\big)\wedge \delta= \Big((\frac{7}{6})^B \frac{m}{M}\Big)\wedge \delta,
\]
where we utilized the fact that the initial sampling probability $q_k^{(0)} = \frac{m}{M}$ for all feature $k$. Hence, as long as $B > (\log(\frac{7}{6}))^{-1}\log\left(\frac{\delta M}{2m}\right)$, we have $q_j^{(B)} > \frac{\delta M}{2m}\frac{m}{M}=\frac{1}{2}\delta$. Therefore, by the definition of $\hat{S}$ in Algorithm~\ref{algo:loco_adamp}, as long as we set the constant $C$ in Theorem~\ref{thm:interaction} as $(\log(\frac{7}{6}))^{-1}$, we have $j\in \hat{S}$. Since the arguments above hold for all $j\in S$, this implies that $S \subseteq \hat{S}$.

\paragraph{No false positives.}
We next show that no noise feature is selected, namely,
\[
\hat S \subseteq S.
\]

For this part, we use the following lemma.

\begin{lem}\label{lem:interaction_noise_prob_bound}
Suppose Assumptions~\ref{ass:minipatch_feature_scaling}, \ref{ass:signal_strength_ratio}, and \ref{ass:noise_err} hold.
For any $j\in S^c$ and $1\le b\le B$, $q_j^{(b)}\le
\frac{1}{2}\delta$.
\end{lem}

\medskip

Applying Lemma~\ref{lem:interaction_noise_prob_bound} at iteration $b=B$, we obtain
\[
q_j^{(B)} \le \frac{1}{2}\delta,
\qquad \forall j\in S^c.
\]
Therefore no noise feature can cross the selection threshold, and hence
\[
\hat S \subseteq S.
\]

Combining the no-false-negative and no-false-positive parts, we conclude that
\[
\hat S = S.
\]
This completes the proof.
\end{proof}

\section{Proof of Lemma~\ref{lem:interaction_q_growth} and Lemma~\ref{lem:interaction_noise_prob_bound}}
Here, we present the proof of Lemma~\ref{lem:interaction_q_growth} and Lemma~\ref{lem:interaction_noise_prob_bound}.

\begin{proof}[Proof of Lemma \ref{lem:interaction_q_growth}]
Consider any $j\in S$. Apply Lemma~\ref{prop:q_lwrbnd} to the $b$th iteration, we have
\begin{equation*}
    q_j^{(b)}\geq \frac{\til{\Delta}_j^{(b)}(m-s\delta)}{\sum_{l>s}\til{\Delta}^{(b)}_{(l)}}\wedge \delta \geq \frac{\til{\Delta}^{(b)}_j(m-s\delta)}{\sum_{l\in S^c}\til{\Delta}^{(b)}_{l}}\wedge \delta,
\end{equation*}
where the second inequality is due to the fact that $\sum_{l\in S^c}\til{\Delta}^{(b)}_{l}\geq \sum_{l>s}\til{\Delta}_{(l)}.$ Therefore, it suffices to show that
\[
\frac{\til{\Delta}^{(b)}_j(m-s\delta)}{\sum_{l\in S^c}\til{\Delta}^{(b)}_{l}}\geq \frac{7}{6} q_j^{(b-1)}.
\]

We note that for any $1\leq k\leq M$, the estimated and shifted feature importance score satisfies
\[
\til{\Delta}^{(b)}_k
=
\Delta_k^{*(b)} - \Delta_{\min}^{*(b)} + \frac{c_0}{M} +\varepsilon_k^{(b)}.
\]
The following lemma provides lower and upper bounds for the population feature importance $\Delta_k^{*(b)}$ when $k\in S$ and $k\in S^c$ separately.

\begin{lem}\label{lem:interaction_Delta_oracle}
Suppose Assumption~\ref{ass:signal_strength_ratio} holds. For any $k\in S\setminus\{u,v\}$,
\begin{equation}
\label{eq:interaction_delta_oracle_signal_bound}
\Delta_k^{*(b)}
\ge
\frac{(1-2e^{-m})q_k^{(b-1)}}{2(1-e^{-m})}\|f_k\|_2^2
-
\frac{2e^{-(m-1)}}{1-e^{-(m-1)}}\frac{q_k^{(b-1)}}{1-\prod_l(1-q_l^{(b-1)})}\,q_u^{(b-1)}q_v^{(b-1)}\,\|f_{uv}\|_2^2.
\end{equation}

For any $k \in \{u,v\}$,
\begin{equation}
\label{eq:interaction_delta_oracle_pair_bound}
\Delta_k^{*(b)}
\ge
q_u^{(b-1)}q_v^{(b-1)}\,\|f_{uv}\|_2^2
-\frac{2e^{-(m-1)}}{1-e^{-(m-1)}}\frac{q_k^{(b-1)}}{1-\prod_l(1-q_l^{(b-1)})}\sum_{l\in S} q_l^{(b-1)}\|f_l\|_2^2
.    
\end{equation}

For any $k\in S^c$,
\begin{equation}
\label{eq:interaction_delta_oracle_noise_bound}
-\frac{2e^{-(m-1)}}{1-e^{-(m-1)}}\frac{q_k^{(b-1)}}{1-\prod_l(1-q_l^{(b-1)})}
\left(
\sum_{l\in S} q_l^{(b-1)}\|f_l\|_2^2 + q_u^{(b-1)}q_v^{(b-1)} \|f_{uv}\|_2^2
\right)
\le
\Delta_k^{*(b)}
\le 0.
\end{equation}
\end{lem}

Since $m\geq 2>\ln{4}$, we can further simplify the bound \eqref{eq:interaction_delta_oracle_signal_bound} in Lemma~\ref{lem:interaction_Delta_oracle} as follows: for $k\in S\setminus\{u,v\}$,
\begin{equation}\label{eq:interaction_delta_oracle_signal_bnd2}
\begin{split}
    \Delta_k^{*(b)}
    \geq&
    \frac{(1-2e^{-m})q_k^{(b-1)}}{2(1-e^{-m})}\|f_k\|_2^2
    -
    \frac{2e^{-(m-1)}}{1-e^{-(m-1)}}\frac{q_k^{(b-1)}}{1-\prod_l(1-q_l^{(b-1)})}\,q_u^{(b-1)}q_v^{(b-1)}\,\|f_{uv}\|_2^2
    \\
    \geq&
    \frac{1}{2}\left(1-\frac{e^{-m}}{1-e^{-m}}\right)q_k^{(b-1)}\|f_k\|_2^2
    -
    \frac{2e}{(1-1/e)(1-e^{-m})}M^{-C_1}q_k^{(b-1)}\,\|f_{uv}\|_2^2
    \\
    \geq&
    \frac{1}{3}q_k^{(b-1)}\|f_k\|_2^2
    -
    10M^{-C_1}q_k^{(b-1)}\,\|f_{uv}\|_2^2,
\end{split}
\end{equation}
where the second inequality is due to Assumption~\ref{ass:minipatch_feature_scaling} and the fact that
\[
\prod_l(1-q_l^{(b-1)})
=
\exp\Big\{\sum_{l=1}^M \log(1-q_l^{(b-1)})\Big\}
\leq
\exp\big\{-\sum_{l=1}^M q_l^{(b-1)}\big\}
=
e^{-m}.
\]

For $k\in\{u,v\}$, 
\begin{equation}\label{eq:interaction_delta_oracle_pair_bnd2}
    \begin{split}
        \Delta_k^{*(b)}
        \geq&
        q_u^{(b-1)}q_v^{(b-1)}\,\|f_{uv}\|_2^2
        -\frac{2e^{-(m-1)}}{1-e^{-(m-1)}}\frac{q_k^{(b-1)}}{1-\prod_l(1-q_l^{(b-1)})}\sum_{l\in S} q_l^{(b-1)}\|f_l\|_2^2
        \\
        \geq&
        q_u^{(b-1)}q_v^{(b-1)}\,\|f_{uv}\|_2^2
        -10M^{-C_1}q_k^{(b-1)}c_1M^{C_1-2}\|f_{uv}\|_2^2
        \\
        =&
        (q_u^{(b-1)}q_v^{(b-1)}\ - \frac{10c_1}{M^2}q_k^{(b-1)})\,\|f_{uv}\|_2^2,
    \end{split}
\end{equation}

where the second inequality is due to Assumption~\ref{ass:minipatch_feature_scaling} and Assumption~\ref{ass:signal_strength_ratio}.
While for the lower bound \eqref{eq:interaction_delta_oracle_noise_bound}, Assumption~\ref{ass:minipatch_feature_scaling} helps to simplify it as follows: for $k\in S^c$,
\begin{equation}\label{eq:interaction_delta_oracle_noise_bnd2}
    \begin{split}
        -\Delta_k^{*(b)}
        \leq&
        \frac{2e^{-(m-1)}}{1-e^{-(m-1)}}\frac{q_k^{(b-1)}}{1-\prod_l(1-q_l^{(b-1)})}
        \left(
        \sum_{l\in S}q_l^{(b-1)}\|f_l\|_2^2
        +
        q_u^{(b-1)}q_v^{(b-1)} \|f_{uv}\|_2^2
        \right)
        \\
        \leq&
        10M^{-C_1}q_k^{(b-1)}
        \left(
        \sum_{l\in S}q_l^{(b-1)}\|f_l\|_2^2
        +
        q_u^{(b-1)}q_v^{(b-1)} \|f_{uv}\|_2^2
        \right).
    \end{split}
\end{equation}

Therefore, for $j\in S\setminus\{u,v\}$, we can write
\begin{equation}\label{eq:interaction_q_numerator_lwr}
    \begin{split}
        \til{\Delta}_j^{(b)}
        \geq&
        \Delta_j^{*(b)} - \Delta_{\min}^{*(b)} +\frac{c_0}{M} +\varepsilon_j^{(b)}
        \\
        \geq&
        \left(
        \frac{1}{3}\|f_j\|_2^2
        -
        10M^{-C_1}\,\|f_{uv}\|_2^2
        -
        \big|\varepsilon_j^{(b)}\big|/q_j^{(b-1)}
        \right)q_j^{(b-1)}
        \\
        \geq&
        \left(
        \frac{1}{6}\|f_j\|_2^2
        -
        10M^{-C_1}\,\|f_{uv}\|_2^2
        \right)q_j^{(b-1)},
    \end{split}
\end{equation}
where we applied Assumption~\ref{ass:interaction_pair_min_sig} in the last line above.

In addition,
\begin{equation}\label{eq:interaction_q_denominator_upp}
    \begin{split}
        \sum_{l\in S^c}\til{\Delta}_l^{(b)}
        =
        &\sum_{l\in S^c}\Delta_l^{*(b)}
        -
        (M-|S|)\Delta_{\min}^{*(b)}
        +
        \frac{c_0(M-|S|)}{M}
        +
        \sum_{l\in S^c}\varepsilon_l^{(b)}
        \\
        \leq&
        10M^{-(C_1-1)}
        \left(
        \sum_{l\in S}\|f_l\|_2^2+\|f_{uv}\|_2^2
        \right)
        +
        c_0+\|\varepsilon_{S^c}^{(b)}\|_1.
    \end{split}
\end{equation}
where the second line utilizes Lemma~\ref{lem:interaction_Delta_oracle} and \eqref{eq:interaction_delta_oracle_noise_bnd2}.

Suppose that we choose the constant $c_1\leq \frac{m-s\delta}{480} \wedge \frac{m(m-s\delta)}{60}$ in Assumption~\ref{ass:signal_strength_ratio}. Then, for $j\in S\setminus\{u,v\}$, we have
\begin{equation*}
    \begin{split}
        \frac{\til{\Delta}^{(b)}_j(m-s\delta)}{\sum_{l\in S^c}\til{\Delta}^{(b)}_{l}}
        \geq&
        \frac{
        \left(\frac{1}{6}-\frac{10c_1}{M^2}\right)
        \|f_j\|_2^2(m-s\delta)
        }{
        20\frac{c_1}{M}\|f_j\|_2^2+c_0+\|\varepsilon_{S^c}^{(b)}\|_1
        }
        q_j^{(b-1)}
        \\
        \geq&
        \frac{
        \left(\frac{1}{6}-\frac{1}{48}\right)
        \|f_j\|_2^2(m-s\delta)
        }{
        \frac{1}{24}(m-s\delta)\|f_j\|_2^2+c_0+\|\varepsilon_{S^c}^{(b)}\|_1
        }
        q_j^{(b-1)}
        \\
        \geq&
        \frac{7}{6}q_j^{(b-1)}.
    \end{split}
\end{equation*}
where the first inequality is due to \eqref{eq:interaction_q_numerator_lwr}, \eqref{eq:interaction_q_denominator_upp}, and Assumption~\ref{ass:signal_strength_ratio}; the second line is due to $\frac{m-s\delta}{M}\le 1$; the third line is due to Assumption~\ref{ass:interaction_pair_min_sig}.

For $j\in\{u,v\}$, by \eqref{eq:interaction_delta_oracle_pair_bnd2},
\begin{equation}
\label{eq:interaction_q_numerator_lwr_pair}
\begin{split}
    \til{\Delta}_j^{(b)}
    \geq&
    \Delta_j^{*(b)}
    -
    \Delta_{\min}^{*(b)}
    +
    \frac{c_0}{M}
    +
    \varepsilon_j^{(b)}
    \\
    \geq&
    (q_u^{(b-1)}q_v^{(b-1)}\ - \frac{10c_1}{M^2}q_j^{(b-1)})\,\|f_{uv}\|_2^2
    -
    \big|\varepsilon_j^{(b)}\big|
    \\
    \geq&
    (\frac{31}{72}
    q_u^{(b-1)}q_v^{(b-1)}- \frac{m-s\delta}{6M}\frac{m}{M}q_j^{(b-1)})
    \|f_{uv}\|_2^2,
\end{split}
\end{equation}
where the third line follows from
Assumption~\ref{ass:interaction_pair_min_sig}.

We now establish \eqref{eq:growth_factor_C_in_interaction_proof} for
$j\in\{u,v\}$ by induction on $b$. In particular, we simultaneously show that
\begin{equation}
\label{eq:interaction_pair_q_lower}
q_u^{(b)}\wedge q_v^{(b)}
\geq
\frac{m}{M},
\qquad
0\leq b\leq B.
\end{equation}

First, consider the base case $b=0$. By initialization,
\[
q_u^{(0)}=q_v^{(0)}=\frac{m}{M},
\]
so \eqref{eq:interaction_pair_q_lower} holds.

We next establish the growth bound from iteration $0$ to iteration $1$.
By Assumption~\ref{ass:minipatch_feature_scaling}, $\frac{m-s\delta}{M} < \frac{1}{4}$, together with base case, 
\eqref{eq:interaction_q_numerator_lwr_pair} gives
\begin{equation}
    \label{eq:interaction_q_numerator_lwr_pair_new}
  \til{\Delta}_j^{(b)} \geq \frac{7}{18}(q_u^{(b-1)}\wedge q_v^{(b-1)})\|f_{uv}\|_2^2 \, q_j^{(b-1)}  
\end{equation}

By Assumption~\ref{ass:signal_strength_ratio},
\[
\sum_{l\in S}\|f_l\|_2^2
+
\|f_{uv}\|_2^2
\leq
c_1M^{C_1-2}\|f_{uv}\|_2^2.
\]
Therefore, \eqref{eq:interaction_q_denominator_upp} gives
\[
\sum_{l\in S^c}\til{\Delta}_l^{(1)}
\leq
\frac{10c_1}{M}\|f_{uv}\|_2^2
+
c_0
+
\varepsilon_{\Delta}.
\]
Since
\[
c_1
\leq
\frac{m(m-s\delta)}{60},
\qquad
c_0
<
\frac{m(m-s\delta)}{12M}\|f_{uv}\|_2^2,
\]
and Assumption~\ref{ass:interaction_pair_min_sig} implies
\[
\|\varepsilon_{S^c}^{(1)}\|_1
<
\frac{(m-s\delta)}{12}(q_u^{(0)}\wedge q_v^{(0)})
\|f_{uv}\|_2^2,
\]
we obtain
\begin{equation}
\label{eq:interaction_q_denominator_upp_pair_base}
\sum_{l\in S^c}\til{\Delta}_l^{(1)}
\leq
\frac{1}{3}
\big(q_u^{(0)}\wedge q_v^{(0)}\big)
(m-s\delta)\|f_{uv}\|_2^2.
\end{equation}
Combining \eqref{eq:interaction_q_numerator_lwr_pair_new} and
\eqref{eq:interaction_q_denominator_upp_pair_base}, for
$j\in\{u,v\}$,
\[
\begin{split}
\frac{
\til{\Delta}_j^{(1)}(m-s\delta)
}{
\sum_{l\in S^c}\til{\Delta}_l^{(1)}
}
&\geq
\frac{
\frac{7}{18}
\big(q_u^{(0)}\wedge q_v^{(0)}\big)
\|f_{uv}\|_2^2
q_j^{(0)}
(m-s\delta)
}{
\frac{1}{3}
\big(q_u^{(0)}\wedge q_v^{(0)}\big)
(m-s\delta)\|f_{uv}\|_2^2
}
\\
&=
\frac{7}{6}q_j^{(0)}.
\end{split}
\]
Hence,
\[
q_j^{(1)}
\geq
\left(
\frac{7}{6}q_j^{(0)}
\right)\wedge\delta,
\qquad
j\in\{u,v\}.
\]
Since $q_j^{(0)}\leq\delta$, this further implies
\[
q_j^{(1)}
\geq
q_j^{(0)}
=
\frac{m}{M},
\]
so \eqref{eq:interaction_pair_q_lower} holds for $b=1$.

Now suppose that \eqref{eq:interaction_pair_q_lower} holds at iteration
$b-1$ for some $2\leq b\leq B$. Then
\[
q_u^{(b-1)}\wedge q_v^{(b-1)}
\geq
\frac{m}{M}.
\]
Consequently, we can get \eqref{eq:interaction_q_numerator_lwr_pair_new} from \eqref{eq:interaction_q_numerator_lwr_pair}.
Besides,
the conditions on $c_1$ and $c_0$ imply
\[
\frac{10c_1}{M}\|f_{uv}\|_2^2
\leq
\frac{1}{6}
\big(q_u^{(b-1)}\wedge q_v^{(b-1)}\big)
(m-s\delta)\|f_{uv}\|_2^2,
\]
and
\[
c_0
<
\frac{1}{12}
\big(q_u^{(b-1)}\wedge q_v^{(b-1)}\big)
(m-s\delta)\|f_{uv}\|_2^2.
\]
Moreover, Assumption~\ref{ass:interaction_pair_min_sig} gives
\[
\|\varepsilon_{S^c}^{(b)}\|_1
<
\frac{1}{12}
\big(q_u^{(b-1)}\wedge q_v^{(b-1)}\big)
(m-s\delta)\|f_{uv}\|_2^2.
\]
Plugging these bounds into
\eqref{eq:interaction_q_denominator_upp} yields
\begin{equation}
\label{eq:interaction_q_denominator_upp_pair}
\sum_{l\in S^c}\til{\Delta}_l^{(b)}
\leq
\frac{1}{3}
\big(q_u^{(b-1)}\wedge q_v^{(b-1)}\big)
(m-s\delta)\|f_{uv}\|_2^2.
\end{equation}
Combining \eqref{eq:interaction_q_numerator_lwr_pair_new} and
\eqref{eq:interaction_q_denominator_upp_pair}, for $j\in\{u,v\}$,
\[
\begin{split}
\frac{
\til{\Delta}_j^{(b)}(m-s\delta)
}{
\sum_{l\in S^c}\til{\Delta}_l^{(b)}
}
&\geq
\frac{
\frac{7}{18}
\big(q_u^{(b-1)}\wedge q_v^{(b-1)}\big)
\|f_{uv}\|_2^2
q_j^{(b-1)}
(m-s\delta)
}{
\frac{1}{3}
\big(q_u^{(b-1)}\wedge q_v^{(b-1)}\big)
(m-s\delta)\|f_{uv}\|_2^2
}
\\
&=
\frac{7}{6}q_j^{(b-1)}.
\end{split}
\]
Therefore,
\[
q_j^{(b)}
\geq
\left(
\frac{7}{6}q_j^{(b-1)}
\right)\wedge\delta,
\qquad
j\in\{u,v\}.
\]
Since $q_j^{(b-1)}\leq\delta$, we also have
\[
q_j^{(b)}
\geq
q_j^{(b-1)}
\geq
\frac{m}{M}.
\]
Thus, \eqref{eq:interaction_pair_q_lower} holds at iteration $b$.
By induction, \eqref{eq:interaction_pair_q_lower} holds for all
$0\leq b\leq B$, and \eqref{eq:growth_factor_C_in_interaction_proof} holds for
$j\in\{u,v\}$ and all $1\leq b\leq B$.

The proof is now complete.
\end{proof}

\begin{proof}[Proof of Lemma \ref{lem:interaction_noise_prob_bound}]

By Algorithm~\ref{algo:IndepSampWeight}, we have

\[
q_j^{(b)}
=
\frac{
m\big(\til{\Delta}_j^{(b)}\wedge t^{*(b)}\big)
}{
\sum_l\big(\til{\Delta}_l^{(b)}\wedge t^{*(b)}\big)
},
\]
where $t^{*(b)}$ is the output of Algorithm~\ref{algo:find_threshold} when applied to $\{\til{\Delta}_l^{(b)}\}_{l=1}^M$. After examining the steps of Algorithm~\ref{algo:find_threshold}, we note that $t^{*(b)}\geq \til{\Delta}_{(\lfloor \frac{m}{\delta}\rfloor +1)}^{(b)} \geq \frac{c_0}{M}$, since $\min_l \til{\Delta}_l^{(b)} \geq \frac{c_0}{M}$ by definition. Therefore, we can write $\sum_l\big(\til{\Delta}_l^{(b)}\wedge t^*\big)\ge M\frac{c_0}{M} = c_0$. Hence, $q_j^{(b)}\le
\frac{m}{c_0}\til{\Delta}_j^{(b)}$.

By definition of $\til{\Delta}_j^{(b)}$,

\begin{equation*}
\begin{split}
    q_j^{(b)}\le &\frac{m}{c_0}\left(\Delta_j^{*(b)}-\Delta_{\min}^{*(b)} + \frac{c_0}{M} + \varepsilon_j^{(b)}\right)\\
    \le &\frac{m}{c_0}\left(10M^{-C_1}(\sum_{l\in S}\|f_l\|_2^2 + \|f_{uv}\|_2^2) +\frac{c_0}{M} + \big|\varepsilon_j^{(b)}\big|\right)\\
    \le & \frac{1}{2}\delta,
\end{split}
\end{equation*}
where the second line is due to \eqref{eq:interaction_delta_oracle_noise_bnd2} and the last line is due to Assumption~\ref{ass:minipatch_feature_scaling} and Assumption~\ref{ass:noise_err}. By the definition of $\hat{S}$ in Algorithm~\ref{algo:loco_adamp}, the proof of Lemma~\ref{lem:interaction_noise_prob_bound} is now complete.
\end{proof}

\begin{proof}[Proof of Lemma~\ref{lem:interaction_Delta_oracle}]
For notational simplicity we fix an iteration $b$ and omit the superscript $(b)$ throughout.
Thus $q_l$ means $q_l^{(b-1)}$, $Q$ means $Q^{(b)}$, and $Q_{\backslash j}$ means
$Q^{(b)}_{\backslash j}$.
Also define
\[
I_{uv}:=\|f_{uv}\|_2^2=\bbE[f_{uv}(\notationchange{X^{(u)}},\notationchange{X^{(v)}})^2].
\]
and
\[
r:=\prod_{l=1}^M(1-q_l).
\]

\paragraph{Step 1: Oracle predictors under the order-two functional ANOVA model.}

The oracle predictor on a feature subset $F$ is
\[
\mu_F^*(\notationchange{X^{(F)}})
=
f_0+\sum_{l\in S} f_l(\notationchange{X^{(l)}})\ind(l\in F)
+
f_{uv}(\notationchange{X^{(u)}},\notationchange{X^{(v)}})\ind(u\in F,v\in F).
\]
Hence
\[
\bbE_{F\sim Q}\mu_F^*(\notationchange{X^{(F)}})
=
f_0+\sum_{l\in S} p_l\,f_l(\notationchange{X^{(l)}})
+
p_{uv}\,f_{uv}(\notationchange{X^{(u)}},\notationchange{X^{(v)}}),
\]
and
\[
\bbE_{F\sim Q_{\backslash j}}\mu_F^*(\notationchange{X^{(F)}})
=
f_0+\sum_{l\in S} p_l^{(-j)}\,f_l(\notationchange{X^{(l)}})
+
p_{uv}^{(-j)}\,f_{uv}(\notationchange{X^{(u)}},\notationchange{X^{(v)}}),
\]
where
\[
p_l:=\bbP_{F\sim Q}(l\in F),
\qquad
p_l^{(-j)}:=\bbP_{F\sim Q}(l\in F\mid j\notin F),
\]
and
\[
p_{uv}:=\bbP_{F\sim Q}(u,v\in F),
\qquad
p_{uv}^{(-j)}:=\bbP_{F\sim Q}(u,v\in F\mid j\notin F).
\]

\paragraph{Step 2: Risk decomposition.}

Write
\[
R_Q:=\bbE\Big(Y-\bbE_{F\sim Q}\mu_F^*(\notationchange{X^{(F)}})\Big)^2,
\qquad
R_{Q,\backslash j}:=\bbE\Big(Y-\bbE_{F\sim Q_{\backslash j}}\mu_F^*(\notationchange{X^{(F)}})\Big)^2.
\]
Because the ANOVA components are centered and pairwise orthogonal, we have
\[
R_Q
=
\sigma_\varepsilon^2
+
\sum_{l\in S}(1-p_l)^2\|f_l\|_2^2
+
(1-p_{uv})^2 I_{uv},
\]
and similarly
\[
R_{Q,\backslash j}
=
\sigma_\varepsilon^2
+
\sum_{l\in S}(1-p_l^{(-j)})^2\|f_l\|_2^2
+
(1-p_{uv}^{(-j)})^2 I_{uv}.
\]
Hence
\begin{equation}\label{eq:Delta_decomp_interaction_anova}
\Delta_j^*
=
\underbrace{
\sum_{l\in S}
\big[(1-p_l^{(-j)})^2-(1-p_l)^2\big]\|f_l\|_2^2
}_{=:A_j}
+
\underbrace{
\big[(1-p_{uv}^{(-j)})^2-(1-p_{uv})^2\big]I_{uv}
}_{=:B_j}.
\end{equation}

\paragraph{Step 3: Bound the main-effect part $A_j$.}

The term $A_j$ is exactly the additive oracle quantity. Since interaction terms have no main effect,
they act like noise term in additive model. By the proof of Lemma~\ref{lem:Delta_oracle},
\[
A_j
\ge
\frac{(1-2e^{-m})q_j}{2(1-e^{-m})}\|f_j\|_2^2,
\qquad j\in S \setminus \{u, v\},
\]
and
\[
-\frac{2e^{-(m-1)}}{1-e^{-(m-1)}}\frac{q_j}{1-r}
\sum_{l\in S} q_l\|f_l\|_2^2
\le
A_j
\le 0,
\qquad j\in S^c \cup \{u, v\}.
\]

It remains only to bound the interaction term $B_j$.

\paragraph{Step 4: Bound the interaction term when $j\notin\{u,v\}$.}

Assume first that $j\notin\{u,v\}$. Since under the Bernoulli construction
\[
\bbP(F=\varnothing)=r=\prod_{l=1}^M(1-q_l),
\]
we have
\[
p_{uv}
=
\bbP_{F\sim Q}(u,v\in F)
=
\frac{q_uq_v}{1-r}.
\]
Let
\[
r_{-j}:=\prod_{l\neq j}(1-q_l),
\qquad\text{so that}\qquad
r=(1-q_j)r_{-j}.
\]
Because $\{u,v\}\cap\{j\}=\varnothing$,
\[
p_{uv}^{(-j)}
=
\bbP_{F\sim Q}(u,v\in F\mid j\notin F)
=
\frac{q_uq_v}{1-r_{-j}}.
\]
Therefore
\[
p_{uv}^{(-j)}-p_{uv}
=
q_uq_v\left(\frac{1}{1-r_{-j}}-\frac{1}{1-r}\right)
=
\frac{q_jq_uq_v\,r_{-j}}{(1-r)(1-r_{-j})}.
\]
Since $\sum_{l\neq j}q_l\ge m-1$, we have
\[
\frac{r_{-j}}{1-r_{-j}}
\le
\frac{e^{-(m-1)}}{1-e^{-(m-1)}}.
\]
Hence
\[
0\le p_{uv}^{(-j)}-p_{uv}
\le
\frac{e^{-(m-1)}}{1-e^{-(m-1)}}
\frac{q_jq_uq_v}{1-r}.
\]
Now
\[
B_j
=
\big[(1-p_{uv}^{(-j)})^2-(1-p_{uv})^2\big]I_{uv}
=
(p_{uv}-p_{uv}^{(-j)})(2-p_{uv}^{(-j)}-p_{uv})I_{uv},
\]
so $B_j\le 0$, and
\[
|B_j|
\le
2(p_{uv}^{(-j)}-p_{uv})I_{uv}
\le
\frac{2e^{-(m-1)}}{1-e^{-(m-1)}}
\frac{q_jq_uq_v}{1-r}\,I_{uv}.
\]
Thus
\begin{equation}\label{eq:Bj_nonpair_anova}
-\frac{2e^{-(m-1)}}{(1-e^{-(m-1)})(1-r)}\,q_jq_uq_v\,I_{uv}
\le
B_j
\le 0.
\end{equation}

\paragraph{Step 5: Bound the interaction term when $j\in\{u,v\}$.}

Assume now that $j=u$; the case $j=v$ is identical.
Under $Q_{\backslash u}$ the set $F$ never contains $u$, so
\[
p_{uv}^{(-u)}=0.
\]
Hence
\[
B_u
=
\big[1-(1-p_{uv})^2\big]I_{uv}
=
(2p_{uv}-p_{uv}^2)I_{uv}
\ge
p_{uv}I_{uv}.
\]
Since
\[
p_{uv}=\frac{q_uq_v}{1-r}\ge q_uq_v,
\]
we obtain
\begin{equation}\label{eq:Bj_pair_anova}
B_u\ge q_uq_v I_{uv}.
\end{equation}
The same bound holds for $B_v$.

\paragraph{Step 6: Combine the bounds.}

We now combine the decomposition \eqref{eq:Delta_decomp_interaction_anova} with the bounds above.

\medskip
\noindent
\emph{Case 1: $j\in S\setminus\{u,v\}$.}
Then $A_j\ge \frac{(1-2e^{-m})q_j}{2(1-e^{-m})}\|f_j\|_2^2$, and
$B_j$ satisfies \eqref{eq:Bj_nonpair_anova}. Therefore
\[
\Delta_j^*
=
A_j+B_j
\ge
\frac{(1-2e^{-m})q_j}{2(1-e^{-m})}\|f_j\|_2^2
-
\frac{2e^{-(m-1)}}{(1-e^{-(m-1)})(1-r)}\,q_jq_uq_v\,I_{uv}.
\]

\medskip
\noindent
\emph{Case 2: $j\in\{u,v\}$.}
Then $A_j\ge -\frac{2e^{-(m-1)}}{1-e^{-(m-1)}}\frac{q_j}{1-r}\sum_{l\in S} q_l\|f_l\|_2^2$, and by
\eqref{eq:Bj_pair_anova},
\[
\Delta_j^*
=
A_j+B_j
\ge
q_uq_v I_{uv}
-\frac{2e^{-(m-1)}}{1-e^{-(m-1)}}\frac{q_j}{1-r}\sum_{l\in S} q_l\|f_l\|_2^2
.
\]

\medskip
\noindent
\emph{Case 3: $j\in S^c$.}
Since $j\notin\{u,v\}$, we have $A_j\le 0$ and $B_j\le 0$. Hence $\Delta_j^*\le 0$. Moreover,
combining the lower bound for $A_j$ with \eqref{eq:Bj_nonpair_anova} 
yields
\[
\Delta_j^*
=
A_j+B_j
\ge
-\frac{2e^{-(m-1)}}{(1-e^{-(m-1)})(1-r)}\,
q_j
\left(
\sum_{l\in S} q_l\|f_l\|_2^2 + q_uq_v I_{uv}
\right).
\]

This proves all claims.
\end{proof}

\section{Detailed Theory and Proofs for the Correlated Feature Setting}
\label{sec:correlateddetailtheory}

Consider a linear model with $M\geq 3$ and $S=\{1\}$:
\[
    Y=\beta \notationchange{X^{(1)}}+\varepsilon,
    \qquad
    \varepsilon\perp X,
    \qquad
    \bbE(\varepsilon)=0,
\]
where $\forall j\in [M]$, $\bbE(\notationchange{X^{(j)}})=0$, $\Var(\notationchange{X^{(j)}})=1$, and $(\notationchange{X^{(1)}},\notationchange{X^{(2)}})\perp \notationchange{X^{(3)}}\perp \notationchange{X^{(4)}}\perp\dots\perp \notationchange{X^{(M)}}$. $(\notationchange{X^{(1)}},\notationchange{X^{(2)}})$ is a correlated pair satisfying $\bbE(\notationchange{X^{(1)}}\mid \notationchange{X^{(2)}})=\rho \notationchange{X^{(2)}}$, which implies $\Corr(\notationchange{X^{(1)}},\notationchange{X^{(2)}})=\rho$. We consider an oracle version of Algorithm~\ref{algo:loco_adamp}, which we refer to as Algorithm~\ref{algo:loco_adamp} (oracle), where in each iteration $b$, we replace the LOCO-LOO feature importance estimate $\hat{\Delta}_j^{(b)}$ by the oracle feature importance score $\Delta_j^{\star(b)}$ in Definition~\ref{def:Delta_defs} where $\mu_F^\star(\notationchange{X^{(F)}}) = \bbE(Y|\notationchange{X^{(F)}}).$

\begin{prop}
\label{prop:correlated}
Consider the linear model setup described above. If Algorithm~\ref{algo:loco_adamp} (oracle) is run for $B$ iterations with squared error loss $\err(Y,\hat Y)=(Y-\hat Y)^2$, then the following statements hold.
\begin{enumerate}
    \item {\bf (Monotonicity)} $\forall\rho\in [-1,1]$, $j\geq 2$, $q_j^{(b)}\leq q_1^{(b)}$ for $b\geq 0$. If $|\rho|<1$, then $\Delta_2^{\star(b)}<\Delta_1^{\star(b)}$.

    \item {\bf (No false negative)} $\forall\rho\in [-1,1]$, if $m \geq 2 \vee (\log M+C_\delta)$ with $C_\delta>0$ sufficiently large depending only on $\delta$ and
    \begin{equation}
        \label{eq:ass_no_false_neg}
        c_0\leq (2-\frac{\delta}{1-e^{-m}})(1-\frac{\delta}{1-e^{\delta-m}})^2(\frac{2\delta}{1+\delta}m-2\delta)\beta^2 ,
    \end{equation}
    then $q_1^{(b)}\geq \big(\frac{1+\delta}{2\delta}q_1^{(b-1)}\big)\wedge \delta$. As a consequence, $S\subseteq \hat S$ whenever
    $B> (\log(\frac{1+\delta}{2\delta}))^{-1}\log\left(\frac{\delta M}{2m}\right)$.

    \item {\bf (No false positive)} $\hat S\subseteq S$ if the following holds:
    \begin{equation}\label{eq:rho_cond}
        \rho^2
        <
        \frac{c_0}{\beta^2}\frac{\frac{1}{2}\delta(M-2)-m+q_1^{(B)}}{M(m-q_1^{(B)}-\frac{1}{2}\delta)}-
        \frac{2\delta e^{-(m-1)}}{(1-e^{-m})(1-e^{-(m-1)})}.
    \end{equation}
\end{enumerate}
\end{prop}
Proposition~\ref{prop:correlated} shows that the true signal feature $\notationchange{X^{(1)}}$ remains the strongest feature despite its correlation with a noise feature. Even when $|\rho|=1$, $\notationchange{X^{(1)}}$ would still be selected provided that $c_0$ is appropriately chosen. The correlated noise feature $\notationchange{X^{(2)}}$ behaves like a weak signal feature: when $|\rho|<1$, its oracle importance is strictly smaller than that of $\notationchange{X^{(1)}}$, and it would not be selected when $\rho$ is sufficiently small as specified in \eqref{eq:rho_cond}, where the R.H.S. is an increasing function of $q_1^{(B)}$ and $B$ when \eqref{eq:ass_no_false_neg} holds.
We note that the sufficient condition we impose on $\rho$ could be much more conservative than what is required in practice: 
as shown in Figure~\ref{fig:corr_theory}, feature $2$ is not selected even when $\rho=0.5$ or $0.7$ when $M=100$, $m=12$.
Furthermore, our empirical results also shows that the adaptive sampling can be helpful for eliminating the correlated proxy: when $\rho=0.7$, the correlated noise feature initially has sampling probability around the selection threshold, but after $q_1^{(1)}$ increases to $0.8$, its sampling probability drops below the threshold in later iterations. More experiment results can be found in Section~\ref{sec:supp_correlated_theory_validation}, where similar patterns are observed.

\begin{proof}[Proof of Proposition~\ref{prop:correlated}]
We first note that in Algorithm~\ref{algo:loco_adamp} (oracle), we update the sampling probability via the following equation:
\begin{equation}
    \label{eq:prob_delta_formula}
    q_j^{(b)}
    =
    \frac{m(\widetilde{\Delta}_j^{(b)}\wedge t^{\star(b)})}{\sum_{l=1}^M (\widetilde{\Delta}_l^{(b)}\wedge t^{\star(b)})},
\end{equation}
where $\widetilde{\Delta}_j^{(b)}$ satisfies
\begin{equation}
    \label{eq:loco_importance_estimation_proxy}
    \widetilde{\Delta}_j^{(b)}
    =
    \Delta_j^{\star(b)}
    -
    \Delta_{\min}^{\star(b)}
    +
    \frac{c_0}{M},
    \qquad
    \Delta_{\min}^{\star(b)}
    =
    \min_{1\leq j\leq M}\Delta_j^{\star(b)},
\end{equation}
and $t^{\star(b)}$ is the threshold determined in Algorithm~\ref{algo:IndepSampWeight} when applied to $\{\widetilde{\Delta}_j^{(b)}\}_{j=1}^M$. In the following, we prove the three claims in order.
\paragraph{Monotonicity.}
We first show that, for all $b\geq 1$ and $j\geq3$,
\[
    \Delta_j^{\star(b)} \le \Delta_1^{\star(b)},
    \qquad
    q_j^{(b)}\leq q_1^{(b)}.
\]
We prove this by induction on $b$. The following lemma provides the exact
oracle LOCO importance formula used throughout the proof.

\begin{lem}\label{lem:Delta_oracle_exact}
Under the model described in Section~\ref{sec:correlated_theory}, at iteration $b$, let $q_j=q_j^{(b-1)}$,
$r^{(b-1)}=\prod_{\ell=1}^M(1-q_\ell)$, and
$r_{-j}^{(b-1)}=\prod_{\ell\neq j}(1-q_\ell)$. Define
\[
    p_1=\frac{q_1}{1-r^{(b-1)}},
    \qquad
    \pi_2=\frac{(1-q_1)q_2}{1-r^{(b-1)}},
    \qquad
    \pi_2^{(-1)}=\frac{q_2}{1-r_{-1}^{(b-1)}},
    \qquad
    p_1^{(-2)}=\frac{q_1}{1-r_{-2}^{(b-1)}}.
\]
Then
\[
    \Delta_1^{\star(b)}
    =
    \beta^2(2p_1-p_1^2)
    \left\{
        1-\rho^2
        \left(2\pi_2^{(-1)}-(\pi_2^{(-1)})^2\right)
    \right\},
\]
\[
    \Delta_2^{\star(b)}
    =
    \beta^2
    \left[
        (1-p_1^{(-2)})^2-(1-p_1)^2
        +
        \rho^2\{(1-p_1)^2(2\pi_2^{(-1)}-(\pi_2^{(-1)})^2)\}
    \right],
\]
and, for $j \geq 3$,
\[
    \Delta_j^{\star(b)}
    =
    \beta^2
    \{
        (1-p_1^{(-j)})^2-(1-p_1)^2
        +
        \rho^2\left[\pi_2^{(-j)}(\pi_2^{(-j)}-2)-\pi_2(\pi_2-2)+2(p_1^{(-j)}\pi_2^{(-j)}-p_1\pi_2)\right]
    \}.
\]
\end{lem}

For the base case, we show that
$\Delta_1^{\star(1)} \ge \Delta_j^{\star(1)}$. By Lemma~\ref{lem:Delta_oracle_exact}
and the identity $\pi_2=(1-p_1)\pi_2^{(-1)}$, we have
\begin{equation}
\label{eq:delta1_minus_deltaj}
\begin{split}
    \Delta_1^{\star(b)}-\Delta_j^{\star(b)}
    &=\beta^2(
    (1-\pi_2^{(-1)})^2
    -
    (1-p_1^{(-j)}-\pi_2^{(-j)})^2 \\
    &\quad
    +
    (1-\rho^2)
    \left\{
        \pi_2^{(-1)}(2-\pi_2^{(-1)})
        -
        \pi_2^{(-j)}
        \left(2(1-p_1^{(-j)})-\pi_2^{(-j)}\right)
    \right\}).
\end{split}    
\end{equation}
We now show that the two terms in \eqref{eq:delta1_minus_deltaj} are nonnegative whenever
$q_j^{(b-1)}\leq q_1^{(b-1)}$. First,
\[
    (1-\pi_2^{(-1)})^2-(1-p_1^{(-j)}-\pi_2^{(-j)})^2\geq 0
\]
since $p_1^{(-j)}+\pi_2^{(-j)}\leq 1$ and $p_1^{(-j)}+\pi_2^{(-j)}\geq \pi_2^{(-1)}$. Second, 
\[
    \pi_2^{(-j)}
    \left(2(1-p_1^{(-j)})-\pi_2^{(-j)}\right)
    \leq
    \pi_2^{(-j)}(2-\pi_2^{(-j)})
    \leq
    \pi_2^{(-1)}(2-\pi_2^{(-1)})
\]
since 
$0<\pi_2^{(-j)}\leq \pi_2^{(-1)}<1$ from $q_j^{(b-1)}\leq q_1^{(b-1)}$ and $x(2-x)$ is increasing on $[0,1]$. Combining the preceding inequalities yields $\Delta_1^{\star(b)}-\Delta_j^{\star(b)}\geq 0$ whenever $q_j^{(b-1)}\leq q_1^{(b-1)}$.

We now continue to complete the induction. For the base case, the initialization satisfies
$q_j^{(0)}=q_1^{(0)}=m/M$. Applying the implication above with $b=1$ gives $\Delta_j^{\star(1)}\leq \Delta_1^{\star(1)}$.
We also have $\widetilde{\Delta}_j^{(1)} \leq \widetilde{\Delta}_1^{(1)}$ from ~\eqref{eq:loco_importance_estimation_proxy}.
The sampling update in Algorithm~\ref{algo:IndepSampWeight} satisfies
\begin{equation}
    \label{eq:prob_delta_formula_first_iteration}
    q_j^{(1)}
    =
    \frac{
        m(\widetilde{\Delta}_j^{(1)}\wedge t^{\star(1)})
    }{
        \sum_{l=1}^M
        (\widetilde{\Delta}_l^{(1)}\wedge t^{\star(1)})
    }.
\end{equation}
Thus, $q_j^{(1)}\leq q_1^{(1)}$. This proves the base case.

Now suppose that $q_j^{(k)}\leq q_1^{(k)}$ for some $k\geq1$. Applying \eqref{eq:delta1_minus_deltaj} with $b=k+1$ gives
$\Delta_j^{\star(k+1)}\leq\Delta_1^{\star(k+1)}$.
Therefore, $\widetilde{\Delta}_j^{(k+1)} \leq \widetilde{\Delta}_1^{(k+1)}$. Using the sampling update again,
\[
    q_j^{(k+1)}
    =
    \frac{
        m(\widetilde{\Delta}_j^{(k+1)}\wedge t^{\star(k+1)})
    }{
        \sum_{l=1}^M
        (\widetilde{\Delta}_l^{(k+1)}\wedge t^{\star(k+1)})
    }
    \leq
    q_1^{(k+1)}.
\]
By induction, for all $b\geq1$ and $j\geq3$,
\[
    \Delta_j^{\star(b)}
    \leq
    \Delta_1^{\star(b)},
    \qquad
    q_j^{(b)}
    \leq
    q_1^{(b)}.
\]

It remains to show that
\[
    \Delta_2^{\star(b)} < \Delta_1^{\star(b)}
    \quad \text{when } |\rho|<1,
    \qquad
    \Delta_1^{\star(b)} = \Delta_2^{\star(b)}
    \quad \text{when } |\rho|=1,
\]
for all $b\geq1$. We again proceed by induction. For the base case, it suffices to compare $\Delta_1^{\star(1)}$ and
$\Delta_2^{\star(1)}$. We have
\begin{equation}
    \label{eq:delta1_minus_delta2}
    \Delta_1^{\star(b)} - \Delta_2^{\star(b)}
    =
    \beta^2\left[(2p_1^{(-2)}-(p_1^{(-2)})^2)-\rho^2(2\pi_2^{(-1)}-(\pi_2^{(-1)})^2)\right].
\end{equation}
At initialization, $q_1^{(0)}=\cdots=q_M^{(0)}=m/M$. Applying
Lemma~\ref{lem:Delta_oracle_exact} with $b=1$ gives
$p_1^{(-2)}=\pi_2^{(-1)}$. Therefore
$\Delta_1^{\star(1)}=\Delta_2^{\star(1)}$ when $|\rho|=1$, while when $|\rho|<1$,
$\Delta_1^{\star(1)}-\Delta_2^{\star(1)}>0$ since
$0<\pi_2^{(-1)}<1$. This proves the base case.

Now suppose that
$\Delta_2^{\star(k)}<\Delta_1^{\star(k)}$ when $|\rho|<1$, and
$\Delta_1^{\star(k)}=\Delta_2^{\star(k)}$ when $|\rho|=1$. We show that the
same claim holds at iteration $k+1$. When $|\rho|=1$, the induction hypothesis, together with \eqref{eq:loco_importance_estimation_proxy} and \eqref{eq:prob_delta_formula}, implies $q_1^{(k)}=q_2^{(k)}$. Applying Lemma~\ref{lem:Delta_oracle_exact} with $b=k+1$, we have $\pi_2^{(-1)}=p_1^{(-2)}$, and hence \eqref{eq:delta1_minus_delta2} gives
$\Delta_1^{\star(k+1)}=\Delta_2^{\star(k+1)}$. When $|\rho|<1$, the induction hypothesis, together with
\eqref{eq:loco_importance_estimation_proxy} and \eqref{eq:prob_delta_formula},
implies $q_1^{(k)}\geq q_2^{(k)}$. Applying
Lemma~\ref{lem:Delta_oracle_exact} with $b=k+1$, we have
$\pi_2^{(-1)}\leq p_1^{(-2)}$. Since $0<\pi_2^{(-1)}\leq p_1^{(-2)}<1$,
\eqref{eq:delta1_minus_delta2} yields
$\Delta_1^{\star(k+1)}-\Delta_2^{\star(k+1)}>0$.

By induction,
\[
    \Delta_1^{\star(b)} > \Delta_2^{\star(b)}
    \quad \text{for } b\geq1 \text{ when } |\rho|<1,
\]
and
\[
    \Delta_1^{\star(b)} = \Delta_2^{\star(b)}
    \quad \text{for } b\geq1 \text{ when } |\rho|=1.
\]
Combining $\Delta_1^{\star(b)}\geq\Delta_2^{\star(b)}$ with
\eqref{eq:prob_delta_formula} gives $q_1^{(b)}\geq q_2^{(b)}$ for all
$b\geq1$.

\paragraph{No false negatives.}
By symmetry of the independent noise variables and the common initialization,
the independent-noise oracle LOCO feature importances remain identical across
iterations. Denote their common value by $\Delta_N^{\star(b)}$. From the
monotonicity part, we have
$\Delta_1^{\star(b)}\geq \Delta_2^{\star(b)}$ and 
$\Delta_1^{\star(b)}\geq \Delta_N^{\star(b)}$. We now show that $\sum_{\ell=1}^M \widetilde{\Delta}_\ell^{(b)}$ is bounded by
$2\widetilde{\Delta}_1^{(b)}+c_0$.

First, if $\Delta_N^{\star(b)}=\Delta_{\min}^{\star(b)}$, then
$\widetilde{\Delta}_3^{(b)} = \cdots = \widetilde{\Delta}_M^{(b)}=c_0/M$.

Therefore,
\[
\sum_{\ell=1}^M \widetilde{\Delta}_\ell^{(b)}
=
\widetilde{\Delta}_1^{(b)}
+
\widetilde{\Delta}_2^{(b)}
+
\frac{M-2}{M}c_0  \
\leq
2\widetilde{\Delta}_1^{(b)}+c_0,
\]
where the inequality uses
$\widetilde{\Delta}_2^{(b)}\leq\widetilde{\Delta}_1^{(b)}$. It remains to consider the case
$\Delta_2^{\star(b)}=\Delta_{\min}^{\star(b)}$. In this case,
\[
\widetilde{\Delta}_2^{(b)}=\frac{c_0}{M},
\qquad
\widetilde{\Delta}_j^{(b)}
=
\Delta_N^{\star(b)}-\Delta_2^{\star(b)}+\frac{c_0}{M},
\quad j\geq3.
\]

The following lemma gives the
required oracle importance bounds.

\begin{lem}\label{lem:Delta_oracle_bound}
Under the linear model described in Section~\ref{sec:correlated_theory}, \[
    -
    \frac{
        2\beta^2 q_1q_2  e^{-(m-1)}
    }{
        (1-e^{-m})(1-e^{-(m-1)})
    }\leq \Delta_2^{\star(b)}
    \leq
    \beta^2\rho^2 (1-q_1) \leq \beta^2\rho^2.
\]
For any independent noise feature $j\geq3$,
\[
    -\frac{
        2\beta^2 q_j e^{-(m-1)}
    }{
        (1-e^{-m})(1-e^{-(m-1)})
    }\leq \Delta_j^{\star(b)}\leq0.
\]
\end{lem}

Since $\Delta_N^{\star(b)}\leq0$, Lemma~\ref{lem:Delta_oracle_bound} gives
\[
(M-2)({\Delta_N^{\star(b)}-\Delta_2^{\star(b)}})
\leq
-(M-2)\Delta_2^{\star(b)} \leq
\frac{
2\beta^2 q_1^{(b-1)}\delta(M-2)e^{-(m-1)}
}{
(1-e^{-m})(1-e^{-(m-1)})
}.
\]
Under the assumption $m\geq \log M+C_\delta$ with $C_\delta>0$ sufficiently
large depending only on $\delta$,
\[
\frac{
2\delta(M-2)e^{-(m-1)}
}{
(1-e^{-m})(1-e^{-(m-1)})
}
\leq
\left(2-\frac{\delta}{1-e^{-m}}\right)
\left(1-\frac{\delta}{1-e^{\delta-m}}\right)^2.
\]
Combining this bound with the lower bound for the signal importance yields
\[
(M-2)({\Delta_N^{\star(b)}-\Delta_2^{\star(b)}})
\leq
\beta^2 q_1^{(b-1)}
\left(2-\frac{\delta}{1-e^{-m}}\right)
\left(1-\frac{\delta}{1-e^{\delta-m}}\right)^2
\leq
\widetilde{\Delta}_1^{(b)}.
\]
We show the last inequality here.
By \eqref{eq:loco_importance_estimation_proxy} and Lemma~\ref{lem:Delta_oracle_bound},
$\Delta_{\min}^{\star(b)}\leq0$. Hence,
\[
\begin{split}
\widetilde{\Delta}_1^{(b)}
&\geq
\Delta_1^{\star(b)} \\
&\geq
\beta^2\frac{q_1^{(b-1)}}{1-r^{(b-1)}}
\left(2-\frac{q_1^{(b-1)}}{1-r^{(b-1)}}\right)
\left(1-\frac{q_2^{(b-1)}}{1-r_{-1}^{(b-1)}}\right)^2 \\
&\geq
\beta^2\frac{q_1^{(b-1)}}{1-r^{(b-1)}}
\left(2-\frac{\delta}{1-e^{-m}}\right)
\left(1-\frac{\delta}{1-e^{\delta-m}}\right)^2 .
\end{split}
\]
Here the second inequality follows from Lemma~\ref{lem:Delta_oracle_exact} and
$\rho^2\leq1$, while the third inequality uses
\[
    r_{-1}
    =
    \prod_{\ell\neq1}(1-q_\ell)
    \leq
    \exp\{-(m-q_1)\}
    \leq
    e^{\delta-m},
    \qquad
    r \le e^{-m}.
\]

Hence,
\[
\begin{split}
\widetilde{\Delta}_1^{(b)}
+
\widetilde{\Delta}_2^{(b)}
+
\sum_{j=3}^M \widetilde{\Delta}_j^{(b)}
&=
\widetilde{\Delta}_1^{(b)}
+
\frac{M-1}{M}c_0
+
(M-2)({\Delta_N^{\star(b)}-\Delta_2^{\star(b)}}) \\
&\leq
2\widetilde{\Delta}_1^{(b)}+c_0.
\end{split}
\]
Thus, in both cases, for every update before the signal feature reaches the
cap, we have
\begin{equation}
\label{eq:q_1_bound}
q_1^{(b)}
=
\frac{
m\widetilde{\Delta}_1^{(b)}
}{
\sum_{\ell=1}^M \widetilde{\Delta}_\ell^{(b)}
}
\geq
\frac{
m\widetilde{\Delta}_1^{(b)}
}{
2\widetilde{\Delta}_1^{(b)}+c_0
}
=
\frac{m}{2+c_0/\widetilde{\Delta}_1^{(b)}}.
\end{equation}

Therefore, 
\begin{equation}
    \label{eq:signal_prob_growth}
    q_1^{(b)}
    \geq
    \frac{m q_1^{(b-1)}}{2\delta+\frac{c_0}{\beta^2(2-\frac{\delta}{1-e^{-m}})(1-\frac{\delta}{1-e^{\delta-m}})^2}}
    \geq
    \Big(\frac{1+\delta}{2\delta}q_1^{(b-1)}\Big)\wedge \delta,
\end{equation}
where the second inequality follows from the assumptions
\[
    \beta^2 \geq
    \frac{c_0}{(2-\frac{\delta}{1-e^{-m}})(1-\frac{\delta}{1-e^{\delta-m}})^2(\frac{2\delta}{1+\delta}m-2\delta)}
\]
and $m\geq 2 >1+\delta$.

Iterating \eqref{eq:signal_prob_growth} gives
\[
q_1^{(B)}\geq
\big((\frac{1+\delta}{2\delta})^B q_1^{(0)}\big)\wedge \delta
=
\Big((\frac{1+\delta}{2\delta})^B \frac{m}{M}\Big)\wedge \delta,
\]
where $q_1^{(0)}=m/M$. Hence, if
$B> (\log(\frac{1+\delta}{2\delta}))^{-1}\log\left(\frac{\delta M}{2m}\right)$, then
$q_1^{(B)}> \frac{\delta M}{2m}\frac{m}{M}=\frac{1}{2}\delta$.

\paragraph{No false positives.}
It suffices to show that $q_j^{(B)} \le \frac{1}{2}\delta$ for all $j\ge3$. We start by showing $q_2^{(B)} \le \frac{1}{2}\delta$

Let $\omega_j^{(B)}=\widetilde{\Delta}_j^{(B)}\wedge t^{\star(B)}$. Since
\[
    q_1^{(B)}
    =
    \frac{m\omega_1^{(B)}}{\sum_{l=1}^M \omega_l^{(B)}},
\]
we have
\[
    \omega_1^{(B)}
    =
    \frac{q_1^{(B)}}{m-q_1^{(B)}}
    \sum_{l=2}^M \omega_l^{(B)}.
\]
Hence,
\begin{equation}
\label{eq:q_2_bound_in_delta}
q_2^{(B)}
=
\frac{m\omega_2^{(B)}}{(1+\frac{q_1^{(B)}}{m-q_1^{(B)}})(\omega_2^{(B)}+\sum_{l=3}^M \omega_l^{(B)})}
\le
\frac{m-q_1^{(B)}}{1+\frac{c_0(M-2)}{M\omega_2^{(B)}}}
\leq
\frac{m-q_1^{(B)}}{1+\frac{c_0(M-2)}{M\widetilde{\Delta}_2^{(B)}}},
\end{equation}
where the first inequality uses $\sum_{l=3}^M \omega_l^{(B)} \ge \frac{M-2}{M}c_0$ and the second inequality uses $\omega_2^{(B)}\leq\widetilde{\Delta}_2^{(B)}$.

Therefore, by \eqref{eq:q_2_bound_in_delta} and
\eqref{eq:loco_importance_estimation_proxy}, it suffices to have
\[
\Delta_2^{\star(B)}
\leq
\frac{(M-2)c_0\frac{1}{2}\delta}{M(m-q_1^{(B)}-\frac{1}{2}\delta)}
+
\Delta_{\min}^{\star(B)}
-
\frac{c_0}{M}.
\]
Applying Lemma~\ref{lem:Delta_oracle_bound} to bound
$\Delta_2^{\star(B)}$ and $\Delta_{\min}^{\star(B)}$, it suffices that
\[
\rho^2 < \frac{c_0}{\beta^2}\frac{\frac{1}{2}\delta(M-1)-m+q_1^{(B)}}{M(m-q_1^{(B)}-\frac{1}{2}\delta)}-
\frac{2\delta e^{-(m-1)}}{(1-e^{-m})(1-e^{-(m-1)})}.
\]
Under the stated assumption, $q_2^{(B)} \le \frac{1}{2}\delta$. 
\\
Now we show $q_j^{(B)} \le \frac{1}{2}\delta$ for all $j\ge3$. By symmetry and $\sum_{l=1}^Mq_l^{(B)} = m$, we have 
\[
q_j^{(B)}=\frac{m-q_1^{(B)}-q_2^{(B)}}{M-2} \leq \frac{m-q_1^{(B)}}{M-2} \le \frac{1}{2}\delta
\] 
for all $j \ge 3$, where the last inequality follows from \eqref{eq:rho_cond}: since $\rho^2\ge0$ and $m-q_1^{(B)}-\frac{1}{2}\delta>0$, the right-hand side of \eqref{eq:rho_cond} must be positive, which implies $\frac{1}{2}\delta(M-2)-m+q_1^{(B)}>0$. Hence $\hat S\subseteq S$. This completes the proof.

\end{proof}

\begin{proof}[Proof of Lemma~\ref{lem:Delta_oracle_exact}]
By definition,
\[
    \mu_F^\star(\notationchange{X^{(F)}})=\bbE(\beta \notationchange{X^{(1)}}\mid \notationchange{X^{(F)}}).
\]
Under the stated assumptions,
\[
    \mu_F^\star(\notationchange{X^{(F)}})
    =
    \begin{cases}
        \beta \notationchange{X^{(1)}}, & 1\in F,\\
        \beta\rho \notationchange{X^{(2)}}, & 1\notin F,\ 2\in F,\\
        0, & 1\notin F,\ 2\notin F.
    \end{cases}
\]
Indeed, if $1\in F$, then $\notationchange{X^{(1)}}$ is observed. If $1\notin F$ but $2\in F$, then
the only observed variable carrying information about $\notationchange{X^{(1)}}$ is $\notationchange{X^{(2)}}$, so
$\bbE(\notationchange{X^{(1)}}\mid \notationchange{X^{(F)}})=\bbE(\notationchange{X^{(1)}}\mid \notationchange{X^{(2)}})=\rho \notationchange{X^{(2)}}$. If neither $1$ nor $2$ is
included, then the variables in $F$ are independent of $\notationchange{X^{(1)}}$, and the
conditional mean is zero.

Define
\[
    p_1:=\bbP_{F\sim Q}(1\in F),
    \qquad
    \pi_2:=\bbP_{F\sim Q}(1\notin F,\ 2\in F),
\]
and, for each feature $j$,
\[
    p_1^{(-j)}
    :=
    \bbP_{F\sim Q}(1\in F\mid j\notin F),
    \qquad
    \pi_2^{(-j)}
    :=
    \bbP_{F\sim Q}(1\notin F,\ 2\in F\mid j\notin F).
\]
Then
\[
    \bbE_{F\sim Q}\mu_F^\star(\notationchange{X^{(F)}})
    =
    \beta p_1\notationchange{X^{(1)}}+\beta\rho\pi_2\notationchange{X^{(2)}},
\]
and
\[
    \bbE_{F\sim Q_{\backslash j}}\mu_F^\star(\notationchange{X^{(F)}})
    =
    \beta p_1^{(-j)}\notationchange{X^{(1)}}+\beta\rho\pi_2^{(-j)}\notationchange{X^{(2)}}.
\]
Since $Q$ is obtained by independent Bernoulli sampling conditioned on
$F\neq\emptyset$,
\[
    p_1=\frac{q_1}{1-r},
    \qquad
    \pi_2=\frac{(1-q_1)q_2}{1-r}.
\]

For a generic pair $(t_1,t_2)$, define
\[
    H(t_1,t_2)
    :=
    (1-t_1)^2+\rho^2t_2^2-2\rho^2(1-t_1)t_2.
\]
Then
\[
\begin{split}
    \bbE\{Y-\beta t_1\notationchange{X^{(1)}}-\beta\rho t_2\notationchange{X^{(2)}}\}^2
    &=
    \sigma_\varepsilon^2
    +
    \beta^2\bbE\{(1-t_1)\notationchange{X^{(1)}}-\rho t_2\notationchange{X^{(2)}}\}^2  \\
    &=
    \sigma_\varepsilon^2+\beta^2H(t_1,t_2),
\end{split}
\]
where we used $\Var(\notationchange{X^{(1)}})=\Var(\notationchange{X^{(2)}})=1$ and
$\operatorname{Cov}(\notationchange{X^{(1)}},\notationchange{X^{(2)}})=\rho$. Therefore,
\begin{equation}
\label{eq:Delta_one_proxy_H_form_complete}
    \Delta_j^\star
    =
    \beta^2
    \left\{
        H(p_1^{(-j)},\pi_2^{(-j)})
        -
        H(p_1,\pi_2)
    \right\}.
\end{equation}

We first consider $j=1$. In this case,
\[
    p_1^{(-1)}=0,
    \qquad
    \pi_2^{(-1)}=\frac{q_2}{1-r_{-1}}.
\]
Moreover,
\[
    1-p_1
    =
    1-\frac{q_1}{1-r}
    =
    \frac{1-r-q_1}{1-r}
    =
    \frac{(1-q_1)(1-r_{-1})}{1-r},
\]
and hence
\[
    \pi_2
    =
    \frac{(1-q_1)q_2}{1-r}
    =
    (1-p_1)\pi_2^{(-1)}.
\]
Thus
\[
    H(p_1,\pi_2)
    =
    (1-p_1)^2H(0,\pi_2^{(-1)}).
\]
By \eqref{eq:Delta_one_proxy_H_form_complete},
\[
\begin{split}
    \Delta_1^\star
    &=
    \beta^2
    \{1-(1-p_1)^2\}
    H(0,\pi_2^{(-1)}) \\
    &=
    \beta^2(2p_1-p_1^2)
    \left\{
        1-\rho^2\left(2\pi_2^{(-1)}-(\pi_2^{(-1)})^2\right)
    \right\}.
\end{split}
\]
This proves the exact expression for $\Delta_1^{\star(b)}$.

We next consider $j=2$. In this case,
\[
    p_1^{(-2)}=\frac{q_1}{1-r_{-2}},
    \qquad
    \pi_2^{(-2)}=0.
\]
Therefore,
\[
\begin{split}
    \Delta_2^\star
    &=
    \beta^2
    \left\{
        H(p_1^{(-2)},0)-H(p_1,\pi_2)
    \right\} \\
    &=
    \beta^2
    \left[
        (1-p_1^{(-2)})^2-(1-p_1)^2
        +
        \rho^2\{2(1-p_1)\pi_2-\pi_2^2\}
    \right] \\
    &=
    \beta^2
    \left[
        (1-p_1^{(-2)})^2-(1-p_1)^2
        +
        \rho^2\{(1-p_1)^2(2\pi_2^{(-1)}-(\pi_2^{(-1)})^2)\}
    \right].
\end{split}
\]
where the last equality uses $\pi_2=(1-p_1)\pi_2^{(-1)}$. This proves the exact
expression for $\Delta_2^{\star(b)}$.

By \eqref{eq:Delta_one_proxy_H_form_complete} and direct computation, for
$j\geq3$,
\[
    \Delta_j^{\star(b)}
    =
    \beta^2
    \{
        (1-p_1^{(-j)})^2-(1-p_1)^2
        +
        \rho^2\left[\pi_2^{(-j)}(\pi_2^{(-j)}-2)-\pi_2(\pi_2-2)+2(p_1^{(-j)}\pi_2^{(-j)}-p_1\pi_2)\right]
    \}.
\]
This proves the exact expression for $\Delta_j^{\star(b)}$ for $j\geq3$. The
proof is complete.
\end{proof}

\begin{proof}[Proof of Lemma~\ref{lem:Delta_oracle_bound}]
For $j=2$, since $1-r\geq1-r_{-2}$, we have $p_1^{(-2)}\geq p_1$, and hence
\[
    (1-p_1^{(-2)})^2-(1-p_1)^2\leq0.
\]
Thus
\[
    \Delta_2^\star
    \leq
    \beta^2\rho^2\{2(1-p_1)\pi_2-\pi_2^2\}
    \leq
    \beta^2\rho^2(1-p_1)^2
    \leq
    \beta^2\rho^2(1-p_1)
    \leq
    \beta^2\rho^2(1-q_1)
    \leq
    \beta^2\rho^2.
\]
For the lower bound, the proxy term is nonnegative, so
\[
\begin{split}
    \Delta_2^\star
    &\geq
    \beta^2\{(1-p_1^{(-2)})^2-(1-p_1)^2\} \\
    &=
    -\beta^2(p_1^{(-2)}-p_1)(2-p_1^{(-2)}-p_1) \\
    &\geq
    -2\beta^2(p_1^{(-2)}-p_1).
\end{split}
\]
Since
\[
    p_1^{(-2)}-p_1
    =
    \frac{q_1q_2r_{-2}}{(1-r)(1-r_{-2})},
\]
and
\[
    r_{-2}\leq e^{-(m-1)},
    \qquad
    r \leq e^{-m},
\]
we obtain
\[
    \Delta_2^\star
    \geq
    -
    \frac{
        2\beta^2 q_1q_2 e^{-(m-1)}
    }{
        (1-e^{-m})(1-e^{-(m-1)})
    }.
\]

Now consider any independent noise feature $j\geq3$. In this case,
\[
    p_1^{(-j)}=\frac{q_1}{1-r_{-j}},
    \qquad
    \pi_2^{(-j)}=\frac{(1-q_1)q_2}{1-r_{-j}}.
\]
Since $r=(1-q_j)r_{-j}$,
\[
    1-r=(1-r_{-j})+q_jr_{-j}.
\]
Therefore,
\[
    p_1^{(-j)}=\lambda p_1,
    \qquad
    \pi_2^{(-j)}=\lambda\pi_2,
    \qquad
    \lambda:=\frac{1-r}{1-r_{-j}}\geq1.
\]
Also, $p_1^{(-j)}+\pi_2^{(-j)}\leq1$, so
$\lambda(p_1+\pi_2)\leq1$. Direct differentiation gives
\[
\begin{split}
    \frac{d}{d\lambda}H(\lambda p_1,\lambda\pi_2)
    =
    -2\left[
        \{1-\lambda(p_1+\pi_2)\}(p_1+\rho^2\pi_2)
        +
        (1-\rho^2)\lambda p_1\pi_2
    \right]
    \leq0.
\end{split}
\]
Hence
\[
    H(p_1^{(-j)},\pi_2^{(-j)})
    =
    H(\lambda p_1,\lambda\pi_2)
    \leq
    H(p_1,\pi_2),
\]
which implies
\[
    \Delta_j^{\star(b)}
    \leq0
\]
for all $j\ge3$.

It remains to bound $|\Delta_j^\star|$. For $t_1,t_2\in[0,1]$,
\[
    \frac{\partial H}{\partial t_1}(t_1,t_2)
    =
    2(t_1-1+\rho^2t_2),
\]
and therefore
\[
    \left|\frac{\partial H}{\partial t_1}(t_1,t_2)\right|\leq2.
\]
Similarly,
\[
    \frac{\partial H}{\partial t_2}(t_1,t_2)
    =
    2\rho^2(t_1+t_2-1),
\]
and therefore
\[
    \left|\frac{\partial H}{\partial t_2}(t_1,t_2)\right|\leq2\rho^2.
\]
By the mean value theorem,
\[
    |\Delta_j^\star|
    \leq
    2\beta^2 |p_1^{(-j)}-p_1|
    +
    2\beta^2\rho^2|\pi_2^{(-j)}-\pi_2|.
\]
Moreover,
\[
    |p_1^{(-j)}-p_1|
    =
    \frac{q_1q_jr_{-j}}{(1-r)(1-r_{-j})},
\]
and
\[
    |\pi_2^{(-j)}-\pi_2|
    =
    \frac{(1-q_1)q_2q_jr_{-j}}{(1-r)(1-r_{-j})}.
\]
Therefore,
\[
    |\Delta_j^\star|
    \leq
    \frac{
        2\beta^2 q_j r_{-j}
        \{q_1+\rho^2(1-q_1)q_2\}
    }{
        (1-r)(1-r_{-j})
    }
\]
for all $j\ge3$. Using
\[
    r\leq e^{-m},
    \qquad
    r_{-j}\leq e^{-(m-1)},
\]
and $q_1+\rho^2(1-q_1)q_2\leq1$ gives the stated independent-noise bound. The
proof is complete.
\end{proof}

\section{Additional Empirical Details and Results}
In the following sections, we provide additional empirical details and experimental results.

\section{Additional Details for Simulation Figures}
\label{sec:intro_fig_sim_details}

We provide additional details for the data generating process and LAMPS configurations to produce empirical results presented in the main text.

\paragraph{LOCO importance scores in Figure~\ref{fig:locomp_split}.}
The data generating process is the same as in Linear Model Experiments part of Section~\ref{sec:simulation_settings} with correlation level fixed at $\rho = 0.5$ and non-permuted Toeplitz design. We compare two LOCO-based feature importance scores. LOCO-MP is implemented using a minipatch ensemble with ordinary least squares without an intercept as the base learner, while LOCO-Split is implemented using Lasso with cross-validation as the base learner. Figure~\ref{fig:locomp_split} reports the resulting LOCO feature importance scores for signal and noise features.

\paragraph{Adaptive sampling probabilities in Figure~\ref{fig:prob_delta_vs_epoch}.}
Figure~\ref{fig:prob_delta_vs_epoch} uses the same linear data-generating model as Figure~\ref{fig:locomp_split}.
LAMPS is run with squared-error loss and ordinary least square without an intercept as the base learner. The minipatch ratios are fixed at $n/N=0.4$ and $m/M=0.12$, with $B=3$ iterations and ensemble size $K$ satisfying the coverage conditions in Section~\ref{sec:tuning}. We set the adaptivity parameter to $\delta=0.8$ and return the selected set $\widehat S=\{j:q_j^{(B)}>\delta/2\}$.

The figure displays the evolution of LOCO importance scores and sampling probabilities across the first three iterations, showing how the adaptive update progressively separates signal features from noise features.

\paragraph{One-signal correlated-noise setting in Figure~\ref{fig:corr_theory}.}
We generate data with $N=500$ observations and $M=100$ features. The predictors follow a zero-mean Gaussian distribution in which only the first two features are correlated:
$$
\Sigma_{11}=\Sigma_{22}=1,\qquad \Sigma_{12}=\Sigma_{21}=\rho,\qquad \rho\in\{0.5,0.7,0.9\},
$$
and all remaining features are independent standard normal noise variables. The response is generated as $Y=\notationchange{\bX_{\cdot,1}^2}$, so that $\notationchange{\bX_{\cdot,1}}$ is the only signal feature and $\notationchange{\bX_{\cdot,2}}$ is a correlated noise proxy. LAMPS uses decision tree regression as the base learner. LAMPS is run with squared-error loss, minipatch ratios $n/N=0.4$ and $m/M=0.12$, $B=10$ iterations, ensemble size $K$ satisfying the coverage conditions in Section~\ref{sec:tuning}, adaptivity parameter $\delta=0.8$, and final selection rule $\widehat S=\{j:q_j^{(B)}>\delta/2\}$. Figure~\ref{fig:corr_theory} illustrates how increasing correlation makes the noise proxy harder to separate from the true signal feature.

\paragraph{Feature selection performance in Figure~\ref{fig:linear-nonoracle}-\ref{fig:interaction-rate}.}

In Figure~\ref{fig:linear-nonoracle}, Lasso \citep{Tibshirani1996Lasso} is run on standardized predictors and tuned by five-fold cross-validation and by eBIC \citep{chen2008extended}. Stability Selection \citep{MeinshausenBuhlmann2010} uses a Lasso
base learner with selection-probability threshold $\pi_{\mathrm{thr}}=0.6$. CPSS follows \citet{ShahSamworth2013} with $B=50$ complementary pairs. Elastic Net \citep{ZouHastie2005ElasticNet} is tuned by five-fold cross-validation on the regularization strength parameter with $l_1$ penalty ratio held fixed at 0.5. All methods operate on standardized predictors and are evaluated under identical simulation replicates as LAMPS.

In Figure~\ref{fig:additive-nonoracle}, for LAMPS, the nonlinear base learner is taken to be either MARS or SpAM. When SpAM is used as the base learner for LAMPS, we fix its internal regularization parameter at $\lambda=0.01$ for computational efficiency. MARS is implemented via the \texttt{earth} package \citep{friedman1991multivariate, MilborrowEarth} with default settings. SpAM follows the \texttt{SAM} implementation \citep{Ravikumar2009SpAM, jiang2026package} with tuning parameter selected by cross-validation. Model-X Knockoffs are implemented using random-forest feature statistic with swap importances as the knockoff score with target FDR levels $q\in\{0.1,0.2,0.3\}$.

In Figure~\ref{fig:interaction-rate}, for LAMPS, when MARS is used as the base learner, the degree parameter is set to 2 to allow MARS to detect possible interaction. The degree parameter for MARS itself as a comparison method is also set to 2. All other competing procedures are the same as in the nonlinear additive study.

\clearpage

\section{Additional Correlated-Theory Validation Results}
\label{sec:supp_correlated_theory_validation}

We further examine the sampling-probability and LOCO feature-importance dynamics in the one-signal correlated-noise setting studied in Section~\ref{sec:correlated_theory}. We consider both linear and nonlinear signal structures. In each setting, we generate $N=500$ observations with $M=25, 100, 200$ features from a zero-mean Gaussian distribution, where only the first two features are correlated:
\[
\Sigma_{11}=\Sigma_{22}=1,\qquad
\Sigma_{12}=\Sigma_{21}=\rho,\qquad
\rho\in\{0.5,0.7,0.9\},
\]
and all remaining features are mutually independent standard normal variables. Thus, $\notationchange{\bX_{\cdot,1}}$ is the signal feature and $\notationchange{\bX_{\cdot,2}}$ is a correlated noise proxy. For the linear setting, the response is generated as $Y=\notationchange{\bX_{\cdot,1}}$ and LAMPS uses linear regression as the base learner; for the nonlinear setting, the response is generated as $Y=\notationchange{\bX_{\cdot,1}^2}$ and the base learner is decision tree regression. All other LAMPS configurations are the same as in Figure~\ref{fig:corr_theory}: squared-error loss, minipatch ratios $n/N=0.4$ and $m/M=0.12$, $B=10$ iterations, constant ensemble size $K$ satisfying the coverage conditions in Section~\ref{sec:tuning}, adaptivity parameter $\delta=0.8$, and final selection rule $\widehat S=\{j:q_j^{(B)}>\delta/2\}$.

\suppfigure
    {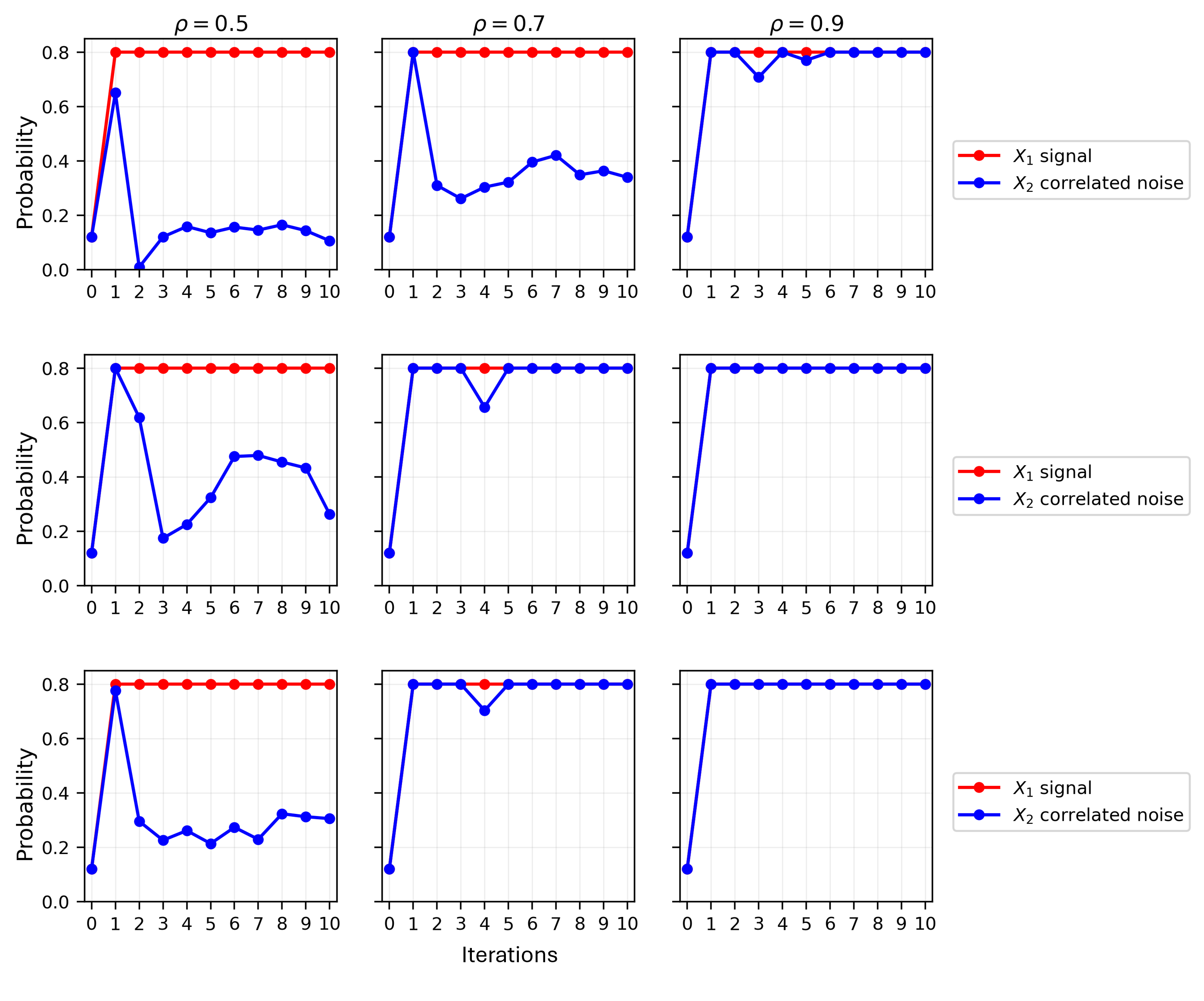}
    {fig:supp_probability_vs_epoch_linear}
    {Sampling probabilities across LAMPS iterations in the linear one-signal correlated-noise setting. Columns correspond to $\rho=0.5,0.7,0.9$, and rows correspond to $M=25,100,200$. The minipatch feature ratio is fixed at $m/M=0.12$, giving $m=4,12,24$, respectively. Red curves show the signal feature $\notationchange{\bX_{\cdot,1}}$, and blue curves show the correlated noise feature $\notationchange{\bX_{\cdot,2}}$. The response is generated as $Y=\notationchange{\bX_{\cdot,1}}$, and linear regression is used as the base learner, with $N=500$. The correlated noise feature receives increasingly large sampling probability as either $\rho$ or $m$ increases, making it more difficult to distinguish from the signal feature, consistent with Proposition~\ref{prop:correlated}.
}

\suppfigure
    {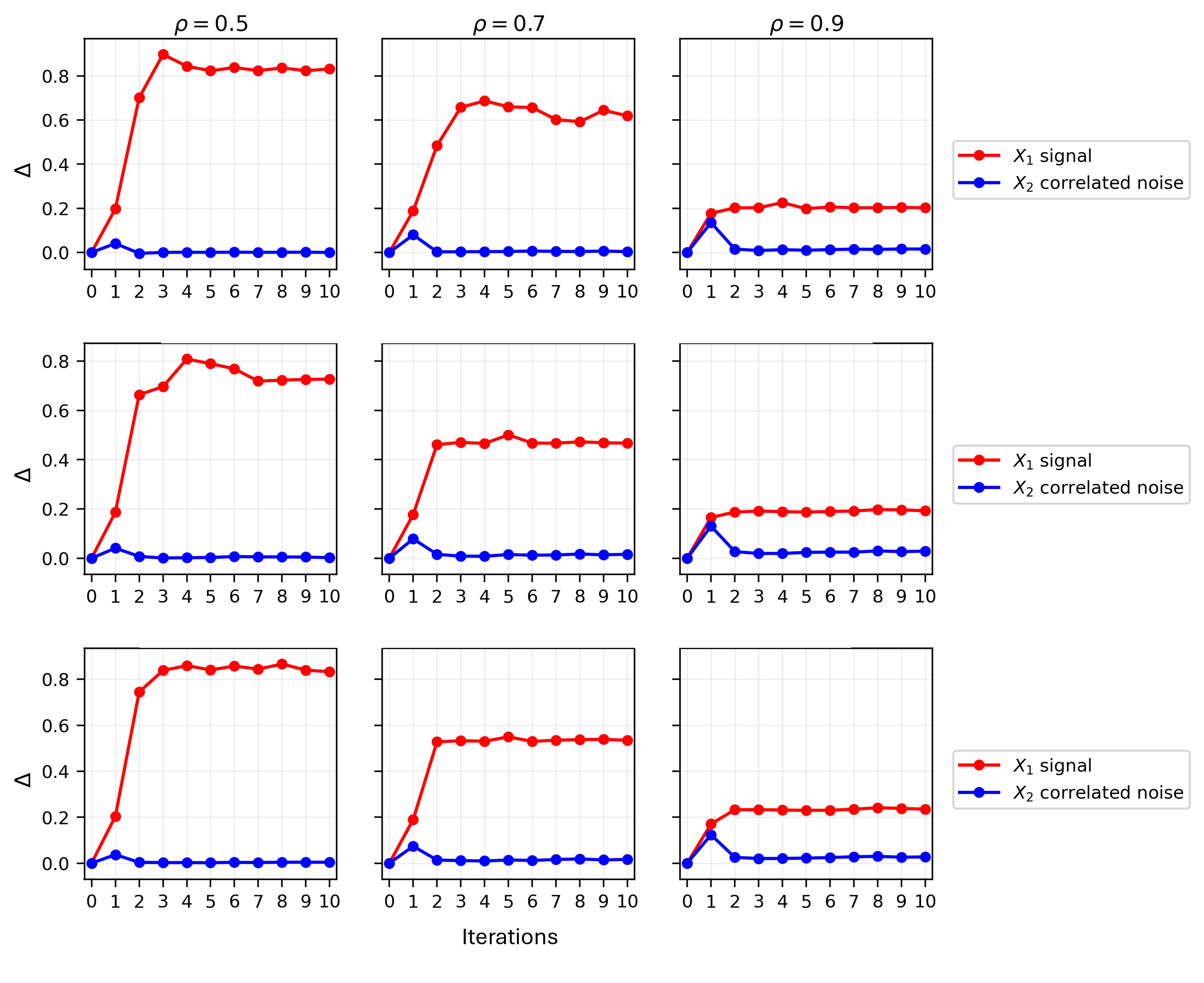}
    {fig:supp_delta_vs_epoch_linear}
    {LOCO feature importance across LAMPS iterations in the linear one-signal correlated-noise setting. Columns correspond to $\rho=0.5,0.7,0.9$, and rows correspond to $M=25,100,200$. The minipatch feature ratio is fixed at $m/M=0.12$, giving $m=4,12,24$, respectively. Red curves show the signal feature $\notationchange{\bX_{\cdot,1}}$, and blue curves show the correlated noise feature $\notationchange{\bX_{\cdot,2}}$. The response is generated as $Y=\notationchange{\bX_{\cdot,1}}$, and linear regression is used as the base learner.}

\suppfigure
    {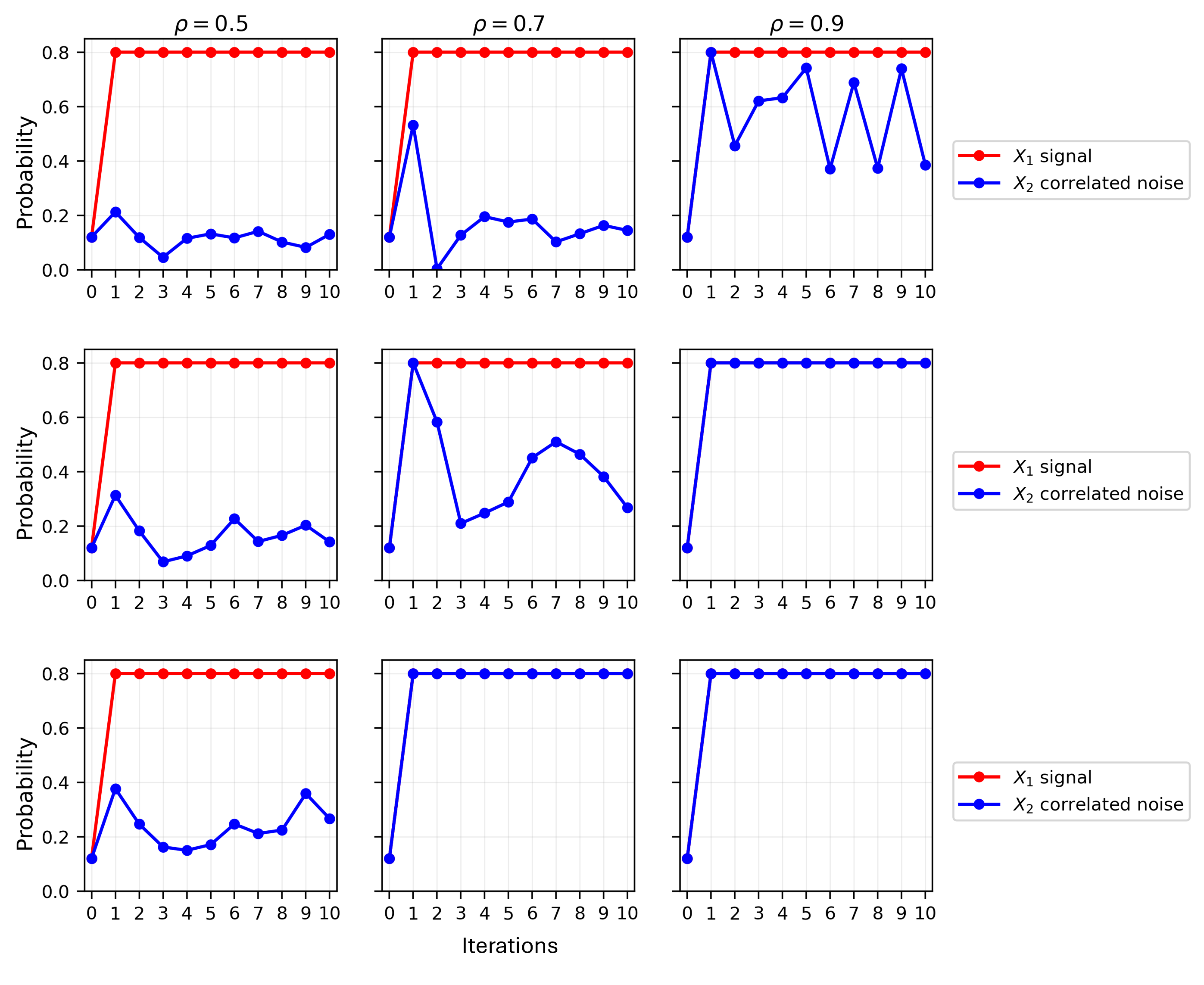}
    {fig:supp_probability_vs_epoch_nonlinear}
    {Sampling probabilities across LAMPS iterations in the nonlinear one-signal correlated-noise setting. Columns correspond to $\rho=0.5,0.7,0.9$, and rows correspond to $M=25,100,200$. The minipatch feature ratio is fixed at $m/M=0.12$, giving $m=4,12,24$, respectively. Red curves show the signal feature $\notationchange{\bX_{\cdot,1}}$, and blue curves show the correlated noise feature $\notationchange{\bX_{\cdot,2}}$. The response is generated as $Y=\notationchange{\bX_{\cdot,1}^2}$, and decision tree regression is used as the base learner. The signal feature remains distinguishable, while the correlated noise feature receives increasingly large sampling probability as $\rho$ increases, consistent with Proposition~\ref{prop:correlated}.}

\suppfigure
    {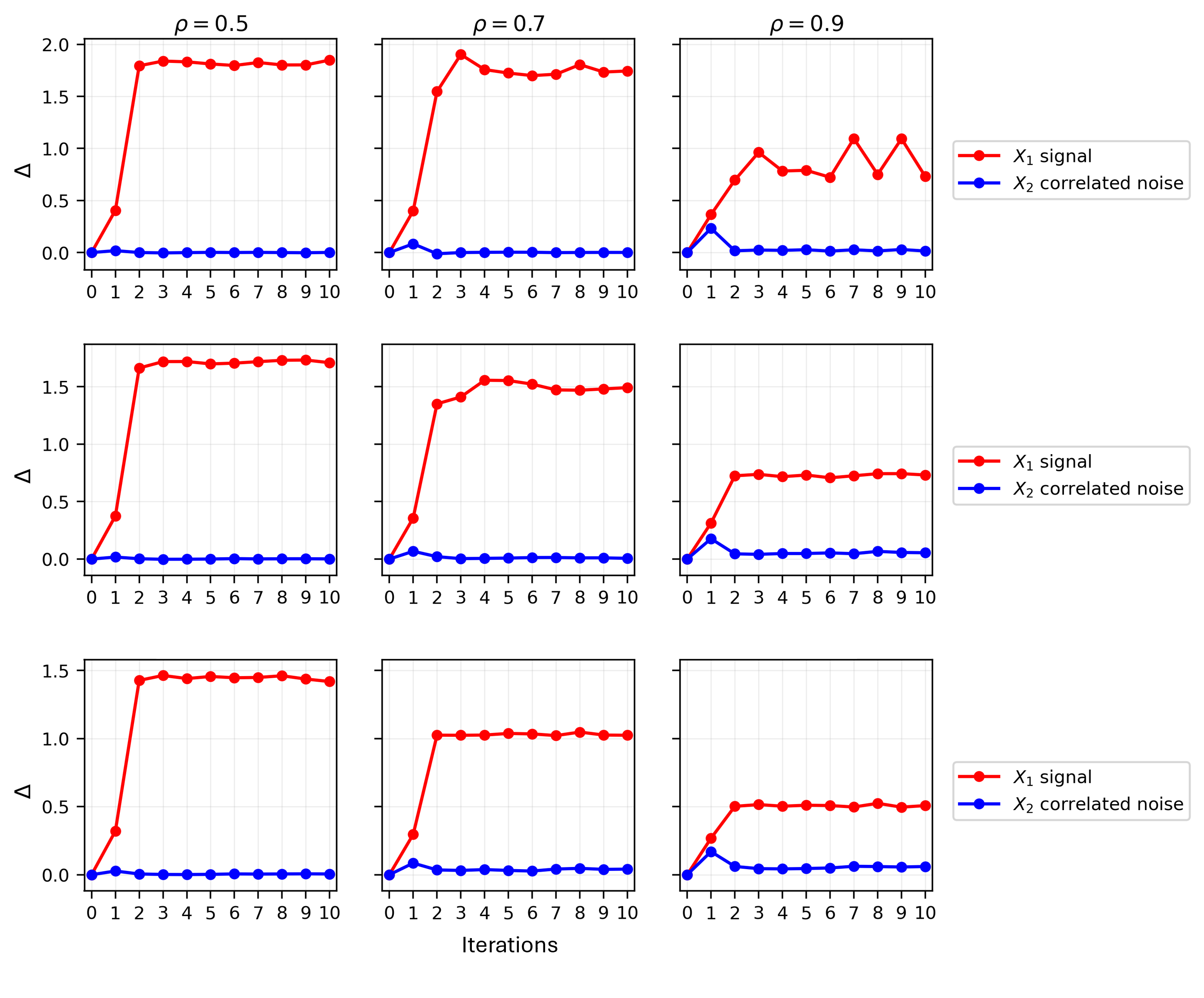}
    {fig:supp_delta_vs_epoch_nonlinear}
    {LOCO feature importance across LAMPS iterations in the nonlinear one-signal correlated-noise setting. Columns correspond to $\rho=0.5,0.7,0.9$, and rows correspond to $M=25,100,200$. The minipatch feature ratio is fixed at $m/M=0.12$, giving $m=4,12,24$, respectively. Red curves show the signal feature $\notationchange{\bX_{\cdot,1}}$, and blue curves show the correlated noise feature $\notationchange{\bX_{\cdot,2}}$. The response is generated as $Y=\notationchange{\bX_{\cdot,1}^2}$, and decision tree regression is used as the base learner.}

\clearpage

\section{Additional Precision and Recall and Oracle Sparsity Tuning Results}
\label{sec:supp_additional_precision_recall}

We present additional precision, recall and oracle sparsity tuning summaries for the linear and nonlinear additive experiments. These figures complement the F1-score results in the main paper. All simulation settings are the same as in Section~\ref{sec:simulation_settings}.

\subsection{Linear Models}

Figures~\ref{fig:supp_linear_oracle} report f1 score for the linear settings with oracle sparsity tuning methods included. For Lasso with oracle tuning, we choose regularization strength parameter so that Lasso gives the desired sparsity. For Elastic Net with oracle tuning, we choose regularization strength parameter and $l_1$ penalty ratio parameter so that Elastic Net gives the desired sparsity. Figures~\ref{fig:supp_linear_precision} and \ref{fig:supp_linear_recall} report precision and recall for the linear settings. The top row corresponds to the non-permuted Toeplitz design, and the bottom row corresponds to the permuted design. Columns represent correlation levels $\rho\in\{0,0.5,0.9\}$. LAMPS yields the best result in most scenarios; only giving lower recall than methods that yield denser selected sets with low precision. When restricting the sparsity level, Lasso and elastic net with oracle sparsity tuning give lower F1, precision, and recall than LAMPS. 

\suppfigure
    {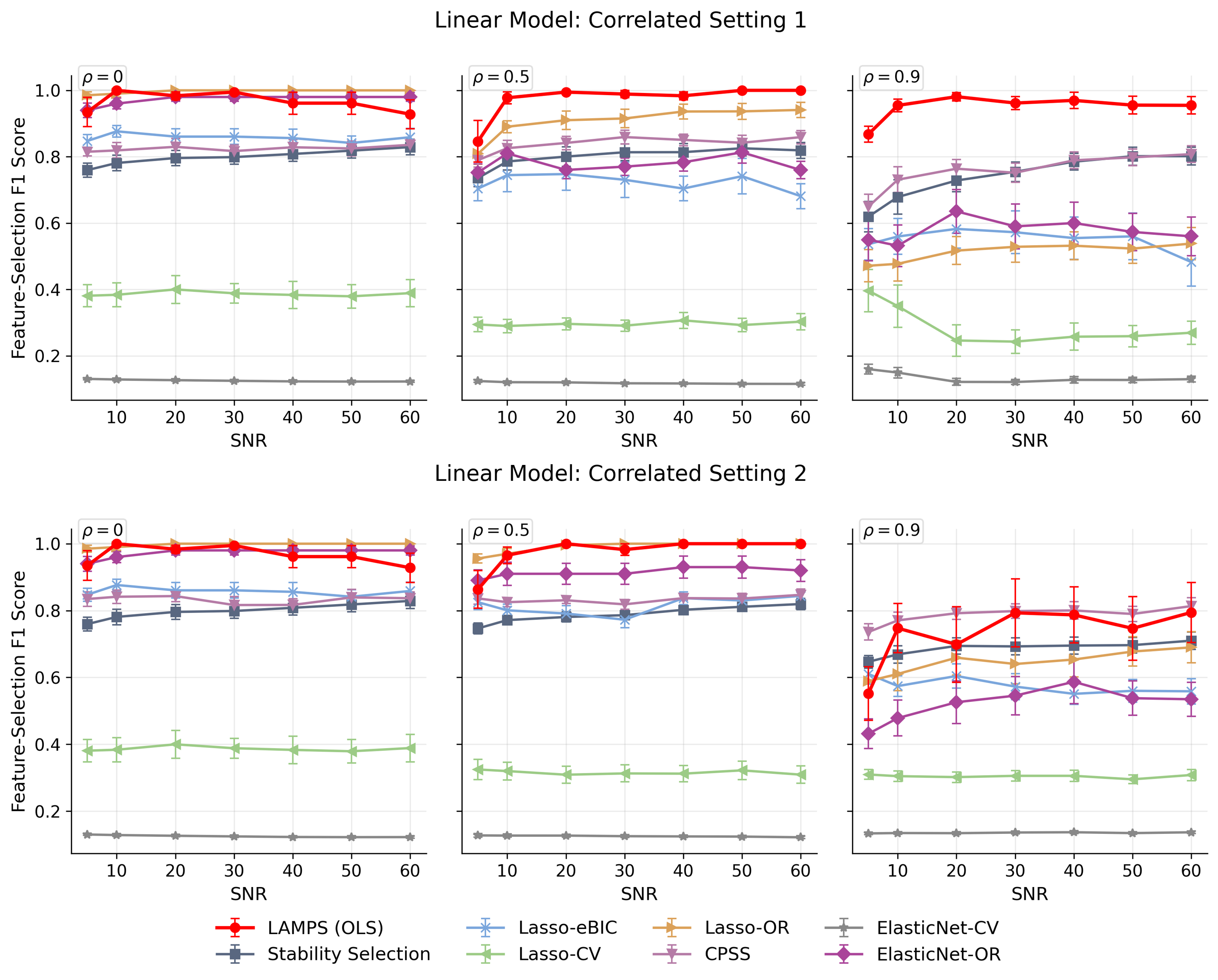}
    {fig:supp_linear_oracle}
    {F1 score in linear models across signal-to-noise ratios. Top row: Linear Setting~1 with non-permuted Toeplitz covariance; bottom row: Linear Setting~2 with permuted covariance. Columns correspond to $\rho\in\{0,0.5,0.9\}$. LAMPS is shown in red, and competing non-oracle and oracle baselines are shown in distinct colors and markers. Error bars represent $\pm$ one standard error over 10 Monte Carlo replicates.}
    
\suppfigure
    {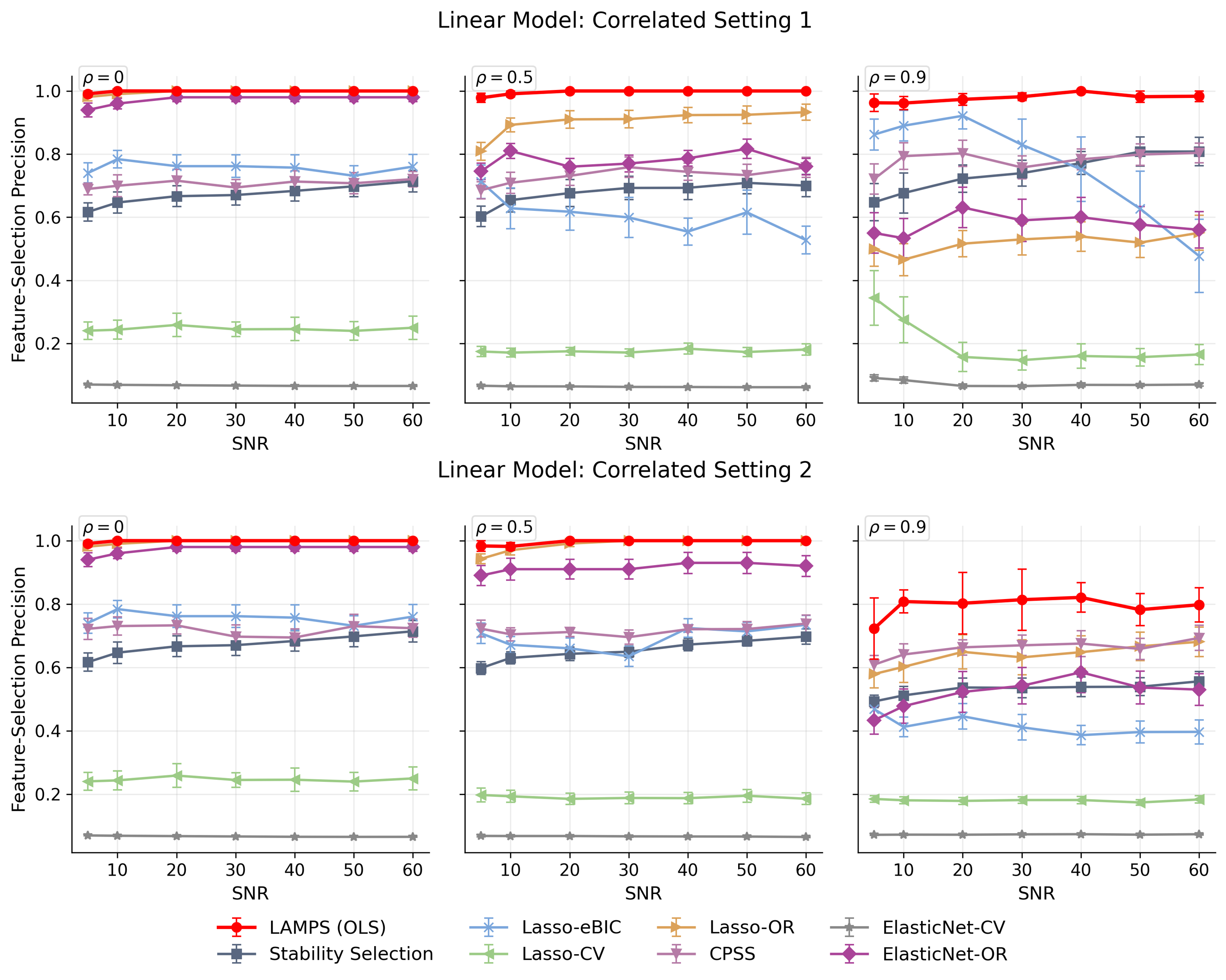}
    {fig:supp_linear_precision}
    {Precision in linear models across signal-to-noise ratios. Top row: Correlated Setting~1 with non-permuted Toeplitz covariance; bottom row: Correlated Setting~2 with permuted covariance. Columns correspond to $\rho\in\{0,0.5,0.9\}$. LAMPS is shown in red, and competing baselines are shown in distinct colors and markers. Error bars represent $\pm$ one standard error over 10 Monte Carlo replicates.}

\suppfigure
    {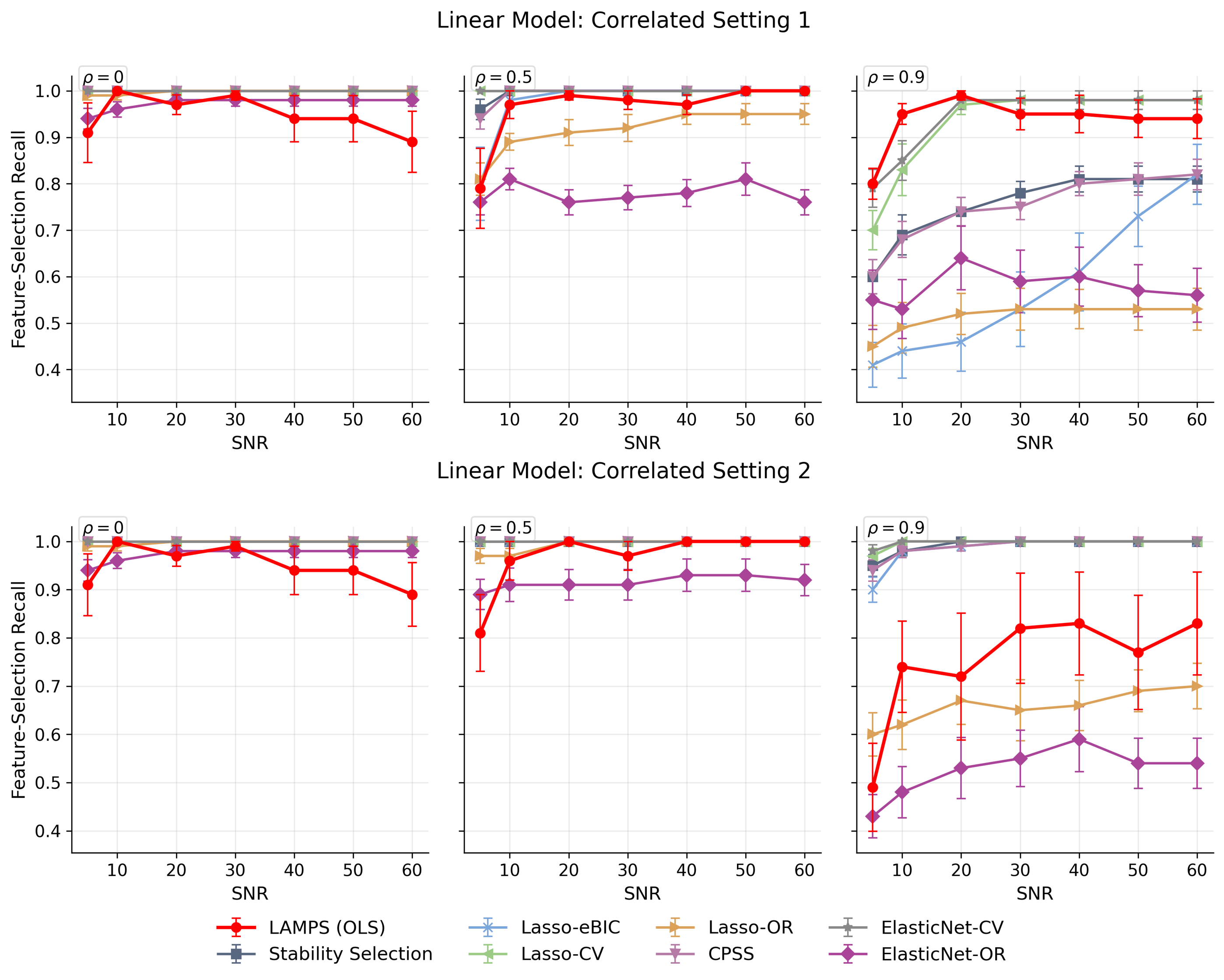}
    {fig:supp_linear_recall}
    {Recall in linear models across signal-to-noise ratios. Top row: Correlated Setting~1 with non-permuted Toeplitz covariance; bottom row: Correlated Setting~2 with permuted covariance. Columns correspond to $\rho\in\{0,0.5,0.9\}$. LAMPS is shown in red, and competing baselines are shown in distinct colors and markers. Error bars represent $\pm$ one standard error over 10 Monte Carlo replicates.}

\subsection{Nonlinear Additive Models}

Figures~\ref{fig:supp_nonlinear_additive_precision_nonoracle}--\ref{fig:supp_nonlinear_additive_f1_oracle} report additional results for nonlinear additive models. LAMPS yields the best precision result in all scenarios; only giving lower recall than SpAM. However, SpAM with cross-validation yields dense selection and lower precision. Thus, we separately look at the oracle sparsity tuning comparison. In oracle tuning comparison, we tune the regularization parameter in SpAM so that it gives the correct sparsity. LAMPS gives the best result in oracle sparsity tuning comparison.

\suppfigure
    {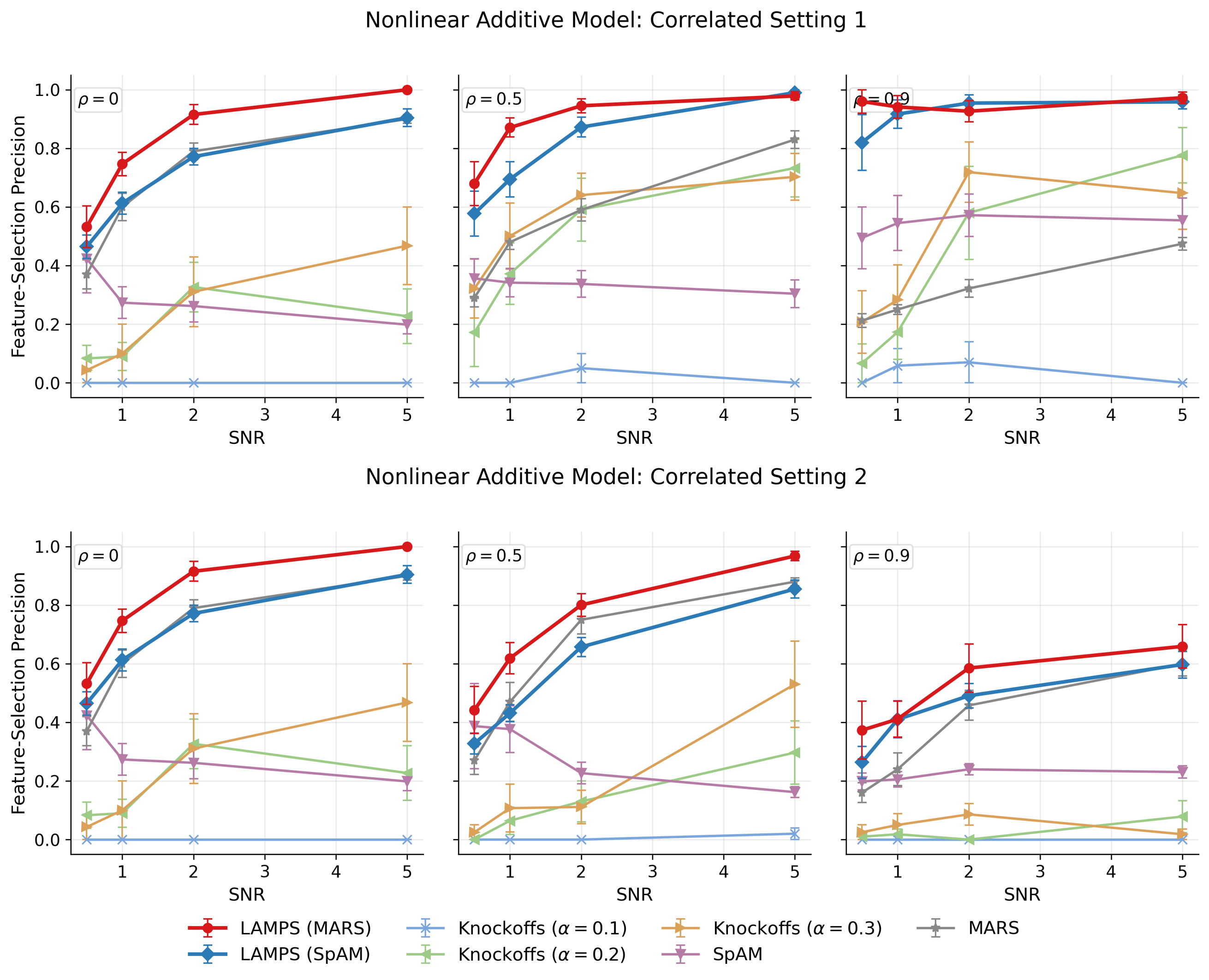}
    {fig:supp_nonlinear_additive_precision_nonoracle}
    {Precision for non-oracle procedures in nonlinear additive models. Top row: Correlated Setting~1 with non-permuted Toeplitz covariance; bottom row: Correlated Setting~2 with permuted covariance. Columns correspond to $\rho\in\{0,0.5,0.9\}$. LAMPS with MARS and SpAM base learners are shown separately, together with competing nonlinear and additive baselines. Error bars represent $\pm$ one standard error over 10 Monte Carlo replicates.}

\suppfigure
    {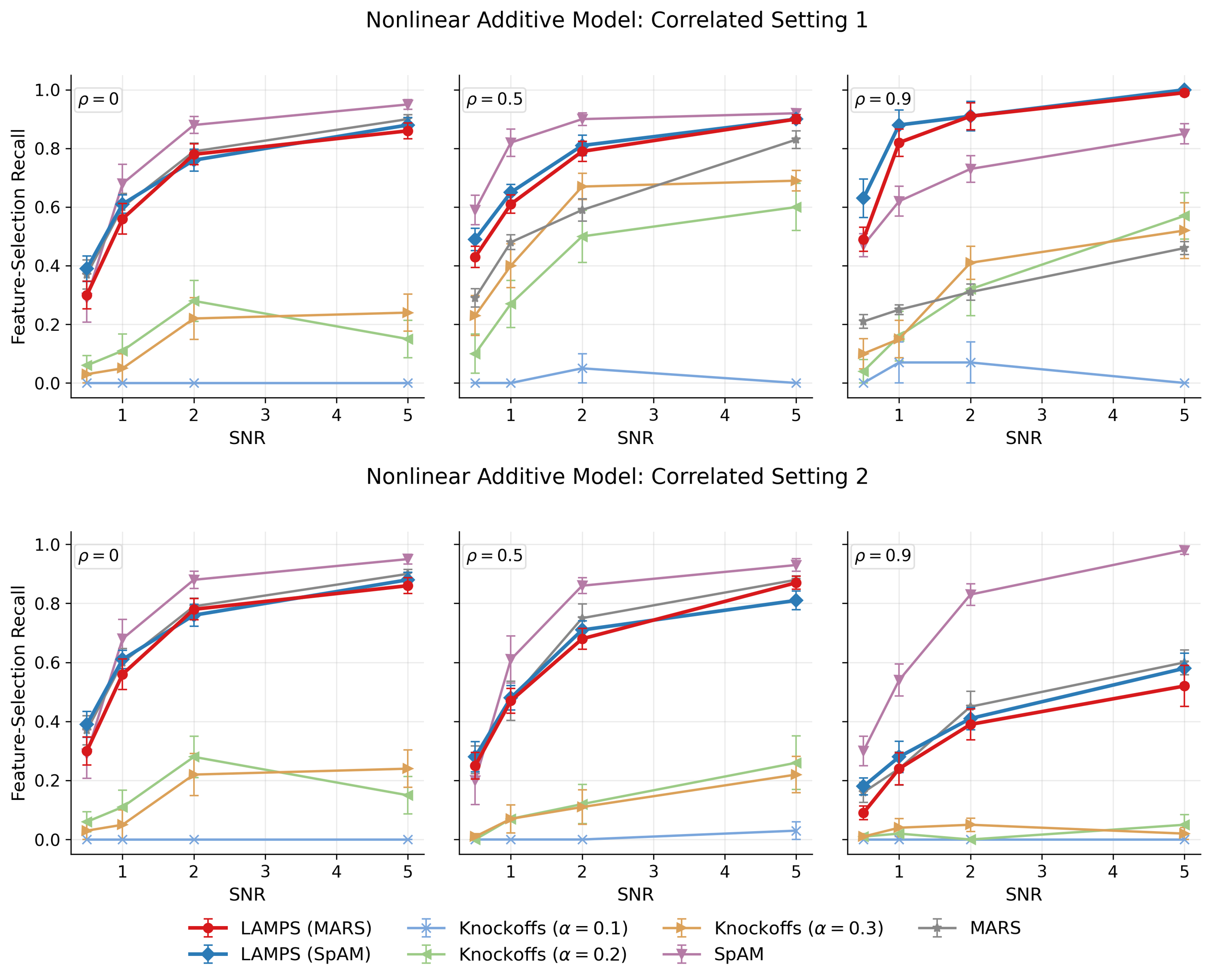}
    {fig:supp_nonlinear_additive_recall_nonoracle}
    {Recall for non-oracle procedures in nonlinear additive models. Top row: Correlated Setting~1 with non-permuted Toeplitz covariance; bottom row: Correlated Setting~2 with permuted covariance. Columns correspond to $\rho\in\{0,0.5,0.9\}$. LAMPS with MARS and SpAM base learners are shown separately, together with competing nonlinear and additive baselines. Error bars represent $\pm$ one standard error over 10 Monte Carlo replicates.}

\suppfigure
    {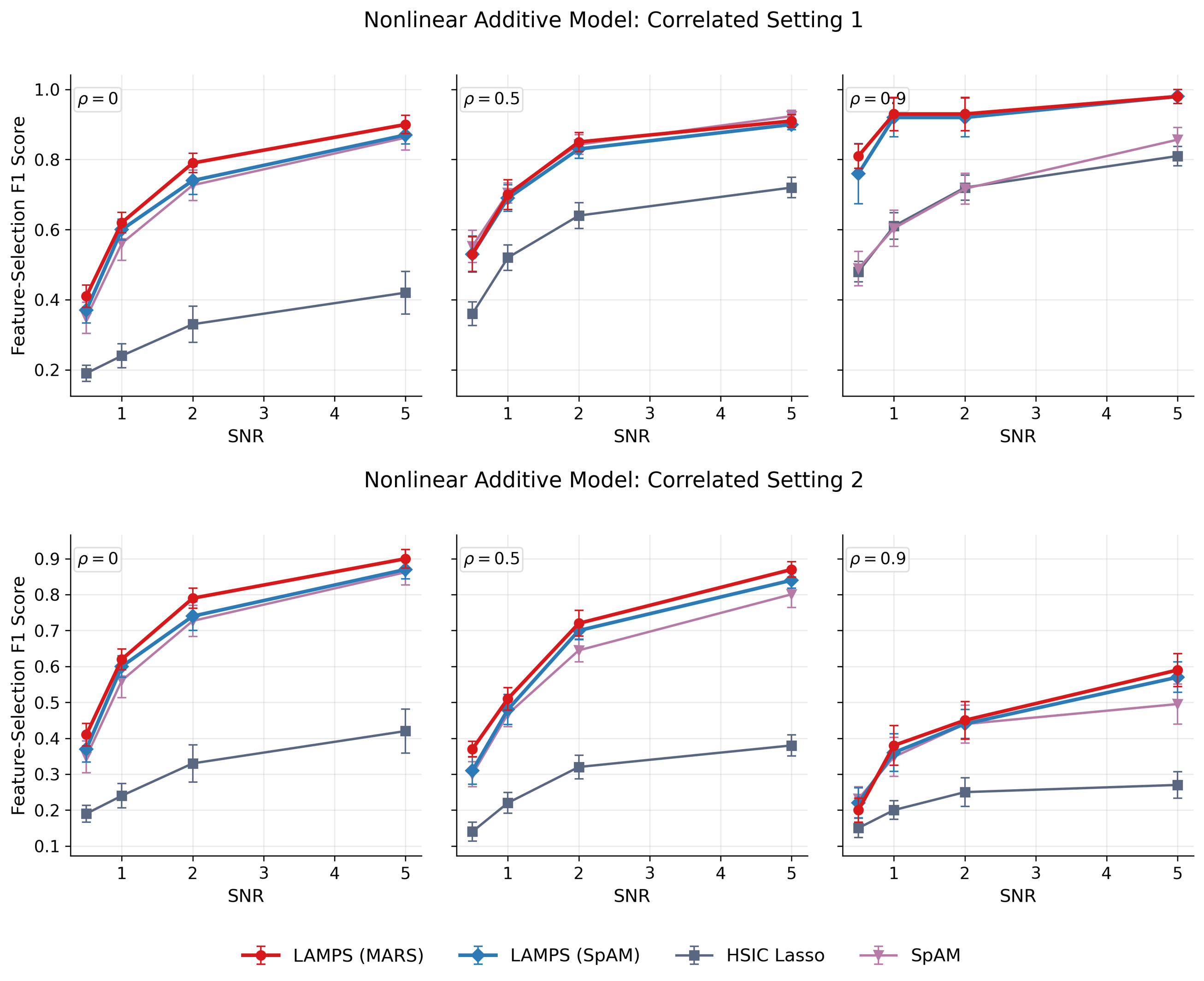}
    {fig:supp_nonlinear_additive_f1_oracle}
    {F1 scores for procedures with oracle-sparsity tuning in nonlinear additive models. Top row: Correlated Setting~1 with non-permuted Toeplitz covariance; bottom row: Correlated Setting~2 with permuted covariance. Columns correspond to $\rho\in\{0,0.5,0.9\}$. LAMPS with MARS and SpAM base learners are shown separately, together with oracle-tuning competitors. MARS and knockoffs are not excluded since they do not come with sparsity tuning. Error bars represent $\pm$ one standard error over 10 Monte Carlo replicates.}

\section{Comparative Results with Data-driven Tuning of Minipatch Sizes}
\label{sec:supp_hyperparameter_tuning}

We evaluate the performance of LAMPS with data-driven tuning of minipatch sizes for the three simulation classes studied in the main paper: linear, nonlinear additive, and nonlinear interaction models. In all three experiments, we fix the correlation level at $\rho=0.5$ and use the non-permuted Toeplitz design. All other simulation settings, base learners, and evaluation protocols are the same as in the corresponding main-text experiments.

Following Section~\ref{sec:tuning}, we parameterize the minipatch sizes as
\[
n=\lfloor n_{\rm ratio}N\rfloor,
\qquad
m=\lfloor m_{\rm ratio}M\rfloor,
\]
and search over
\[
n_{\rm ratio}\in\{0.3,0.4,0.5\},
\qquad
m_{\rm ratio}\in\{0.1,0.15,0.2\}.
\]
For each Monte Carlo replicate, we select the pair $(n_{\rm ratio},m_{\rm ratio})$ that minimizes the final accepted LAMPS LOO error. The selected pair is then used for feature selection in that replicate. For each minipatch size candidate, the ensemble size $K$ is recomputed using the coverage rule in Section~\ref{sec:tuning}.

Figures~\ref{fig:supp_hyper_linear_f1}--\ref{fig:supp_hyper_nonlinear_interaction_rate} report the resulting performance. For the linear and nonlinear additive settings, we report F1 score across signal-to-noise ratios. For the nonlinear interaction setting, we report the interaction detection rate across interaction strengths. Overall, the LOO-based tuning rule gives stable performance across the three model classes.

\suppfigure[0.76\textwidth]
    {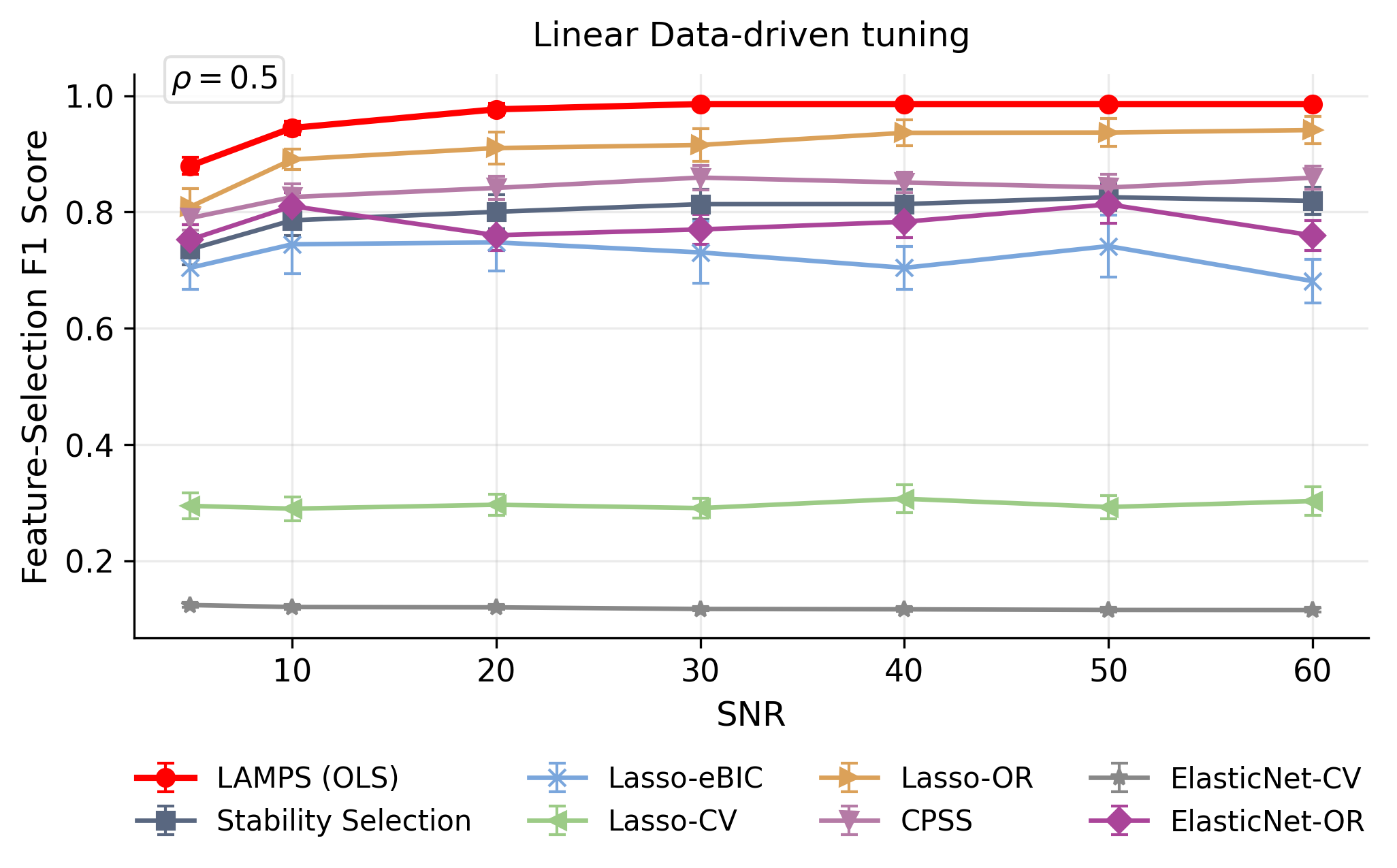}
    {fig:supp_hyper_linear_f1}
    {F1 scores of LAMPS with data-driven tuning of minipatch sizes, presented together with baseline methods in the linear model setup. This figure corresponds to Correlated Setting~1 with non-permuted Toeplitz covariance and correlation level $\rho=0.5$. Error bars represent $\pm$ one standard error over 10 replicates.}

\suppfigure[0.76\textwidth]
    {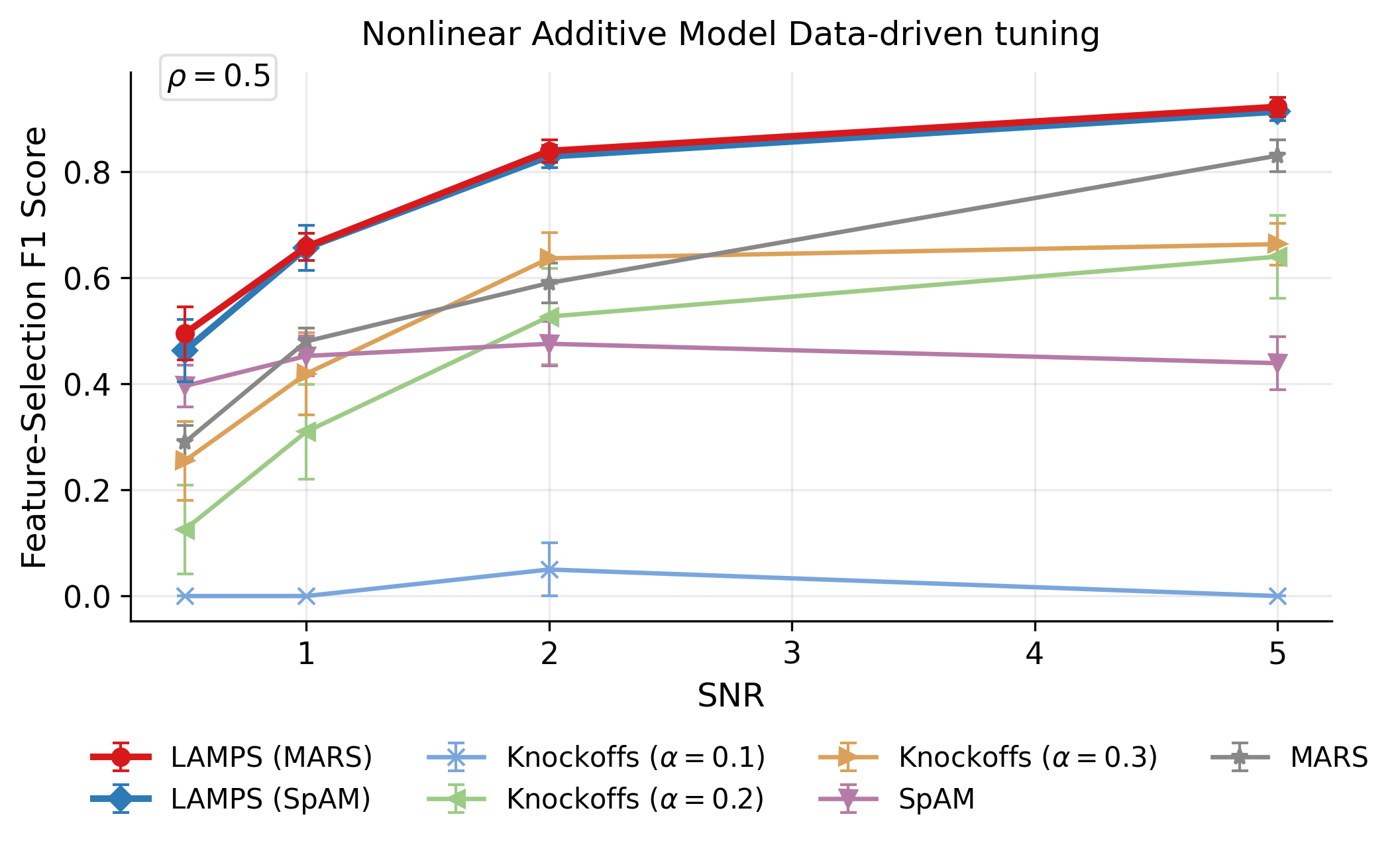}
    {fig:supp_hyper_nonlinear_additive_f1}
    {F1 scores of LAMPS with data-driven tuning of minipatch sizes, presented together with baseline methods in the non-linear additive model setup. This figure corresponds to Correlated Setting~1 with non-permuted Toeplitz covariance and correlation level $\rho=0.5$. Error bars represent $\pm$ one standard error over 10 replicates.}

\suppfigure[0.76\textwidth]
    {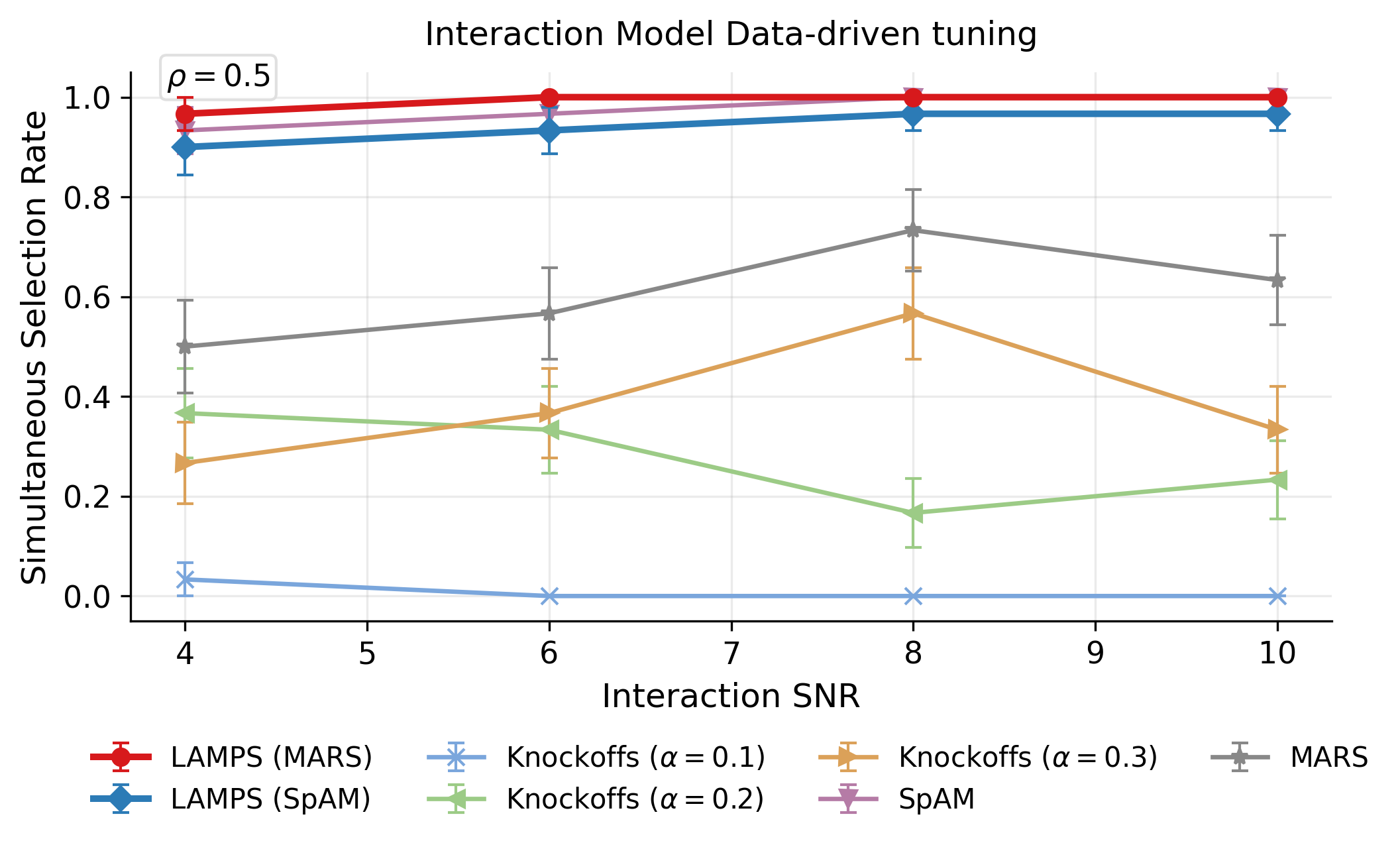}
    {fig:supp_hyper_nonlinear_interaction_rate}
    {Interaction detection rate (proportion of both interaction features been selected) of LAMPS with data-driven tuning of minipatch sizes, presented together with baseline methods in the non-linear interaction model setup. This figure corresponds to Correlated Setting~1 with non-permuted Toeplitz covariance and correlation level $\rho=0.5$. Error bars represent $\pm$ one standard error over 30 replicates.}

\clearpage

\section{Additional Details on the Case Study}
\label{sec:supp_case_study_details}

We present the gene selection results by LAMPS in the ROSMAP case study. The selection frequencies across the 10 splits are reported in Table~\ref{tab:rosmap_selection_frequency}. Among the frequently selected genes, RPL41P1 has been reported in immune-related expression analyses in neurological contexts, suggesting a possible neuroimmune link to cognitive processes~\citep{wieland2022epstein}. AC104841.1 is located near SEPT2 and may have local regulatory relevance for neuronal structure or cognitive development~\citep{tada2007role,gil2020regulation}. RP11-599B13.6 has domain similarity to VAMP2, a protein involved in neural signaling~\citep{schoch2001snare}. ENSG00000262526.2 encodes I3L401, which is annotated as similar to proteins involved in transcriptional regulation. CCL2 has been linked to cognition through neuroinflammation and dementia-related pathology~\citep{westin2012ccl2}, and S100A4 is involved in inflammatory and neuronal response pathways that may be relevant to brain function~\citep{serrano2019s100a4}.

\begin{table}[t]
\centering
\caption{\small Selection frequencies of genes identified by LAMPS (MARS) in the ROSMAP experiment across the 10 cross-validation splits.
The table reports all genes selected in at least one split, ordered by decreasing selection frequency.}
\label{tab:rosmap_selection_frequency}
\begin{tabular}{clc}
\toprule
Rank & Gene & Selection frequency (out of 10) \\
\midrule
1 & RPL41P1 & 10 \\
2 & AC104841.1 & 10 \\
3 & RP11-599B13.6 & 9 \\
4 & ENSG00000262526.2 & 9 \\
5 & CCL2 & 3 \\
6 & S100A4 & 3 \\
7 & MT-TV & 2 \\
8 & HPR & 2 \\
9 & ENSG00000247457.9 & 1 \\
\bottomrule
\end{tabular}
\end{table}

\typeout{get arXiv to do 4 passes: Label(s) may have changed. Rerun}
\end{document}